\documentclass{article}
\usepackage{titletoc}

\usepackage{etoolbox}
\usepackage{comment}
\newtoggle{colt}

\newcommand{\arxiv}[1]{\iftoggle{colt}{}{#1}}
\newcommand{\loose}{\looseness=-1}

\arxiv{
\usepackage[letterpaper, left=1in, right=1in, top=1in,
bottom=1in]{geometry}
  \usepackage{parskip}
}

\PassOptionsToPackage{hypertexnames=false}{hyperref}  %

\arxiv{
\usepackage[dvipsnames]{xcolor}
\usepackage[colorlinks=true, linkcolor=blue!70!black, citecolor=blue!70!black,urlcolor=black,breaklinks=true]{hyperref}
}

\usepackage{amsmath}

\usepackage{microtype}
\usepackage{hhline}

\usepackage{amsthm}
\usepackage{mathtools}
\usepackage{bbm}
\usepackage{amsfonts}
\usepackage{amssymb}
\usepackage[nameinlink,capitalize]{cleveref}

\makeatletter
\newcommand{\neutralize}[1]{\expandafter\let\csname c@#1\endcsname\count@}
\makeatother

\usepackage{algorithm}

\arxiv{
\usepackage{natbib}
\bibliographystyle{plainnat}
\bibpunct{(}{)}{;}{a}{,}{,}
}

\usepackage{xpatch}

\usepackage{thm-restate}
\usepackage{autonum}
\declaretheorem[name=Theorem,parent=section]{theorem}
\declaretheorem[name=Lemma,parent=section, numberlike=theorem]{lemma}
\declaretheorem[name=Assumption, parent=section, numberlike=theorem]{assumption}
\declaretheorem[name=Definition, parent=section, numberlike=theorem]{definition}

\declaretheorem[name=Corollary, parent=section, numberlike=theorem]{corollary}

\declaretheorem[name=Remark, parent=section, numberlike=theorem]{remark}
\declaretheorem[name=Proposition, parent=section, numberlike=theorem]{proposition}

\usepackage{crossreftools}
  \newcommand{\creftitle}[1]{\crtcref{#1}}

\makeatletter
  \renewenvironment{proof}[1][Proof]%
  {%
   \par\noindent{\bfseries\upshape {#1.}\ }%
  }%
  {\qed\newline}
  \makeatother

\xpatchcmd{\proof}{\itshape}{\normalfont\proofnameformat}{}{}
\newcommand{\proofnameformat}{\bfseries}

\newcommand{\pref}[1]{\cref{#1}}
\newcommand{\pfref}[1]{Proof of \pref{#1}}

\renewcommand{\eqref}[1]{\texorpdfstring{\hyperref[#1]{(\ref*{#1})}}{(\ref*{#1})}}
\crefformat{equation}{#2Eq.\,(#1)#3}
\Crefformat{equation}{#2Eq.\,(#1)#3}

\Crefformat{figure}{#2Figure~#1#3}
\Crefformat{assumption}{#2Assumption~#1#3}

\Crefname{assumption}{Assumption}{Assumptions}

\crefname{fact}{Fact}{Facts}

\Crefformat{figure}{#2Figure #1#3}
\Crefformat{assumption}{#2Assumption #1#3}

\usepackage{xparse}

\ExplSyntaxOn
\DeclareDocumentCommand{\XDeclarePairedDelimiter}{mm}
 {
  \__egreg_delimiter_clear_keys: %
  \keys_set:nn { egreg/delimiters } { #2 }
  \use:x %
   {
    \exp_not:n {\NewDocumentCommand{#1}{sO{}m} }
     {
      \exp_not:n { \IfBooleanTF{##1} }
       {
        \exp_not:N \egreg_paired_delimiter_expand:nnnn
         { \exp_not:V \l_egreg_delimiter_left_tl }
         { \exp_not:V \l_egreg_delimiter_right_tl }
         { \exp_not:n { ##3 } }
         { \exp_not:V \l_egreg_delimiter_subscript_tl }
       }
       {
        \exp_not:N \egreg_paired_delimiter_fixed:nnnnn 
         { \exp_not:n { ##2 } }
         { \exp_not:V \l_egreg_delimiter_left_tl }
         { \exp_not:V \l_egreg_delimiter_right_tl }
         { \exp_not:n { ##3 } }
         { \exp_not:V \l_egreg_delimiter_subscript_tl }
       }
     }
   }
 }

\keys_define:nn { egreg/delimiters }
 {
  left      .tl_set:N = \l_egreg_delimiter_left_tl,
  right     .tl_set:N = \l_egreg_delimiter_right_tl,
  subscript .tl_set:N = \l_egreg_delimiter_subscript_tl,
 }

\cs_new_protected:Npn \__egreg_delimiter_clear_keys:
 {
  \keys_set:nn { egreg/delimiters } { left=.,right=.,subscript={} }
 }

\cs_new_protected:Npn \egreg_paired_delimiter_expand:nnnn #1 #2 #3 #4
 {%
  \mathopen{}
  \mathclose\c_group_begin_token
   \left#1
   #3
   \group_insert_after:N \c_group_end_token
   \right#2
   \tl_if_empty:nF {#4} { \c_math_subscript_token {#4} }
 }
\cs_new_protected:Npn \egreg_paired_delimiter_fixed:nnnnn #1 #2 #3 #4 #5
 {
  \mathopen{#1#2}#4\mathclose{#1#3}
  \tl_if_empty:nF {#5} { \c_math_subscript_token {#5} }
 }
\ExplSyntaxOff

\XDeclarePairedDelimiter{\supnorm}{
  left=\lVert,
  right=\rVert,
  subscript=\infty
  }

\usepackage[utf8]{inputenc} %
\usepackage[T1]{fontenc}    %
\usepackage{url}            %
\usepackage{booktabs}       %
\usepackage{amsfonts}       %
\usepackage{nicefrac}       %
\usepackage{microtype}      %

\usepackage{tocloft}            %

\usepackage{makecell}
\usepackage{amsmath}
\usepackage{amsthm}

\usepackage{upgreek}

\usepackage{enumitem}

\usepackage{breakcites}

\usepackage{mathrsfs}

\usepackage{algorithm}
\usepackage{verbatim}
\usepackage[noend]{algpseudocode}
\newcommand{\multiline}[1]{\parbox[t]{\dimexpr\linewidth-\algorithmicindent}{#1}}

\usepackage{multicol}

\usepackage{colortbl}
\usepackage{setspace}

\usepackage{transparent}

\usepackage{inconsolata}
\usepackage[scaled=.90]{helvet}
\usepackage{xspace}

\usepackage{pifont}
\usepackage{bm}

\usepackage{tablefootnote}

\DeclareFontFamily{U}{jkpmia}{}
\DeclareFontShape{U}{jkpmia}{m}{it}{<->s*jkpmia}{}
\DeclareFontShape{U}{jkpmia}{bx}{it}{<->s*jkpbmia}{}
\DeclareMathAlphabet{\mathfrak}{U}{jkpmia}{m}{it}
\SetMathAlphabet{\mathfrak}{bold}{U}{jkpmia}{bx}{it}

\newcommand{\abs}[1]{\left\lvert#1\right\rvert}
\DeclarePairedDelimiter{\brk}{[}{]}
\DeclarePairedDelimiter{\crl}{\{}{\}}
\DeclarePairedDelimiter{\prn}{(}{)}

\let\Pr\undefined

\DeclareMathOperator{\En}{\mathbb{E}}

\DeclareMathOperator{\Pr}{Pr}

\DeclareMathOperator*{\argmin}{arg\,min} %
\DeclareMathOperator*{\argmax}{arg\,max}

\newcommand{\wt}[1]{\widetilde{#1}}
\newcommand{\wh}[1]{\widehat{#1}}
\newcommand{\wb}[1]{\widebar{#1}}

\def\ddefloop#1{\ifx\ddefloop#1\else\ddef{#1}\expandafter\ddefloop\fi}
\def\ddef#1{\expandafter\def\csname bb#1\endcsname{\ensuremath{\mathbb{#1}}}}
\ddefloop ABCDEFGHIJKLMNOPQRSTUVWXYZ\ddefloop
\def\ddefloop#1{\ifx\ddefloop#1\else\ddef{#1}\expandafter\ddefloop\fi}
\def\ddef#1{\expandafter\def\csname b#1\endcsname{\ensuremath{\mathbf{#1}}}}
\ddefloop ABCDEFGHIJKLMNOPQRSTUVWXYZ\ddefloop
\def\ddef#1{\expandafter\def\csname sf#1\endcsname{\ensuremath{\mathsf{#1}}}}
\ddefloop ABCDEFGHIJKLMNOPQRSTUVWXYZ\ddefloop
\def\ddef#1{\expandafter\def\csname c#1\endcsname{\ensuremath{\mathcal{#1}}}}
\ddefloop ABCDEFGHIJKLMNOPQRSTUVWXYZ\ddefloop
\def\ddef#1{\expandafter\def\csname h#1\endcsname{\ensuremath{\widehat{#1}}}}
\ddefloop ABCDEFGHIJKLMNOPQRSTUVWXYZ\ddefloop
\def\ddef#1{\expandafter\def\csname hc#1\endcsname{\ensuremath{\widehat{\mathcal{#1}}}}}
\ddefloop ABCDEFGHIJKLMNOPQRSTUVWXYZ\ddefloop
\def\ddef#1{\expandafter\def\csname t#1\endcsname{\ensuremath{\widetilde{#1}}}}
\ddefloop ABCDEFGHIJKLMNOPQRSTUVWXYZ\ddefloop
\def\ddef#1{\expandafter\def\csname tc#1\endcsname{\ensuremath{\widetilde{\mathcal{#1}}}}}
\ddefloop ABCDEFGHIJKLMNOPQRSTUVWXYZ\ddefloop
\def\ddefloop#1{\ifx\ddefloop#1\else\ddef{#1}\expandafter\ddefloop\fi}
\def\ddef#1{\expandafter\def\csname scr#1\endcsname{\ensuremath{\mathscr{#1}}}}
\ddefloop ABCDEFGHIJKLMNOPQRSTUVWXYZ\ddefloop

\newcommand{\ind}{\mathbbm{1}}    %

\newcommand{\eps}{\epsilon}
\newcommand{\veps}{\varepsilon}

\newcommand{\ldef}{\vcentcolon=}

\newcommand{\densobs}{expert densities\xspace}
\newcommand{\DensObs}{Expert Densities\xspace}

\newcommand{\chis}{$\chi^2$\xspace}

\newcommand{\Capx}{C_{\texttt{apx}}}

\newcommand{\vepsstatsn}{\vepsstat^2(n)}

\newcommand{\vepsmis}{\veps_{\texttt{mis}}}

\newcommand{\pia}{\pi^{\afrak}}
\newcommand{\pib}{\pi^{\bfrak}}

\newcommand{\sigmastar}{\sigma_{\pistar}}

\newcommand{\Qstar}{Q^{\pistar}}
\newcommand{\Vstar}{V^{\pistar}}

\newcommand{\loglossbc}{\texttt{LogLossBC}\xspace}
\newcommand{\boostedloglossbc}{\texttt{BoostedLogLossBC}\xspace}
\newcommand{\smoothedloglossbc}{\texttt{SmoothedLogLossBC}\xspace}
\newcommand{\rhobc}{\texttt{RhoEstimatorBC}\xspace}
\newcommand{\layerrhobc}{\texttt{LayeredRhoBC}\xspace}

\newcommand{\Unif}{\mathsf{Unif}}

\newcommand{\bfi}{\mathbf{i}}

\newcommand{\pibar}{\wb{\pi}}

  \newcommand{\afrak}{\mathfrak{a}}
  \newcommand{\bfrak}{\mathfrak{b}}
  
  \newcommand{\xfrak}{\mathfrak{x}}
  \newcommand{\yfrak}{\mathfrak{y}}

\newcommand{\vepsstat}{\veps_{\texttt{stat}}}

\renewcommand{\emptyset}{\varnothing}

\newcommand{\obs}{o}

\newcommand{\M}[1]{^{{\scriptscriptstyle M}}}  %

\newcommand{\pistar}{\pi^{\star}}

\newcommand{\pihat}{\wh{\pi}}

\newcommand{\thetahat}{\wh{\theta}}

\newcommand{\algcommentlight}[1]{\textcolor{blue!70!black}{\transparent{0.5}\small{\texttt{\textbf{//\hspace{2pt}#1}}}}}

\newcommand{\approxleq}{\lesssim}

\renewcommand{\ind}[1]{^{{\scriptscriptstyle#1}}}

\newcommand{\bigoh}{O}
\newcommand{\bigoht}{\wt{O}}
\newcommand{\bigom}{\Omega}
\newcommand{\bigomt}{\wt{\Omega}}

\renewcommand{\Pr}{\bbP}

\newcommand{\poly}{\mathrm{poly}}

\newcommand{\Dkl}[2]{D_{\mathsf{KL}}\prn*{#1\,\|\,#2}}

\newcommand{\Dhel}[2]{D_{\mathsf{H}}\prn*{#1,#2}}

\newcommand{\Dhels}[2]{D^{2}_{\mathsf{H}}\prn*{#1,#2}}

\newcommand{\Dchis}[2]{D_{\chi^2}\prn*{#1\dmid{}#2}}
\newcommand{\DchisX}[3]{D_{\chi^2}\prn[#1]{#2\dmid{}#3}}

\newcommand{\Dtv}[2]{D_{\mathsf{TV}}\prn*{#1,#2}}
\newcommand{\Dtvs}[2]{D^2_{\mathsf{TV}}\prn*{#1,#2}}

\newcommand{\DhelsX}[3]{D^{2}_{\mathsf{H}}\prn[#1]{#2,#3}}
\newcommand{\DhelX}[3]{D_{\mathsf{H}}\prn[#1]{#2,#3}}
\newcommand{\DtvX}[3]{D_{\mathsf{TV}}\prn[#1]{#2,#3}}

\newcommand{\Ber}{\mathrm{Ber}}

\newcommand{\dmid}{\;\|\;}

\def\multiset#1#2{\ensuremath{\left(\kern-.3em\left(\genfrac{}{}{0pt}{}{#1}{#2}\right)\kern-.3em\right)}}

\newcommand{\grad}{\nabla}

\newcommand{\iid}{i.i.d.\xspace}

\renewcommand{\emptyset}{\varnothing}

\newcommand{\pbar}{\overline{p}}
\newcommand{\phat}{\wh{p}}

\newcommand{\MA}{\mathcal{A}}
\newcommand{\norm}[1]{\left \lVert #1 \right \rVert}
\DeclareMathOperator*{\EE}{\mathbb{E}}
\newcommand{\RR}{\mathbb{R}}
\newcommand{\NN}{\mathbb{N}}
\newcommand{\st}{\star}
\newcommand{\BP}{\mathbb{P}}
\newcommand{\BQ}{\mathbb{Q}}

\newcommand{\Alg}{\texttt{Alg}}
\newcommand{\MX}{\mathcal{X}}
\newcommand{\MO}{\mathcal{O}}

\newcommand{\MD}{\mathcal{D}}
\DeclareMathOperator{\Bin}{Bin}

\DeclareMathOperator*{\Prr}{Pr}
\DeclareMathOperator{\Law}{Law}
\DeclareMathOperator{\Maj}{\mathsf{MAJ}}
\newcommand{\phiver}{\phi^{\mathsf{ver}}}

\newcommand{\ME}{\mathcal{E}}
\newcommand{\xtest}{x^{\mathsf{test}}}
\newcommand{\ytest}{y^{\mathsf{test}}}
\newcommand{\Adv}{\texttt{Adv}}

\newcommand{\epapx}{\varepsilon_{\mathsf{apx}}}
\newcommand{\mol}[2]{\mathsf{mol}_{#1}(#2)}
\newcommand{\epstat}{\varepsilon_{\mathsf{stat}}}
\newcommand{\MS}{\mathcal{S}}
\newcommand{\MF}{\mathcal{F}}
\newcommand{\MG}{\mathcal{G}}

\newcommand{\emploss}{\mathcal{\wh L}}
\newcommand{\epopt}{\varepsilon_{\mathsf{opt}}}

\newcommand{\Ovec}{\mathcal{O}_{\mathsf{vec}}}
\newcommand{\Oproj}{\mathcal{O}_{\mathsf{proj}}}
\newcommand{\epproj}{\varepsilon_{\mathsf{proj}}}
\newcommand{\PGD}{\texttt{PGD}}
\newcommand{\MY}{\mathcal{Y}}
\newcommand{\alphatil}{\widetilde\alpha}
\newcommand{\betatil}{\widetilde\beta}
\newcommand{\KernRho}{\texttt{KernelizedRho}}
\DeclareMathOperator{\Proj}{Proj}
\DeclareMathOperator{\vspan}{span}

\newcommand{\Blarge}{B_{\mathsf{large}}}

\newcommand{\ChunkKR}{\texttt{ChunkKR}}

\newcommand{\pihi}{\pi_h(a_h\ind{i}\mid{} s_h\ind{i})}
\newcommand{\piphi}{\pi'_h(a_h\ind{i}\mid{} s_h\ind{i})}
\newcommand{\pistarhi}{\pistar_h(a_h\ind{i}\mid{} s_h\ind{i})}
\newcommand{\pibarhi}{\pibar_h(a_h\ind{i}\mid{} s_h\ind{i})}

\newcommand{\Nlog}{\mathcal{N}_{\mathsf{log}}}
\newcommand{\Bdot}{B_{\mathsf{dot}}}
\newcommand{\MC}{\mathcal{C}}
\newcommand{\uf}{\mathfrak{u}}
\newcommand{\ball}{\mathbb{B}_2}
\newcommand{\uop}[2]{\uf(#1,#2)}
\newcommand{\pstar}{p^\st}
\newcommand{\ee}{\EE}
\newcommand{\pp}{\Pr}
\newcommand{\denbound}{W}
\newcommand{\IL}{IL}
\newcommand{\BALM}{\texttt{BALM}\xspace}
\newcommand{\GDALM}{\texttt{GAALM}\xspace}

\let\underbar\undefined
\makeatletter
\let\save@mathaccent\mathaccent
\newcommand*\if@single[3]{%
  \setbox0\hbox{${\mathaccent"0362{#1}}^H$}%
  \setbox2\hbox{${\mathaccent"0362{\kern0pt#1}}^H$}%
  \ifdim\ht0=\ht2 #3\else #2\fi
  }
\newcommand*\rel@kern[1]{\kern#1\dimexpr\macc@kerna}
\newcommand*\widebar[1]{\@ifnextchar^{{\wide@bar{#1}{0}}}{\wide@bar{#1}{1}}}
\newcommand*\underbar[1]{\@ifnextchar_{{\under@bar{#1}{0}}}{\under@bar{#1}{1}}}
\newcommand*\wide@bar[2]{\if@single{#1}{\wide@bar@{#1}{#2}{1}}{\wide@bar@{#1}{#2}{2}}}
\newcommand*\under@bar[2]{\if@single{#1}{\under@bar@{#1}{#2}{1}}{\under@bar@{#1}{#2}{2}}}
\newcommand*\wide@bar@[3]{%
  \begingroup
  \def\mathaccent##1##2{%
    \let\mathaccent\save@mathaccent
    \if#32 \let\macc@nucleus\first@char \fi
    \setbox\z@\hbox{$\macc@style{\macc@nucleus}_{}$}%
    \setbox\tw@\hbox{$\macc@style{\macc@nucleus}{}_{}$}%
    \dimen@\wd\tw@
    \advance\dimen@-\wd\z@
    \divide\dimen@ 3
    \@tempdima\wd\tw@
    \advance\@tempdima-\scriptspace
    \divide\@tempdima 10
    \advance\dimen@-\@tempdima
    \ifdim\dimen@>\z@ \dimen@0pt\fi
    \rel@kern{0.6}\kern-\dimen@
    \if#31
      \overline{\rel@kern{-0.6}\kern\dimen@\macc@nucleus\rel@kern{0.4}\kern\dimen@}%
      \advance\dimen@0.4\dimexpr\macc@kerna
      \let\final@kern#2%
      \ifdim\dimen@<\z@ \let\final@kern1\fi
      \if\final@kern1 \kern-\dimen@\fi
    \else
      \overline{\rel@kern{-0.6}\kern\dimen@#1}%
    \fi
  }%
  \macc@depth\@ne
  \let\math@bgroup\@empty \let\math@egroup\macc@set@skewchar
  \mathsurround\z@ \frozen@everymath{\mathgroup\macc@group\relax}%
  \macc@set@skewchar\relax
  \let\mathaccentV\macc@nested@a
  \if#31
    \macc@nested@a\relax111{#1}%
  \else
    \def\gobble@till@marker##1\endmarker{}%
    \futurelet\first@char\gobble@till@marker#1\endmarker
    \ifcat\noexpand\first@char A\else
      \def\first@char{}%
    \fi
    \macc@nested@a\relax111{\first@char}%
  \fi
  \endgroup
}
\newcommand*\under@bar@[3]{%
  \begingroup
  \def\mathaccent##1##2{%
    \let\mathaccent\save@mathaccent
    \if#32 \let\macc@nucleus\first@char \fi
    \setbox\z@\hbox{$\macc@style{\macc@nucleus}_{}$}%
    \setbox\tw@\hbox{$\macc@style{\macc@nucleus}{}_{}$}%
    \dimen@\wd\tw@
    \advance\dimen@-\wd\z@
    \divide\dimen@ 3
    \@tempdima\wd\tw@
    \advance\@tempdima-\scriptspace
    \divide\@tempdima 10
    \advance\dimen@-\@tempdima
    \ifdim\dimen@>\z@ \dimen@0pt\fi
    \rel@kern{0.6}\kern-\dimen@
    \if#31
      \underline{\rel@kern{-0.6}\kern\dimen@\macc@nucleus\rel@kern{0.4}\kern\dimen@}%
      \advance\dimen@0.4\dimexpr\macc@kerna
      \let\final@kern#2%
      \ifdim\dimen@<\z@ \let\final@kern1\fi
      \if\final@kern1 \kern-\dimen@\fi
    \else
      \underline{\rel@kern{-0.6}\kern\dimen@#1}%
    \fi
  }%
  \macc@depth\@ne
  \let\math@bgroup\@empty \let\math@egroup\macc@set@skewchar
  \mathsurround\z@ \frozen@everymath{\mathgroup\macc@group\relax}%
  \macc@set@skewchar\relax
  \let\mathaccentV\macc@nested@a
  \if#31
    \macc@nested@a\relax111{#1}%
  \else
    \def\gobble@till@marker##1\endmarker{}%
    \futurelet\first@char\gobble@till@marker#1\endmarker
    \ifcat\noexpand\first@char A\else
      \def\first@char{}%
    \fi
    \macc@nested@a\relax111{\first@char}%
  \fi
  \endgroup
}
\makeatother

\usepackage[suppress]{color-edits}
 \addauthor{df}{ForestGreen}
 \addauthor{ab}{red}
 \addauthor{dr}{purple}
 \addauthor{ak}{BurntOrange}
 \addauthor{ah}{olive}
 \newcommand{\dfc}[1]{}
 
 \newcommand{\ah}[1]{\ahcomment{#1}}

\makeatletter
\let\OldStatex\Statex
\renewcommand{\Statex}[1][3]{%
  \setlength\@tempdima{\algorithmicindent}%
  \OldStatex\hskip\dimexpr#1\@tempdima\relax}
\makeatother

 \usepackage{accents}

\let\oldparagraph\paragraph

\renewcommand{\paragraph}[1]{\oldparagraph{#1.}}

\arxiv{

\title{\Large Computational-Statistical Tradeoffs at the Next-Token Prediction Barrier: \\
Autoregressive and Imitation Learning under Misspecification
}

}

\date{}

\begin{document}
\maketitle\vspace{-5em}
\begin{center}
\large
\setlength{\tabcolsep}{20pt}
\begin{tabular}{ccc}
\makecell{Dhruv Rohatgi \\ \small{\texttt{drohatgi@mit.edu}}}
&
\makecell{Adam Block \\ \small{\texttt{blockadam@microsoft.com}}}
&
\makecell{Audrey Huang \\ \small{\texttt{audreyh5@illinois.edu}}}
\end{tabular}\vspace{1em}
\begin{tabular}{cc}
\makecell{Akshay Krishnamurthy \\ \small{\texttt{akshaykr@microsoft.com}}}
&
\makecell{Dylan J. Foster \\ \small{\texttt{dylanfoster@microsoft.com}}}
\end{tabular}
\end{center}
\vspace{1em}

\begin{abstract}

  Next-token prediction with the logarithmic loss is a cornerstone of autoregressive sequence modeling, but, in practice, suffers from \emph{error amplification}, where errors in the model compound and generation quality degrades as sequence length $H$ increases. From a theoretical perspective, this phenomenon should not appear in \emph{well-specified} settings, and, indeed, a growing body of empirical work hypothesizes that \emph{misspecification}, where the learner is not sufficiently expressive to represent the target distribution, may be the root cause. Under misspecification---where the goal is to learn as well as the best-in-class model up to a multiplicative approximation factor $\Capx\geq{}1$---we confirm that $\Capx$ indeed grows with $H$ for next-token prediction, lending theoretical support to this empirical hypothesis.
  We then ask whether this mode of error amplification is avoidable algorithmically, computationally, or information-theoretically, and uncover inherent computational-statistical tradeoffs.
  
We show: \textbf{(1)} Information-theoretically, one can avoid error amplification and achieve $\Capx=O(1)$. \textbf{(2)} Next-token prediction can be made robust\arxiv{ so as} to achieve $\Capx=\bigoht(H)$, representing moderate error amplification, but this is an inherent barrier: \emph{any} next-token prediction-style objective must suffer $\Capx=\Omega(H)$. \textbf{(3)} For the natural testbed of autoregressive \emph{linear} models, \emph{no computationally efficient algorithm} can achieve sub-polynomial approximation factor $\Capx=e^{(\log H)^{1-\Omega(1)}}$; however, at least for binary token spaces, one can smoothly trade compute for statistical power and improve on $\Capx=\Omega(H)$ in sub-exponential time.
  Our results have consequences in the more general setting of imitation learning, where the widely-used behavior cloning\arxiv{ algorithm} generalizes next-token prediction.\loose

\end{abstract}
\arxiv{
\setcounter{tocdepth}{0}
}

\section{Introduction}
\label{sec:intro}

Next-token prediction with the logarithmic loss \citep{shannon1951prediction} is a cornerstone of autoregressive sequence modeling---particularly language
model pre-training \citep{vaswani2017attention,radford2019language}. It estimates a distribution over sequences $(a_1,\ldots,a_H)$ by jointly fitting a sequence
  of conditional models
  $\pihat(a_h\mid{}a_{1:h-1})$ to
  maximize log-likelihood.
  This method is appealing in its simplicity and scalability, but seemingly ignores the \emph{feedback loop}
  inherent to autoregressive generation, whereby outputs sampled from
  the learned model depend on tokens previously generated by
    the same (possibly imperfect) model. This can
  lead to the widely-observed phenomenon known as \emph{error
    amplification} (or \emph{exposure
  bias}), where small inaccuracies in the conditional model
$\pihat(a_h\mid{}a_{1:h-1})$ compound, leading to out-of-distribution
sequences with poor performance on downstream tasks of interest
\citep{holtzman2019curious,braverman2020calibration,arora2022exposure,block2023butterfly};
some have speculated this to be a fundamental limitation \citep{lecun2023large,bachmann2024pitfalls}.

\ah{I think Akshay makes a good point that perhaps we should give ourselves more credit for 
(now in my words) "showing for the first time that log-loss with misspecification causes error compounding, 
and that offline algorithms are sufficient to achieve good guarantees" }
\ah{My other comment is that it could help our delivery to combine, rather than differentiate (as it currently reads to me), 
the causes for misspecification in robotics vs language modeling.}
Next-token prediction can be seen as a special case of
    \emph{behavior cloning}, a fundamental approach to the more general
    problem of \emph{imitation learning} (IL) \citep{pomerleau1988alvinn}, for which similar compounding errors
    (e.g., a learned policy
    for a self-driving car
    slowly drifting off of the road) have been observed    \citep{ross2010efficient,laskey2017dart,block2023butterfly}. Here,
    a growing body of empirical work
\citep{bansal2018chauffeurnet,de2019causal,spencer2021feedback}
    suggests that error amplification may
    arise from \emph{misspecification}, where the learned policy is not
    sufficiently powerful to represent the target policy. For example, in applications of IL to robotics, there may be issues of partial
    observability or privileged information \citep{de2019causal}\arxiv{---e.g., if the conditional
    distribution $\pistar_h$ depends on the full history, but the
    model $\pihat_h$ is Markovian---} and in language modeling,
    misspecification may arise when using a
    model of limited capacity to represent a
    complex distribution (e.g., the distribution over all
    text on the internet) \citep{braverman2020calibration}, or when trying to distill a powerful
    teacher into a weaker student \citep{touvron2023llama,team2024gemini}. However, there is little theoretical understanding of the impact of misspecification in IL.\loose

    In this work, we draw inspiration from the \IL{}
    literature \citep{ross2010efficient,rajaraman2020toward,rajaraman2021provably,block2024provable,foster2024behavior}, and
    quantify error amplification through the effect of
    \emph{horizon} (sequence length) on model performance. Through
    this lens, recent work of
\citet{foster2024behavior} shows that in the absence of
    misspecification, next-token prediction with the log-loss can avoid error
    amplification entirely. Yet, under misspecification, there are simple problem instances (cf. \cref{cor:unbounded}) where it %
    fails to learn a non-trivial model,
    even when a good model exists and optimization
    error is not a concern. %
    This motivates us to investigate \emph{whether error amplification is fundamental in autoregressive sequence modeling and \IL{} under misspecification}. Concretely, we ask whether next-token prediction with the log-loss suffers from:\loose
\begin{itemize}[noitemsep, topsep=0.25pt]%
      \item[\textbf{(a)}] An \emph{algorithmic} limitation, which we can hope to 
        mitigate through (efficient)
        algorithmic interventions alone (e.g., by modifying the
        next-token prediction objective)?%
      \item[\textbf{(b)}] A \emph{computational} limitation, in the sense that
        there is enough information in the training data to avoid error amplification, but extracting it is
        computationally intractable?%
      \item[\textbf{(c)}] An \emph{information-theoretic/statistical} limitation, in the
      sense that there is simply not enough information in the
      training data to
      avoid error amplification?%
      \loose
    \end{itemize}
We show that error amplification is information-theoretically
avoidable; moreover, non-trivial algorithmic interventions to
next-token prediction \emph{are} possible, but there is a fundamental
limit to the improvement that can be achieved by efficient algorithms, at what we call the \emph{next-token prediction barrier}.

\arxiv{\subsection{Error Amplification in Next-Token Prediction under
  Misspecification}}

For the exposition, we focus on autoregressive sequence modeling, and defer discussion of the more general \IL{} setting to \cref{sec:setting}. %
The goal is to learn a
conditional distribution/model $\pistar: \cX \to \Delta(\cA^H)$, where $\cX$
is the context space, $\cA$ is a token space, and $H$ is the horizon. By Bayes' rule, any model $\pi:\cX\to\Delta(\cA^H)$ can be represented autoregressively \arxiv{in terms of $H$ token-level conditional distributions $\pi_h: \cX\times\cA^{h-1} \to \Delta(\cA)$:\loose
\begin{align}
  \pi(a_{1:H}\mid{} x) = \prod_{h=1}^H \pi_h(a_h\mid{}
  x,a_{1:h-1}).
  \label{eq:bayes}
\end{align}}%
For a fixed context distribution $\mu \in \Delta(\cX)$ and any model
$\pi$, we write $\BP^\pi$ to denote the distribution of sequences $(x,
a_{1:H})$ induced by sampling $x \sim \mu$ and $a_{1:H} \sim
\pi(\cdot\mid{}x)$. 

Given a model class $\Pi$ (represented by, e.g., transformers or
other deep networks) and a dataset
$\crl*{(x\ind{i},a_{1:H}\ind{i})}_{i=1}^n$ assumed to be sampled i.i.d. from $\BP^{\pistar}$, \emph{next-token prediction with the logarithmic loss}
(e.g., \citet{radford2019language}) solves the following optimization problem: 
\begin{align}
  \label{eq:bc}
  \pihat \in \argmax_{\pi\in\Pi}
  \sum_{i=1}^{n}\sum_{h=1}^{H}\log\prn*{\pi_h(a_h\ind{i}\mid{}x\ind{i},a_{1:h-1}\ind{i})}.
\end{align}
For the more general imitation learning setting, this
  coincides with \emph{behavior cloning} (\cref{sec:basic}).
As noted by
\citet{foster2024behavior}, the objective in \cref{eq:bc} is equivalent
to maximum likelihood estimation (MLE) over the distribution family
$\crl*{\bbP^{\pi}}_{\pi\in\Pi}$, so that standard MLE guarantees
\arxiv{\citep{wong1995probability,Sara00,zhang2006from}} imply convergence in
Hellinger distance---a standard metric for distribution estimation
defined via $\Dhels{\bbP}{\bbQ}=\int(\sqrt{\mathrm{d}\bbP}-\sqrt{\mathrm{d}\bbQ})^2$---when the problem is \emph{realizable/well-specified} in the sense
that $\pistar\in\Pi$:\arxiv{\footnote{This result follows from a well-known
  connection between the moment generating function for the
  logarithmic loss and Hellinger distance\arxiv{ (and other Renyi
  divergences)}. Importantly, this holds with no assumption on
  boundedness of the densities.\loose}}
\begin{proposition}[\citet{foster2024behavior}]
      \label{prop:mle_finite}
      Whenever $\pistar\in\Pi$, the estimator $\pihat$ in \cref{eq:bc} satisfies that $\DhelsX{\big}{\bbP^{\pihat}}{\bbP^{\pistar}} \leq
  2\log(\abs{\Pi}\delta^{-1})/n$ with
      probability at least $1-\delta$.\footnote{For simplicity, we work with finite classes
        $\Pi$, following a
        common convention in reinforcement learning theory
        \citep{agarwal2019reinforcement}.  \cref{prop:mle_finite} (and later results) extends to infinite classes via standard covering arguments.
      }\loose
\end{proposition}

This result yields \emph{horizon-independent} guarantees on generation
performance (as long as the expressivity of $\Pi$ is controlled,
e.g. via parameter sharing). Namely, for any function
$r(x,a_{1:H})\in\brk*{0,1}$ measuring quality of generated sequences (e.g.,\arxiv{ text coherence,
chatbot quality, or} correctness of generated proofs\arxiv{ or code}), we have\loose
\begin{align}
  \label{eq:rollout}
  \En_{\pihat}\brk*{r(x,a_{1:H})} \geq
  \En_{\pistar}\brk*{r(x,a_{1:H})}
  - \DhelX{\big}{\bbP^{\pihat}}{\bbP^{\pistar}},
\end{align}
so by \cref{prop:mle_finite}, the quality improves as $n\to\infty$, with no dependence on the horizon $H$. %
\loose

\paragraph{Error amplification  under misspecification}
Unfortunately, if the model class is misspecified, i.e. $\pistar \not
\in \Pi$, the guarantees above break down. A trivial failure occurs when densities for models in $\Pi$ are not bounded away from $0$, allowing
the loss in \cref{eq:bc} to take value $-\infty$, and leading to arbitrarily bad performance.\footnote{If $\BP^{\pistar}$ is $\veps$-close to $\Pi$ in
  \emph{\chis-divergence}, then next-token prediction can 
  avoid error amplification. Concretely,
    \citet{foster2024behavior} show that \cref{eq:bc} achieves $\DhelsX{\big}{\bbP^{\pihat}}{\bbP^{\pistar}}
  \approxleq{} \frac{\log(\abs{\Pi}\delta^{-1})}{n} +
  \inf_{\pi\in\Pi}\DchisX{\big}{\bbP^{\pi}}{\bbP^{\pistar}}$.
  However, 
  \chis-divergence can be infinite even when Hellinger distance is small.} A more troubling issue is that it can be the case (cf. \cref{prop:log-loss-lb}) that all $\pi\in\Pi$ have well-behaved densities, yet the estimator $\pihat$ in \cref{eq:bc} incurs explicit horizon dependence: 
    \begin{equation}
      \Dhels{\BP^{\pihat}}{\BP^{\pistar}}\geq \Omega(\min(1, \veps^2 H)) \quad\text{while}\quad \min_{\pi\in\Pi}\Dhels{\BP^{\pi}}{\BP^{\pistar}} \leq \veps^2.
      \label{eq:failure-h}
    \end{equation}
  That is,
  even though the best model in $\Pi$ is
    $\veps$-suboptimal with respect to generation performance (via
    \cref{eq:rollout}), next-token prediction with the log-loss yields a model whose
    generation performance degrades with $H$---a marked departure from the well-specified setting. One of our initial contributions is a sharp characterization of this phenomenon.

    \arxiv{
      \begin{remark}[Connection to imitation learning]
      In imitation learning (IL), the goal is to learn a \emph{policy}
      $\pihat$ that matches the distribution of an \emph{expert
        policy} in a Markov decision process. Autoregressive
      sequence modeling can be viewed as a special case of this problem, associating sequence models with policies in a \emph{token-level MDP}, and the next-token prediction objective in \cref{eq:bc} is a special
      case of \emph{behavior cloning}, the most basic and widely used
      algorithm in IL.
       Understanding the impact of horizon/sequence length on
performance is a central theme in \IL{}
\citep{ross2010efficient,rajaraman2020toward,rajaraman2021provably,foster2024behavior}.
 Further, as discussed in \cref{sec:setting}, the estimation in Hellinger distance is directly connected to \IL{}
performance. While
      we focus on autoregressive modeling in this section for the
      purpose of exposition, we present our main results in sections
      that follow in the general \IL{} framework; see
      \cref{sec:setting} for a formal overview. \loose\end{remark} %
        }

\begin{remark}[Terminology for next-token prediction]
      \label{rem:next_token}
\arxiv{  Throughout the paper, we} use the term \emph{next-token prediction}
  to refer to the broader paradigm of minimizing any sum of token-wise or per-timestep loss
  functions. Next-token prediction with the logarithmic loss \ahreplace{, defined in \cref{eq:bc},}{(\cref{eq:bc})} represents the most widely
  used instantiation.\arxiv{ \citet{foster2024behavior} show that
  the logarithmic loss enjoys benefits in horizon
  dependence over other standard losses (e.g., square or indicator)
  even in the well-specified setting, motivating our focus on it in
  this exposition.}\loose
\end{remark}

\subsection{Our Question: Agnostic Guarantees for Hellinger Distance}

With the goal of mitigating error amplification (i.e., avoiding the failures discussed above), we ask whether it is possible to %
achieve \emph{agnostic estimation} guarantees with respect to
sequence-level Hellinger distance. Concretely, consider any
model class $\Pi$, and let $\pistar$ be an unknown model
which may or may not lie in $\Pi$. We would like a learning algorithm
that---given $n$ \iid trajectories drawn from $\bbP^{\pistar}$---produces $\pihat$ satisfying the following agnostic estimation guarantee with high probability:\loose
\begin{equation} \Dhels{\BP^{\pihat}}{\BP^{\pistar}} \leq \Capx \cdot \min_{\pi\in\Pi} \Dhels{\BP^{\pi}}{\BP^{\pistar}} + \vepsstatsn.
\label{eq:hels-goal-intro}
\end{equation}
Here, $\vepsstatsn$ represents statistical error with
$\vepsstatsn\to{}0$ as $n\to\infty$, and should ideally be not much
larger than in the well-specified setting (i.e.,
$\vepsstatsn\approxleq\frac{\log(\abs{\Pi}\delta^{-1})}{n}$ for a
finite class). Meanwhile, $\vepsmis^2\ldef{}\min_{\pi\in\Pi} \Dhels{\BP^{\pi}}{\BP^{\pistar}}$ represents \emph{irreducible error} for
estimation, since any proper learning algorithm selecting
$\pihat\in\Pi$ must (trivially) have
$\Dhels{\BP^{\pihat}}{\BP^{\pistar}} \geq  \vepsmis^2.$
The parameter $\Capx\geq{}1$ is an \emph{approximation
ratio}; if $\Capx=1$, then $\pihat$ is no worse at approximating
$\pistar$ than the best model in $\Pi$
asymptotically, but this may be too much to ask (for either statistical
or computational reasons).

By \cref{eq:failure-h}, next-token prediction with the log-loss incurs $\Capx \geq \Omega(H)$ even for well-behaved $\Pi$; it incurs $\Capx=\infty$ in the worst case (cf. \cref{cor:unbounded}). %
Restating our central question, we ask: \emph{what is the tightest
  approximation ratio $\Capx$ that can be achieved} \textbf{(a)} via practical interventions to the next-token prediction
    objective; \textbf{(b)}  via any computationally efficient algorithm; and \textbf{(c)} via \emph{any} algorithm, irrespective of
    computational efficiency?

\paragraph{Computational testbed: Autoregressive linear models}
To formalize questions of computational efficiency, our testbed 
will be the class $\Pi$ of \emph{autoregressive linear models}, defined by a known feature map $\phi: \MX \times \MA^\star \to \RR^d$. For each parameter $\theta \in\Theta\subset \RR^d$, the model $\pi_\theta=(\pi_{\theta,h})_{h=1}^H$ is defined by 
\begin{equation}
  \label{eq:linear}
  \pi_{\theta,h}(a_h\mid{}x,a_{1:h-1}) \propto \exp(\langle \theta,
    \phi(x,a_{1:h})\rangle).
\end{equation}
Recall that in practice \citep{radford2019language}, autoregressive
sequence models (e.g., based on transformers) typically generate each token by sampling from a
softmax distribution determined by a linear combination of learned
features. \cref{eq:linear} is a simplification
where we freeze the features, but it can still capture rich
non-Markovian structure (depending on the choice of feature map).
In this setting, the log-loss objective (\cref{eq:bc}) is concave in parameter space with efficiently computable
gradients, so it can be efficiently optimized. In conjunction with \cref{prop:mle_finite} (generalized to infinite model classes), it follows that learning \emph{well-specified} autoregressive linear models is end-to-end computationally tractable, under appropriate norm bounds.

\begin{proposition}[informal; see \cref{cor:linear-wellspec-logloss}]\label{prop:gdlogloss-intro}
Let $\Pi := \{\pi_\theta: \theta\in\Theta\}$ for a convex set $\Theta\subseteq \RR^d$. Given $n$ i.i.d. samples from $\BP^{\pistar}$ for some $\pistar \in\Pi$, projected gradient ascent on \cref{eq:bc} can be implemented in time $\poly(n,d,H,|\MA|)$ and yields $\wh\theta \in \Theta$ such that, \arxiv{with high probability}, %
\arxiv{
\[\Dhels{\BP^{\pi_{\wh\theta}}}{\BP^{\pistar}}
\leq \widetilde{O}\left(\frac{d}{n}\right).\]
}
\end{proposition}

This algorithm can still be \arxiv{efficiently implemented} when $\pistar \not
\in \Pi$, but may suffer from the statistical issues in the
prequel a priori; even in this concrete setting, the computational-statistical tradeoffs are unclear.

\subsection{Contributions}

We illuminate
the computational-statistical tradeoffs inherent to autoregressive
sequence modeling and imitation learning under misspecification. While error
amplification can be avoided information-theoretically
($\Capx\ll{}H$), the regime $\Capx=\bigom(H)$ represents a fundamental barrier that no computationally efficient algorithm
can substantially surpass.
Our results apply to both next-token prediction and imitation learning, which we formally introduce and relate in \cref{sec:setting}. %
\loose

\vspace{0.3em}
\noindent\textbf{The statistical gold standard avoids error amplification (\cref{sec:optimal}).}
As a starting point that motivates our main
results, we show that the $\rho$-estimator of \cite{baraud2017new,baraud2018rho} can be applied in the general imitation
learning setting, which addresses question \textbf{(c)} above: information-theoretically,
\cref{eq:hels-goal-intro} is achievable with
$\Capx=O(1)$. %
Unfortunately,
the $\rho$-estimator is computationally impractical compared to traditional methods, as it requires min-max optimization.\loose

\vspace{0.3em}
\noindent\textbf{Robustifying the log-loss, and a barrier to further improvement (\cref{sec:next_token}).}
Toward practical algorithms that mitigate error amplification,
we explore whether better bounds on $\Capx$ can be achieved by
modifying the log-loss in imitation learning and next-token prediction (i.e., question \textbf{(a)} above).
First, we give sharp upper and lower bounds on the performance of
the log-loss, revealing that $\Capx$ depends not just on the horizon $H$, but also on (i) the failure
probability $\delta$, and (ii) a lower bound on the densities of $\pi\in\Pi$. 
We alleviate dependence on (i) via a new cross validation procedure, and dependence on (ii) by smoothing the objective, 
given access to per-timestep \emph{\densobs} from the expert model $\pistar$.
These results constitute a practical method that achieves $\Capx=\bigoht(H)$, and we uncover a fundamental barrier to further
improvement: \emph{any} next-token
prediction objective (cf. \cref{rem:next_token}), including those used in online imitation learning algorithms, 
must suffer $\Capx = \Omega(H)$.\loose

\vspace{0.3em}
\noindent\textbf{Computational-statistical tradeoffs at the next-token prediction barrier (\cref{sec:computational}).}
Can clever algorithm design circumvent the
$\Capx=\bigom(H)$ barrier, without sacrificing computational
efficiency (cf. question \textbf{(b)} above)? To make the question concrete, we focus on autoregressive
linear models, where our preceding improvements to next-token
prediction achieve $\Capx=\bigoht(H)$
in polynomial time. On the negative side, we show that achieving $\Capx =2^{\log^{1-\Omega(1)}(H)}$ is computationally hard under a standard
cryptographic assumption.
On the positive side, we show that it \emph{is}
possible to smoothly trade computation for statistical power, at least when $\abs{\cA}=2$: for
any constant $K$, there is a polynomial-time algorithm 
with $\Capx \leq \lceil H/K\rceil$; this is achieved through an
\emph{improper} relaxation to the $\rho$-estimator based on kernel
approximation \citep{shalev2011learning}.\loose

  \arxiv{
Taken together, we view our results as a promising first step toward a
computational theory of autoregressive sequence modeling and imitation
learning; we highlight open problems and future
directions in \cref{sec:conclusion}.\loose
}

\section{Problem Setting: Autoregression and Imitation Learning}\label{sec:setting}

As mentioned in the prequel, we present our main results in a general 
\emph{imitation learning} (IL) setting which encompasses autoregressive sequence modeling as a special case. %
This allows us to present
  our results---which we expect to find broader use in \IL{}---in the most general form possible.\loose

\vspace{0.3em}
\noindent\textbf{Basic notation.}
  For an integer $n\in\bbN$, we let $[n]$ denote the set
  $\{1,\dots,n\}$. For a set $\cX$, we let $\Delta(\cX)$ denote the
  set of all probability distributions over $\cX$. We adopt
    standard big-oh notation and write $f=\bigoht(g)$ to denote that
    $f = \bigoh(g\cdot{}\max\crl*{1,\mathrm{polylog}(g)})$ and
    $a\approxleq{}b$ as shorthand for $a=\bigoh(b)$. \loose

\vspace{0.3em}
\noindent\textbf{Markov decision processes.} We consider \IL{}{} 
in a (reward-free) Markov decision process (MDP) given by a tuple $M = (H, \MS,\MA,
(\BP_h)_{h \in \crl{0,\ldots,H-1}})$ where $\MS$ is
the\arxiv{ (potentially large)} \ahdelete{\emph}{state space}; $\MA$ is the
\ahdelete{\emph}{action space}; $\BP_0 \in \Delta(\MS)$ is the \ahdelete{\emph}{initial
  state distribution}; and for each $h \in [H-1]$, $\BP_h: \MS \times \MA
\to \Delta(\MS)$ is the \ahdelete{\emph}{transition distribution} at step $h$.
A (randomized) \ahdelete{\emph}{policy} $\pi$ is a
collection of mappings $\pi_h: \MS \to \Delta(\MA)$ for $h \in [H]$,
with $\pi_h(a_h \mid{} s_h)$ denoting the density of $\pi_h(s_h)$ at
$a_h$. Each policy $\pi$ in the MDP $M$ induces a distribution
$\BP^{\pi}$ over \ahdelete{\emph}{trajectories} $(s_1,a_1,\dots,s_H,a_H)$ defined
as follows. First, sample $s_1 \sim \BP_0$. Then, for each $1 \leq h <
H$, sample $a_h \sim \pi_h(\cdot\mid{}s_h)$ and $s_{h+1} \sim
\BP_h(s_h,a_h)$. For any \arxiv{real-valued }function $f$ on trajectories, we write
$\EE^{\pi}[f]$ to denote the expectation of $f(s_{1:H},a_{1:H})$
under $(s_{1:H},a_{1:H}) \sim \BP^{\pi}$.\loose

Our running example will be the \emph{autoregressive MDP}. For a context space $\MX$ (with context distribution $\mu\in\Delta(\MX)$), token space $\MA$, and $H \in \NN$, the $H$-step autoregressive MDP has state space $\MX\times\MA^\st$ and action space $\MA$, where $\MA^\st$ is the set of all finite-length strings formed by concatenation of elements of $\MA$. The initial distribution is $\mu$, and the transition dynamics are defined by deterministic concatenation: $s_{h+1} \gets (s_h,a_h) = ((x,a_{1:h-1}),a_h)$. The autoregressive MDP is \emph{accretive}, in the sense that $(s_1,a_1),\dots,(s_{h-1},a_{h-1})$ is a measurable function of $s_h$.

\vspace{0.3em}
\noindent\textbf{Imitation learning (IL).} In (offline) imitation learning \citep{pomerleau1988alvinn,ross2010efficient,foster2024behavior}, we are given a dataset $\cD=\crl*{o\ind{i}}_{i=1}^{n}$ of $n$
trajectories
$o\ind{i}=(s_1\ind{i},a_1\ind{i}, \ldots, s_H\ind{i},a_H\ind{i})$
sampled \iid by executing an \emph{expert policy}
$\pistar=\crl*{\pistar_h:\cS\to\Delta(\cA)}_{h=1}^{H}$ in the
underlying MDP. For an (unknown) reward function $r_h: \MS\times\MA
\to \RR$ measuring quality at some task of interest, the goal of
\IL{} is typically formulated as \emph{regret
  minimization}: Given a policy class $\Pi$,  we aim to learn a policy
$\pihat$ such that the \emph{regret} $J(\pistar;r) - J(\pihat;r)$ is
minimized; here $J(\pi;r) :=
\EE^{\pi}\brk[\big]{\sum_{h=1}^H r_h(s_h,a_h)}$ denotes the \emph{value}
of the policy $\pi$ in MDP $M$. %
\ahreplace{ 
We emphasize that the MDP $M$ itself (i.e., the transition
distribution) is not known to the learner in this framework.
}{Neither the MDP nor its transitions are known to the learner.
}

\subsection{Equivalence of Regret Minimization with Distribution Learning}\label{sec:regret-vs-learning}

A priori, regret minimization seems unrelated to the task of minimizing Hellinger distance. %
However, since the
rewards are never observed by the learner in the \IL{}{} protocol, it turns out that there is a close connection.
Concretely, suppose the rewards $r$ are normalized so that
$\sum_{h=1}^{H}r_h\in\brk*{0,R}$ for a parameter $R>0$. We refer to
such a reward function as \emph{$R$-bounded}, and for simplicity take $R=1$. Then as observed by \cite{foster2024behavior}, for any accretive MDP, it holds that %
\begin{align}
  \label{eq:tv_equivalence}
  \sup_{r: \text{$1$-bounded}}
  \crl*{J(\pistar;r) - J(\pihat;r)}
  = \DtvX{\big}{\bbP^{\pihat}}{\bbP^{\pistar}},
\end{align}
where
$\Dtv{\bbP}{\bbQ}=\frac{1}{2}\int\abs{\mathrm{d}\bbP-\mathrm{d}\bbQ}$
is the\arxiv{ total variation }distance. Thus, \IL{}{} is a form of structured distribution learning where we aim
to learn the law of the trajectory induced by $\pistar$.

\vspace{0.3em}
\noindent\textbf{Hellinger vs total variation.} \cref{eq:tv_equivalence} suggests minimizing TV-distance. However, 
Hellinger distance is
equivalent up to a quadratic factor
($\Dtvs{\bbP}{\bbQ}\approxleq\Dhels{\bbP}{\bbQ}\approxleq\Dtv{\bbP}{\bbQ}$),
so the guarantee from \cref{eq:hels-goal-intro} does
approximately minimize TV-distance when $\vepsmis^2 \geq
\Omega(1)$. We focus on Hellinger distance because it leads to a tighter statistical theory---see \cref{sec:comparing} for additional motivation---but we do not see this as a critical conceptual distinction. The key point is that via \cref{eq:tv_equivalence}, any agnostic estimation error bound
as in \cref{eq:hels-goal-intro} leads to a bound on regret of order
$\vepsstat(n) + \Capx^{1/2}\vepsmis$.\footnote{As discussed in \cref{sec:comparing}, tighter variance-dependent bounds are also possible.} Notably, such a bound depends on
$\vepsstat(n)$ and $\vepsmis$ in a \emph{horizon-independent} fashion
whenever $\Capx=\bigoh(1)$, motivating our goal of mitigating error amplification. %

\vspace{0.3em}
\noindent\textbf{Autoregressive sequence modeling as \IL{}.} With the
perspective above, the autoregressive sequence modeling formulation in
\cref{sec:intro} is simply
\IL{}{} in the autoregressive MDP. %
Each model $\pi:\cX\to\Delta(\cA^H)$ in the model class is a policy $(\pi_h)_{h=1}^H$ in the policy class, where $\pi_h: \cX\times\cA^{h-1}\to\Delta(\cA)$ is the conditional distribution $\pi(a_h\mid{}x,a_{1:h-1})$. In the same way, the true model $\pistar$ is the expert policy. The learned policy $\pihat$ yields a model via autoregressive generation on any initial context. %
\loose
\subsection{Basic Algorithms: Next-Token Prediction and Behavior
  Cloning}
\label{sec:basic}

The next-token prediction objective in \cref{eq:bc}
specializes a canonical IL algorithm, \emph{behavior cloning with the
  logarithmic loss} ($\loglossbc$) to the autoregressive setting. For
general \IL{} with a policy class $\Pi$, $\loglossbc$ takes as input trajectories $o\ind{i} = (s_1\ind{i},a_1\ind{i},\dots,s_H\ind{i},a_H\ind{i})$, and outputs the policy\loose
  \begin{equation}
    \label{eq:bc_general}
    \pihat := \argmax_{\pi \in \Pi} \wh L(\pi) \quad\text{where}\quad
  \wh L(\pi) := \sum_{i=1}^n \sum_{h=1}^H \log 
  (\pihi).\end{equation}
$\loglossbc$
  enjoys the guarantee in \cref{prop:mle_finite} for arbitrary MDPs \citep{foster2024behavior}.

\section{An Inefficient Algorithm with Optimal Misspecification Tolerance}
\label{sec:optimal}

We first ask if it is possible to avoid error amplification \emph{information-theoretically}, irrespective of computational practicality. We find that the \emph{$\rho$-estimator}, a recent agnostic estimation technique from the statistics literature \citep{baraud2018rho}, yields an imitation learning algorithm that achieves near-optimal misspecification tolerance (i.e., achieves $\Capx=\bigoh(1)$), while matching the performance guarantee for \loglossbc in \cref{prop:mle_finite} in the well-specified setting.
For a policy class $\Pi$, we define \emph{$\rho$-estimator behavior cloning} ($\rhobc$) to be the algorithm that, given trajectories $\cD=\crl*{o\ind{i}}_{i=1}^{n}$ with $o\ind{i} = (s_1\ind{i},a_1\ind{i},\dots,s_H\ind{i},a_H\ind{i})$, returns\footnote{The $\rho$-\emph{estimator} is named thus in \citet{baraud2018rho} because it is presented in terms of the function $\rho(x) \ldef \tau(1/x^2)$.  We find the current parameterization more convenient for our purposes.}  \loose
\begin{align}\label{eq:rhobc}
    \pihat := \argmin_{\pi \in \Pi} \sup_{\pi' \in \Pi} 
\sum_{i=1}^n \tau\left(
\prod_{h=1}^{H}\frac{\pihi}{\piphi}\right), \qquad \text{where } \qquad \tau(x) \ldef \frac{\sqrt{1/x} - 1}{\sqrt{1/x} + 1}.
\end{align}
We have the following guarantee, which shows that $\rhobc$ achieves $\Capx = \bigoh(1)$. %

\begin{theorem}\label{thm:rho-il}
Fix an MDP $M$, a policy class $\Pi$, and an expert policy $\pistar$. Let $n \in \NN$ and $\delta > 0$. \arxiv{Let $\cD=\{o\ind{i}\}_{i=1}^n$ be i.i.d. trajectories $o\ind{i} = (s_1\ind{i},a_1\ind{i},\dots,s_H\ind{i},a_H\ind{i})$ from $\bbP^{\pistar}$. Then the} policy $\pihat$ produced by $\rhobc$
satisfies, with probability at least $1-\delta$,{\setlength{\abovedisplayskip}{3pt}
\setlength{\belowdisplayskip}{3pt}
\begin{equation} \Dhels{\BP^{\pihat}}{\bbP^{\pistar}} \lesssim \frac{\log(|\Pi|\delta^{-1})}{n} + \min_{\pi \in \Pi} \Dhels{\BP^\pi}{\bbP^{\pistar}}.
\label{eq:rho-il-thm}\end{equation}}
\end{theorem}
We defer the proof of \Cref{thm:rho-il} to \Cref{app:rho}; briefly, it follows by applying a guarantee for $\rho$-estimators \citep{baraud2018rho} to the family of distributions $\cP=\{\BP^\pi: \pi \in\Pi\}$. Since%
\arxiv{\loose\[
  \frac{\BP^\pi(o)}{\BP^{\pi'}(o)} = \prod_{h=1}^H \frac{\pi_h(a_h\mid{} s_h)}{\pi'_h(a_h\mid{} s_h)}
\]}
 for any policies $\pi,\pi'$ and trajectory $o = (s_1,a_1,\dots,s_H,a_H)$, \cref{eq:rhobc} implicitly applies the $\rho$-estimator to $\cP$, in spite of the fact that the transition probabilities are unknown. The function $\tau(x)$ can be viewed as a better-behaved replacement for the negative log likelihood, $-\log(x)$, that (a) is uniformly bounded (allowing tight concentration under misspecification), yet (b) enjoys similar statistical properties,
\ahreplace{; in particular, the expectation of $\tau$ can be related to the Hellinger distance (\cref{lemma:rho-estimator-bounds}).}{because $\tau$, in expectation, can be related to the Hellinger distance (\cref{lemma:rho-estimator-bounds}).}\footnote{Note if we replace $\tau(x)$ with $-\log(x)$ in \cref{eq:rhobc}, \arxiv{the inner maximization problem becomes irrelevant and the algorithm} coincides with \loglossbc.}\loose

The $\rhobc$ algorithm has some similarity to recent work in imitation learning based on \emph{inverse reinforcement learning} (IRL) \citep{ho2016generative,ke2021imitation,swamy2021moments}\arxiv{; while the precise setting for these IRL-based algorithms is different, they also solve a minimax problem in order to minimize some $f$-divergence between the expert and learned policy. Compared to these works, which require online interaction with the MDP or knowledge of the dynamics, $\rhobc$ remains fully offline in the sense that \emph{no interaction with the MDP or expert is required}. Further, the derivation of the algorithm is somewhat different: IRL-style algorithms are typically derived from a variational representation for the $f$-divergence under consideration, while $\rhobc$---per the discussion above---is better understood as a smoothed or better-behaved generalization of maximum likelihood.}%

While the statistical performance of $\rhobc$ is essentially optimal, it is substantially less attractive when viewed through a computational lens: 
the product over $H$ ratios and the min-max optimization make it impractical compared to $\loglossbc$.\footnote{Interestingly, we show in \cref{sec:density} that the additional difficulty of a min-max objective (as opposed to a single minimization problem as in \cref{eq:bc}) can be overcome if we assume the learner has access to \emph{\densobs}\arxiv{, in which case the maximal $\pi'$ in \eqref{eq:rhobc} can be replaced by the true density $\pistar$}. %
}
Thus, while $\rhobc$ is our gold standard for statistical performance, we will need to look further for practical algorithms.\loose

\section{Next-Token Prediction under Misspecification: Improvements
  and Limitations}
\label{sec:next_token}

With $\rhobc$ as a statistical skyline, we return to the most widely-used \IL{} algorithm: behavior cloning with the logarithmic loss (\loglossbc; \cref{eq:bc_general}), an instance of next-token prediction for the general \IL{} setting.
We show that simple algorithmic tweaks can improve its performance substantially, but the performance of $\rhobc$ cannot be matched: $\Capx = \Omega(H)$ is a barrier for \emph{any} next-token prediction algorithm (cf. \cref{rem:next_token}). Proofs are deferred to \cref{app:ntp}.

\subsection{Sharp Analysis of Log-Loss Behavior Cloning}
\label{sec:logloss}

We start by giving a tight analysis for log-loss behavior cloning\arxiv{, with no modifications,} in the general IL setting.
While $\loglossbc$ can fail to achieve any finite approximation ratio in pathological examples (\cref{cor:unbounded}), we show that it achieves bounded \arxiv{approximation ratio} whenever density ratios of the form $\nicefrac{\pistar_h(a\mid{}s)}{\pi_h(a\mid{}s)}$ are bounded, an assumption satisfied in many settings including autoregressive linear models.\arxiv{ Formally, we consider the following assumption.}\loose %

\begin{definition}[Density ratio bound]\label{ass:density-bound}
For $\denbound \geq 2$, we say that a policy $\pistar$ is $\denbound$-bounded with respect to policy class $\Pi$ if \arxiv{
\[\max_{\pi \in \Pi} \max_{(s,a) \in \cS \times \cA} \max_{h \in [H]} \frac{\pistar_h(a\mid{}s)}{\pi_h(a\mid{}s)} \leq \denbound.\]}
\end{definition}

For example, if $\min_{s,a,h}\pi_h(a\mid{}s) \geq 1/\denbound$ for all $\pi\in\Pi$, then any policy $\pistar$ is $\denbound$-bounded with respect to $\Pi$. We show that $\loglossbc$ has approximation ratio roughly $\Capx \approx H\log \denbound$.\loose

\begin{theorem}\label{thm:log-loss-bounded}
Fix an MDP $M$, a policy class $\Pi$, and an expert policy $\pistar$. Suppose that $\pistar$ is $\denbound$-bounded with respect to $\Pi$ (\cref{ass:density-bound}) for some $\denbound \geq 2$. Let $n \in \NN$ and $\delta > 0$.\arxiv{ Let $\{o\ind{i}\}_{i=1}^n$ be i.i.d. trajectories $o\ind{i} = (s_1\ind{i},a_1\ind{i},\dots,s_H\ind{i},a_H\ind{i})$ from $\BP^{\pistar}$.} Then the policy $\pihat$ produced by $\loglossbc$ (\cref{eq:bc_general}) satisfies, with probability at least $1-\delta$,
\begin{equation} \Dhels{\BP^{\pihat}}{\bbP^{\pistar}} \lesssim \frac{\log(|\Pi|\delta^{-1})}{n} + \frac{\log{}\denbound\log(\delta^{-1})}{n} + \frac{H\log{}\denbound}{\delta} \cdot \min_{\pi \in \Pi} \Dhels{\BP^\pi}{\bbP^{\pistar}}.
\end{equation}
\end{theorem}
Concretely, the approximation ratio scales as $\Capx=\frac{H\log \denbound}{\delta}$ for failure probability $\delta$; note the polynomial rather than logarithmic scaling in $\delta^{-1}$. It is possible to avoid the dependence of $\Capx$ on $\delta^{-1}$ at the cost of an additional factor of $H$ in the statistical rate (cf. \cref{thm:log-loss-bounded-app}), but this horizon dependence may be undesirable. We remark that while $\loglossbc$ can be interpreted as maximum likelihood on trajectories, \arxiv{and hence analyzed directly at the sequence level, }\cref{thm:log-loss-bounded} is \emph{not} a corollary of existing analyses for maximum likelihood in terms of e.g., $\chi^2$-misspecification \citep[Proposition B.1]{foster2024behavior}: naively converting $\min_{\pi\in\Pi} \Dchis{\bbP^{\pistar}}{\bbP^{\pi}}$ to Hellinger misspecification via the density ratio bound would incur a factor of $\Capx \approx \denbound^H$. The proof of \cref{thm:log-loss-bounded} fundamentally uses the sequential structure of the IL setting.\loose

  \subsection{Robustifying Next-Token Prediction via Cross-Validation and Smoothing}
  \label{sec:improvements}

There are two shortcomings of \cref{thm:log-loss-bounded}, even ignoring the fact that the approximation ratio scales with $H$ (which, as we will show later, is essentially necessary). First, to get the optimal rate, we pay a factor of $1/\delta$ in the approximation ratio. Second, the theorem only holds under $\denbound$-boundedness. The following result
shows that both of these shortcomings are inherent, and not artifacts of the analysis.\loose

\begin{proposition}[Tightness of \creftitle{thm:log-loss-bounded}]\label{prop:log-loss-prob-lb}
  Fix any $H \in \NN$ and $\denbound \geq 2$ and $\delta \in (0,1/2)$, and set $n_0 := H\log{}\denbound$. There is an $H$-step autoregressive MDP $M$, a policy class $\Pi$ of size $|\Pi| = 2$, and an expert policy $\pistar$ such that $\pistar$ is $\denbound$-bounded with respect to $\Pi$ (\cref{ass:density-bound}), with the following property. Given $n_0$ i.i.d. trajectories $o\ind{i} = (x\ind{i},a_{1:H}\ind{i})$ from $\BP^{\pistar}$, the estimator $\pihat$ produced by $\loglossbc$ satisfies $\Dhels{\BP^{\pihat}}{\bbP^{\pistar}} \gtrsim 1$ with probability at least $\delta$, even though $\min_{\pi\in\Pi} \Dhels{\BP^\pi}{\bbP^{\pistar}} \lesssim \frac{\delta}{H\log R}.$
  \end{proposition}

This result shows that the tradeoff discussed after \cref{thm:log-loss-bounded} is tight: either $\Capx = \Omega(1/\delta)$, or the statistical rate must scale as $\Omega(\nicefrac{(H\log{} \denbound)}{n})$. Additionally, one can show that  the dependence of $\Capx$ on $H\log{}\denbound$ is necessary even as $n \to \infty$ (\cref{prop:log-loss-lb}). Next, we present two algorithmic modifications to $\loglossbc$ that avoid these shortcomings: (1) boosting the success probability via cross-validation, and (2) addressing unbounded density ratios through access to \emph{\densobs}.\loose
  
  \subsubsection{Boosting to High Probability via $\rho$-Estimator Cross-Validation}
We can boost \loglossbc to achieve a high probability guarantee 
  (without $\poly(1/\delta)$ dependence in the approximation factor, and without worsening the statistical rate)
  by first running \loglossbc on $K$ independent partitions of the data 
  to obtain an intermediate policy class $\Pi' = \crl*{\pihat_{1},\ldots,\pihat_K}$,
  then running \rhobc with $\Pi'$ to output the final policy $\pihat$.
  We call the resulting algorithm \boostedloglossbc.
  
\arxiv{Formally, given a parameter $\delta>0$, a policy class $\Pi$, and $n$ trajectories $o\ind{i} = (s_1\ind{i},a_1\ind{i},\dots,s_H\ind{i},a_H\ind{i})$, consider the algorithm $\boostedloglossbc$ defined by the following procedure:
\begin{enumerate}\arxiv{\setlength\itemsep{0.3em}} 
\item Divide the dataset $\cD=(o\ind{i})_{i=1}^n$ into $K := 2\log(2/\delta)$ disjoint equal-sized folds $\cD^1,\dots,\cD^K$.
\item For each $1 \leq i \leq K/2$, compute policy $\pihat^i$ by applying $\loglossbc$ with dataset $\cD^i$ and policy class $\Pi$.
\item Output the policy $\pihat$ obtained by applying $\rhobc$ with dataset $\cD^{K/2+1} \cup \dots \cup \cD^{K}$ and policy class $\Pi'=\{\pihat^1,\dots,\pihat^{K/2}\}$.
\end{enumerate}
The main guarantee for this algorithm is as follows.
}%

\begin{corollary}\label{prop:boosted-bc}
Fix an MDP $M$, a policy class $\Pi$, and an expert policy $\pistar$. Suppose that $\pistar$ is $\denbound$-bounded with respect to $\Pi$ (\cref{ass:density-bound}) for some $\denbound \geq 2$. \arxiv{Let $n \in \NN$ and $\delta \in (0,1/2)$. Let $\cD=\{o\ind{i}\}_{i=1}^n$ be i.i.d. trajectories $o\ind{i} = (s_1\ind{i},a_1\ind{i},\dots,s_H\ind{i},a_H\ind{i})$ from $\BP^{\pistar}$. Then the} policy $\pihat$ produced by $\boostedloglossbc$ satisfies, with probability at least $1-\delta$,
\begin{equation} \Dhels{\BP^{\pihat}}{\bbP^{\pistar}} \lesssim \frac{\log(|\Pi|\denbound)\log(\delta^{-1})}{n} + H\log{}\denbound \cdot \min_{\pi \in \Pi} \Dhels{\BP^\pi}{\bbP^{\pistar}}.
\end{equation}
\end{corollary}

\arxiv{We emphasize that }$\boostedloglossbc$ has minimal computational overhead over $\loglossbc$: while $\rhobc$ is computationally unattractive for general policy classes, for a \emph{finite} policy class of size $K$ it can be computed in time $O(K^2 nH)$ through enumeration; we take $K=\bigoh(\log(\delta^{-1}))$. As a result, $\boostedloglossbc$ can be implemented \emph{provably efficiently} for autoregressive linear models,  giving a baseline for the computational-statistical tradeoffs that we explore in \cref{sec:computational}---see \cref{cor:linear-misspec-logloss}.\loose

\subsubsection{Addressing Unbounded Densities via Smoothing}

Next, suppose that in addition to the usual expert trajectories, we have access to \emph{\densobs} of the form $\pistar_h(a\ind{i}_h\mid{}s\ind{i}_h)$. Such access is a natural assumption for the task of \emph{expert distillation}, where we aim to distill a large model $\pistar$ into a smaller model $\pihat$ \citep{hinton2015distilling}. Given access to such observations, we can remove the dependence on the density ratio $\denbound$ through the following algorithm, which we refer to as \smoothedloglossbc: For a parameter $\lambda\in(0,1)$, output the policy:{\setlength{\abovedisplayskip}{3pt}
\setlength{\belowdisplayskip}{3pt}
\begin{align}
  \label{eq:smoothed_bc}
  \pihat := \argmax_{\pi\in\Pi} \sum_{i=1}^n \sum_{h=1}^H \log ((1-\lambda)\pihi +\lambda \pistarhi).
  \end{align}}%
  $\smoothedloglossbc$ can be viewed as applying $\loglossbc$ to an augmented policy class $\Pi^\lambda$ where each policy is mixed with $\pistar$; this has some similarity to knowledge distillation objectives in the literature \citep{hinton2015distilling,lopez2015unifying}, but mixes the teacher's logits with the student's instead of mixing them with the labels. Since $\pistar$ is $(1/\lambda)$-bounded with respect to $\Pi^\lambda$---with no assumptions on the original policy class---\cref{thm:log-loss-bounded} implies the following improved guarantee.\loose
  \begin{corollary}
    \label{thm:layerwise_smoothing}
Fix an MDP $M$, a policy class $\Pi$, and an expert policy $\pistar$. Let $n \in \NN$ and $\delta \in (0,1/2)$.\arxiv{ Let $\cD=\{o\ind{i}\}_{i=1}^n$ be i.i.d. trajectories $o\ind{i} = (s_1\ind{i},a_1\ind{i},\pistar(a_1\ind{i}\mid{}s_1\ind{i}),\dots,s_H\ind{i},a_H\ind{i},\pistar(a_H\ind{i}\mid{}s_H\ind{i}))$ from $\BP^{\pistar}$.} The policy $\pihat$ produced by $\smoothedloglossbc$ with smoothing parameter $\lambda = 1/(H^2n)$ satisfies, with probability at least $1-\delta$,\loose
\begin{equation} \Dhels{\BP^{\pihat}}{\bbP^{\pistar}} \lesssim \frac{\log(|\Pi|\delta^{-1})}{n} + \frac{\log(Hn)\log(\delta^{-1})}{n} + \frac{H\log(Hn)}{\delta} \cdot \min_{\pi \in \Pi} \Dhels{\BP^\pi}{\bbP^{\pistar}}.
\end{equation}
\end{corollary}
We emphasize that the loss function in \cref{eq:smoothed_bc}---like the vanilla next-token prediction loss itself---is concave in policy space, though it may not be concave in parameter space in general. This estimator can also be boosted to succeed with high probability via cross-validation; we omit the details.

\subsection{A Barrier for Next-Token Prediction}
\label{sec:limits}

While cross-validation and smoothing mitigate certain shortcomings of $\loglossbc$, the main weakness remains: the approximation ratio $\Capx$ scales linearly in the horizon $H$. We now show that $\Capx = \Omega(H)$
cannot be surpassed by \emph{any} next-token prediction algorithm, i.e., any method that minimizes a sequence of token-level or per-timestep losses.
  Formally, we introduce the abstract notion of an \emph{iterative learner} that is given $\pistar$ directly, but is limited in how it can be used.
\begin{definition}\label{def:iterative}
  For a given MDP $M$ and policy class $\Pi$, an \emph{iterative learner} is an algorithm that, for any expert policy $\pistar$, produces an estimate $\pihat \in \Pi$ ``autoregressively'' as follows: for $h = 1,\dots,H$, it defines $\pihat_h$ as some (potentially randomized) function of $\pistar_{1:h}$ and $\pihat_{1:h-1}$.\arxiv{\footnote{Notably, the iterative learner can draw samples from $\pistar_{1:h}$ (it has full knowledge of the underlying MDP $M$) and compute any function thereof. We require the learner to be proper as otherwise, it could output $\pihat_h := \pistar_h$, since there is no statistical error.}}%
\end{definition}

This definition is most meaningful if the policy class $\Pi$ has no \emph{parameter sharing} across layers, i.e., there are families $\Pi_{1:H}$ so that $\pi\in\Pi$ if and only if $\pi_h \in \Pi_h$. In this case, any estimator defined by a loss function that decomposes additively across layers---including $\loglossbc$ and $\smoothedloglossbc$, but not $\rhobc$---is an iterative learner (\cref{prop:iterative-simulation}), though an iterative learner has additional flexibility (e.g., $\pihat_h$ may depend on $\pihat_{1:h-1}$ in some clever way). This flexibility
notwithstanding, we show that any iterative learner incurs linear dependence on $H$.

\begin{theorem}%
  \label{thm:ntp-lb}
  Fix $H \in \NN$. There is an $H$-step autoregressive MDP $M$ and a policy class
  $\Pi$ with no parameter sharing, so that for any iterative learner, there exists a policy $\pistar$ such that\loose
  {
\begin{equation} \EE\left[\Dhels{\BP^{\pihat}}{\BP^{\pistar}}\right] \geq \Omega(H) \cdot \min_{\pi \in \Pi} \Dhels{\BP^{\pi}}{\BP^{\pistar}}\label{eq:ntp-lb}\end{equation}
}where $\pihat$ is the (potentially random) output of the iterative learner, and $\min_{\pi \in \Pi} \Dhels{\BP^{\pi}}{\BP^{\pistar}} = 2^{-H}$. %
\end{theorem}

\cref{thm:ntp-lb} implies that $\Capx = \Omega(H)$ is a barrier for estimators defined by layer-wise loss functions, regardless of how many samples they are given.\footnote{Since $\log|\Pi|=\Omega(H)$ for any class with no parameter sharing, the rate term $\vepsstat(n)$ will scale with $H$ for any estimator, but \cref{thm:ntp-lb} holds in an infinite-data limit $n \to \infty$, so it is fundamentally a statement about $\Capx$.} In fact, since $\pihat_h$ may depend on $\pihat_{1:h-1}$, \cref{thm:ntp-lb} even applies to some \emph{online/interactive} imitation learning algorithms, e.g., \texttt{Forward} \citep{ross2011reduction}. The caveat of \cref{thm:ntp-lb} is that the misspecification in the construction is exponentially small; finding a stronger construction is an interesting technical question. We prove the result by embedding a ``consistency game'' in the learning task---see \cref{sec:limits-app}.%
\arxiv{ 

}
$\loglossbc$ (nearly) matches \cref{eq:ntp-lb} under either bounded densities or with smoothing. We remark that a \emph{layerwise} version of $\rhobc$ matches \cref{eq:ntp-lb} with no assumptions  (\cref{prop:layerwise-rho}), but unlike $\loglossbc$, it requires optimizing an objective that is non-convex even for autoregressive linear models.\loose

\section{Computational-Statistical Tradeoffs for Misspecification
  Tolerance}
\label{sec:computational}
Our results in \cref{sec:next_token} show that to beat the $\Capx=\Omega(H)$ barrier, we need to move beyond next-token prediction entirely. However, they leave open the possibility of a completely different algorithm that gets a better guarantee without sacrificing computational efficiency. To investigate this possibility, we restrict our focus to \arxiv{autoregressive sequence modeling, specifically to }the \emph{autoregressive linear models} defined in \cref{eq:linear}; working in this simple, concrete setting allows us to formalize questions of computational complexity.

\paragraph{Notation and computational framework} Fix sets $\MX,\MA$ with $|\MA|<\infty$, and parameters $d,H\in\NN$. Let $M$ be an $H$-step autoregressive MDP with context space $\MX$ and action space $\MA$. Let $\phi:\MX\times\MA^\st \to \RR^d$ be a given $d$-dimensional feature map, and let $\Theta\subset\RR^d$ be a convex parameter set. We consider the policy class $\Pi := \{\pi_\theta: \theta \in \Theta\}$ where $\pi_\theta = (\pi_{\theta,h})_{h=1}^H$ is the autoregressive linear policy defined as in \cref{eq:linear}. We assume that in $\poly(d,H)$ time, a learning algorithm can (i) query $\phi(x,a_{1:h})$ for any given $(x,a_{1:h}) \in \MX\times\MA^\st$ (with $h \leq H$), and (ii) compute the Euclidean projection of any point $\theta \in \RR^d$ onto $\Theta$. In addition, we assume the
following norm bounds.

\begin{assumption}[Norm bounds]\label{ass:linear-norm-bounds-main}
  Let $B, \Bdot \geq{}1$ be parameters. It holds that $\norm{\phi(x,a_{1:h})}_2, \norm{\theta}_2 \leq B$ and $|\langle \phi(x,a_{1:h}),\theta\rangle| \leq \Bdot$ for all $(x,a_{1:h}) \in \MX\times\MA^\st$ and $\theta\in\Theta$.\footnote{
    While $\Bdot\leq{}B$, the upper bounds we present for next-token prediction scale polynomially in $\Bdot$, yet logarithmically in $B$, so we separate these parameters to accommodate situations where $B\gg\Bdot$.
    }
\end{assumption}
$\boostedloglossbc$ is end-to-end computationally efficient in this setting. Moreover, any density of any policy in $\Pi$ can be lower bounded by $\frac{1}{|\MA|}\exp(-2\Bdot)$. Thus, a (straightforward) generalization of \cref{prop:boosted-bc} implies a guarantee for efficient learning in the presence of misspecification, where the approximation ratio scales with the horizon $H$ and the inner product bound $\Bdot$.\loose

\begin{proposition}\label{cor:linear-misspec-logloss}
  Suppose that \cref{ass:linear-norm-bounds-main} holds with parameters $B,\Bdot\geq 1$. There is a $\poly(n,d,H,|\MA|,B)$-time algorithm that takes $n$ i.i.d. samples $(x^i,a^i_{1:H})_{i=1}^n$ from $\BP^{\pistar}$ for any unknown policy $\pistar$, and outputs $\pihat\in\Pi$ so that with probability at least $1-\delta$, %
\arxiv{\begin{align}
         \Dhels{\BP^{\pihat}}{\BP^{\pistar}} &\lesssim \frac{(d\log(BHn) + \Bdot + \log|\MA|)\log(\delta^{-1})}{n} + (\Bdot+\log|\MA|)H \cdot \min_{\pi\in\Pi} \Dhels{\BP^\pi}{\BP^{\pistar}}.\end{align}
}
\end{proposition}

See \cref{sec:alm-gd} for the proof. Unfortunately, even for $\Bdot,|\MA| = O(1)$, the approximation ratio scales with $\Omega(H)$. In \cref{sec:comp-lb}, we show that this dependence cannot be improved substantially for polynomial-time algorithms, but in \cref{sec:comp-ub} we show that---at least when $|\MA| = 2$---there is a non-trivial \emph{trade-off} achievable between time complexity and approximation ratio.

\subsection{Computational Lower Bounds for Optimal Misspecification Tolerance}\label{sec:comp-lb}

Our main result for this section is a computational lower bound for learning autoregressive linear models based on hardness of \emph{Learning Parities with Noise (LPN)} (\cref{assumption:lpn}; see \cref{subsec:comp-lb-setting}).

\begin{theorem}\label{thm:comp-lb-main}
Suppose the sub-exponential decisional LPN hypothesis (\cref{assumption:lpn}) holds. Fix any $c,C>0$. Then \textbf{{no learning algorithm $\Alg$ has the following guarantee}}. Suppose $|\MA|=2$ and \cref{ass:linear-norm-bounds-main} holds with parameters $B=\sqrt{d}$ and $\Bdot=1$; then when given $T = (dH/\epsilon)^C$ i.i.d. samples from $\BP^{\pistar}$ for some unknown policy $\pistar$, the time complexity of $\Alg$ is $O(T)$ and the output is an $O(T)$-time conditional sampler for a policy $\pihat$ such that, with probability at least $9/10$,
\[\Dhels{\BP^{\pihat}}{\BP^{\pistar}} \lesssim \epsilon + e^{\log^{1-c} (\max(d,H))} \cdot \min_{\pi \in \Pi} \Dhels{\BP^\pi}{\BP^{\pistar}}.\]
\end{theorem}

\paragraph{Implications} \cref{thm:comp-lb-main} shows that, under a plausible cryptographic assumption (\citet{yu2021smoothing}; see \cref{subsec:comp-lb-setting} for details), it is impossible to dramatically bypass the next-token prediction barrier in polynomial time (concretely, the result rules out $\Capx\leq{}e^{\log^{1-c}(H)}$ when $d\geq{}H$). It also implies computational hardness of regret minimization for worst-case unknown reward (\cref{remark:il-hardness}). We emphasize that the result applies to \emph{improper} learners, i.e., $\pihat$ does not itself need to be autoregressive linear, but does leave open the possibility of achieving $\Capx=H^{c}$ for some $c<1$, or $\Capx=\poly(d)$.\footnote{However, our ultimate interest is in broader policy classes, and a learner with strong dependence on the dimension in the autoregressive linear setting seems unlikely to be more broadly applicable.} \arxiv{It also does not apply if the learner is given access to the conditional densities of $\pistar$ (the setting of $\smoothedloglossbc$).} We emphasize that since \loglossbc is provably efficient for the class $\Pi$ (\Cref{prop:gdlogloss-intro}), this result implies that, even if we assume access to an \emph{oracle} for maximum likelihood (a common approach when working with general function classes \citep{foster2021statistical,foster2023foundations}), there is no hope for an \emph{oracle-efficient} algorithm achieving a better approximation factor.
  
\loose

\paragraph{Proof overview}
To prove \cref{thm:comp-lb-main}, we adapt an argument of \cite{diakonikolas2022hardness} that gives LPN-based hardness of agnostic PAC learning for a neuron with softmax activation function (concretely, their result implies that for $H = 1$, the approximation ratio of any polynomial-time autoregressive learner must scale with $\Bdot$ when the dimension $d$ is large). Our construction is similar, but ``spreads'' the signal in the noisy parity distribution across the $H$ steps of the autoregressive sequence model. For each individual step, the conditional distribution is much closer to uniform, so we can take $\Bdot=O(1)$, thereby isolating the impact of $H$ on the approximation ratio from the impact of $\Bdot$. We defer a more detailed overview and the formal proof to \cref{app:comp-lb}.

\subsection{A Computational-Statistical Tradeoff for Autoregressive Linear Models}\label{sec:comp-ub}

An interesting question left open by \cref{thm:comp-lb-main} is whether there exist polynomial time algorithms that achieve approximation guarantees of the form $\Capx=H^{c}$ for $c\in(0,1)$, i.e., in the regime between the $\Capx=\bigom(H)$ barrier for next-token prediction and the sub-polynomial $\Capx\ll\poly(H)$ region ruled out by the theorem. For our final result, we give some positive evidence in this direction, showing that for the special case of autoregressive linear models with $\abs{\cA}=2$, there exists an efficient algorithm based on an improper relaxation of the $\rho$-estimator (the \emph{chunked, kernelized $\rho$-estimator}, or $\ChunkKR$) that achieves $\Capx=\bigoh(\nicefrac{H}{K})$ for any constant $K$.\loose

\begin{theorem}[Informal; see \cref{thm:chunk-kr-main}]\label{thm:chunk-kr-informal}
Fix $\MA = \{0,1\}$ and suppose that \cref{ass:linear-norm-bounds-main} holds with parameters $B,\Bdot$. There is an algorithm $\ChunkKR$ (\cref{alg:chunkkr}) with the following property. For any $\delta \in (0,1/2)$, $\epsilon \in (0,1)$, and $K \in [H]$, there is some $n = (2^{B^2+K}H/\epsilon)^{O(B^2 K)} \log(H/\delta)$ so that if $(x\ind{i},a\ind{i}_{1:H})_{i=1}^n$ are i.i.d. samples from $\BP^{\pistar}$ for some unknown $\pistar$, then with probability at least $1-\delta$, the output $\pihat\gets \ChunkKR((x^{(i)},a^{(i)})_{i=1}^n, \epsilon)$ is computed in time $\poly(n,H)$ and satisfies
\[\Dhels{\BP^{\pihat}}{\BP^{\pistar}}
\lesssim \epsilon + \frac{H}{K} \min_{\pi\in\Pi} \Dhels{\BP^{\pi}}{\BP^{\pistar}}.\]
\end{theorem}
For example, when $B=\bigoh(1)$, $\ChunkKR$ (with $K=H^{1/3}$) achieves a sublinear approximation factor $\Capx=\bigoh(H^{2/3})$ (beating next-token prediction) in subexponential time $e^{\bigoht(H^{2/3})}$, for $\eps\geq1/\poly(H)$.
\loose
\arxiv{We leave (i) a sharper understanding of computational-statistical tradeoffs, and (ii) developing similar tradeoffs for general classes $\Pi$ as directions for future work.}

\paragraph{Overview of algorithm design and proof techniques}
The $\ChunkKR$ algorithm in \cref{thm:chunk-kr-informal} uses two key algorithmic ideas: ``chunking'' the sequence into blocks \citep{chi2023diffusion,zhao2023learning,block2024provable}, and applying an improper, kernel-based relaxation to each chunk. The first idea\arxiv{, which may be of practical value and is reminiscent of tokenization,} is to learn the autoregressive model in chunks of size $K$: i.e., learn $\BP^{\pistar}(a_{iK+1:(i+1)K}\mid{}x,a_{1:iK})$ separately for each $i$. If, for each chunk, we can efficiently learn with approximation ratio $O(1)$, then by standard information-theoretic bounds, the combined model has $\Capx = O(H/K)$.\loose

With this insight, it remains to design an algorithm for learning misspecified autoregressive linear models with horizon $K \ll H$ that is efficient, yet achieves $\Capx=\bigoh(1)$---note that for this subproblem, we are allowed time complexity exponential in the horizon $K$ (but not in the dimension $d$). To achieve this, we implement $\rhobc$ via a generalization of the kernel-relaxation technique of \citet{shalev2011learning}, showing that we can approximately represent $\pi_{\theta}(a_{1:K}\mid{}x)$ as a function in an infinite-dimensional RKHS of bounded norm. %
After applying this relaxation, the $\rho$-estimator for each chunk becomes convex-concave in parameter space, and the resulting min-max program can be solved in polynomial time via projected gradient-descent-ascent (using the so-called ``kernel trick''). See \cref{app:comp-ub} for the full algorithm description and formal proof.\loose

\arxiv{
\section{Conclusion}
\label{sec:conclusion}
Our results highlight the computational-statistical tradeoffs inherent to autoregressive
sequence modeling and imitation learning under misspecification and
show that while some further improvement to the next-token prediction
objective may be possible, there is little hope of developing
efficient algorithmic interventions that offer substantial
improvement beyond the next-token prediction barrier at $\Capx=\bigom(H)$---at least in a worst-case sense. More broadly, we view
our results as a first step toward a
computational theory of autoregressive sequence modeling and imitation
learning. Natural questions for future research include:
\begin{itemize}
\item \emph{Beyond offline imitation learning.} To what extent can our hardness
  results for learning under misspecification be bypassed through additional side
  information or access
  to the expert? 
  For example, can online/interactive algorithms that do not correspond to iterative learners (\cref{def:iterative}) 
  bypass the
  $\Capx=\bigom(H)$ barrier for next-token prediction?
\item \emph{Computational-statistical tradeoffs for general policy
    classes.} Our computational-statistical tradeoff for
  autoregressive linear models in \cref{sec:computational} is achieved
  through rather specialized algorithmic techniques---particularly
  the use of kernel-based approximation. Is there any hope of
  efficiently achieving similar tradeoffs for general policy
  parameterizations (assuming, e.g., access to an oracle for maximum likelihood)?\loose
\item \emph{Beyond additive misspecification.} While additive
  misspecification is a simple and well-studied solution concept, it
  is not clear whether this notion is meaningful for 
  autoregressive sequence modeling applications like language model
  pre-training. Are there more natural notions of
  misspecification---possibly with different algorithm design
  principles---that allow for non-trivial guarantees even when
  additive misspecification is large or constant? 
  \end{itemize}

  \subsection*{Acknowledgements}
  We thank Sivaraman Balakrishnan and Cyril Zhang for helpful discussions.

}

\bibliography{refs}

\begin{thebibliography}{85}
\providecommand{\natexlab}[1]{#1}
\providecommand{\url}[1]{\texttt{#1}}
\expandafter\ifx\csname urlstyle\endcsname\relax
  \providecommand{\doi}[1]{doi: #1}\else
  \providecommand{\doi}{doi: \begingroup \urlstyle{rm}\Url}\fi

\bibitem[Acharya et~al.(2015)Acharya, Diakonikolas, Hegde, Li, and Schmidt]{acharya2015fast}
Jayadev Acharya, Ilias Diakonikolas, Chinmay Hegde, Jerry~Zheng Li, and Ludwig Schmidt.
\newblock Fast and near-optimal algorithms for approximating distributions by histograms.
\newblock In \emph{Symposium on Principles of Database Systems}, 2015.

\bibitem[Acharya et~al.(2017)Acharya, Diakonikolas, Li, and Schmidt]{acharya2017sample}
Jayadev Acharya, Ilias Diakonikolas, Jerry Li, and Ludwig Schmidt.
\newblock Sample-optimal density estimation in nearly-linear time.
\newblock In \emph{Symposium on Discrete Algorithms}, 2017.

\bibitem[Agarwal et~al.(2019)Agarwal, Jiang, and Kakade]{agarwal2019reinforcement}
Alekh Agarwal, Nan Jiang, and Sham~M Kakade.
\newblock Reinforcement learning: Theory and algorithms.
\newblock \url{https://rltheorybook.github.io/}, 2019.
\newblock Version: January 31, 2022.

\bibitem[Alekhnovich(2003)]{alekhnovich2003more}
Michael Alekhnovich.
\newblock More on average case vs approximation complexity.
\newblock In \emph{Symposium on Foundations of Computer Science}, 2003.

\bibitem[Applebaum et~al.(2009)Applebaum, Cash, Peikert, and Sahai]{applebaum2009fast}
Benny Applebaum, David Cash, Chris Peikert, and Amit Sahai.
\newblock Fast cryptographic primitives and circular-secure encryption based on hard learning problems.
\newblock In \emph{Advances in Cryptology}, 2009.

\bibitem[Arora et~al.(2022)Arora, El~Asri, Bahuleyan, and Cheung]{arora2022exposure}
Kushal Arora, Layla El~Asri, Hareesh Bahuleyan, and Jackie Chi~Kit Cheung.
\newblock Why exposure bias matters: An imitation learning perspective of error accumulation in language generation.
\newblock In \emph{Findings of the Association for Computational Linguistics}, 2022.

\bibitem[Bachmann and Nagarajan(2024)]{bachmann2024pitfalls}
Gregor Bachmann and Vaishnavh Nagarajan.
\newblock The pitfalls of next-token prediction.
\newblock \emph{arXiv:2403.06963}, 2024.

\bibitem[Bansal et~al.(2018)Bansal, Krizhevsky, and Ogale]{bansal2018chauffeurnet}
Mayank Bansal, Alex Krizhevsky, and Abhijit Ogale.
\newblock Chauffeurnet: Learning to drive by imitating the best and synthesizing the worst.
\newblock \emph{arXiv:1812.03079}, 2018.

\bibitem[Baraud and Birg{\'e}(2018)]{baraud2018rho}
Yannick Baraud and Lucien Birg{\'e}.
\newblock Rho-estimators revisited: {G}eneral theory and applications.
\newblock \emph{The Annals of Statistics}, 2018.

\bibitem[Baraud et~al.(2017)Baraud, Birg{\'e}, and Sart]{baraud2017new}
Yannick Baraud, Lucien Birg{\'e}, and Mathieu Sart.
\newblock A new method for estimation and model selection: $\rho$-estimation.
\newblock \emph{Inventiones mathematicae}, 2017.

\bibitem[Barnes(2023)]{barnes2023world}
Matt Barnes.
\newblock {World scale inverse reinforcement learning in Google Maps}.
\newblock \url{https://research.google/blog/world-scale-inverse-reinforcement-learning-in-google-maps/}, 2023.
\newblock [Online; accessed 26-Oct-2024].

\bibitem[Bilodeau et~al.(2023)Bilodeau, Foster, and Roy]{bilodeau2021minimax}
Blair Bilodeau, Dylan~J Foster, and Daniel~M Roy.
\newblock Minimax rates for conditional density estimation via empirical entropy.
\newblock \emph{Annals of Statistics}, 2023.

\bibitem[Birg{\'e}(2006)]{birge2006model}
Lucien Birg{\'e}.
\newblock Model selection via testing: an alternative to (penalized) maximum likelihood estimators.
\newblock In \emph{Annales de l'IHP Probabilit{\'e}s et statistiques}, 2006.

\bibitem[Block et~al.(2024{\natexlab{a}})Block, Foster, Krishnamurthy, Simchowitz, and Zhang]{block2023butterfly}
Adam Block, Dylan~J Foster, Akshay Krishnamurthy, Max Simchowitz, and Cyril Zhang.
\newblock Butterfly effects of {SGD} noise: Error amplification in behavior cloning and autoregression.
\newblock \emph{International Conference on Learning Representations}, 2024{\natexlab{a}}.

\bibitem[Block et~al.(2024{\natexlab{b}})Block, Jadbabaie, Pfrommer, Simchowitz, and Tedrake]{block2024provable}
Adam Block, Ali Jadbabaie, Daniel Pfrommer, Max Simchowitz, and Russ Tedrake.
\newblock Provable guarantees for generative behavior cloning: Bridging low-level stability and high-level behavior.
\newblock \emph{Advances in Neural Information Processing Systems}, 2024{\natexlab{b}}.

\bibitem[Blum et~al.(2003)Blum, Kalai, and Wasserman]{blum2003noise}
Avrim Blum, Adam Kalai, and Hal Wasserman.
\newblock Noise-tolerant learning, the parity problem, and the statistical query model.
\newblock \emph{Journal of the ACM}, 2003.

\bibitem[Bousquet et~al.(2019)Bousquet, Kane, and Moran]{bousquet2019optimal}
Olivier Bousquet, Daniel Kane, and Shay Moran.
\newblock The optimal approximation factor in density estimation.
\newblock In \emph{Conference on Learning Theory}, 2019.

\bibitem[Braverman et~al.(2020)Braverman, Chen, Kakade, Narasimhan, Zhang, and Zhang]{braverman2020calibration}
Mark Braverman, Xinyi Chen, Sham Kakade, Karthik Narasimhan, Cyril Zhang, and Yi~Zhang.
\newblock Calibration, entropy rates, and memory in language models.
\newblock In \emph{International Conference on Machine Learning}, 2020.

\bibitem[Bubeck(2015)]{bubeck2015convex}
S{\'e}bastien Bubeck.
\newblock Convex optimization: Algorithms and complexity.
\newblock \emph{Foundations and Trends{\textregistered} in Machine Learning}, 2015.

\bibitem[Cheng and Boots(2018)]{cheng2018convergence}
Ching-An Cheng and Byron Boots.
\newblock Convergence of value aggregation for imitation learning.
\newblock In \emph{International Conference on Artificial Intelligence and Statistics}, 2018.

\bibitem[Cheng et~al.(2019)Cheng, Yan, Theodorou, and Boots]{cheng2019accelerating}
Ching-An Cheng, Xinyan Yan, Evangelos Theodorou, and Byron Boots.
\newblock Accelerating imitation learning with predictive models.
\newblock In \emph{International Conference on Artificial Intelligence and Statistics}, 2019.

\bibitem[Cheng et~al.(2020)Cheng, Kolobov, and Agarwal]{cheng2020policy}
Ching-An Cheng, Andrey Kolobov, and Alekh Agarwal.
\newblock Policy improvement via imitation of multiple oracles.
\newblock \emph{Advances in Neural Information Processing Systems}, 2020.

\bibitem[Chi et~al.(2023)Chi, Feng, Du, Xu, Cousineau, Burchfiel, and Song]{chi2023diffusion}
Cheng Chi, Siyuan Feng, Yilun Du, Zhenjia Xu, Eric Cousineau, Benjamin Burchfiel, and Shuran Song.
\newblock Diffusion policy: Visuomotor policy learning via action diffusion.
\newblock \emph{arXiv:2303.04137}, 2023.

\bibitem[Choudhury et~al.(2018)Choudhury, Bhardwaj, Arora, Kapoor, Ranade, Scherer, and Dey]{choudhury2018data}
Sanjiban Choudhury, Mohak Bhardwaj, Sankalp Arora, Ashish Kapoor, Gireeja Ranade, Sebastian Scherer, and Debadeepta Dey.
\newblock Data-driven planning via imitation learning.
\newblock \emph{The International Journal of Robotics Research}, 2018.

\bibitem[De~Haan et~al.(2019)De~Haan, Jayaraman, and Levine]{de2019causal}
Pim De~Haan, Dinesh Jayaraman, and Sergey Levine.
\newblock Causal confusion in imitation learning.
\newblock \emph{Advances in Neural Information Processing Systems}, 2019.

\bibitem[Devroye and Lugosi(2001)]{devroye2001combinatorial}
Luc Devroye and G{\'a}bor Lugosi.
\newblock \emph{Combinatorial methods in density estimation}.
\newblock Springer Science \& Business Media, 2001.

\bibitem[Diakonikolas(2016)]{diakonikolas2016learning}
Ilias Diakonikolas.
\newblock Learning structured distributions.
\newblock \emph{Handbook of Big Data}, 2016.

\bibitem[Diakonikolas et~al.(2022{\natexlab{a}})Diakonikolas, Kane, Manurangsi, and Ren]{diakonikolas2022hardness}
Ilias Diakonikolas, Daniel Kane, Pasin Manurangsi, and Lisheng Ren.
\newblock Hardness of learning a single neuron with adversarial label noise.
\newblock In \emph{International Conference on Artificial Intelligence and Statistics}, 2022{\natexlab{a}}.

\bibitem[Diakonikolas et~al.(2022{\natexlab{b}})Diakonikolas, Kontonis, Tzamos, and Zarifis]{diakonikolas2022learning}
Ilias Diakonikolas, Vasilis Kontonis, Christos Tzamos, and Nikos Zarifis.
\newblock Learning a single neuron with adversarial label noise via gradient descent.
\newblock In \emph{Conference on Learning Theory}, 2022{\natexlab{b}}.

\bibitem[Duchi et~al.(2008)Duchi, Shalev-Shwartz, Singer, and Chandra]{duchi2008efficient}
John Duchi, Shai Shalev-Shwartz, Yoram Singer, and Tushar Chandra.
\newblock Efficient projections onto the l 1-ball for learning in high dimensions.
\newblock In \emph{International Conference on Machine learning}, 2008.

\bibitem[Foster and Rakhlin(2023)]{foster2023foundations}
Dylan~J Foster and Alexander Rakhlin.
\newblock Foundations of reinforcement learning and interactive decision making.
\newblock \emph{arXiv:2312.16730}, 2023.

\bibitem[Foster et~al.(2021)Foster, Kakade, Qian, and Rakhlin]{foster2021statistical}
Dylan~J Foster, Sham~M Kakade, Jian Qian, and Alexander Rakhlin.
\newblock The statistical complexity of interactive decision making.
\newblock \emph{arXiv:2112.13487}, 2021.

\bibitem[Foster et~al.(2024{\natexlab{a}})Foster, Block, and Misra]{foster2024behavior}
Dylan~J Foster, Adam Block, and Dipendra Misra.
\newblock Is behavior cloning all you need? understanding horizon in imitation learning.
\newblock \emph{Advances in Neural Information Processing Systems}, 2024{\natexlab{a}}.

\bibitem[Foster et~al.(2024{\natexlab{b}})Foster, Han, Qian, and Rakhlin]{foster2024online}
Dylan~J Foster, Yanjun Han, Jian Qian, and Alexander Rakhlin.
\newblock Online estimation via offline estimation: An information-theoretic framework.
\newblock \emph{arXiv:2404.10122}, 2024{\natexlab{b}}.

\bibitem[Gollakota et~al.(2024)Gollakota, Gopalan, Klivans, and Stavropoulos]{gollakota2024agnostically}
Aravind Gollakota, Parikshit Gopalan, Adam Klivans, and Konstantinos Stavropoulos.
\newblock Agnostically learning single-index models using omnipredictors.
\newblock \emph{Advances in Neural Information Processing Systems}, 2024.

\bibitem[Golowich et~al.(2024{\natexlab{a}})Golowich, Moitra, and Rohatgi]{golowich2024exploration}
Noah Golowich, Ankur Moitra, and Dhruv Rohatgi.
\newblock Exploration is harder than prediction: Cryptographically separating reinforcement learning from supervised learning.
\newblock \emph{arXiv:2404.03774}, 2024{\natexlab{a}}.

\bibitem[Golowich et~al.(2024{\natexlab{b}})Golowich, Moitra, and Rohatgi]{golowich2024exploring}
Noah Golowich, Ankur Moitra, and Dhruv Rohatgi.
\newblock Exploring and learning in sparse linear mdps without computationally intractable oracles.
\newblock In \emph{Symposium on Theory of Computing}, 2024{\natexlab{b}}.

\bibitem[Gupta et~al.(2017)Gupta, Davidson, Levine, Sukthankar, and Malik]{gupta2017cognitive}
Saurabh Gupta, James Davidson, Sergey Levine, Rahul Sukthankar, and Jitendra Malik.
\newblock Cognitive mapping and planning for visual navigation.
\newblock In \emph{Conference on Computer Vision and Pattern Recognition}, 2017.

\bibitem[Han et~al.(2015)Han, Jiao, and Weissman]{han2015minimax}
Yanjun Han, Jiantao Jiao, and Tsachy Weissman.
\newblock Minimax estimation of discrete distributions under l1 loss.
\newblock \emph{IEEE Transactions on Information Theory}, 2015.

\bibitem[Hinton et~al.(2015)Hinton, Vinyals, and Dean]{hinton2015distilling}
Geoffrey Hinton, Oriol Vinyals, and Jeff Dean.
\newblock Distilling the knowledge in a neural network.
\newblock \emph{arXiv:1503.02531}, 2015.

\bibitem[Ho and Ermon(2016)]{ho2016generative}
Jonathan Ho and Stefano Ermon.
\newblock Generative adversarial imitation learning.
\newblock \emph{Advances in Neural Information Processing Systems}, 2016.

\bibitem[Holtzman et~al.(2019)Holtzman, Buys, Du, Forbes, and Choi]{holtzman2019curious}
Ari Holtzman, Jan Buys, Li~Du, Maxwell Forbes, and Yejin Choi.
\newblock The curious case of neural text degeneration.
\newblock \emph{arXiv:1904.09751}, 2019.

\bibitem[Ke et~al.(2021)Ke, Choudhury, Barnes, Sun, Lee, and Srinivasa]{ke2021imitation}
Liyiming Ke, Sanjiban Choudhury, Matt Barnes, Wen Sun, Gilwoo Lee, and Siddhartha Srinivasa.
\newblock Imitation learning as f-divergence minimization.
\newblock In \emph{Algorithmic Foundations of Robotics}, 2021.

\bibitem[Kearns et~al.(1994)Kearns, Mansour, Ron, Rubinfeld, Schapire, and Sellie]{kearns1994learnability}
Michael Kearns, Yishay Mansour, Dana Ron, Ronitt Rubinfeld, Robert~E Schapire, and Linda Sellie.
\newblock On the learnability of discrete distributions.
\newblock In \emph{Symposium on Theory of Computing}, 1994.

\bibitem[Kelly et~al.(2019)Kelly, Sidrane, Driggs-Campbell, and Kochenderfer]{kelly2019hg}
Michael Kelly, Chelsea Sidrane, Katherine Driggs-Campbell, and Mykel~J Kochenderfer.
\newblock Hg-dagger: Interactive imitation learning with human experts.
\newblock In \emph{International Conference on Robotics and Automation}, 2019.

\bibitem[Kim et~al.(2013)Kim, Farahmand, Pineau, and Precup]{kim2013learning}
Beomjoon Kim, Amir-massoud Farahmand, Joelle Pineau, and Doina Precup.
\newblock Learning from limited demonstrations.
\newblock \emph{Advances in Neural Information Processing Systems}, 2013.

\bibitem[Laskey et~al.(2017)Laskey, Lee, Fox, Dragan, and Goldberg]{laskey2017dart}
Michael Laskey, Jonathan Lee, Roy Fox, Anca Dragan, and Ken Goldberg.
\newblock Dart: Noise injection for robust imitation learning.
\newblock In \emph{Conference on Robot Learning}, 2017.

\bibitem[Le~Cam(1990)]{le1990maximum}
Lucien Le~Cam.
\newblock Maximum likelihood: an introduction.
\newblock \emph{International Statistical Review/Revue Internationale de Statistique}, 1990.

\bibitem[LeCun(2023)]{lecun2023large}
Yann LeCun.
\newblock Do large language models need sensory grounding for meaning and understanding.
\newblock In \emph{Workshop on Philosophy of Deep Learning}, 2023.

\bibitem[Lerasle(2019)]{lerasle2019lecture}
Matthieu Lerasle.
\newblock Lecture notes: Selected topics on robust statistical learning theory.
\newblock \emph{arXiv:1908.10761}, 2019.

\bibitem[Lopez-Paz et~al.(2015)Lopez-Paz, Bottou, Sch{\"o}lkopf, and Vapnik]{lopez2015unifying}
David Lopez-Paz, L{\'e}on Bottou, Bernhard Sch{\"o}lkopf, and Vladimir Vapnik.
\newblock Unifying distillation and privileged information.
\newblock \emph{arXiv:1511.03643}, 2015.

\bibitem[Lum et~al.(2024)Lum, Matak, Makoviychuk, Handa, Allshire, Hermans, Ratliff, and Van~Wyk]{lum2024dextrah}
Tyler Ga~Wei Lum, Martin Matak, Viktor Makoviychuk, Ankur Handa, Arthur Allshire, Tucker Hermans, Nathan~D Ratliff, and Karl Van~Wyk.
\newblock Dextrah-g: Pixels-to-action dexterous arm-hand grasping with geometric fabrics.
\newblock \emph{arXiv:2407.02274}, 2024.

\bibitem[Mehta(2017)]{mehta2017fast}
Nishant~A Mehta.
\newblock Fast rates with high probability in exp-concave statistical learning.
\newblock \emph{International Conference on Artificial Intelligence and Statistics}, 2017.

\bibitem[Mossel and Roch(2005)]{mossel2005learning}
Elchanan Mossel and S{\'e}bastien Roch.
\newblock Learning nonsingular phylogenies and hidden markov models.
\newblock In \emph{Symposium on Theory of Computing}, 2005.

\bibitem[Pfrommer et~al.(2022)Pfrommer, Zhang, Tu, and Matni]{pfrommer2022tasil}
Daniel Pfrommer, Thomas Zhang, Stephen Tu, and Nikolai Matni.
\newblock Tasil: Taylor series imitation learning.
\newblock \emph{Advances in Neural Information Processing Systems}, 2022.

\bibitem[Pietrzak(2012)]{pietrzak2012cryptography}
Krzysztof Pietrzak.
\newblock Cryptography from learning parity with noise.
\newblock In \emph{International Conference on Current Trends in Theory and Practice of Computer Science}, 2012.

\bibitem[Polyanskiy and Wu(2024)]{polyanskiy2024information}
Yury Polyanskiy and Yihong Wu.
\newblock \emph{Information theory: From coding to learning}.
\newblock Cambridge University Press, 2024.

\bibitem[Pomerleau(1988)]{pomerleau1988alvinn}
Dean~A Pomerleau.
\newblock Alvinn: An autonomous land vehicle in a neural network.
\newblock \emph{Advances in Neural Information Processing Systems}, 1988.

\bibitem[Radford et~al.(2019)Radford, Wu, Child, Luan, Amodei, and Sutskever]{radford2019language}
Alec Radford, Jeffrey Wu, Rewon Child, David Luan, Dario Amodei, and Ilya Sutskever.
\newblock Language models are unsupervised multitask learners.
\newblock \emph{OpenAI blog}, 2019.

\bibitem[Rajaraman et~al.(2020)Rajaraman, Yang, Jiao, and Ramchandran]{rajaraman2020toward}
Nived Rajaraman, Lin Yang, Jiantao Jiao, and Kannan Ramchandran.
\newblock Toward the fundamental limits of imitation learning.
\newblock \emph{Advances in Neural Information Processing Systems}, 2020.

\bibitem[Rajaraman et~al.(2021{\natexlab{a}})Rajaraman, Han, Yang, Liu, Jiao, and Ramchandran]{rajaraman2021value}
Nived Rajaraman, Yanjun Han, Lin Yang, Jingbo Liu, Jiantao Jiao, and Kannan Ramchandran.
\newblock On the value of interaction and function approximation in imitation learning.
\newblock \emph{Advances in Neural Information Processing Systems}, 2021{\natexlab{a}}.

\bibitem[Rajaraman et~al.(2021{\natexlab{b}})Rajaraman, Han, Yang, Ramchandran, and Jiao]{rajaraman2021provably}
Nived Rajaraman, Yanjun Han, Lin~F Yang, Kannan Ramchandran, and Jiantao Jiao.
\newblock Provably breaking the quadratic error compounding barrier in imitation learning, optimally.
\newblock \emph{arXiv:2102.12948}, 2021{\natexlab{b}}.

\bibitem[Ross and Bagnell(2010)]{ross2010efficient}
St{\'e}phane Ross and Drew Bagnell.
\newblock Efficient reductions for imitation learning.
\newblock In \emph{International Conference on Artificial Intelligence and Statistics}, 2010.

\bibitem[Ross and Bagnell(2014)]{ross2014reinforcement}
Stephane Ross and J~Andrew Bagnell.
\newblock Reinforcement and imitation learning via interactive no-regret learning.
\newblock \emph{arXiv:1406.5979}, 2014.

\bibitem[Ross et~al.(2011)Ross, Gordon, and Bagnell]{ross2011reduction}
St{\'e}phane Ross, Geoffrey Gordon, and Drew Bagnell.
\newblock A reduction of imitation learning and structured prediction to no-regret online learning.
\newblock In \emph{International Conference on Artificial Intelligence and Statistics}, 2011.

\bibitem[Ross et~al.(2013)Ross, Melik-Barkhudarov, Shankar, Wendel, Dey, Bagnell, and Hebert]{ross2013learning}
St{\'e}phane Ross, Narek Melik-Barkhudarov, Kumar~Shaurya Shankar, Andreas Wendel, Debadeepta Dey, J~Andrew Bagnell, and Martial Hebert.
\newblock Learning monocular reactive uav control in cluttered natural environments.
\newblock In \emph{International Conference on Robotics and Automation}, 2013.

\bibitem[Shalev-Shwartz and Ben-David(2014)]{shalev2014understanding}
Shai Shalev-Shwartz and Shai Ben-David.
\newblock \emph{Understanding machine learning: From theory to algorithms}.
\newblock Cambridge University Press, 2014.

\bibitem[Shalev-Shwartz et~al.(2011)Shalev-Shwartz, Shamir, and Sridharan]{shalev2011learning}
Shai Shalev-Shwartz, Ohad Shamir, and Karthik Sridharan.
\newblock Learning kernel-based halfspaces with the 0-1 loss.
\newblock \emph{SIAM Journal on Computing}, 2011.

\bibitem[Shannon(1951)]{shannon1951prediction}
Claude~E Shannon.
\newblock Prediction and entropy of printed english.
\newblock \emph{Bell System Technical Journal}, 1951.

\bibitem[Spencer et~al.(2021)Spencer, Choudhury, Venkatraman, Ziebart, and Bagnell]{spencer2021feedback}
Jonathan Spencer, Sanjiban Choudhury, Arun Venkatraman, Brian Ziebart, and J~Andrew Bagnell.
\newblock Feedback in imitation learning: The three regimes of covariate shift.
\newblock \emph{arXiv:2102.02872}, 2021.

\bibitem[Sun et~al.(2017)Sun, Venkatraman, Gordon, Boots, and Bagnell]{sun2017deeply}
Wen Sun, Arun Venkatraman, Geoffrey~J Gordon, Byron Boots, and J~Andrew Bagnell.
\newblock Deeply aggrevated: Differentiable imitation learning for sequential prediction.
\newblock In \emph{International Conference on Machine Learning}, 2017.

\bibitem[Swamy et~al.(2021)Swamy, Choudhury, Bagnell, and Wu]{swamy2021moments}
Gokul Swamy, Sanjiban Choudhury, J~Andrew Bagnell, and Steven Wu.
\newblock Of moments and matching: A game-theoretic framework for closing the imitation gap.
\newblock In \emph{International Conference on Machine Learning}, 2021.

\bibitem[Team(2024)]{team2024gemini}
Gemini Team.
\newblock Gemini 1.5: Unlocking multimodal understanding across millions of tokens of context.
\newblock \emph{arXiv:2403.05530}, 2024.

\bibitem[Touvron et~al.(2023)Touvron, Martin, Stone, Albert, Almahairi, Babaei, Bashlykov, Batra, Bhargava, Bhosale, Bikel, Blecher, Ferrer, Chen, Cucurull, Esiobu, Fernandes, Fu, Fu, Fuller, Gao, Goswami, Hartshorn, Hosseini, Hou, Inan, Kardas, Kerkez, Khabsa, Kloumann, Korenev, Koura, \~Lachaux, Lavril, Lee, Liskovich, Lu, Mao, Martinet, Mihaylov, Mishra, Molybog, Nie, Poulton\, Reizenstein, Rungta, Saladi, Schelten, Silva, Smith, Subramanian, Tan, Tang, Taylor, \~Williams, Kuan, Xu, Yan, Zarov, Zhang, Fan, Kambadur, Narang, Rodriguez, Stojnic, Edunov\, and Scialom]{touvron2023llama}
Hugo Touvron, Louis Martin, Kevin Stone, Peter Albert, Amjad Almahairi, Yasmine Babaei, Nikolay Bashlykov, Soumya Batra, Prajjwal Bhargava, Shruti Bhosale, \~Dan Bikel, Lukas Blecher, Cristian~Canton Ferrer, Moya Chen, Guillem Cucurull, David Esiobu, Jude Fernandes, Jeremy Fu, Wenyin Fu, Brian Fuller, Cynthia Gao, Naman Goswami, Vedanuj a\ nd~Goyal, Anthony Hartshorn, Saghar Hosseini, Rui Hou, Hakan Inan, Marcin Kardas, Viktor Kerkez, Madian Khabsa, Isabel Kloumann, Artem Korenev, Punit~Singh Koura, Marie-Anne \~Lachaux, Thibaut Lavril, Jenya Lee, Diana Liskovich, Yinghai Lu, Yuning Mao, Xavier Martinet, Todor Mihaylov, Pushkar Mishra, Igor Molybog, Yixin Nie, Andrew Poulton\, Jeremy Reizenstein, Rashi Rungta, Kalyan Saladi, Alan Schelten, Ruan Silva, Eric~Michael Smith, Ranjan Subramanian, Xiaoqing~Ellen Tan, Binh Tang, Ross Taylor, Adina \~Williams, Jian~Xiang Kuan, Puxin Xu, Zheng Yan, Iliyan Zarov, Yuchen Zhang, Angela Fan, Melanie Kambadur, Sharan Narang, Aurelien Rodriguez, Robert Stojnic, Sergey
  Edunov\, and Thomas Scialom.
\newblock Llama 2: Open foundation and fine-tuned chat models.
\newblock \emph{arXiv:2307.09288}, 2023.

\bibitem[van~de Geer(2000)]{Sara00}
Sara~A. van~de Geer.
\newblock \emph{Empirical Processes in {M}-{E}stimation.}
\newblock Cambridge University Press, 2000.

\bibitem[Vaswani et~al.(2017)Vaswani, Shazeer, Parmar, Uszkoreit, Jones, Gomez, Kaiser, and Polosukhin]{vaswani2017attention}
Ashish Vaswani, Noam Shazeer, Niki Parmar, Jakob Uszkoreit, Llion Jones, Aidan~N Gomez, {\L}ukasz Kaiser, and Illia Polosukhin.
\newblock Attention is all you need.
\newblock \emph{Advances in Neural Information Processing Systems}, 2017.

\bibitem[Wainwright(2019)]{wainwright2019high}
Martin~J Wainwright.
\newblock \emph{High-dimensional statistics: A non-asymptotic viewpoint}.
\newblock Cambridge University Press, 2019.

\bibitem[Wong and Shen(1995)]{wong1995probability}
Wing~Hung Wong and Xiaotong Shen.
\newblock Probability inequalities for likelihood ratios and convergence rates of sieve mles.
\newblock \emph{The Annals of Statistics}, 1995.

\bibitem[Yan et~al.(2021)Yan, Boots, and Cheng]{yan2021explaining}
Xinyan Yan, Byron Boots, and Ching-An Cheng.
\newblock Explaining fast improvement in online imitation learning.
\newblock In \emph{Uncertainty in Artificial Intelligence}, 2021.

\bibitem[Yang and Barron(1998)]{yang1998asymptotic}
Yuhong Yang and Andrew~R Barron.
\newblock An asymptotic property of model selection criteria.
\newblock \emph{IEEE Transactions on Information Theory}, 1998.

\bibitem[Yu and Zhang(2021)]{yu2021smoothing}
Yu~Yu and Jiang Zhang.
\newblock Smoothing out binary linear codes and worst-case sub-exponential hardness for {LPN}.
\newblock In \emph{Advances in Cryptology}, 2021.

\bibitem[Yu et~al.(2019)Yu, Zhang, Weng, Guo, and Li]{yu2019collision}
Yu~Yu, Jiang Zhang, Jian Weng, Chun Guo, and Xiangxue Li.
\newblock Collision resistant hashing from sub-exponential learning parity with noise.
\newblock In \emph{International Conference on the Theory and Application of Cryptology and Information Security}, 2019.

\bibitem[Zhang(2006)]{zhang2006from}
Tong Zhang.
\newblock From $\epsilon$-entropy to {KL}-entropy: Analysis of minimum information complexity density estimation.
\newblock \emph{The Annals of Statistics}, 2006.

\bibitem[Zhao et~al.(2023)Zhao, Kumar, Levine, and Finn]{zhao2023learning}
Tony~Z Zhao, Vikash Kumar, Sergey Levine, and Chelsea Finn.
\newblock Learning fine-grained bimanual manipulation with low-cost hardware.
\newblock \emph{arXiv:2304.13705}, 2023.

\bibitem[Zhuang et~al.(2023)Zhuang, Fu, Wang, Atkeson, Schwertfeger, Finn, and Zhao]{zhuang2023robot}
Ziwen Zhuang, Zipeng Fu, Jianren Wang, Christopher Atkeson, Soeren Schwertfeger, Chelsea Finn, and Hang Zhao.
\newblock Robot parkour learning.
\newblock \emph{arXiv:2309.05665}, 2023.

\end{thebibliography}

\clearpage

\appendix

\renewcommand{\contentsname}{Contents of Appendix}
\addtocontents{toc}{\protect\setcounter{tocdepth}{2}}
{
 \hypersetup{hidelinks}
 \tableofcontents
}

\newpage
\part{Additional Discussion and Results}

\section{Additional Related Work}
\label{sec:related}

In this section we discuss additional related work not already covered
in detail.

\paragraph{Imitation learning} In the empirical literature on
imitation learning,
error amplification in next-token prediction and behavior
  cloning can be mitigated empirically to some extent through interactive access
  to the target distribution (demonstrating expert)
  \citep{ross2013learning,kim2013learning,gupta2017cognitive,bansal2018chauffeurnet,laskey2017dart,
    choudhury2018data,kelly2019hg,barnes2023world,zhuang2023robot,lum2024dextrah}
  or additional side information
  \citep{pfrommer2022tasil,block2024provable}. However, such access
  may not always be realistic or practical. Given the ubiquity of next-token
  prediction, our work focuses on the purely offline setting, seeking
  to understand whether error amplification can be
  mitigated without collecting additional data.

On the theoretical side, various improved imitation learning
procedures have been proposed with or without additional interactive
access or side information  
\citep{ross2010efficient,ross2011reduction,ross2014reinforcement,sun2017deeply,cheng2018convergence,cheng2020policy,cheng2019accelerating,yan2021explaining,spencer2021feedback}. Comparing
these results under misspecification is somewhat subtle, as many use
different, incomparable notions of supervised learning error, and
passing between these different notions often incurs additional
dependence on the horizon $H$. To our knowledge, the only work that
provides tight guarantees for general policy class $\Pi$, even in the
realizable/well-specified case, is \citet{foster2024behavior}, though
various works provide tight guarantees for specific (e.g., tabular or linear)
policy classes
\citep{rajaraman2020toward,rajaraman2021value,rajaraman2021provably}.

An important conceptual distinction is that---following
\citet{foster2024behavior}---we focus on estimating the
trajectory-level distribution $\bbP^{\pistar}$, which readily
translates to guarantees on generation performance in a horizon-free
fashion. A complementary approach used in many theoretical works
\citep{rajaraman2020toward,swamy2021moments} is to estimate
\emph{occupancy measures} given by
\[
  d_h^{\pi}(s,a)
  = \bbP^{\pi}\brk*{s_h=s, a_h=a}.
\]
Note that in the autoregressive setting, we have
$d_{H}^{\pi}\equiv\bbP^{\pi}$, since the final state fully determines the entire trajectory. For general MDPs, we are not aware of any
techniques based on occupancy measure estimation that give tight
dependence on horizon for general policy classes $\Pi$ even in the well-specified setting, irrespective of
computation.\loose

  \paragraph{Agnostic estimation in theoretical computer science}
  Agnostic estimation in $f$-divergences (particularly total
  variation distance) has been investigated in the theoretical
  computer science literature, and efficient algorithms have been
  identified for many specific distribution families of
  interest---particularly over low-dimensional or discrete domains
  \citep{acharya2015fast,diakonikolas2016learning,acharya2017sample,bousquet2019optimal}. Our
  results for autoregressive linear models are most closely related to a line of work on agnostically
  learning generalized linear models
  \citep{shalev2011learning,diakonikolas2022hardness,diakonikolas2022learning,gollakota2024agnostically},
  which corresponds to a special case when $H=1$ and $\abs{\cA}=2$
  (though the loss function in these works are different from the
  Hellinger distance objective we consider); notably our hardness results build on
  \citet{diakonikolas2022hardness} and our algorithms build on
  \citet{shalev2011learning}. On the hardness side, an important distinction is that
  our lower bounds aim to isolate the effect of the horizon $H$
  while controlling other problem-dependent parameters such as the
  norm of the weights.

  \paragraph{Misspecified estimation in statistics}
  Motivated by the insufficiency of maximum likelihood estimation under
  misspecification \citep{le1990maximum,birge2006model}, guarantees
  for misspecified distribution estimation in $f$-divergences
  like total variation distance and Hellinger distance have received
  some investigation in statistics \citep{devroye2001combinatorial,baraud2017new,baraud2018rho}, with the \emph{Scheff\'e
    tournament} \citep{devroye2001combinatorial} as perhaps the most
  well-known technique for general distribution classes. This line of
  work is not concerned with computational efficiency. In addition,
  while some techniques can be used essentially as-is for the general imitation
  learning setting we consider \citep{baraud2018rho}, not all
  techniques (including the Scheff\'e
    tournament itself) can be applied without knowledge of the
  underlying MDP dynamics.\loose

\section{Comparing Hellinger Distance to Other Misspecification Notions}
\label{sec:comparing}

Our emphasis on Hellinger distance
(versus total variation distance) is additionally motivated by recent results
of \citet{foster2024behavior}, which show that Hellinger distance
leads to tighter problem-dependent regret bounds that improve over \cref{eq:tv_equivalence}:\loose
\begin{proposition}[\cite{foster2024behavior}]\label{prop:hellinger-to-il}
For any $R$-bounded reward function $r$ and policies
  $\pistar$ and $\pihat$,
  \begin{align}
    \label{eq:bc_stochastic}
    J(\pistar;r)-J(\pihat;r)
    \approxleq{} \sqrt{\sigmastar^2\cdot\DhelsX{\big}{\bbP^{\pihat}}{\bbP^{\pistar}}}
    + \bigoht\prn*{R}\cdot\DhelsX{\big}{\bbP^{\pihat}}{\bbP^{\pistar}},
  \end{align}
  where
  $\sigmastar^2=\sum_{h=1}^{H}\En^{\pistar}\brk*{(\Vstar_h(x_h)-\Qstar_h(x_h,a_h))^2}\leq{}R^2$
  is the \emph{expert variance}.\footnote{We
    define $V_h^{\pi}(s)\coloneqq\En^{\pi}\brk[\big]{\sum_{h'=h}^{H}r_{h'}\mid{}x_h=s}$ and
    $Q_h^{\pi}(s,a)\coloneqq\En^{\pi}\brk[\big]{\sum_{h'=h}^{H}r_{h'}\mid{}s_h=s,
      a_h=a}$ as the\arxiv{ state- and state-action value functions for a
    policy} $\pi$.}
  Further, if $\pistar$ is deterministic, then
  for all $\pihat$, %
  \arxiv{\begin{align}
        J(\pistar;r)-J(\pihat;r)
    \leq
    4R\cdot{}\DhelsX{\big}{\bbP^{\pihat}}{\bbP^{\pistar}}.
  \end{align}}
  Finally, for any accretive MDP, i.e. with the property that
  $(x_1,a_1),\ldots,(x_{H-1},a_{H-1})\subset x_H$, there exists a
  reward function for which each inequality is tight up to logarithmic factors.
\end{proposition}

In particular, consider the extreme case where the expert is deterministic, and hence $\sigmastar = 0$. Then for any estimator $\pihat$ that satisfies 
\[\Dhels{\BP^{\pihat}}{\BP^{\pistar}} \leq \Capx \cdot \vepsmis^2 + \vepsstat^2(n),\]
where $\vepsmis^2 = \min_{\pi\in\Pi} \Dhels{\BP^\pi}{\BP^{\pistar}}$ is the irreducible misspecification error, the regret of $\pihat$ can be bounded as
\begin{align}
J(\pistar;r)-J(\pihat;r) &\lesssim R \cdot (\Capx \vepsmis^2 + \vepsstat^2(n)).
\end{align}
Moreover, whenever the underlying MDP is accretive (as is the case for the autoregressive MDP) and $\Capx = O(1)$, this bound is asymptotically optimal up to logarithmic factors, since by \cref{prop:hellinger-to-il} it holds that
\[\min_{\pi\in\Pi} J(\pistar;r)-J(\pi;r) \geq \bigomt\left(R \cdot \min_{\pi\in\Pi} \Dhels{\BP^\pi}{\BP^{\pistar}}\right) = \bigomt(R \cdot \vepsmis^2).\]
More generally, similar guarantees hold whenever the variance of the
expert policy is sufficiently small. As final motivation, we observe
that fast statistical rates are achievable in Hellinger distance (as
in \cref{prop:mle_finite}), but generically unachievable for
TV-distance \citep{han2015minimax}. We remark that measuring
  misspecification through information-theoretic divergences as we do
  here may be overly pessimistic if the reward function belongs to a
  class with known structure (e.g., linear rewards); understanding
  the role of misspecification in this setting is an interesting
  direction for future work.\loose

\begin{remark}[KL-Divergence]
  Another natural divergence to use for distribution estimation is
  KL-divergence; for example, one might aim to minimize
  $\Dkl{\bbP^{\pistar}}{\bbP^{\pihat}}$ and measure misspecification
  via $\min_{\pi\in\Pi}\Dkl{\bbP^{\pistar}}{\bbP^{\pi}}$. However,
  even in the well-specified case, it is not possible to perform
  distribution estimation in KL-divergence (for general classes $\Pi$)
  without making assumptions on boundedness of the densities under
  consideration (e.g., \citet{bilodeau2021minimax}), which is not required for Hellinger distance
  (\cref{prop:mle_finite}). Moreover, $\min_{\pi\in\Pi}\Dkl{\bbP^{\pistar}}{\bbP^{\pi}}$ can easily be infinite even when the Hellinger misspecification is arbitrarily small.
\end{remark}

\section{Further Benefits of \DensObs}
\label{sec:density}

In \cref{sec:improvements}, we showed that given access to
\densobs $\pistar_h(a_h\ind{i}\mid{}s_h\ind{i})$ for the examples
$(s_h\ind{i},a_h\ind{i})$, one can smooth $\loglossbc$ to achieve improved
misspecification tolerance $\Capx=\bigoht(H)$; the resulting method $\smoothedloglossbc$ could be of practical interest, even though we do not know how to implement it efficiently in our testbed of autoregressive linear models. In this section, we
present two algorithms with \emph{optimal} misspecification tolerance
$\Capx=\bigoh(1)$ enabled by \densobs; both algorithms seem likely somewhat impractical (and both are
computationally inefficient for autoregressive linear models), but they do slightly simplify the $\rho$-estimator,
and may be of independent interest.

\paragraph{Logarithmic loss with trajectory-level smoothing}
Given expert dataset
$\cD=\crl*{(s_h\ind{i},a_h\ind{i},\pistar_h(a_h\ind{i}\mid{}s_h\ind{i}))}_{i\in\brk{n}}$,
consider the following estimator, which we refer to as
\emph{log-loss behavior cloning with trajectory-level smoothing}:
\begin{align}
  \label{eq:traj_smoothing}
  \pihat=\argmax_{\pi\in\Pi}\sum_{i=1}^{n}\log\prn*{\prod_{h=1}^{H}\pi_h(a_h\ind{i}\mid{}s_h\ind{i})
  + \prod_{h=1}^{H}\pistar_h(a_h\ind{i}\mid{}s_h\ind{i})}.
\end{align}
To see why this algorithm is natural, we observe that it can be viewed
as an instance of maximum likelihood over the class
$\crl*{\frac{1}{2}(\bbP^{\pi}+\bbP^{\pistar})}_{\pi\in\Pi}$ of
smoothed trajectory distributions; indeed, for any policy $\pi$ and trajectory $o = (s_1,a_1,\dots,s_H,a_H)$, we have
\begin{align}
  \log\prn*{\frac{1}{2}\prn*{\bbP^{\pi}(o) + \bbP^{\pistar}(o)}}%
  &=\log\prn*{\frac{1}{2}\prn*{\prod_{h=1}^{H}\bbP_h(s_{h+1}\mid{}s_h,a_h)\pi_h(a_h\mid{}s_h)
  +
    \prod_{h=1}^{H}\bbP_h(s_{h+1}\mid{}s_h,a_h)\pistar_h(a_h\mid{}s_h)}}\\
  &=\log\prn*{\prod_{h=1}^{H}\pi_h(a_h\mid{}s_h)
    + \prod_{h=1}^{H}\pistar_h(a_h\mid{}s_h)}
    + \log\prn*{\frac{1}{2}\prod_{h=1}^{H}\bbP_h(s_{h+1}\mid{}s_h,a_h)}.\label{eq:traj_equiv}
\end{align}%
Note that the second term above is independent of the policy being
optimized over, and hence does not affect the maximizer. By applying the results of \citet{foster2024behavior} and performing
some elementary manipulations, we can deduce the following result.
\begin{proposition}
  \label{prop:traj_smoothing}
Fix an MDP $M$, a policy class $\Pi$, and an expert policy $\pistar$. For i.i.d. trajectories $o\ind{1},\dots,o\ind{n}$ from $\bbP^{\pistar}$, the policy $\pihat$ in \cref{eq:traj_smoothing} satisfies, with
probability at least $1-\delta$, 
\begin{align}
    \Dhels{\bbP^{\pihat}}{\bbP^{\pistar}} 
  \approxleq{} \frac{\log(\abs{\Pi}\delta^{-1})}{n}
      +   \min_{\pi\in\Pi}\Dhels{\bbP^{\pistar}}{\bbP^{\pi}}.
\end{align}
\end{proposition}
That is, trajectory-level smoothing substantially improves over
layer-wise smoothing (cf. \cref{thm:layerwise_smoothing}), achieving $\Capx=\bigoh(1)$ and matching the
result for the $\rho$-estimator in \cref{thm:rho}. Our results in \cref{sec:computational} show that in a worst-case
sense, one should not hope to implement the objective in
\cref{prop:traj_smoothing}, but it is certainly simpler than the
$\rho$-estimator itself, and may be interesting to explore further.

\begin{proof}[\pfref{prop:traj_smoothing}]
By Proposition B.1 in \citet{foster2024behavior} and \cref{eq:traj_equiv}, we have that with
probability at least $1-\delta$,
\begin{align}
  \label{eq:4}
\Dhels{\frac{1}{2}(\bbP^{\pihat}+\bbP^{\pistar})}{\bbP^{\pistar}}
    &\approxleq{} \frac{\log(\abs{\Pi}\delta^{-1})}{n}
      +
      \min_{\pi\in\Pi}\Dchis{\bbP^{\pistar}}{\frac{1}{2}(\bbP^{\pi}+\bbP^{\pistar})}.
\end{align}
For any policy $\pi\in\Pi$, we have
\begin{align}
  \Dchis{\bbP^{\pistar}}{\frac{1}{2}(\bbP^{\pi}+\bbP^{\pistar})}
  &\leq  \Dkl{\bbP^{\pistar}}{\frac{1}{2}(\bbP^{\pi}+\bbP^{\pistar})} \\ &\lesssim \Dhels{\bbP^{\pistar}}{\frac{1}{2}(\bbP^{\pi}+\bbP^{\pistar})}
\end{align}
where the second inequality uses the fact that $\Dkl{\BP}{\BQ} \leq (2+\log(W))\Dhels{\BP}{\BQ}$ whenever $\BP(z)/\BQ(z) \leq W$ for all $z$ \citep[Lemma 4]{yang1998asymptotic}. Finally, by \cref{lemma:avg-hels} we have for any $\pi\in\Pi$ that
\[
  \Dhels{\frac{1}{2}(\bbP^{\pi}+\bbP^{\pistar})}{\bbP^{\pistar}} \approxleq{} \Dhels{\bbP^{\pi}}{\bbP^{\pistar}} \approxleq{} \Dhels{\frac{1}{2}(\bbP^{\pi}+\bbP^{\pistar})}{\bbP^{\pistar}}.
\]
Combining the above bounds completes the proof.  
\end{proof}

\paragraph{Reducing the $\rho$-estimator to a single maximization problem}
Recall that the $\rho$-estimator solves a min-max problem of the form
\[
\pihat = \argmin_{\pi \in \Pi} \sup_{\pi' \in \Pi} 
\sum_{i=1}^n \tau\left(
  \prod_{h=1}^{H}\frac{\pi_h(a_h\ind{i}\mid{}s_h\ind{i})}{\pi'_h(a_h\ind{i}\mid{}s_h\ind{i})}
\right)
\]
for $\tau(x) \ldef \frac{\sqrt{1/x} - 1}{\sqrt{1/x} + 1}$. Given
access to \densobs
$\cD=\crl*{(x_h\ind{i},a_h\ind{i},\pistar_h(a_h\ind{i}\mid{}x_h\ind{i}))}_{i\in\brk{n}}$,
we consider the following variant of the $\rho$-estimator, which
simplifies to a single minimization problem:
\begin{align}
  \label{eq:rho_simple}
  \pihat = \argmin_{\pi \in \Pi}
\sum_{i=1}^n \tau\left(
  \prod_{h=1}^{H}\frac{\pi_h(a_h\ind{i}\mid{}s_h\ind{i})}{\pistar_h(a_h\ind{i}\mid{}s_h\ind{i})}
\right).
\end{align}
Like the $\rho$-estimator itself, this algorithm achieves
$\Capx=\bigoh(1)$, as shown below.
\begin{proposition}
  \label{prop:rho_simple}
Fix an MDP $M$, a policy class $\Pi$, and an expert policy $\pistar$. For i.i.d. trajectories $o\ind{1},\dots,o\ind{n}$ from $\bbP^{\pistar}$, the policy $\pihat$ in \cref{eq:rho_simple} satisfies, with
probability at least $1-\delta$, 
\begin{align}
    \Dhels{\bbP^{\pihat}}{\bbP^{\pistar}} 
  \approxleq{} \frac{\log(\abs{\Pi}\delta^{-1})}{n}
      +   \min_{\pi\in\Pi}\Dhels{\bbP^{\pistar}}{\bbP^{\pi}}.
\end{align}
\end{proposition}
As above, we do not expect to be able to implement
the simplified objective efficiently in general, but it may be of
further interest.

\begin{proof}[\pfref{prop:rho_simple}]
  We first observe that $\abs{\tau(a)} \leq 1$ for all $a \geq 0$.
  Hence, by Freedman's inequality, for any fixed $\pi$ and any $\lambda \geq{}2$, with probability at least $1-\delta$,
  \begin{align}
    \abs{\frac 1n \cdot \sum_{i = 1}^n \tau\left( \frac{\pp^{\pi}(\obs^i)}{\pp^{\pistar}(\obs^i)} \right) - \ee_{\obs \sim \pp^{\pistar}}\left[ \tau\left(\frac{\pp^{\pi}(\obs)}{\pp^{\pistar}(\obs)}\right) \right]} \leq \frac{1}{\lambda} \cdot \ee_{\obs \sim \pp^{\pistar}}\left[ \tau^2\left(\frac{\pp^{\pi}(\obs)}{\pp^{\pistar}(\obs)}\right) \right] + \frac{2(1 + \lambda) \log\left( 2 /\delta \right)}{n}.
  \end{align}
  Taking a union bound over all $\pi \in \Pi$ and setting $\lambda = 16 \sqrt{2}$, we then have that there is an event $\cE$ occurring with probability at least $1 - \delta$ such that for all $\pi \in \Pi$, it holds that
  \begin{align}
    \abs{\frac 1n \cdot \sum_{i = 1}^n \tau\left( \frac{\pp^{\pi}(\obs^i)}{\pp^{\pistar}(\obs^i)} \right) - \ee_{\obs \sim \pp^{\pistar}}\left[ \tau\left(\frac{\pp^{\pi}(\obs)}{\pp^{\pistar}(\obs)}\right) \right]} \leq \frac{1}{16 \sqrt{2}} \cdot \ee_{\obs \sim \pp^{\pistar}}\left[ \tau^2\left(\frac{\pp^{\pi}(\obs)}{\pp^{\pistar}(\obs)}\right) \right] + \frac{33 \sqrt{2} \log\left( 2 \abs{\Pi}/\delta \right)}{n};
  \end{align}
  we condition on this event moving forward.  We now note that
  \begin{align}
    \frac{\pp^\pi(\obs)}{\pp^{\pistar}(\obs)} = \prod_{h = 1}^H \frac{\bbP_h(s_{h+1} \mid{} a_h, s_h) \pi_h(a_h | s_h)}{\bbP_h(s_{h+1} \mid{} a_h, s_h) \pistar_h(a_h | s_h)} = \prod_{h = 1}^H \frac{\pi_h(a_h \mid{} s_h)}{\pistar_h(a_h \mid{} s_h)},
  \end{align}
  and thus
  \begin{align}
    \pihat = \argmin_{\pi \in \Pi} \sum_{i = 1}^n \tau\left( \frac{\pp^\pi(\obs)}{\pp^{\pistar}(\obs)} \right).
  \end{align}
  Now, let $\pibar = \argmin_{\pi \in \Pi} \Dhels{\pp^{\pi}}{\pp^{\pistar}}$ and observe that by \Cref{lemma:rho-estimator-bounds}, it holds for all $\pi \in \Pi$ that
  \begin{align}
    \frac{3}{8} \cdot \Dhels{\pp^{\pistar}}{\pp^{\pi}} &\leq \ee_{\obs \sim \pp^{\pistar}}\left[ \tau\left( \frac{\pp^{\pi}(\obs)}{\pp^{\pistar}(\obs)} \right) \right] \leq 4 \cdot \Dhels{\pp^{\pistar}}{\pp^{\pi}}
    \end{align} 
    and
    \begin{align}
    \ee\left[ \tau^2\left( \frac{\pp^{\pi}(\obs)}{\pp^{\pistar}(\obs)} \right) \right] &\leq 3 \sqrt{2} \cdot \Dhels{\pp^{\pistar}}{\pp^{\pi}}.
  \end{align}
  Thus in the event $\cE$, we compute
  \begin{align}
    \frac{3}{8} \cdot \Dhels{\pp^{\pistar}}{\pp^{\pihat}} &\leq \ee_{\obs \sim \pp^{\pistar}}\left[ \tau\left( \frac{\pp^{\pihat}(\obs)}{\pp^{\pistar}(\obs)} \right) \right] \\
    &\leq \frac 1n \cdot \sum_{i = 1}^n \tau\left( \frac{\pp^{\pihat}(\obs^i)}{\pp^{\pistar}(\obs^i)} \right) +  \frac{1}{16 \sqrt{2}} \cdot \ee_{\obs \sim \pp^{\pistar}}\left[ \tau^2\left(\frac{\pp^{\pihat}(\obs)}{\pp^{\pistar}(\obs)}\right) \right] + \frac{33 \sqrt{2} \log\left( 2 \abs{\Pi}/\delta \right)}{n} \\
    &\leq \frac 1n \cdot \sum_{i = 1}^n \tau\left( \frac{\pp^{\pibar}(\obs^i)}{\pp^{\pistar}(\obs^i)} \right) +  \frac{1}{16 \sqrt{2}} \cdot \ee_{\obs \sim \pp^{\pistar}}\left[ \tau^2\left(\frac{\pp^{\pihat}(\obs)}{\pp^{\pistar}(\obs)}\right) \right] + \frac{33 \sqrt{2} \log\left( 2 \abs{\Pi}/\delta \right)}{n} \\
    &\leq \ee_{\obs \sim \pp^{\pistar}}\left[ \tau\left( \frac{\pp^{\pibar}(\obs)}{\pp^{\pistar}(\obs)} \right) \right]+  \frac{1}{16 \sqrt{2}} \cdot \ee_{\obs \sim \pp^{\pistar}}\left[ \tau^2\left(\frac{\pp^{\pihat}(\obs)}{\pp^{\pistar}(\obs)}\right) \right] \\ 
    &\quad+ \frac{1}{16 \sqrt{2}} \cdot \ee_{\obs \sim \pp^{\pistar}}\left[ \tau^2\left(\frac{\pp^{\pibar}(\obs)}{\pp^{\pistar}(\obs)}\right) \right] +  \frac{66 \sqrt{2} \log\left( 2 \abs{\Pi}/\delta \right)}{n} \\
    &\leq 4 \cdot \Dhels{\pp^{\pistar}}{\pp^{\pibar}} + \frac{3}{16} \cdot \Dhels{\pp^{\pistar}}{\pp^{\pihat}}  + \frac{3}{16} \cdot \Dhels{\pp^{\pistar}}{\pp^{\pibar}} +  \frac{66 \sqrt{2} \log\left( 2 \abs{\Pi}/\delta \right)}{n},
  \end{align}
  where the first and penultimate inequalities follow from the preceding display, the second and fourth inequalities follow from the Bernstein calculation above, and the third inequality follows from the definition of $\pihat$.  Rearranging the above and plugging in the definition of $\pibar$ concludes the proof.
\end{proof}

\subsection{Supporting Technical Lemmas}

\begin{lemma}\label{lemma:avg-hels}
Let $\BP,\BQ \in \Delta(\MX)$ be distributions. Then
\[\Dhels{\BP}{\frac{\BP+\BQ}{2}} \lesssim \Dhels{\BP}{\BQ} \lesssim \Dhels{\BP}{\frac{\BP+\BQ}{2}}.\]
\end{lemma}

\begin{proof}[\pfref{lemma:avg-hels}]
For any real numbers $0 \leq x,y$, we have $x-y = (\sqrt{x} - \sqrt{y})(\sqrt{x} + \sqrt{y})$, so on the one hand,
\[\frac{|x-y|}{2\max(\sqrt{x},\sqrt{y})} \leq |\sqrt{x}-\sqrt{y}| \leq \frac{|x-y|}{\max(\sqrt{x},\sqrt{y})}.\]
On the other hand, $\frac{x-y}{2} = (\sqrt{x} - \sqrt{\frac{x+y}{2}})(\sqrt{x} + \sqrt{\frac{x+y}{2}})$, so
\begin{align}\frac{|x-y|}{4\max(\sqrt{x},\sqrt{y})} \leq \frac{|x-y|}{4\max(\sqrt{x},\sqrt{\frac{x+y}{2}})} \leq \left|\sqrt{x} - \sqrt{\frac{x+y}{2}}\right| &\leq \frac{|x-y|}{2\max(\sqrt{x},\sqrt{\frac{x+y}{2}})} \\&\leq \frac{|x-y|}{\sqrt{2} \max(\sqrt{x},\sqrt{y})}.\end{align}
It follows that
\[\left|\sqrt{x} - \sqrt{\frac{x+y}{2}}\right| \lesssim |\sqrt{x}-\sqrt{y}| \lesssim \left|\sqrt{x} - \sqrt{\frac{x+y}{2}}\right|.\]
The claim now follows from the definition of Hellinger distance, i.e. $\Dhels{\BP}{\BQ} = \int (\sqrt{\BP} - \sqrt{\BQ})^2$.
\end{proof}

\newpage

\part{Proofs}

\section{Supporting Results}

This section of the appendix contains proofs for various supporting
and secondary results. In \cref{sec:alm-gd}, we show $\boostedloglossbc$ can be implemented computationally efficiently for autoregressive linear models, and achieves approximation ratio $\Capx = \bigoht(H)$ with high probability (\cref{cor:linear-misspec-logloss-app}). In \cref{sec:ll-failure-large-n} we prove a lower bound on the approximation ratio of $\loglossbc$ in the large-sample regime (\cref{prop:log-loss-lb}), complementing \cref{prop:log-loss-prob-lb}. In \cref{sec:layerwise-rho} we show that a layerwise version of $\rhobc$ achieves the optimal approximation ratio among next-token prediction algorithms (\cref{prop:layerwise-rho}), matching our lower bound from \cref{thm:ntp-lb}.

\subsection{Next-Token Prediction for Autoregressive Linear Models}\label{sec:alm-gd}

In this section we study the autoregressive linear setting as formally introduced in \cref{sec:computational}, and prove \cref{cor:linear-misspec-logloss} (restated below as \cref{cor:linear-misspec-logloss-app}) by analyzing $\BALM$ (\cref{alg:balm}), which simply implements $\boostedloglossbc$, using projected gradient ascent in parameter space (\cref{alg:gdlogloss}) to approximately implement the invocations of $\loglossbc$. 

To restate the setting, let $\MX$ and $\MA$ be sets where $|\MA|<\infty$. Fix $H \in \NN$, and let $M$ be the $H$-step autoregressive MDP with context space $\MX$, action space $\MA$, and some initial context distribution $\MD \in \Delta(\MX)$. We define an autoregressive policy class $\Pi := \{\pi_\theta: \theta \in \Theta\}$ where $\Theta \subseteq \RR^d$ is a convex parameter set, and each policy $\pi_\theta = (\pi_{\theta,h)})_{h=1}^H$ is defined by
\begin{equation} \pi_{\theta,h}(a_h\mid{}x,a_{1:h-1}) := \frac{\exp(\langle \phi(x,a_{1:h}), \theta\rangle)}{\sum_{a'_h\in\MA} \exp(\langle \phi(x,a_{1:h-1},a'_h),\theta\rangle)}.\label{eq:auto-linear-app}
\end{equation}
We assume that in $\poly(d,H)$ time we can (a) query $\phi(x,a_{1:h})$ for any given $(x,a_{1:h}) \in \MX\times\MA^\st$ (with $h \leq H$), and (b) compute the Euclidean projection $\Proj_\Theta[\theta] = \argmin_{\theta'\in\Theta}\norm{\theta-\theta'}_2$ of any point $\theta \in \RR^d$ onto $\Theta$. We also make the following norm bound assumption.

\begin{algorithm}[t]
\caption{\BALM: Boosted Log-Loss Optimization for Autoregressive Linear Models}\label{alg:balm}
\begin{algorithmic}[1]
\State\textbf{input:} Samples $(x\ind{i},a\ind{i}_{1:H})_{i=1}^n$; iteration complexity $T$, desired failure probability $\delta$.
\State Partition $(x\ind{i},a\ind{i}_{1:H})_{i=1}^n$ into $K := 2\log(2/\delta)$ disjoint equal-sized folds $\cD^1,\dots,\cD^K$.
\For{$1 \leq k \leq K/2$}
    \State Compute $\theta^k \gets \GDALM(\MD^k, T)$.\hfill\algcommentlight{Approximates $\loglossbc$; see \cref{alg:gdlogloss}.}
\EndFor
\For{$1 \leq k,k' \leq K/2$}
    \State Compute 
    \[f_{k,k'} \gets \sum_{\ell=K/2+1}^K \sum_{(x,a_{1:H})\in\MD^\ell} \tau\left(\prod_{h=1}^H \frac{\pi_{\theta^k,h}(a_h\mid{}x,a_{1:h-1})}{\pi_{\theta^{k'},h}(a_h\mid{}x,a_{1:h-1})}\right).\]
\EndFor
\State \textbf{return} $\theta^{\widehat k}$ where $\widehat k := \argmin_{k\in[K/2]} \max_{k'\in[K/2]} f_{k,k'}$.
\end{algorithmic}
\end{algorithm}

\begin{algorithm}[t]
\caption{\GDALM: Gradient Ascent on Log Likelihood for Autoregressive Linear Models}\label{alg:gdlogloss}
\begin{algorithmic}[1]
\State\textbf{input:} Samples $(x\ind{i},a\ind{i}_{1:H})_{i=1}^n$; iteration complexity $T$.
\State Set $\theta^{(1)} := 0 \in \RR^d$ and $\eta := \frac{1}{2nH\sqrt{T}}$.
\For{$1 \leq t < T$}
    \State Compute
\[g^{(t)} := \sum_{i=1}^n \sum_{h=1}^H \left(\phi(x\ind{i},a\ind{i}_{1:h}) - \frac{\sum_{a_h'\in\MA} \phi(x\ind{i},a\ind{i}_{1:h-1},a'_h) \exp(\langle \phi(x\ind{i},a\ind{i}_{1:h-1},a'_h),\theta^{(t)}\rangle)}{\sum_{a_h'\in\MA} \exp(\langle \phi(x\ind{i},a\ind{i}_{1:h-1},a'_h),\theta^{(t)}\rangle)}\right).\]
\State Set $\theta^{(t+1)} := \Proj_\Theta[\theta^{(t)} + \eta g^{(t)}]$. \hfill\algcommentlight{Euclidean projection onto $\Theta$.}
\EndFor
\State \textbf{return} $\thetahat := \frac{1}{T}\sum_{t=1}^t \theta^{(t)}$.
\end{algorithmic}
\end{algorithm}

\begin{assumption}[Norm bounds]\label{ass:linear-norm-bounds}
Let $B>0$ be a parameter. It holds that $\norm{\phi(x,a_{1:h})}_2 \leq B$ for all $(x,a_{1:h}) \in \MX\times\MA^\st$ and $\norm{\theta}_2 \leq B$ for all $\theta\in\Theta$. Moreover, $|\langle \phi(x,a_{1:h}),\theta\rangle| \leq \Bdot$ for all $(x,a_{1:h})\in\MX\times\MA^\st$ and $\theta \in \Theta$. 
\end{assumption}

Obviously, we can always take $\Bdot := B^2$; however, we separate these parameters because the time complexity will scale with $\poly(B)$, whereas the approximation ratio will only scale with $\Bdot$, and the latter can be much smaller in natural settings (e.g. if $\phi(x,a_{1:h}) \in [-1,1]^d$ and $\Theta$ is the $\ell_1$ ball). The following proposition states that \cref{alg:gdlogloss} (which is simply projected gradient ascent on the next-token prediction log-loss in parameter space $\Theta$) is both computationally efficient and achieves a non-trivial statistical guarantee even in the presence of misspecification:

\begin{proposition}[Restatement of \cref{cor:linear-misspec-logloss}]\label{cor:linear-misspec-logloss-app}
Suppose that \cref{ass:linear-norm-bounds} holds with parameters $B,\Bdot\geq 1$. Let $(x^i,a^i_{1:H})_{i=1}^n$ be i.i.d samples from $\BP^{\pistar}$ for any unknown policy $\pistar$. Then for any $\delta \in (0,1/2)$, the output $\thetahat$ of $\BALM((x^i,a^i_{1:H})_{i=1}^n, 2B^4 H^2 n^2,\delta)$ satisfies $\thetahat \in \Theta$ and, with probability at least $1-\delta$,
\begin{align}
  \Dhels{\BP^{\pi_{\thetahat}}}{\BP^{\pistar}} &\lesssim \frac{(d\log(BHn) + \Bdot + \log|\MA|)\log(1/\delta)}{n} + (\Bdot+\log|\MA|)H \cdot \min_{\pi\in\Pi} \Dhels{\BP^\pi}{\BP^{\pistar}}.\end{align}
Moreover, the time complexity of the algorithm is $\poly(n,d,H,|\MA|,B,\log(1/\delta))$.
\end{proposition}

To prove \cref{cor:linear-misspec-logloss-app}, we start by analyzing the subroutine $\GDALM$ (\cref{alg:gdlogloss}), which approximately implements $\loglossbc$. In particular, we show that the log-loss is concave in parameter space and invoke a standard guarantee for projected gradient ascent (\cref{lemma:gd}) to prove that the output of $\GDALM$ is an approximate maximizer of the log-loss: %

\begin{lemma}\label{lemma:gd}
Suppose that \cref{ass:linear-norm-bounds} holds with parameters $B,\Bdot>0$. Fix $n \in \NN$ and let $(x\ind{i},a\ind{i}_{1:H})_{i=1}^n$ be arbitrary elements of $\MX\times\MA^H$. Then the output $\thetahat$ of $\GDALM$ (\cref{alg:gdlogloss}) with samples $(x\ind{i},a\ind{i}_{1:H})_{i=1}^n$ and iteration complexity $T$ satisfies $\thetahat \in \Theta$ and
\[\sum_{i=1}^n\sum_{h=1}^H \log \pi_{\thetahat,h}(a\ind{i}_h\mid{}x\ind{i},a\ind{i}_{1:h-1}) \geq \max_{\theta\in\Theta}\sum_{i=1}^n\sum_{h=1}^H \log \pi_{\theta,h}(a\ind{i}_h\mid{}x\ind{i},a\ind{i}_{1:h-1}) - \frac{2B^2 Hn}{\sqrt{T}}.\]
\end{lemma}

\begin{proof}[\pfref{lemma:gd}]
The guarantee $\thetahat \in \Theta$ is immediate from the projection step. Next, observe that \cref{alg:gdlogloss} is performing projected gradient ascent with projection set $\Theta$ and loss function
\begin{align}
\wh L(\theta) 
&:= \sum_{i=1}^n\sum_{h=1}^H \log \pi_{\theta,h}(a\ind{i}_h\mid{}x\ind{i},a\ind{i}_{1:h-1}) \\ 
&= \sum_{i=1}^n\sum_{h=1}^H \langle \phi(x\ind{i},a\ind{i}_{1:h}),\theta\rangle - \log \sum_{a_h'\in\MA} \exp(\langle \phi(x\ind{i},a\ind{i}_{1:h-1},a'_h),\theta\rangle).
\end{align}
Indeed, for any $\theta \in \Theta$ we can write
\begin{align}
\grad_\theta \wh L(\theta) 
&= \sum_{i=1}^n \sum_{h=1}^H \left(\phi(x\ind{i},a\ind{i}_{1:h}) - \frac{\sum_{a_h'\in\MA} \phi(x\ind{i},a\ind{i}_{1:h-1},a'_h) \exp(\langle \phi(x\ind{i},a\ind{i}_{1:h-1},a'_h),\theta\rangle)}{\sum_{a_h'\in\MA} \exp(\langle \phi(x\ind{i},a\ind{i}_{1:h-1},a'_h),\theta\rangle)}\right) \\ 
&= \sum_{i=1}^n \sum_{h=1}^H \left(\phi(x\ind{i},a\ind{i}_{1:h}) - \EE_{a'_h \sim \pi_\theta(\cdot\mid{}x\ind{i},a\ind{i}_{1:h-1})}[\phi(x\ind{i},a\ind{i}_{1:h-1},a'_h)]\right), \label{eq:grad-linear-exp}
\end{align}
and furthermore
\begin{align}
\grad^2_\theta \wh L(\theta) 
&= \sum_{i=1}^n\sum_{h=1}^H\Bigg( \EE_{a'_h \sim \pi_\theta(\cdot\mid{}x\ind{i},a\ind{i}_{1:h-1})}[\phi(x\ind{i},a\ind{i}_{1:h-1},a'_h)]\EE_{a'_h \sim \pi_\theta(\cdot\mid{}x\ind{i},a\ind{i}_{1:h-1})}[\phi(x\ind{i},a\ind{i}_{1:h-1},a'_h)]^\top \\ 
&\qquad\qquad\qquad- \EE_{a'_h \sim \pi_\theta(\cdot\mid{}x\ind{i},a\ind{i}_{1:h-1})}\left[\phi(x\ind{i},a\ind{i}_{1:h-1},a'_h)\phi(x\ind{i},a\ind{i}_{1:h-1},a'_h)^\top\right]\Bigg).\label{eq:hessian-linear-exp}
\end{align}
From \cref{eq:hessian-linear-exp} and the fact that any covariance matrix is positive semi-definite, we see that $\grad^2\theta \wh L(\theta) \preceq 0$ for all $\theta$, and hence $L$ is concave. By \cref{ass:linear-norm-bounds}, we know that $\Theta$ is contained in a Euclidean ball of norm $B$ centered at $\theta^{(1)} = 0$. Moreover, by \cref{eq:grad-linear-exp} it is clear that $\wh L$ is $2nHB$-Lipschitz. The lemma statement now follows from standard analyses of projected gradient ascent (e.g. \cite[Theorem 3.2]{bubeck2015convex}). 
\end{proof}

We can now prove \cref{cor:linear-misspec-logloss-app} by essentially repeating the original analysis of $\boostedloglossbc$ (\cref{prop:boosted-bc}), and using with a lower bound on the densities of autoregressive linear models. One minor differnce is that we cannot directly use \cref{thm:log-loss-bounded}, since the policy class is infinite and there is non-zero optimization error, but we actually prove a more general version (\cref{thm:log-loss-bounded-app}) that handles both of these complications---it suffices to bound the covering number of the policy class, which we do in \cref{lemma:auto-linear-cover}.
\vspace{1em}

\begin{proof}[Proof of \cref{cor:linear-misspec-logloss-app}]
Fix any $1 \leq k \leq K/2$. Consider the invocation of \GDALM (\cref{alg:gdlogloss}) on the dataset $\cD^k$ with iteration complexity $T = 2B^4 H^2 n^2$. Moreover, by \cref{lemma:gd}, we have that
\[\sum_{(x,a_{1:H})\in\cD^k} \sum_{h=1}^H \log \pi_{\thetahat,h}(a_h\mid{}x,a_{1:h-1}) \geq \max_{\theta\in\Theta}\sum_{(x,a_{1:H})\in\cD^k}\sum_{h=1}^H \log \pi_{\theta,h}(a_h\mid{}x,a_{1:h-1}) - 1.\]
Thus, $\thetahat$ is a solution to $1$-approximate $\loglossbc$ with dataset $\cD^k$, as defined in \cref{sec:log-loss-app}. By \cref{lemma:auto-linear-density-bound}, the expert policy $\pistar$ is $|\MA|\exp(2\Bdot)$-bounded with respect to $\Pi$ (\cref{ass:density-bound}). We now apply the second guarantee of \cref{thm:log-loss-bounded-app} with dataset size $n' := n/(2\log(2/\delta))$, cover discretization $\epsilon := 1/Hn$, optimization error $\epopt := 1$, and density bound $\denbound := |\MA|\exp(2\Bdot)$. We get that with probability at least $1/2$,
\begin{align}
\Dhels{\BP^{\pi_{\theta^k}}}{\BP^{\pistar}} &\lesssim \frac{2}{n'} + \frac{\log(2\Nlog(\Pi,1/(Hn)))}{n'} + \frac{\log(e\denbound)\log(2)}{n'} \\
&\qquad+ 2H\log(e\denbound) \cdot \min_{\pi\in\Pi} \Dhels{\BP^\pi}{\BP^{\pistar}} \\ 
&\lesssim \frac{(d\log(BHn) + \Bdot + \log|\MA|)\log(1/\delta)}{n} + (\Bdot+\log|\MA|)H \cdot \min_{\pi\in\Pi} \Dhels{\BP^\pi}{\BP^{\pistar}}
\end{align}
where the second inequality uses \cref{lemma:auto-linear-cover}. Now by independence of $\cD^1,\dots,\cD^{K/2}$, it holds with probability at least $1 - (1/2)^{K/2} = 1-\delta/2$ that there is at least one $k \in [K/2]$ satisfying the above bound. Condition on this event. Observing that the final steps of \cref{alg:balm} precisely implement $\rhobc$ with dataset $\cD^{K/2+1}\sqcup\dots\sqcup\cD^K$ and policy class $\{\pi_{\theta^1},\dots,\pi_{\theta^{K/2}}\}$, applying \cref{thm:rho-il} gives that with probability at least $1-\delta/2$,
\begin{align}
\Dhels{\BP^{\pi_{\theta^{\widehat k}}}}{\BP^{\pistar}} &\lesssim \frac{\log(K/\delta)}{n} + \min_{k\in[K/2]} \Dhels{\BP^{\pi_{\theta^k}}}{\BP^{\pistar}} \\ 
&\lesssim \frac{\log(1/\delta)}{n} + \min_{k\in[K/2]} \Dhels{\BP^{\pi_{\theta^k}}}{\BP^{\pistar}}\end{align}
where the second inequality is because $K \lesssim 1/\delta$. By the union bound, we have with probability at least $1-\delta$ that
\begin{align}
\Dhels{\BP^{\pi_{\theta^{\widehat k}}}}{\BP^{\pistar}} &\lesssim \frac{(d\log(BHn) + \Bdot + \log|\MA|)\log(1/\delta)}{n} + (\Bdot+\log|\MA|)H \cdot \min_{\pi\in\Pi} \Dhels{\BP^\pi}{\BP^{\pistar}}\end{align}
as needed. We know that $\theta_{\widehat k} \in \Theta$ by \cref{lemma:gd}. Finally, we analyze the time complexity of the algorithm. Each iteration of $\GDALM$ has time complexity $\poly(n,H,d,|\MA|)$, so the overall time complexity of $\GDALM$ is $\poly(n,H,d,|\MA|,B)$. It follows that the $K$ invocations of $\GDALM$ require time $\poly(n,H,d,|\MA|,B,\log(1/\delta))$. For each $k,k' \in [K/2]$, $f_{k,k'}$ can be computed in time $\poly(n,d,H,|\MA|)$, since each conditional density can be computed using $|\MA|+1$ queries to the feature map $\phi$. It follows that the overall time complexity is $\poly(n,H,d,|\MA|,B,\log(1/\delta))$.
\end{proof}

\cref{cor:linear-misspec-logloss-app} specializes to the well-specified setting as follows:

\begin{proposition}\label{cor:linear-wellspec-logloss}
Suppose that $\norm{\phi(x,a_{1:h})}_2 \leq \sqrt{d}$ for all $(x,a_{1:h}) \in \MX\times\MA^\star$ and $\norm{\theta}_2\leq \sqrt{d}$ for all $\theta \in \Theta$. There is a $\poly(n,d,H,|\MA|)$-time algorithm that takes $n$ i.i.d. samples $(x^i,a^i_{1:H})_{i=1}^n$ from $\BP^{\pistar}$ for any unknown policy $\pistar \in \Pi$, and outputs $\pihat\in\Pi$ so that with probability at least $1-\delta$,
\begin{align}
\Dhels{\BP^{\pihat}}{\BP^{\pistar}} &\lesssim \frac{(d+\log|\MA|)\log(dHn/\delta)}{n}.\end{align}
\end{proposition}

\begin{proof}[\pfref{cor:linear-wellspec-logloss}]
Immediate by setting $\min_{\pi\in\Pi} \Dhels{\BP^\pi}{\BP^{\pistar}} = 0$ in \cref{cor:linear-misspec-logloss}.
\end{proof}

\subsubsection{Supporting lemmas}

The following lemma shows that the policy class $\Pi$ has a small cover, in the sense of \cref{def:cover}:

\begin{lemma}\label{lemma:auto-linear-cover}
Suppose that \cref{ass:linear-norm-bounds} holds with parameters $B,\Bdot>0$. For any $\epsilon>0$, it holds that $\Nlog(\Pi,\epsilon) \leq (6B/\epsilon)^d$.
\end{lemma}

\begin{proof}[\pfref{lemma:auto-linear-cover}]
Since $\Theta$ is contained in the $d$-dimensional Euclidean ball, there is a set $\Theta' \subset \Theta$ be a of size at most $(6B^2/\epsilon)^d$, such that for every $\theta\in\Theta$ there is some $\theta' \in \Theta'$ with $\norm{\theta-\theta'}_2 \leq \epsilon/(2B)$. Define $\Pi' = \{\pi_{\theta'}: \theta'\in\Theta'\}$. For any $\theta,\theta' \in \Theta$ with $\norm{\theta-\theta'}_2 \leq \epsilon/(2B)$, and any $(x,a_{1:h})\in\MX\times\MA^\star$, observe that
\[\log\frac{\pi_\theta(a_h\mid{}x,a_{1:h-1})}{\pi_{\theta'}(a_h\mid{}x,a_{1:h-1})} = \langle \phi(x,a_{1:h}),\theta-\theta') + \log \frac{\sum_{a'_h\in\MA} \exp(\langle \phi(x,a_{1:h-1},a'_h),\theta'\rangle)}{\sum_{a'_h\in\MA} \exp(\langle \phi(x,a_{1:h-1},a'_h),\theta\rangle)}.\]
The first term has magnitude at most $B\norm{\theta-\theta'}_2 \leq \epsilon/2$. For the second term, note that for any $a_h'\in\MA$,
\begin{align}
\exp(\langle \phi(x,a_{1:h-1},a'_h),\theta'\rangle) 
&= \exp(\langle \phi(x,a_{1:h-1},a'_h),\theta\rangle)\exp(\langle \phi(x,a_{1:h-1},a'_h),\theta'-\theta\rangle) \\ 
&\leq \exp(\langle \phi(x,a_{1:h-1},a'_h),\theta\rangle) \cdot \exp(\epsilon/2)
\end{align}
and
\[\exp(\langle \phi(x,a_{1:h-1},a'_h),\theta'\rangle) \geq \exp(\langle \phi(x,a_{1:h-1},a'_h),\theta\rangle) \cdot \exp(-\epsilon/2).\]
It follows that 
\[\frac{\sum_{a'_h\in\MA} \exp(\langle \phi(x,a_{1:h-1},a'_h),\theta'\rangle)}{\sum_{a'_h\in\MA} \exp(\langle \phi(x,a_{1:h-1},a'_h),\theta\rangle)} \in [\exp(-\epsilon/2),\exp(\epsilon/2)]\]
and hence the second term is bounded  in magnitude by $\epsilon/2$ as well. We conclude that
\[\log\frac{\pi_\theta(a_h\mid{}x,a_{1:h-1})}{\pi_{\theta'}(a_h\mid{}x,a_{1:h-1})}.\]
This shows that $\Pi'$ is an $\epsilon$-cover for $\Pi$.
\end{proof}

\begin{lemma}\label{lemma:auto-linear-density-bound}
Suppose that \cref{ass:linear-norm-bounds} holds with parameters $B,\Bdot>0$. Let $(x,a_{1:h}) \in \MX\times\MA^\star$ and let $\theta\in\Theta$. Then $\pi_{\theta,h}(a_h\mid{}x,a_{1:h-1}) \geq \frac{1}{|\MA|\exp(2\Bdot)}$.
\end{lemma}

\begin{proof}[\pfref{lemma:auto-linear-density-bound}]
We have $\exp(\langle \phi(x,a_{1:h}),\theta\rangle) \geq \exp(-\Bdot)$ whereas $\sum_{a_h'\in\MA} \exp(\langle \phi(x,a_{1:h-1},a_h'),\theta\rangle) \leq |\MA|\exp(\Bdot)$. The result follows from \cref{eq:auto-linear-app}.
\end{proof}

\subsection{Failure of $\loglossbc$ in Large-Sample Regime}\label{sec:ll-failure-large-n}

The following result (\cref{prop:log-loss-lb}) shows that the approximation ratio of
$\loglossbc$ necessarily scales with $H\log(\denbound)$, where $H$ is the
horizon and $\denbound$ is the density bound parameter from
\cref{ass:density-bound}----even as $n\to\infty$. This result is incomparable to \cref{prop:log-loss-prob-lb-app}, where the lower bound scales with $1/\delta$ (where $\delta$ is the failure probability) but the number of samples $n$ is not allowed to grow. We use \cref{prop:log-loss-lb} to show that the approximation ratio of $\loglossbc$ can be arbitrarily bad without a density bound (\cref{cor:unbounded}).
    \begin{proposition}\label{prop:log-loss-lb}
  Fix any $n,H \in \NN$ and $\denbound \geq 2$. Let $\veps \in (0, 1/(1+H\log(\denbound)))$. Suppose that $n \geq 8/\veps$. There is an $H$-step autoregressive MDP $M$, a policy class $\Pi$ of size $|\Pi| = 2$, and an expert policy $\pistar$ such that $\pistar$ is $\denbound$-bounded with respect to $\Pi$ (\cref{ass:density-bound}), with the following property. Given $n$ i.i.d. trajectories $o\ind{i} = (x\ind{i},a_{1:H}\ind{i})$ from $\BP^{\pistar}$, the estimator $\pihat$ produced by $\loglossbc$ satisfies, with probability at least $1-1/e$,
  \[\Dhels{\BP^{\pihat}}{\bbP^{\pistar}} \gtrsim H\log(\denbound) \cdot \veps\]
  while $\min_{\pi\in\Pi} \Dhels{\BP^\pi}{\bbP^{\pistar}} \leq \veps$.
  \end{proposition}

  \begin{proof}[\pfref{prop:log-loss-lb}]
Let $M$ be the $H$-step autoregressive MDP with context space $\cX = \{\perp,\xfrak,\yfrak\}$, action space $\cA = \{\afrak,\bfrak\}$, and context distribution $\rho \ \in\Delta(\cX)$ with $\rho(\yfrak) = \veps$ and $\rho(\xfrak) = H\log(\denbound) \cdot \veps$. Recall that any policy in an autoregressive MDP is (uniquely) identified by a conditional distribution $\cX \to \Delta(\cA^H)$. Define $\pistar$ so that 
  \[ 
  \Pr^{\pistar}[a_{1:H} \mid{} x] = \begin{cases} 
  \mathbbm{1}[a_{1:H}=(\afrak,\dots,\afrak)] &\text{ if } x \in \{\perp,\xfrak\} \\ 
  \mathbbm{1}[a_{1:H} = (\bfrak,\dots,\bfrak)] & \text{ if } x = \yfrak 
  \end{cases}.
  \]
  Define $\pia$ so that
\[ 
  \Pr^{\pia}[a_{1:H} \mid{} x] = \begin{cases} 
  \mathbbm{1}[a_{1:H}=(\afrak,\dots,\afrak)] &\text{ if } x\in\{\perp,\xfrak\} \\ 
  \prod_{h=1}^H \left(1-\frac{1}{\denbound}\right)^{\mathbbm{1}[a_h=\afrak]}\left(\frac{1}{\denbound}\right)^{\mathbbm{1}[a_h=\bfrak]} & \text{ if } x = \yfrak 
  \end{cases}.
  \]
  Define $\pib$ so that
  \[ 
  \Pr^{\pib}[a_{1:H} \mid{} x] = \begin{cases} 
  \mathbbm{1}[a_{1:H}=(\afrak,\dots,\afrak)] &\text{ if } x=\perp \\ 
  \left(\frac{4}{5}\right)^{\mathbbm{1}[a_1=\afrak]} \left(\frac{1}{5}\right)^{\mathbbm{1}[a_1=\bfrak]}\mathbbm{1}[a_{2:H}=(\afrak,\dots,\afrak)] &\text{ if } x=\xfrak \\ 
  \mathbbm{1}[a_{1:H} = (\bfrak,\dots,\bfrak)] & \text{ if } x = \yfrak 
  \end{cases}.
  \]
  Define $\Pi := \{\pia,\pib\}$. Observe that $\Pr^{\pia}[a_h=\afrak \mid{} x=\xfrak, a_{1:h-1}=a'_{1:h-1}] = 1$ and $\Pr^{\pia}[a_h=\bfrak \mid{} x=\yfrak, a_{1:h-1}=a'_{1:h-1}] \geq 1/\denbound$ for any $h \in [H]$ and $a'_{1:h-1} \in \cA^{h-1}$. Moreover $\Pr^{\pib}[a_h=\afrak \mid{} x=\xfrak, a_{1:h-1}=a'_{1:h-1}] \geq 4/5 \geq 1/\denbound$ and $\Pr^{\pib}[a_h=\bfrak \mid{} x=\yfrak, a_{1:h-1}=a'_{1:h-1}] = 1$ for any $h \in [H]$ and $a'_{1:h-1} \in \cA^{h-1}$. Moreover, $\pia(\cdot\mid{}\perp) = \pib(\cdot\mid{}\perp) = \pistar(\cdot\mid{}\perp)$. Thus, $\pistar$ is $\denbound$-bounded with respect to $\Pi$.

  Now consider $n$ i.i.d. trajectories $o\ind{i} = (x\ind{i},a_{1:H}\ind{i})$ from $\BP^{\pistar}$. By choice of the context distribution $\rho$, we have $\Pr[x\ind{i}=\yfrak] = \veps$ for each $i \in [n]$. Let $\cE$ be the event that $n_\yfrak := \#\{i \in [n]: x\ind{i} = \yfrak\} \geq \frac{n\veps}{2}$ and $n_\xfrak := \#\{i \in [n]: x\ind{i} = \xfrak\} \leq 2nH\log(\denbound)\cdot\veps$. By Chernoff bounds and the assumption that $n \geq 8/\veps$, we have 
  \[\Pr[\cE] \geq 1 - 2e^{-\frac{n\veps}{8}} \geq 1 - 2/e.\]
  Condition on the event $\cE$ henceforth. By definition of $\pistar$, we know that $a_{1:H}\ind{i} = (\afrak,\dots,\afrak)$ whenever $x\ind{i} \in \{\perp,\xfrak\}$, and conversely $a_{1:H}\ind{i} = (\bfrak,\dots,\bfrak)$ whenever $x\ind{i} = \yfrak$. Thus, we have
  \begin{align} 
  \sum_{i=1}^n \sum_{h=1}^H \log \pia_h(a_h\ind{i} \mid{} x\ind{i}, a_{1:h-1}\ind{i})
  &= n_\yfrak \sum_{h=1}^H \log \pia_h(\bfrak \mid{} \yfrak, \bfrak,\dots,\bfrak) \\ 
  &= n_\yfrak H \log (1/\denbound) \\ 
  &\leq -\frac{nH\log(\denbound) \cdot \veps}{2}
  \end{align} 
  since $n_\yfrak \geq \frac{n\veps}{2}$ and $\denbound \geq 1$. On the other hand,
  \begin{align}
  \sum_{i=1}^n \sum_{h=1}^H \log \pib_h(a_h\ind{i} \mid{} x\ind{i}, a_{1:h-1}\ind{i})
  &= n_\xfrak \sum_{h=1}^H \log \pib_h(\afrak \mid{} \xfrak, \afrak,\dots,\afrak) \\
  &= n_\xfrak \log(4/5) \\
  &\geq -2nH\log(\denbound)\log(5/4) \cdot \veps.
  \end{align}
 Since $2\log(5/4) < 1/2$, it follows from the definition of $\loglossbc$ that $\pihat = \pib$. However, 
  \[\Dhels{\BP^{\pistar}}{\BP^{\pib}} \geq \rho(\xfrak) \Dtv{\pistar(\cdot\mid{}\xfrak)}{\pib(\cdot\mid{}\xfrak)}^2 \geq \Omega(H\log(\denbound) \cdot \veps)\]
  whereas
  \[\Dhels{\bbP^{\pistar}}{\BP^{\pia}} = \EE_{x \sim \rho} \Dhels{\bbP^{\pistar}(\cdot\mid{}x)}{\BP^{\pia}(\cdot\mid{}x)} \leq \rho(\yfrak) = \veps.\]
  The claim follows.
  \end{proof}

The following result asserts that without a density bound, the approximation ratio of $\loglossbc$ can be arbitrarily poor. The proof is immediate from \cref{prop:log-loss-lb} by taking $\denbound := e^{\frac{1}{H}(\frac{1}{\veps} - 1)}$. Notice that since the result applies for arbitrarily large sample complexity $n$, it is fundamentally a statement about the approximation ratio (and not the statistical rate).

\begin{proposition}\label{cor:unbounded}
  Fix any $n,H \in \NN$ and $\veps \in (0,1/2)$. Suppose that $n \geq 8/\veps$. There is an $H$-step autoregressive MDP $M$, a policy class $\Pi$ of size $|\Pi| = 2$, and an expert policy $\pistar$, with the following property. Given $n$ i.i.d. trajectories $o\ind{i} = (x\ind{i},a_{1:H}\ind{i})$ from $\BP^{\pistar}$, the estimator $\pihat$ produced by $\loglossbc$ satisfies, with probability at least $1-1/e$,
  \[\Dhels{\BP^{\pihat}}{\bbP^{\pistar}} \gtrsim \Omega(1)\]
  while $\min_{\pi\in\Pi} \Dhels{\BP^\pi}{\bbP^{\pistar}} \leq \veps$.
  \end{proposition}

\subsection{A Statistically Optimal Next-Token Prediction Algorithm}\label{sec:layerwise-rho}

In this section, we design a next-token prediction algorithm (i.e. iterative learner---see \cref{def:iterative}) that achieves $\Capx = O(H)$, which by \cref{thm:ntp-lb} is the best possible approximation ratio for any next-token prediction algorithm. In comparison, $\loglossbc$ requires assuming either a bound on density ratios, or query access to the density of $\pistar$, in order to achieve this guarantee. We emphasize that this result is mainly of interest statistically, and from the perspective of understanding the limits of next-token prediction---the algorithm is likely not efficiently implementable for autoregressive linear models.

For simplicity, we assume that the policy class $\Pi$ has no parameter sharing, as defined below. Note that \cref{thm:ntp-lb} also applies in this setting. Moreover, the assumption is nearly without loss of generality from a statistical perspective, since if $\Pi$ does have parameter sharing then one can define a new policy class $\overline{\Pi} := \Pi_1 \times \dots \times \Pi_H$ where $\Pi_h$ is the class of possible conditional distributions at layer $h$. Since $|\overline{\Pi}| \leq |\Pi|^H$, this will worsen the statistical rate by a factor of at most $H$, but $\overline{\Pi}$ has no parameter sharing so the below result then applies.

\begin{definition}
A policy class $\Pi$ \emph{has no parameter sharing} if there are sets $\Pi_1,\dots,\Pi_h$ so that $\pi = (\pi_h)_h \in \Pi$ if and only if $\pi_h \in \Pi_h$ for all $h \in [H]$.
\end{definition}

For a policy class $\Pi$ with no parameter sharing, \layerrhobc takes as input trajectories $o\ind{1},\dots,o\ind{n}$ where $o\ind{i} = (s_1\ind{i},a_1\ind{i},\dots,s_H\ind{i},a_H\ind{i})$, and outputs the policy $\pihat = (\pihat_h)_{h=1}^H$ defined by
\[\pihat_h := \argmin_{\pi_h \in \Pi_h} \sup_{\pi'_h \in \Pi_h} \sum_{i=1}^n\tau\left( \frac{\pihi}{\piphi}\right).\]

\begin{proposition}\label{prop:layerwise-rho}
Fix an MDP $M$, a policy class $\Pi$ with no parameter sharing, and an expert policy $\pistar$. Let $n \in \NN$ and $\delta \in (0,1/2)$. Let $\{o\ind{i}\}_{i=1}^n$ be i.i.d. trajectories $o\ind{i} = (s_1\ind{i},a_1\ind{i},\dots,s_H\ind{i},a_H\ind{i})$ from $\BP^{\pistar}$. Then the policy $\pihat$ produced by \layerrhobc satisfies, with probability at least $1-\delta$,
\begin{equation} \Dhels{\BP^{\pihat}}{\bbP^{\pistar}} \lesssim \frac{\log(|\Pi|) + H\log(H/\delta)}{n} + H \cdot \min_{\pi \in \Pi} \Dhels{\BP^\pi}{\bbP^{\pistar}}.
\end{equation}
\end{proposition}

\begin{proof}[\pfref{prop:layerwise-rho}]
For each $h \in [H]$ and $\pi \in \Pi\cup\{\pistar\}$, let $\BP^{\pi}_{1:h}$ denote the distribution of the prefix $(s_1,a_1,\dots,s_h,a_h)$ of a trajectory $(s_1,a_1,\dots,s_H,a_H)$ drawn from $\BP^\pi$. Let $\BP^{\pistar \circ_h \pi_h}_{1:h}$ denote the distribution of $(s_1,a_1,\dots,s_h,a_h)$ when $(s_1,a_1,\dots,s_h)$ is drawn from $\BP^{\pistar}$ and $a_h \sim \pi_h(\cdot\mid{}s_h)$. Define the family of distributions $\cP := \{\BP^{\pistar\circ_h\pi_h}_{1:h}: \pi_h \in \Pi_h\}$. Observe that for any $\pi_h, \pi'_h \in\Pi_h$ and trajectory prefix $(s_1,a_1,\dots,s_h,a_h)$, we have
\[\frac{\BP^{\pistar\circ_h\pi_h}_{1:h}(s_1,a_1,\dots,s_h)}{\BP^{\pistar\circ_h\pi'_h}_{1:h}(s_1,a_1,\dots,s_h)} = \frac{\pi_h(a_h\mid{}s_h)}{\pi'_h(a_h\mid{}s_h)}.\]
Thus, for each $h \in [H]$, by applying \cref{thm:rho} with family $\cP$, we have with probability at least $1-\delta/H$ that 
\begin{align} 
\Dhels{\BP^{\pistar\circ_h\pihat_h}_{1:h}}{\BP^{\pistar}_{1:h}} &\lesssim \frac{\log(H|\Pi_h|/\delta)}{n} + \min_{\pi_h\in\Pi_h} \Dhels{\BP^{\pistar\circ_h\pi_h}_{1:h}}{\BP^{\pistar}_{1:h}} \\ 
&= \frac{\log(H|\Pi_h|/\delta)}{n} + \min_{\pi=(\pi_k)_k\in\Pi} \Dhels{\BP^{\pistar\circ_h\pi_h}_{1:h}}{\BP^{\pistar}_{1:h}}.
\end{align}
Condition on the event that this bound holds for all $h \in [H]$, which occurs with probability at least $1-\delta$. Let $\pibar := \argmin_{\pi \in \Pi} \Dhels{\BP^\pi}{\bbP^{\pistar}}$. We have
\begin{align}
\Dhels{\BP^{\pihat}}{\BP^{\pistar}}
&\lesssim \sum_{h=1}^H \En^{\pistar}\left[\Dhels{\pihat_h(\cdot\mid{}s_h)}{\pistar_h(\cdot\mid{}s_h)}\right] \\ 
&= \sum_{h=1}^H \Dhels{\BP^{\pistar\circ_h\pihat_h}_{1:h}}{\BP^{\pistar}_{1:h}} \\ 
&\lesssim \frac{\log(|\Pi|)+H\log(H/\delta)}{n} + \sum_{h=1}^H \min_{\pi \in \Pi} \Dhels{\BP^{\pistar\circ_h\pi_h}_{1:h}}{\BP^{\pistar}_{1:h}} \\ 
&\lesssim \frac{\log(|\Pi|)+H\log(H/\delta)}{n} + \sum_{h=1}^H \min_{\pi \in \Pi} \Dhels{\BP^\pi_{1:h}}{\BP^{\pistar}_{1:h}} \\ 
&\leq \frac{\log(|\Pi|)+H\log(H/\delta)}{n} + H\cdot \min_{\pi \in \Pi} \Dhels{\BP^\pi}{\BP^{\pistar}}
\end{align}
where the first inequality is by \cref{lemma:hell-chain-bound}, the third inequality is by \cref{lemma:hell-reverse-chain-bound}, and the fourth inequality is by the data processing inequality. 
\end{proof}

\newpage
\section{Proof of Theorem \ref*{thm:rho-il} ($\rho$-Estimator)}\label{app:rho}

In this section, we prove \cref{thm:rho-il}, which is an immediate
corollary of \Cref{thm:rho}, a result of \citet{baraud2018rho}) which we prove for completeness below.% 
First, recall the function $\tau: (0,\infty) \to \RR$ defined to be
\begin{align}
  \tau(x) = \frac{\sqrt{1/x} - 1}{\sqrt{1/x} + 1},
\end{align}
and note that $\abs{\tau(x)} \leq 1$ for all $x > 0$.  The utility of
the $\tau$ function can be captured in the following lemma (originally
from \citet{baraud2018rho}), demonstrating that its expectation can be related to Hellinger distances.
\begin{lemma}[see e.g. {\citet[Theorem 97]{lerasle2019lecture}}]\label{lemma:rho-estimator-bounds}
For any set $\MX$ and densities $p,q,\pstar \in \Delta(\MX)$, it holds that
\begin{equation} -4\Dhels{\pstar}{q} + \frac{3}{8} \Dhels{\pstar}{p} \leq \EE_{x \sim \pstar}\left[\tau\left(\frac{p(x)}{q(x)}\right)\right] \leq 4\Dhels{\pstar}{p} - \frac{3}{8} \Dhels{\pstar}{q}
\label{eq:tau-hels}
\end{equation}
and
\begin{equation}
\EE_{x \sim \pstar}\left[\tau^2\left(\frac{p(x)}{q(x)}\right)\right] \leq 3\sqrt{2}\left(\Dhels{p^\st}{p} + \Dhels{p^\st}{q}\right).
\label{eq:tau-sq-hels}
\end{equation}
\end{lemma}
Following \citet{baraud2018rho,lerasle2019lecture} and using \Cref{lemma:rho-estimator-bounds}, we can now prove the following theorem on misspecified distribution learning in Hellinger distance.
\begin{theorem}\label{thm:rho}
Fix a set $\cX$, a family of distributions $\cP \subset \Delta(\cX)$, a distribution $p^\st \in \Delta(\cX)$. Let $n \in \NN$ and $\delta \in (0,1/2)$. Let $x\ind{1},\dots,x\ind{n}$ be $n$ i.i.d. samples from $p^\st$. Then the $\rho$-estimator 
\[\phat := \argmin_{p \in \cP} \sup_{q \in \cP} 
\sum_{i=1}^n \tau\left(\frac{p(x\ind{i})}{q(x\ind{i})}
\right)\]
satisfies, with probability at least $1-\delta$,
\begin{equation} \Dhels{\phat}{p^\st} \lesssim \frac{\log(|\cP|/\delta)}{n} + \min_{p\in\cP} \Dhels{p}{p^\st}.
\label{eq:rho-thm}\end{equation}
\end{theorem}

\begin{proof}[\pfref{thm:rho}]
Note that $\tau$ has range in $[-1,1]$. By Bernstein's inequality and a union bound over $p,q \in \cP$, there is an event $\cE$ that occurs with probability at least $1-\delta$, in which for all $p,q \in \cP$,
\begin{equation}\left| \sum_{i=1}^n \tau\left(\frac{p(x\ind{i})}{q(x\ind{i})}
\right) - n \cdot \EE_{x \sim \pstar}\left[\tau\left(\frac{p(x)}{q(x)}\right)\right]\right| \leq \frac{n}{4} \cdot \EE_{x \sim \pstar}\left[\tau^2\left(\frac{p(x)}{q(x)}\right)\right] + 4\log(4|\cP|/\delta).\label{eq:tau-generalization}\end{equation}
Condition on the event $\cE$ henceforth. Let $\pbar := \argmin_{p \in \cP} \Dhels{p}{\pstar}$. Then
\begin{align}
\frac{3}{8}\Dhels{\pstar}{\phat}
&\leq 4\Dhels{\pstar}{\pbar} + \EE_{x \sim \pstar}\left[\tau\left(\frac{\phat(x)}{\pbar(x)}\right)\right] \\ 
&\leq 4\Dhels{\pstar}{\pbar} + \frac{1}{n} \sum_{i=1}^n\tau\left(\frac{\phat(x\ind{i})}{\pbar(x\ind{i})}\right) + \frac{1}{12\sqrt{2}} \EE_{x \sim \pstar}\left[\tau^2\left(\frac{\phat(x)}{\pbar(x)}\right)\right] + \frac{12\sqrt{2}\log(4|\cP|/\delta)}{n} \\ 
&\leq \frac{17}{4}\Dhels{\pstar}{\pbar} + \frac{1}{4}\Dhels{\pstar}{\phat} + \frac{1}{n}\sum_{i=1}^n\tau\left(\frac{\phat(x\ind{i})}{\pbar(x\ind{i})}\right) + \frac{12\sqrt{2}\log(4|\cP|/\delta)}{n}
\end{align}
where the first inequality is by \cref{eq:tau-hels} of \cref{lemma:rho-estimator-bounds}, the second inequality is by \cref{eq:tau-generalization}, and the third inequality is by \cref{eq:tau-sq-hels} of \cref{lemma:rho-estimator-bounds}. Rearranging, we get
\begin{align}
\frac{1}{8}\Dhels{\pstar}{\phat}
&\leq \frac{17}{4}\Dhels{\pstar}{\pbar} + \frac{1}{n}\sum_{i=1}^n\tau\left(\frac{\phat(x\ind{i})}{\pbar(x\ind{i})}\right) + \frac{12\sqrt{2}\log(4|\cP|/\delta)}{n} \\ 
&\leq \frac{17}{4}\Dhels{\pstar}{\pbar} + \sup_{q \in \cP} \frac{1}{n}\sum_{i=1}^n\tau\left(\frac{\phat(x\ind{i})}{q(x\ind{i})}\right) + \frac{12\sqrt{2}\log(4|\cP|/\delta)}{n} \\ 
&\leq \frac{17}{4}\Dhels{\pstar}{\pbar} + \sup_{q \in \cP} \frac{1}{n}\sum_{i=1}^n\tau\left(\frac{\pbar(x\ind{i})}{q(x\ind{i})}\right) + \frac{12\sqrt{2}\log(4|\cP|/\delta)}{n}.\label{eq:rho-hels-intermediate}
\end{align}
Now for any $q \in \cP$, we have
\begin{align}
\frac{1}{n}\sum_{i=1}^n\tau\left(\frac{\pbar(x\ind{i})}{q(x\ind{i})}\right)
&\leq \EE_{x \sim \pstar}\left[\tau\left(\frac{\pbar(x)}{q(x)}\right)\right] + \frac{1}{12\sqrt{2}} \EE_{x \sim \pstar}\left[\tau^2\left(\frac{\pbar(x)}{q(x)}\right)\right] + \frac{12\sqrt{2}\log(4|\cP|/\delta)}{n} \\ 
&\leq 4\Dhels{\pstar}{\pbar} - \frac{3}{8} \Dhels{\pstar}{q} + \frac{1}{4}\left(\Dhels{\pstar}{\pbar} + \Dhels{\pstar}{q}\right) + \frac{12\sqrt{2}\log(4|\cP|/\delta)}{n} \\ 
&\leq \frac{17}{4}\Dhels{\pstar}{\pbar} + \frac{12\sqrt{2}\log(4|\cP|/\delta)}{n}
\end{align}
where the first inequality is by \cref{eq:tau-generalization} and the second inequality is by \cref{lemma:rho-estimator-bounds}. Substituting into \cref{eq:rho-hels-intermediate}, we get
\[\Dhels{\pstar}{\wh p} \leq 68 \Dhels{\pstar}{\pbar} + \frac{192\sqrt{2}\log(4|\cP|/\delta)}{n}\]
as claimed.
\end{proof}
We can now prove \cref{thm:rho-il} as a corollary of \cref{thm:rho}.
\begin{proof}[\pfref{thm:rho-il}]
  Note that for any policies $\pi,\pi'$ and trajectory $o=(s_1,a_1,\dots,s_H,a_H)$,
  \begin{align}
    \frac{\pp^\pi(\obs)}{\pp^{\pi'}(\obs)} = \prod_{h = 1}^H \frac{\bbP_h(s_{h+1} \mid{} a_h, s_h) \pi_h(a_h \mid{} s_h)}{\bbP_h(s_{h+1} \mid{} a_h, s_h) \pi_h'(a_h \mid{} s_h)} = \prod_{h = 1}^H \frac{\pi_h(a_h \mid{} s_h)}{\pi_h'(a_h \mid{} s_h)},
  \end{align}
  and thus, from \cref{eq:rhobc},
  \begin{align}
    \pihat = \argmin_{\pi \in \Pi} \sup_{\pi'}\sum_{i = 1}^n \tau\left( \frac{\pp^\pi(\obs)}{\pp^{\pi'}(\obs)} \right).
  \end{align}
  The result then follows from \cref{thm:rho} by letting $\cP = \left\{ \pp^{\pi} 
  \mid{} \pi \in \Pi \right\}$, $\pstar = \pp^{\pistar}$ and observing that $\phat = \pp^{\pihat}$ by the preceding display.
\end{proof}

\newpage
\section{Proofs from Section \ref*{sec:next_token} (Next-Token Prediction)}\label{app:ntp}

This section gives proofs for the main results from
\cref{sec:next_token}. In \cref{sec:log-loss-app} we prove \cref{thm:log-loss-bounded}, a sharp statistical analysis of $\loglossbc$ under a density bound assumption. In \cref{sec:improvements-app}, we prove \cref{prop:log-loss-prob-lb}, which proves statistical lower bounds for $\loglossbc$; \cref{prop:boosted-bc}, which shows that $\loglossbc$ can be boosted to high probability success via cross validation; and \cref{thm:layerwise_smoothing}, which provides a statistical analysis of $\smoothedloglossbc$. In \cref{sec:limits-app}, we prove \cref{thm:ntp-lb}, which shows that no next-token prediction algorithm can achieve $\Capx = o(H)$.

\subsection{Proofs from Section \ref*{sec:logloss} (Log-Loss Next-Token Prediction)}\label{sec:log-loss-app}

In this section, we prove a generalization of \cref{thm:log-loss-bounded} which allows for (a) infinite policy classes and (b) optimization error, since both will be useful for the setting of autoregressive linear models.

\begin{definition}\label{def:cover}
Fix a policy class $\Pi$ on state space $\MS$ and action space $\MA$. For $\eps>0$, we say that $\Pi' \subset \Pi$ is an \emph{$\eps$-cover} of $\Pi$ if for each $\pi\in\Pi$, there is some $\pi'\in\Pi'$ with $\log(\pi(a\mid{}s)/\pi'(a\mid{}s)) \leq \eps$ for all $a\in\MA$ and $s\in\MS$. We write $\Nlog(\Pi,\eps)$ to denote the cardinality of the smallest $\eps$-cover of $\Pi$.
\end{definition}

For a policy class $\Pi$ and a parameter $\epopt>0$, $\epopt$-approximate $\loglossbc$ takes as input trajectories $o\ind{1},\dots,o\ind{n}$ where $o\ind{i} = (s_1\ind{i},a_1\ind{i},\dots,s_H\ind{i},a_H\ind{i})$, and outputs some policy $\pihat$ satisfying
\[\wh L(\pihat) \geq \max_{\pi \in \Pi} \wh L(\pi) - \epopt \quad\text{where}\quad \wh L(\pi) := \sum_{i=1}^n \sum_{h=1}^H \log (\pihi).\]

\begin{theorem}[Full version of \cref{thm:log-loss-bounded}]\label{thm:log-loss-bounded-app}
Fix an MDP $M$, a policy class $\Pi$, and an expert policy $\pistar$. Suppose that $\pistar$ is $\denbound$-bounded with respect to $\Pi$ (\cref{ass:density-bound}) for some $\denbound \geq 1$. Let $n \in \NN$ and $\epsilon,\epopt,\delta > 0$. Let $\{o\ind{i}\}_{i=1}^n$ be i.i.d. trajectories $o\ind{i} = (s_1\ind{i},a_1\ind{i},\dots,s_H\ind{i},a_H\ind{i})$ from $\BP^{\pistar}$. Then any policy $\pihat$ produced by $\epopt$-approximate $\loglossbc$ satisfies, with probability at least $1-\delta$,
\begin{align} \Dhels{\BP^{\pihat}}{\bbP^{\pistar}} &\lesssim \frac{\epopt}{n} + H\epsilon+\frac{\log(\Nlog(\Pi,\epsilon)/\delta)}{n} + \frac{H\log(eW)\log(1/\delta)}{n} \\
&+ H\log(eW) \cdot \min_{\pi \in \Pi} \Dhels{\BP^\pi}{\bbP^{\pistar}}.
\label{eq:log-loss-thm-1}\end{align}
Additionally, $\pihat$ satisfies, with probability at least $1-\delta$,
\begin{align} \Dhels{\BP^{\pihat}}{\bbP^{\pistar}} &\lesssim \frac{\epopt}{n} + H\epsilon+\frac{\log(\Nlog(\Pi,\epsilon)/\delta)}{n} + \frac{\log(eW)\log(1/\delta)}{n} \\&+ \frac{H\log(eW)}{\delta} \cdot \min_{\pi \in \Pi} \Dhels{\BP^\pi}{\bbP^{\pistar}}.
\label{eq:log-loss-thm-2}\end{align}
\end{theorem}

In particular, \cref{thm:log-loss-bounded} follows from
\cref{eq:log-loss-thm-2} by taking $\epsilon = \epopt = 0$. Notice that \cref{eq:log-loss-thm-1} avoids dependence on $1/\delta$ in the approximation ratio, but incurs an extra factor of $H$ in the statistical rate. 

\paragraph{Proof overview} The proofs of the two bounds \cref{eq:log-loss-thm-1,eq:log-loss-thm-2} are largely similar; the difference is that \cref{eq:log-loss-thm-1} is derived by applying Bernstein's inequality in the final step, whereas \cref{eq:log-loss-thm-2} uses Markov's inequality. In both cases, the first observation is that by a standard argument (\cref{lemma:hels-to-lhat}), it suffices to bound the empirical excess risk of the best-in-class model $\pibar := \argmin_{\pi\in\Pi} \Dhels{\BP^\pi}{\BP^{\pistar}}$:
\begin{equation}\wh L(\pistar) - \wh L(\pibar) = \sum_{i=1}^n\sum_{h=1}^H \log\frac{\pistarhi}{\pibarhi}.\label{eq:lhat-overview}\end{equation}
\cref{eq:lhat-overview} can be interpreted as an empirical analogue of $\Dkl{\bbP^{\pistar}}{\bbP^{\pibar}}$, and in prior work it is upper bounded in terms of the (population-level) $\chi$-squared divergence $\Dchis{\bbP^{\pistar}}{\bbP^{\pibar}}$. However, even under $W$-boundedness, this divergence cannot be bounded by Hellinger distance without paying a factor of $W^H$. Instead, our goal is to upper bound \cref{eq:lhat-overview} in terms of the \emph{sum of conditional squared Hellinger distances}, i.e.
\[\En^{\pistar}\left[\sum_{h=1}^H \Dhels{\pistar_h(\cdot\mid{}s_h)}{\pibar_h(\cdot\mid{}s_h)}\right],\]
which can be upper bounded by $O(H) \cdot \Dhels{\bbP^{\pistar}}{\bbP^{\pibar}}$ by a standard information-theoretic argument (\cref{cor:hell-reverse-chain-bound}). To achieve this, we use $W$-boundedness together with a more layer-wise concentration argument. The main technical subtlety is that $W$-boundedness only gives an upper bound on the terms in \cref{eq:lhat-overview} (they could still be arbitrarily negative), which is problematic for naive concentration arguments; however, since an upper bound is ultimately what we care about, this can be fixed by appropriately ``truncating'' the logarithm prior to concentration. We now proceed to the formal proof.

\begin{proof}[\pfref{thm:log-loss-bounded-app}]
Define $\pibar := \argmin_{\pi\in\Pi} \Dhels{\BP^\pi}{\bbP^{\pistar}}$, and define $f: [0,\infty) \to \RR$ by 
\[f(t) := \begin{cases} 
\log(t) & \text{ if } t \geq 1 \\ t - 1 & \text{ if } t < 1 
\end{cases}.\] Then we have
\begin{align} 
\wh L(\pistar) - \wh L(\pihat) - \epopt
&\leq \wh L(\pistar) - \wh L(\pibar) \\ 
&= \sum_{i=1}^n \sum_{h=1}^H \log \frac{\pistarhi}{\pibarhi} \\ 
&\leq \sum_{i=1}^n \sum_{h=1}^H f\left(\frac{\pistarhi}{\pibarhi}\right)\label{eq:lhat-bound}
\end{align}
where the first inequality is by definition of $\pihat$, and the second inequality uses that $f(t) \geq \log(t)$ for all $t \geq 0$. Define $Z_{i,h} = f\left(\frac{\pistarhi}{\pibarhi}\right)$. By \cref{ass:density-bound}, we have $\frac{\pistarhi}{\pibarhi} \in [0,\denbound]$ and hence $|Z_{i,h}| \leq 1+\log(\denbound)$ almost surely. Consider the filtration $(\cF_{i,h})_{i,h}$ where $\cF_{i,h}$ is induced by $o\ind{1},\dots,o\ind{i-1}$ and $s_1\ind{i},a_1\ind{i},\dots,s_h\ind{i},a_h\ind{i},s_{h+1}\ind{i}$. Then the sequence of random variables $(Z_{i,h})_{i,h}$ is adapted to this filtration. By Freedman's inequality, there is an event $\cE_1$ that occurs with probability at least $1-\delta/3$, under which we have 
\begin{align}
\sum_{i=1}^n \sum_{h=1}^H Z_{i,h} \leq \sum_{i=1}^n \sum_{h=1}^H \EE[Z_{i,h}\mid{} \cF_{i,h-1}] + \frac{1}{1+\log(\denbound)} \sum_{i=1}^n \sum_{h=1}^H \EE[Z_{i,h}^2\mid{} \cF_{i,h-1}] + (1+\log(\denbound))\log(3/\delta)
\end{align}
where for notational convenience we write $\cF_{i,0}$ to denote $\cF_{i-1,H}$. Now observe that for any $i,h$,
\begin{align}
\EE[\exp(-Z_{i,h}) \mid{} \cF_{i,h-1}]
&= \EE\left[\exp\left(-f\left(\frac{\pistarhi}{\pibarhi}\right)\right) \middle|\, s_h\ind{i}\right] \\ 
&\leq \EE\left[\exp\left(-\log\left(\frac{\pistarhi}{\pibarhi}\right)\right) \middle|\, s_h\ind{i}\right] \\ 
&= \EE_{a_h \sim \pistar_h(\cdot\mid{} s_h\ind{i})}\left[\frac{\pibar_h(a_h\mid{}s_h\ind{i})}{\pistar_h(a_h\mid{}s_h\ind{i})}\right] \\ 
&= 1
\end{align}
where the inequality again uses that $f(t) \geq \log(t)$ for all $t \geq 0$. Therefore \cref{lemma:central-to-bernstein} with $\eta := 1$ and $V := 1+\log(\denbound)$ gives that
\[\EE[Z_{i,h}^2\mid{} \cF_{i,h-1}] \leq 4(2+\log(\denbound))\EE[Z_{i,h}\mid{} \cF_{i,h-1}].\]
We conclude that in event $\cE_1$,
\begin{align} 
\sum_{i=1}^n \sum_{h=1}^H Z_{i,h} &\leq 9\sum_{i=1}^n \sum_{h=1}^H \EE[Z_{i,h}\mid{} \cF_{i,h-1}] + (1+\log(\denbound))\log(3/\delta) \\ 
&= 9\sum_{i=1}^n \sum_{h=1}^H \EE_{a_h \sim \pistar_h(\cdot\mid{} s_h\ind{i})}\left[f\left(\frac{\pistar_h(a_h\mid{}s_h\ind{i})}{\pibar_h(a_h\mid{}s_h\ind{i})}\right)\right] + (1+\log(\denbound))\log(3/\delta) \\ 
&\leq 9(4+\log(\denbound))\sum_{i=1}^n \sum_{h=1}^H \Dhels{\pistar_h(\cdot\mid{}s_h\ind{i})}{\pibar_h(\cdot\mid{}s_h\ind{i})} + (1+\log(\denbound))\log(3/\delta)\label{eq:e1-bound}
\end{align}
where the final inequality is by \cref{lemma:f-to-hels} and again uses \cref{ass:density-bound}. Next, by Bernstein's inequality applied to the i.i.d. random variables $\sum_{h=1}^H \Dhels{\pistar_h(\cdot\mid{}s_h\ind{i})}{\pibar_h(\cdot\mid{}s_h\ind{i})}$ for $i = 1,\dots,n$, there is an event $\cE_2$ that occurs with probability at least $1-\delta/3$ under which 
\begin{align} 
\sum_{i=1}^n \sum_{h=1}^H \Dhels{\pistar_h(\cdot\mid{}s_h\ind{i})}{\pibar_h(\cdot\mid{}s_h\ind{i})} &\leq n\En^{\pistar}\left[\sum_{h=1}^H \Dhels{\pistar_h(\cdot\mid{}s_h)}{\pibar_h(\cdot\mid{}s_h)}\right] \\
&\qquad+ \frac{n}{H}\En^{\pistar}\left[\left(\sum_{h=1}^H \Dhels{\pistar_h(\cdot\mid{}s_h)}{\pibar_h(\cdot\mid{}s_h)}\right)^2\right] \\ 
&\qquad+ H\log(3/\delta) \\ 
&\leq 2n\En^{\pistar}\left[\sum_{h=1}^H \Dhels{\pistar_h(\cdot\mid{}s_h)}{\pibar_h(\cdot\mid{}s_h)}\right] + H\log(3/\delta).\label{eq:e2-bound}
\end{align}
Additionally, by Markov's inequality, there is an event $\cE_2'$ that occurs with probability at least $1-\delta/3$ under which 
\begin{align} 
\sum_{i=1}^n \sum_{h=1}^H \Dhels{\pistar_h(\cdot\mid{}s_h\ind{i})}{\pibar_h(\cdot\mid{}s_h\ind{i})} 
&\leq \frac{3n}{\delta}\En^{\pistar}\left[\sum_{h=1}^H \Dhels{\pistar_h(\cdot\mid{}s_h)}{\pibar_h(\cdot\mid{}s_h)}\right].\label{eq:e2p-bound}
\end{align}
Finally, by \cref{lemma:hels-to-lhat}, there is an event $\cE_3$ that occurs with probability at least $1-\delta/3$ under which
\begin{align}\Dhels{\BP^{\pihat}}{\bbP^{\pistar}}\leq 4H\epsilon + \frac{4\log(3\Nlog(\Pi,\epsilon)/\delta)}{n} + \frac{2}{n}\left(\wh L(\pistar) - \wh L(\pihat)\right).\label{eq:e3-bound}\end{align}
Combining \eqref{eq:e3-bound} with \cref{eq:lhat-bound,eq:e1-bound,eq:e2-bound} we get that in the event $\cE_1 \cap \cE_2 \cap \cE_3$, which occurs with probability at least $1-\delta$,
\begin{align}
\Dhels{\BP^{\pihat}}{\bbP^{\pistar}} &\lesssim \frac{2\epopt}{n} + 4H\epsilon + \frac{\log(3\Nlog(\Pi,\epsilon)/\delta)}{n} + \frac{(1+\log(\denbound))H\log(3/\delta)}{n} \\&+ (1+\log(\denbound))\En^{\pistar}\left[\sum_{h=1}^H \Dhels{\pistar_h(\cdot\mid{}s_h)}{\pibar_h(\cdot\mid{}s_h)}\right].
\end{align}
The result \eqref{eq:log-loss-thm-1} now follows from the above bound and \cref{cor:hell-reverse-chain-bound}: in particular, applying \cref{cor:hell-reverse-chain-bound} to the distributions $\bbP := \bbP^{\pistar}$ and $\bbQ := \bbP^{\pibar}$ gives
\[H \cdot \Dhels{\bbP^{\pistar}}{\bbP^{\pibar}} \gtrsim \En^{\pistar}\left[\sum_{h=1}^H \Dhels{\pistar_h(\cdot\mid{}s_h)}{\pibar_h(\cdot\mid{}s_h)} + \Dhels{\bbP_h(\cdot\mid{}s_h,a_h)}{\bbP_h(\cdot\mid{}s_h,a_h)}\right],\]
and the terms involving the transition probabilities all vanish.

Similarly, combining \eqref{eq:e3-bound} with \cref{eq:lhat-bound,eq:e1-bound,eq:e2p-bound} we get that in the event $\cE_1 \cap \cE'_2 \cap \cE_3$, which occurs with probability at least $1-\delta$,
\begin{align}
\Dhels{\BP^{\pihat}}{\bbP^{\pistar}} &\lesssim \frac{2\epopt}{n} + 4H\epsilon + \frac{\log(3\Nlog(\Pi,\epsilon)/\delta)}{n} + \frac{(1+\log(\denbound))\log(3/\delta)}{n} \\
&+ \frac{(1+\log(\denbound))}{\delta}\En^{\pistar}\left[\sum_{h=1}^H \Dhels{\pistar_h(\cdot\mid{}s_h)}{\pibar_h(\cdot\mid{}s_h)}\right].
\end{align}
The result \eqref{eq:log-loss-thm-2} follows from this bound and \cref{cor:hell-reverse-chain-bound}.
\end{proof}

\subsubsection{Supporting Lemmas}

The following result is implicit in the proof of \citet[Proposition B.1]{foster2024behavior}. We include the proof for completeness.

\begin{lemma}[\cite{foster2024behavior}]\label{lemma:hels-to-lhat}
In the setting of \cref{thm:log-loss-bounded}, it holds with probability at least $1-\delta$ that 
\[\Dhels{\BP^{\pihat}}{\bbP^{\pistar}}\leq 4H\epsilon + \frac{4\log(\Nlog(\Pi,\epsilon)/\delta)}{n} + \frac{2}{n}\left(\wh L(\pistar) - \wh L(\pihat)\right).\]
\end{lemma}

\begin{proof}[\pfref{lemma:hels-to-lhat}]
Let $\Pi'$ be an $\eps$-cover for $\Pi$ (\cref{def:cover}) and fix $\pihat' \in \Pi'$ with $\log(\pihat(a\mid{}s)/\pihat'(a\mid{}s)) \leq \eps$ for all $a\in\MA$ and $s\in\MS$. Note that $\log(\BP^{\pihat}(o)/\BP^{\pihat'}(o)) \leq H\epsilon$ for any trajectory $o$, and hence $\Dkl{\BP^{\pihat}}{\BP^{\pihat'}} \leq H\epsilon$. For each $\pi'\in\Pi'$ and $i \in [n]$, define the random variable
\[X_i(\pi') := \sum_{h=1}^H \log \frac{\pistarhi}{\piphi} = \log \frac{\BP^{\pistar}(o\ind{i})}{\BP^{\pi'}(o\ind{i})}.\]
Note that the random variables $X_1(\pi'),\dots,X_n(\pi')$ are independent and identically distributed. Thus, by an exponential Markov bound and the union bound, it holds with probability at least $1-\delta$ that for all $\pi'\in\Pi'$,
\begin{align}
\frac{1}{2}\left(\wh L(\pi') - \wh L(\pistar)\right)
&= \log\left(\prod_{i=1}^n e^{-\frac{1}{2}X_i(\pi')}\right) \\ 
&\leq \log(|\Pi'|/\delta) + n\cdot \log\left(\EE[e^{-\frac{1}{2}X_1(\pi')}]\right).\label{eq:exp-markov}
\end{align}
Condition on this event henceforth. For any $\pi'\in\Pi'$,
\begin{align}
\log\left(\EE[e^{-\frac{1}{2}X_1(\pi')}]\right) 
&= \log \EE_{o\sim\pistar}\left[\exp\left(-\frac{1}{2} \log \frac{\BP^{\pistar}(o)}{\BP^{\pi'}(o)}\right)\right]\\ 
&= \log\left(1 - \frac{1}{2}\Dhels{\BP^{\pistar}}{\BP^{\pi'}}\right) \\ 
&\leq -\frac{1}{2} \Dhels{\BP^{\pistar}}{\BP^{\pi'}}.
\end{align}
Setting $\pi' := \pihat'$ and substituting into \cref{eq:exp-markov}, we get that
\begin{equation} \wh L(\pihat') - \wh L(\pistar) \leq 2\log(|\Pi'|/\delta) - n\Dhels{\BP^{\pistar}}{\BP^{\pihat'}}.\label{eq:piphat-bound}\end{equation}
Therefore
\begin{align}
\Dhels{\BP^{\pihat}}{\BP^{\pistar}}
&\leq 2\Dhels{\BP^{\pihat}}{\BP^{\pihat'}} + 2\Dhels{\BP^{\pihat'}}{\BP^{\pistar}} \\ 
&\leq 2H\epsilon + \frac{4\log(|\Pi'|/\delta)}{n} + \frac{2}{n}\left(\wh L(\pistar) - \wh L(\pihat')\right) \\ 
&\leq 4H\epsilon + \frac{4\log(|\Pi'|/\delta)}{n} + \frac{2}{n}\left(\wh L(\pistar) - \wh L(\pihat)\right)
\end{align}
where the first inequality uses that $\Dhel{\cdot}{\cdot}$ is a metric; the second inequality uses the fact that $\Dhels{\BP^{\pihat}}{\BP^{\pihat'}} \leq \Dkl{\BP^{\pihat}}{\BP^{\pihat'}} \leq \epsilon$ as well as \cref{eq:piphat-bound}; and the third inequality uses that
\[\wh L(\pihat) - \wh L(\pihat') = \sum_{i=1}^n \log\frac{\BP^{\pihat}(o\ind{i})}{\BP^{\pihat'}(o\ind{i})} \leq nH\epsilon.\]
This completes the proof.
\end{proof}

We also use the following supporting lemmas in the proof of \cref{thm:log-loss-bounded-app}. \cref{lemma:f-to-hels} shows that even though we ``truncated'' the logarithm, we can still upper bound the corresponding $f$-divergence in terms of Hellinger distance, under a density ratio bound; it is a modification of e.g. \cite[Lemma 4]{yang1998asymptotic}.

\begin{lemma}[Central-to-Bernstein \citep{mehta2017fast}]\label{lemma:central-to-bernstein}
  Let $X\in\brk{-V,V}$ be a random variable with
  $\En\brk*{\exp(-\eta{}X)}\leq{}1$. Then $\En\brk*{X^2}\leq{} 4(\eta^{-1}+V)\En\brk*{X}$.
\end{lemma}

\begin{lemma}\label{lemma:f-to-hels}
Define $f: [0,\infty) \to \RR$ by 
\[f(t) := \begin{cases} 
\log(t) & \text{ if } t \geq 1 \\ t - 1 & \text{ if } t < 1 
\end{cases}.\]
For any set $\cX$ and densities $p,q \in \Delta(\cX)$ with $V := \sup_{x \in \cX} \frac{p(x)}{q(x)}$, it holds that
\[\EE_{x \sim p} f\left(\frac{p(x)}{q(x)}\right) \leq (4+\log(V)) \Dhels{p}{q}.\]
\end{lemma}

\begin{proof}[\pfref{lemma:f-to-hels}]
We have
\begin{align}
\EE_{x \sim p} \left[f\left(\frac{p(x)}{q(x)}\right)\right]
&= \EE_{x \sim q} \left[\frac{p(x)}{q(x)} f\left(\frac{p(x)}{q(x)}\right)\right] \\ 
&= \EE_{x \sim q} \left[\frac{p(x)}{q(x)} f\left(\frac{p(x)}{q(x)}\right) - \frac{p(x)}{q(x)} + 1\right] \\ 
&= \EE_{x \sim q} \left[h\left(\frac{p(x)}{q(x)}\right)\right]
\end{align}
where $h(t) := tf(t) - t + 1$. For any $t \in [0,1]$, we have
\begin{align}
h(t)
&= (1 - t)^2 \\ 
&= (1 + \sqrt{t})^2 (1 - \sqrt{t})^2 \\ 
&\leq 4 (1 - \sqrt{t})^2.
\end{align}
Next, observe that for any $t \geq 1$, since $\log(t) = 2\log\sqrt{t} \leq 2\sqrt{t}-2$, we have
\begin{align}
h(t)
&= t\log(t) + 1 - t \\ 
&= t\log(t) + (1-\sqrt{t})(1+\sqrt{t}) \\ 
&\leq t\log(t) + (1-\sqrt{t})(1+\sqrt{t}) + (2\sqrt{t}-1)(2\sqrt{t}-2-\log(t)) \\ 
&= (\sqrt{t} - 1)^2(3 + \log(t)).
\end{align}
Since $p(x)/q(x) \in [0,V]$ for all $x \in \cX$, we get that
\begin{align}
\EE_{x \sim p} \left[f\left(\frac{p(x)}{q(x)}\right)\right] 
&= \EE_{x \sim q} \left[h\left(\frac{p(x)}{q(x)}\right)\right] \\
&\leq (4 + \log(V))\EE_{x \sim q} \left[\left(\sqrt{\frac{p(x)}{q(x)}} - 1\right)^2\right] \\ 
&= (4 + \log(V)) \Dhels{p}{q}
\end{align}
as claimed.
\end{proof}

  \subsection{Proofs from Section \ref*{sec:improvements} (Improvements
    to Next-Token Prediction)}\label{sec:improvements-app}

Here we prove \cref{prop:log-loss-prob-lb}, \cref{prop:boosted-bc}, and \cref{thm:layerwise_smoothing}. The following result shows that $\loglossbc$ necessarily incurs either a factor of $H\log(\denbound)$ in the rate (where $H$ is the horizon and $\denbound$ is the norm bound in \cref{ass:density-bound}), or has approximation ratio scaling with $1/\delta$.

\begin{proposition}[Restatement of \cref{prop:log-loss-prob-lb}]\label{prop:log-loss-prob-lb-app}
  Fix any $H \in \NN$ and $\denbound \geq 2$ and $\delta \in (0,1/2)$, and set $n_0 := H\log(\denbound)$. There is an $H$-step autoregressive MDP $M$, a policy class $\Pi$ of size $|\Pi| = 2$, and an expert policy $\pistar$ such that $\pistar$ is $\denbound$-bounded with respect to $\Pi$ (\cref{ass:density-bound}), with the following property. Given $n_0$ i.i.d. trajectories $o\ind{i} = (x\ind{i},a_{1:H}\ind{i})$ from $\BP^{\pistar}$, the estimator $\pihat$ produced by $\loglossbc$ satisfies, with probability at least $\delta$,
  \[\Dhels{\BP^{\pihat}}{\bbP^{\pistar}} \gtrsim \frac{H\log(\denbound)}{\delta} \cdot \min_{\pi\in\Pi} \Dhels{\BP^\pi}{\bbP^{\pistar}}.\]
  \end{proposition}

  \begin{proof}[\pfref{prop:log-loss-prob-lb-app}]
  Let $M$ be the $H$-step autoregressive MDP with context space $\cX = \{\xfrak,\yfrak\}$, action space $\cA = \{\afrak,\bfrak\}$, and context distribution $\rho \ \in\Delta(\cX)$ with $\rho(\yfrak) = 2\delta/(H\log \denbound)$. Define $\pistar$ so that 
  \[ 
  \Pr^{\pistar}[a_{1:H} \mid{} x] = \begin{cases} 
  \mathbbm{1}[a_{1:H}=(\afrak,\dots,\afrak)] &\text{ if } x =\xfrak \\ 
  \mathbbm{1}[a_{1:H} = (\bfrak,\dots,\bfrak)] & \text{ if } x = \yfrak 
  \end{cases}.
  \]
  Define $\pia$ so that
\[ 
  \Pr^{\pia}[a_{1:H} \mid{} x] = \begin{cases} 
  \mathbbm{1}[a_{1:H}=(\afrak,\dots,\afrak)] &\text{ if } x=\xfrak \\ 
  \prod_{h=1}^H \left(1-\frac{1}{\denbound}\right)^{\mathbbm{1}[a_h=\afrak]}\left(\frac{1}{\denbound}\right)^{\mathbbm{1}[a_h=\bfrak]} & \text{ if } x = \yfrak 
  \end{cases}.
  \]
  Define $\pib$ so that
  \[ 
  \Pr^{\pib}[a_{1:H} \mid{} x] = \begin{cases}
  \left(\frac{4}{5}\right)^{\mathbbm{1}[a_1=\afrak]} \left(\frac{1}{5}\right)^{\mathbbm{1}[a_1=\bfrak]}\mathbbm{1}[a_{2:H}=(\afrak,\dots,\afrak)] &\text{ if } x=\xfrak \\ 
  \mathbbm{1}[a_{1:H} = (\bfrak,\dots,\bfrak)] & \text{ if } x = \yfrak 
  \end{cases}.
  \]
  Define $\Pi := \{\pia,\pib\}$. Observe that $\Pr^{\pia}[a_h=\afrak \mid{} x=\xfrak, a_{1:h-1}=a'_{1:h-1}] = 1$ and $\Pr^{\pia}[a_h=\bfrak \mid{} x=\yfrak, a_{1:h-1}=a'_{1:h-1}] \geq 1/\denbound$ for any $h \in [H]$ and $a'_{1:h-1} \in \cA^{h-1}$. Moreover $\Pr^{\pib}[a_h=\afrak \mid{} x=\xfrak, a_{1:h-1}=a'_{1:h-1}] \geq 4/5 \geq 1/\denbound$ and $\Pr^{\pib}[a_h=\bfrak \mid{} x=\yfrak, a_{1:h-1}=a'_{1:h-1}] = 1$ for any $h \in [H]$ and $a'_{1:h-1} \in \cA^{h-1}$. Thus, $\pistar$ is $\denbound$-bounded with respect to $\Pi$.

  Consider $n$ i.i.d. trajectories $o\ind{i} = (x\ind{i},a\ind{i}_{1:H})$ from $\BP^{\pistar}$. By choice of the context distribution $\rho$, we have $\Pr[x\ind{i}=\yfrak] = 2\delta/(H\log \denbound)$ for each $i \in [n]$. Let $\cE$ be the event that $n_\yfrak := \#\{i \in [n]: x\ind{i} = \yfrak\} \geq 1$. Then 
  \[\Pr[\cE] \geq 1 - \left(1 - \frac{2\delta}{H\log(\denbound)}\right)^n \geq 1 - e^{-2\delta} \geq \delta\]
  by choice of $n$. Condition on the event $\cE$. Again by choice of $n$, we have $n_\yfrak \geq \frac{n}{2H\log(\denbound)}$. Thus, we have
  \begin{align} 
  \sum_{i=1}^n \sum_{h=1}^H \log \pia_h(a_h\ind{i} \mid{} x\ind{i}, a_{1:h-1}\ind{i})
  &= n_\yfrak \sum_{h=1}^H \log \pia_h(\bfrak \mid{} \yfrak, \bfrak,\dots,\bfrak) \\ 
  &= n_\yfrak H \log (1/\denbound) \\ 
  &\leq -\frac{n}{2}
  \end{align} 
  since $n_\yfrak \geq \frac{n}{2H\log(\denbound)}$ and $\denbound \geq 1$. On the other hand,
  \begin{align}
  \sum_{i=1}^n \sum_{h=1}^H \log \pib_h(a_h\ind{i} \mid{} x\ind{i}, a_{1:h-1}\ind{i})
  &= (n-n_\yfrak) \sum_{h=1}^H \log \pib_h(\afrak \mid{} \xfrak, \afrak,\dots,\afrak) \\
  &= (n-n_\yfrak) \log(4/5) \\
  &\geq -n\log(5/4).
  \end{align}
  Since $\log(5/4) < 1/2$, it follows from the definition of $\loglossbc$ that $\pihat = \pib$. Moreover, $\Dhels{\BP^{\pistar}}{\BP^{\pib}} \geq \Omega(1)$ whereas $\Dhels{\BP^{\pistar}}{\BP^{\pia}} \leq \delta/(H\log \denbound)$. The claim follows.
  \end{proof}

  \begin{proof}[Proof of \cref{prop:boosted-bc}]%
  
For each $1 \leq i \leq K/2 = \log(2/\delta)$, applying the second guarantee of \cref{thm:log-loss-bounded-app} with dataset size $n/(2\log(2/\delta))$ and failure probability $1/2$ gives that with probability at least $1/2$,
\[\Dhels{\BP^{\pihat^i}}{\bbP^{\pistar}} \lesssim \frac{\log(2|\Pi|)\log(2/\delta)}{n} + \frac{(1+\log(\denbound))\log(2/\delta)}{n} + H(1+\log(\denbound)) \cdot \min_{\pi\in\Pi} \Dhels{\BP^\pi}{\bbP^{\pistar}}.\]
Thus, with probability at least $1-(1/2)^{K/2} = 1-\delta/2$, there exists at least one $i \in [K/2]$ that satisfies the above bound. Condition on this event. By the guarantee of \cref{thm:rho-il} with dataset size $n/2$, failure probability $\delta/2$, and policy class $\{\pihat^1,\dots,\pihat^{K/2}\}$, with probability at least $1-\delta/2$ the output of the algorithm $\pihat$ satisfies
\begin{align} 
\Dhels{\BP^{\pihat}}{\bbP^{\pistar}} &\lesssim \frac{\log(K/\delta)}{n} + \Dhels{\BP^{\pihat^i}}{\bbP^{\pistar}}  \\ 
&\lesssim \frac{\log(2|\Pi|\denbound)\log(1/\delta) + \log\log(1/\delta)}{n} + H(1+\log(\denbound)) \cdot \min_{\pi\in\Pi} \Dhels{\BP^\pi}{\bbP^{\pistar}} \\ 
&\lesssim \frac{\log(2|\Pi|\denbound)\log(1/\delta)}{n} + H(1+\log(\denbound)) \cdot \min_{\pi\in\Pi} \Dhels{\BP^\pi}{\bbP^{\pistar}}
\end{align}
as claimed.
\end{proof}

\begin{proof}[Proof of \cref{thm:layerwise_smoothing}]
For each $\pi \in \Pi$ let $\pi^\lambda$ denote the ``smoothed'' policy defined by
\[\pi^\lambda_h(a_h\mid{}s_h) := (1-\lambda)\pi(a_h\mid{}s_h) + \lambda\pistar(a_h\mid{}s_h),\]
and define $\Pi^\lambda := \{\pi^\lambda: \pi \in \Pi\}$. Then applying $\smoothedloglossbc$ with policy class $\Pi$ is the same as applying $\loglossbc$ with policy class $\Pi^\lambda$ (save for outputting $\pihat$ rather than $\pihat^\lambda$). Moreover, for any $\pi^\lambda \in \Pi^\lambda$ and any $h \in [H]$, $s_h \in \cS$, and $a_h \in \cA$, it holds that
\[\frac{\pistar_h(a_h\mid{}s_h)}{\pi^\lambda_h(a_h\mid{}s_h)} = \frac{\pistar_h(a_h\mid{}s_h)}{(1-\lambda)\pi_h(a_h\mid{}s_h) + \lambda \pistar_h(a_h\mid{}s_h)} \leq \frac{1}{\lambda}.\]
Thus, $\pistar$ is $1/\lambda$-bounded with respect to $\Pi^\lambda$. By the first guarantee of \cref{thm:log-loss-bounded} with $\denbound := 1/\lambda = H^2 n$, we have with probability at least $1-\delta$ that 
\begin{equation} \Dhels{\BP^{\pihat^\lambda}}{\bbP^{\pistar}} \lesssim \frac{\log(|\Pi|/\delta)}{n} + \frac{H\log(Hn)\log(1/\delta)}{n} + H\log(Hn) \cdot \min_{\pi\in\Pi} \Dhels{\BP^{\pi^\lambda}}{\bbP^{\pistar}}.\label{eq:smoothed-intermediate}
\end{equation}
Since $\Dhel{\cdot}{\cdot}$ is a metric, we have for any $\pi \in \Pi$ that
\begin{align} 
\Dhels{\BP^{\pi^\lambda}}{\bbP^{\pistar}} 
&\leq 2\left(\Dhels{\BP^{\pi^\lambda}}{\BP^\pi} + \Dhels{\BP^\pi}{\bbP^{\pistar}}\right) \\
&\leq 4\Dtv{\BP^{\pi^\lambda}}{\BP^\pi} + 2\Dhels{\BP^\pi}{\bbP^{\pistar}} \\
&\leq 4\lambda H + 2\Dhels{\BP^\pi}{\bbP^{\pistar}}
\end{align}
where the final inequality can be derived by coupling a trajectory drawn from $\BP^\pi$ with a trajectory drawn from $\BP^{\pi^\lambda}$: at each step, the trajectories deviate with probability at most $\lambda$. By a symmetric argument, we also have
\[\Dhels{\BP^{\pihat}}{\bbP^{\pistar}} \leq 4\lambda H + 2\Dhels{\BP^{\pihat^\lambda}}{\bbP^{\pistar}}.\]
Substituting the preceding bounds into \cref{eq:smoothed-intermediate} and using that $\lambda = 1/(H^2 n)$, we get that
\[\Dhels{\BP^{\pihat}}{\bbP^{\pistar}} \preceq \frac{\log(|\Pi|/\delta)}{n} + \frac{H\log(Hn)\log(1/\delta)}{n} + H\log(Hn) \cdot \min_{\pi \in \Pi} \Dhels{\BP^\pi}{\bbP^{\pistar}}\]
as claimed.
\end{proof}

\subsection{Proofs from Section \ref*{sec:limits} (Limits of Next-Token Prediction)}\label{sec:limits-app}

Recall from \cref{def:iterative} that an \emph{iterative learner} for a policy class $\Pi$ and expert policy $\pistar$ is any algorithm that produces $\pihat \in \Pi$ by iteratively defining each token-level conditional distribution $\pihat_h$ in terms of $\pihat_{1:h-1}$ and $\pistar_{1:h}$. In this section we explain why this is a natural definition and then prove \cref{thm:ntp-lb}, which is a lower bound on the approximation ratio of any iterative learner.

\cref{def:iterative} is most natural for policy classes with no parameter sharing, as formally defined below.

\begin{definition}
A policy class $\Pi$ \emph{has no parameter sharing} if there are sets $\Pi_1,\dots,\Pi_h$ so that $\pi = (\pi_h)_h \in \Pi$ if and only if $\pi_h \in \Pi_h$ for all $h \in [H]$.
\end{definition}

For such policy classes, we show that \emph{any} estimator defined by minimizing a layer-wise loss---like $\loglossbc$ and $\smoothedloglossbc$---can be simulated by an iterative learner. Thus, \cref{thm:ntp-lb} applies to all such algorithms.

\begin{proposition}\label{prop:iterative-simulation}
For any MDP $M$ and policy class $\Pi$ with no parameter sharing, there is an iterative learner that, for any expert policy $\pistar$, simulates the execution of $\loglossbc$ on i.i.d. trajectories from $\BP^{\pistar}$. Moreover, the same holds for any estimator of the form
\[\pihat := \argmin_{\pi\in\Pi} \sum_{i=1}^n \sum_{h=1}^H L_h(\pi_h(a_h\ind{i}\mid{}s_h\ind{i}), \pistar_h(a_h\ind{i}\mid{}s_h\ind{i})),\]
where $L_{1:H}$ are arbitrary real-valued loss functions.
\end{proposition}

\begin{proof}[\pfref{prop:iterative-simulation}]
We give a proof for the general case, which clearly contains $\loglossbc$ via the loss function $L_h(p,q) = -\log p$. Since $\Pi$ has no parameter sharing, it suffices to draw $n$ i.i.d. trajectories from $\BP^{\pistar}$ and, for each $h \in [H]$, compute the following estimator, all within the computational framework of \cref{def:iterative}:
\[\pihat_h := \argmin_{\pi_h\in\Pi_h} \sum_{i=1}^n L_h(\pi_h(a_h\ind{i}\mid{}s_h\ind{i}),\pistar_h(a_h\ind{i}\mid{}s_h\ind{i})).\]
We compute $\pihat_1,\dots,\pihat_H$ in order. Since we know $M$, we can draw i.i.d. initial states $s_1\ind{1},\dots,s_1\ind{n}$. For each $h \in [H]$, we use knowledge of $\pistar_h$ to draw $a_h\ind{i} \sim \pistar_h(\cdot\mid{}s_h\ind{i})$ for each $i \in [n]$. Since we know $\Pi_h$, we can now compute $\pihat_h$ as above. We then use knowledge of the MDP to draw the next states $s_{h+1}\ind{i} \sim \BP(\cdot\mid{}s_h\ind{i},a_h\ind{i})$. By construction, $\pihat_h$ is a (random) function of $\pistar_{1:h}$, as needed.
\end{proof}

We now prove \cref{thm:ntp-lb}, restated below.

\begin{theorem}[Restatement of \cref{thm:ntp-lb}]\label{thm:ntp-lb-app}
Fix $H \in \NN$ and sets $\MX := \{0,1\}^H$ and $\MA := \{0,1,\perp\}$. Let $\MD := \Unif(\MX)$, and let $M$ be the $H$-step autoregressive MDP with initial context space $\MD$ and action space $\MA$. There is a policy class $\Pi$ with no parameter sharing, so that for any iterative learner, there exists a policy $\pistar$ such that
\[\EE\left[\Dhels{\BP^{\pihat}}{\BP^{\pistar}}\right] \geq \Omega(H) \cdot \min_{\pi \in \Pi} \Dhels{\BP^{\pi}}{\BP^{\pistar}}\]
where $\pihat$ is the (potentially random) output of the iterative learner.
\end{theorem}

The basic idea is to embed a ``consistency game'' in the distribution learning problem. The expert policy is a fixed (but a priori unknown) action sequence $z$, and each conditional distribution class $\Pi_h$ contains a policy $\pi_h$ for each possible action sequence, so each $\pihat_h$ computed by the learner can be thought of as a ``guess'' for $z$. Minimizing Hellinger distance requires the guess at step $h$ to match $z$ on the first $h$ actions, but also being \emph{consistent}, i.e. minimizing the number of different guesses made (across the $H$ steps). Since the iterative learner must determine $\pihat_{1:h}$ without knowledge of $z_{h+1:H}$, this is provably impossible. We now make this idea formal.

\begin{proof}[Proof of \cref{thm:ntp-lb-app}]
Define $\Gamma := \{0,1\}^H$. Define a policy class $\Pi := \{\pi^{\gamma_{1:H}}: \gamma_{1:H} \in \Gamma^H\}$ where for each $\gamma_{1:H} \in \Gamma^H$, the policy $\pi^{\gamma_{1:H}} = (\pi^{\gamma_{1:H}}_h)_{h=1}^H$ is defined by
\[\pi^{\gamma_{1:H}}_h(a_h\mid{}x,a_{1:h-1}) := f_{h,\gamma_h}(a_h\mid{}x,a_{1:h-1}) := \begin{cases} 
\mathbbm{1}[a_h = (\gamma_h)_h] & \text{ if } (x \neq \gamma_h) \land (a_{1:h-1} = (\gamma_h)_{1:h-1}) \\ 
\mathbbm{1}[a_h = \perp] & \text{ otherwise}
\end{cases}.\]
Note that $\Pi$ indeed has no parameter sharing since $\pi^{\gamma_{1:H}}_h$ is solely a function of $\gamma_h$. 

Fix any iterative learner. Consider selecting $\pistar$ randomly via the following procedure. Draw $z \sim \Unif(\{0,1\}^H)$, and set $\pistar := \pi^z$ where $\pi^z(a\mid{}x) := \mathbbm{1}[a=z]$ for all $a \in \MA^H$ and $x\in\MX$ (recall that a policy in an autoregressive MDP can be equivalently identified by a sequence-level conditional distribution $\MX \to \MA^H$). With this random choice of $\pistar$, let $\pihat$ be the random output of the iterative learner.

By definition of an iterative learner, we always have $\pihat\in\Pi$, so there are some (random) $\wh \gamma^{1:H} \in \Gamma^H$ with $\pihat = \pi^{\wh\gamma_{1:H}}$. We can characterize $\Dhels{\BP^{\pihat}}{\BP^{\pistar}}$ in terms of $z$ and $\wh\gamma_{1:H}$: %
\begin{itemize}
    \item If there is any $h \in [H]$ with $(\wh \gamma_h)_{1:h} \neq z_{1:h}$, then $\Dhels{\BP^{\pihat}}{\BP^{\pistar}} = 1$. Indeed, pick the first such $h$. We have
    \[\BP^{\pihat}[a_{1:H}=z_{1:H}] \leq \EE_{x\sim\MD} \pihat_h(z_h\mid{}x,z_{1:h-1}).\]
    But for any $x\in\MX$, for the partial trajectory $a_{1:h-1} = z_{1:h-1}$, if $(\wh\gamma_h)_{1:h-1} \neq z_{1:h-1}$ then $\pihat$ plays action $\perp$ at step $h$. Otherwise $(\wh\gamma_h)_h \neq z_h$, so again $\pihat$ does not play $z_h$ at step $h$. Together with the above inequality, this shows that $\BP^{\pihat}[a_{1:H}=z_{1:H}] = 0$, i.e. $\BP^{\pihat}$ and $\BP^{\pistar}$ have disjoint supports.
    \item If $(\wh\gamma_h)_{1:h} = z_{1:h}$ for all $h \in [H]$, then \begin{equation} \Dhels{\BP^{\pihat}}{\BP^{\pistar}} = \EE_{x\sim\MD}[\Dhels{\pihat(\cdot\mid{}x)}{\pistar(\cdot\mid{}x)}] = 2^{-H}|\{\wh\gamma_h: h \in [H]\}|.\label{eq:consensus-bound}\end{equation}
    Indeed, condition on any $x \not \in \{\wh\gamma_h:h\in[H]\}$. For each $h \in [H]$, we have
    \[\pihat_h(z_h\mid{}x,z_{1:h-1}) = \pihat_h((\wh\gamma_h)_h\mid{}x,(\wh\gamma_h)_{1:h-1}) = 1,\]
    so inductively we have $\BP^{\pihat}[a_{1:H}=z_{1:H}\mid{}x] = 1$. Conversely, if $x \in \{\wh\gamma_h:h\in[H]\}$ then it is clear that $\BP^{\pihat}[a_{1:H}=z_{1:H}\mid{}x] = 0$.
\end{itemize}
For any $h \in [H]$, by definition of an iterative learner (and the fact that $z_1,\dots,z_H$ are independent) we have that $\wh\gamma_h$ is independent of $z_{h+1:H}$. Hence, for any $h < k$,
\begin{align}
\Pr[(\wh\gamma_h = \wh\gamma_k) \land (\Dhels{\BP^{\pihat}}{\BP^{\pistar}} < 1)]
&\leq \Pr[(\wh\gamma_h = \wh\gamma_k) \land ((\wh\gamma_h)_{1:h} = z_{1:h}) \land ((\wh\gamma_k)_{1:k} = z_{1:k})] \\ 
&\leq \Pr[(\wh\gamma_h)_{h+1:k} = z_{h+1:k}] \\ 
&= \frac{1}{2^{k-h}}.\label{eq:agreement-ub}
\end{align}
Next, observe that the following inequality holds with probability $1$, since either $2^H \Dhels{\BP^{\pihat}}{\BP^{\pistar}} = \{\wh\gamma_h: h \in [H]\}$ or else $\Dhels{\BP^{\pihat}}{\BP^{\pistar}} = 1$:
\[2^H \Dhels{\BP^{\pihat}}{\BP^{\pistar}} \geq \frac{1}{2}\left(|\{\wh\gamma_h: h \in [H]\}| + H \cdot \mathbbm{1}[\Dhels{\pistar}{\pistar} = 1]\right)\]
Using this bound, we get
\allowdisplaybreaks
\begin{align}
&2^H\EE[\Dhels{\BP^{\pihat}}{\BP^{\pistar}}] \\
&\geq \frac{1}{2}\EE\left[|\{\wh\gamma_h: h \in [H]\}| + H \cdot \mathbbm{1}[\Dhels{\pistar}{\pistar} = 1]\right] \\
&\geq \frac{1}{2}\sum_{h \in [H], \text{$h$ odd}} \Pr[\wh\gamma_h \not \in \{\wh\gamma_1,\dots,\wh\gamma_{h-2}\}] + \Pr[\Dhels{\BP^{\pihat}}{\BP^{\pistar}}=1] \\ 
&\geq \frac{1}{2}\sum_{h \in [H], \text{$h$ odd}} \Pr[(\wh\gamma_h \not \in \{\wh\gamma_1,\dots,\wh\gamma_{h-2}\}) \lor (\Dhels{\BP^{\pihat}}{\BP^{\pistar}}=1)] \\ 
&= \frac{1}{2} \sum_{h \in [H], \text{$h$ odd}}\left(1 - \Pr[(\wh\gamma_h  \in \{\wh\gamma_1,\dots,\wh\gamma_{h-2}\}) \land (\Dhels{\BP^{\pihat}}{\BP^{\pistar}}<1)]\right) \\
&\geq \frac{1}{2}\sum_{h \in [H], \text{$h$ odd}}\left(1 - \sum_{k=1}^{h-2} \Pr[(\wh \gamma_h = \wh \gamma_{k-2}) \land (\Dhels{\BP^{\pihat}}{\BP^{\pistar}}<1)]\right) \\ 
&\geq \frac{1}{2}\sum_{h \in [H], \text{$h$ odd}}\left(1 - \sum_{k=1}^{h-2} \frac{1}{2^{k-h}}\right) \\ 
&\geq \frac{H}{8}
\end{align}
where the penultimate inequality is by \cref{eq:agreement-ub}. However, if we define $\pibar := \pi^{\overline{\gamma}_{1:H}} \in \Pi$ where $\overline{\gamma}_h := z$ for all $h \in [H]$, then $\Dhels{\pibar}{\pistar} = 1/2^H$ by \cref{eq:consensus-bound}. Thus,
\[\EE\left[\Dhels{\BP^{\pihat}}{\BP^{\pistar}} - \frac{H}{8} \min_{\pi\in\Pi}\Dhels{\pi}{\pistar}\right] \geq 0.\]
Note that the expectation is over the randomness of $z$ and the interactive learner. It follows that there is some fixed choice of $\pistar$ for which
\[\EE\left[\Dhels{\BP^{\pihat}}{\BP^{\pistar}}\right] \geq  \frac{H}{8} \min_{\pi\in\Pi}\Dhels{\pi}{\pistar}\]
where the expectation is over the randomness of the interactive learner.
\end{proof}

\newpage
\section{Proof of Theorem \ref*{thm:comp-lb-main} (Computational Lower Bound)}\label{app:comp-lb}

In this section we prove \cref{thm:comp-lb-main}, which asserts that learning misspecified autoregressive linear models with optimal approximation ratio requires super-polynomial time under sub-exponential hardness of the \emph{Learning Parities with Noise (LPN)} problem. In \cref{subsec:comp-lb-setting}, we formally describe the problem setting and restate the theorem. In \cref{subsec:comp-lb-overview}, we give a proof overview, expanding on the overview given in \cref{sec:comp-lb}, and introduce relevant notation. In \cref{subsec:generating-samples,subsec:bounding-misspec,subsec:policy-to-parity} we put together the key lemmas for the proof, and in \cref{subsec:comp-lb-proof} we complete the proof.

\subsection{Formal Problem Setting and Theorem Statement}\label{subsec:comp-lb-setting}

\paragraph{Problem setting} A learning algorithm $\Alg$ for (misspecified) autoregressive linear models operates in the following computational framework. Let $\MX,\MA$ be sets with $|\MA|<\infty$, and let $d,H \in \NN$. Let $M$ be the $H$-step autoregressive MDP with context space $\MX$, action space $\MA$, and some initial context distribution $\MD \in \Delta(\MX)$. Let $\phi: \MX\times\MA^\st \to \RR^d$ be a $d$-dimensional feature mapping, and let $\Theta \subset \RR^d$ be a convex parameter set. Let $(o\ind{i})_{i=1}^T$ be trajectories $o\ind{i} = (x\ind{i},a_{1:H}\ind{i})$. The algorithm $\Alg$ receives input $((o\ind{i})_{i=1}^T, \MA, \epsilon)$, and it has access to the following computational oracles:
\begin{enumerate}
\item Given $h \in \{0,\dots,H\}$ and $(x,a_{1:h}) \in \MX\times\MA^h$, query $\phi(x,a_{1:h})$.
\item Given $\theta \in \RR^d$, query $\Proj_\Theta(\theta) = \argmin_{\theta'\in\Theta} \norm{\theta-\theta'}_2$.
\end{enumerate}
The algorithm $\Alg$ is required to output a policy $\pihat = (\pihat_h)_{h=1}^H$ where each $\pihat_h$ is represented as a circuit $\MC_{\pihat_h}$: that is, given $(x,a_{1:h-1}) \in \MX\times\MA^{h-1}$, the distribution of $\MC_{\pihat_h}(x,a_{1:h-1},r)$ for independent randomness $r$ is $\pihat_h(\cdot\mid{}x,a_{1:h-1})$. Note that the feature map $\phi$ is not explicitly specified to the learner in this framework, and must be accessed through querying the first oracle above.

\paragraph{Hardness assumption: Learning Parities with Noise (LPN)}
We define the noisy parity distribution with noise level $\eta$ as follows.
\begin{definition}[Noisy parity distribution]
  \label{def:parity}
Fix $n \in \NN$, $S\subseteq [n]$, and $\eta \in (0, 1/2)$. We let $\BP^\st_{n,S,\eta}$ denote the distribution of $(x,y)$ where $x \sim \Unif(\{-1,1\}^n)$ and $y = (-1)^\xi \prod_{i\in S} x_i$ for an independent random variable $\xi \sim \Ber(1/2-\eta)$. We further let $\BQ^\st_n$ denote the distribution of $(x,y)$ where $x\sim \Unif(\{-1,1\}^n)$ and $y \sim \Unif(\{-1,1\})$ are independent.
\end{definition}

The following assumption asserts that it requires near-exponential time to distinguish between samples from $\BP^\st_{n,S,\eta}$ and $\BQ^\st_{n}$ for an unknown set $S \subseteq [n]$.

\begin{assumption}[Sub-exponential hardness of decisional LPN]\label{assumption:lpn}
Fix any constant $\eta>0$. Suppose that $\Alg^\MD$ is an algorithm that takes as input a sampling oracle for a distribution $\MD \in \Delta(\{-1,1\}^n \times \{-1,1\})$ and produces an output in $\{0,1\}$. Suppose that the following guarantees hold:
\begin{itemize}
\item For any $S \subseteq [n]$, $\EE[\Alg^{\BP^\st_{n,S,1,\eta}}] \geq 5/8$.
\item $\EE[\Alg^{\BQ_{n,1}}] \leq 1/2$.
\end{itemize}
Then $\Alg^\MD$ has time complexity $2^{n^{1-o(1)}}$.
\end{assumption}

While \cref{assumption:lpn} is phrased in terms of a decision task, this task is polynomial-time equivalent to the task of \emph{learning} noisy parities, via standard boosting and self-reducibility arguments. The conjectural computational hardness of LPN has seen extensive use in cryptography \citep{alekhnovich2003more,applebaum2009fast,pietrzak2012cryptography} and learning theory \citep{kearns1994learnability,mossel2005learning,golowich2024exploration}.  While $2^{n^{1-o(1)}}$-hardness (as opposed to, say, hardness for \emph{some} sub-exponential function) is a stronger assumption than what is used in many of these works, the fastest known algorithm for LPN has time complexity $2^{O(n/\log n)}$ \citep{blum2003noise}.  For further discussion, see e.g. \cite{yu2021smoothing} and references therein.\loose

Under \cref{assumption:lpn}, we show that efficiently learning misspecified autoregressive linear models inherently leads to error amplification.

\begin{theorem}[Restatement of \cref{thm:comp-lb-main}]\label{thm:comp-lb-app}
Fix any $c,C>0$ and let $\Alg$ be a learning algorithm for autoregressive linear models with the following guarantee. Suppose $|\MA|=2$ and \cref{ass:linear-norm-bounds-main} holds with parameters $B=\sqrt{d}$ and $\Bdot=1$; then for any policy $\pistar$, if $(o\ind{i})_{i=1}^n$ are $T = (dH/\epsilon)^C$ i.i.d. trajectories from $\BP^{\pistar}$, the time complexity of $\Alg((o\ind{i})_{i=1}^T,\MA,\epsilon)$ is $O(T)$\footnote{Note that $\Alg$ is not required to read the entire input. It would be equivalent to allow for time complexity $\poly(T)$, or to give $\Alg$ a sampling oracle for $\BP^{\pistar}$.} and the output $\pihat$ satisfies, with probability at least $9/10$,
\[\Dhels{\BP^{\pihat}}{\BP^{\pistar}} \lesssim \epsilon + e^{(\log \max(d,H))^{1-c}} \cdot \min_{\pi \in \Pi} \Dhels{\BP^\pi}{\BP^{\pistar}}.\]
Then \cref{assumption:lpn} is false.
\end{theorem}

It is straightforward to check from the proof that, under the weaker assumption of $2^{\sqrt{n}+\epsilon}$-hardness of LPN \citep{yu2019collision,yu2021smoothing} for a given $\eps>)$, there exists some constant $c = c(\epsilon)>0$ such that $\Capx \leq e^{(\log \max(d,H))^c}$ is impossible for any computationally efficient learner. We remark that \cref{thm:comp-lb-app} does not apply if the learner has access to \densobs; resolving this (either with an efficient algorithm or an improved lower bound) is an interesting open problem.

\subsection{Proof Overview and Definitions}\label{subsec:comp-lb-overview}

The proof of \cref{thm:comp-lb-main} is inspired by the main result of \cite{diakonikolas2022hardness}. When translated to our setting, their results essentially show that for $H=1$ and large $B$, any computationally efficient agnostic learner must pay super-constant misspecification factor. The main idea of \cite{diakonikolas2022hardness} is to consider the problem of learning a noisy parity function over the uniform distribution. In the notation of autoregressive sequence modelling (with horizon $1$), the context space is $\MX := \{-1,1\}^n$ and the action space $\MA := \{-1,1\}$. The features $\phi$ are defined by the degree-$t$ Veronese map, i.e. all $d = O(n^t)$ degree-$\leq t$ monomials on the context $x \in \MX$. While polynomial approximation of an $n$-variable parity function uniformly on its domain would require degree nearly $n$, a concentration argument shows that there is a degree-$t=\tilde{O}(\sqrt{n})$ polynomial that approximates the parity function on \emph{most} of the domain. Hence, with a feature mapping of dimension $d=n^{\tilde{O}(\sqrt{n})}$, the policy class has small misspecification. Since learning noisy parities is believed to require $\exp(n^{1-o(1)})$ time, this rules out a polynomial-time algorithm for the agnostic learning problem.

Unfortunately, the polynomial approximation argument requires the policy class to have very large norm bound (concretely, $B \sim \poly(d)$). In our setting, we are interested in how the misspecification factor scales with $H$ when $B = O(1)$. Since $\loglossbc$ achieves $\Capx = O(B)$ in a computationally efficient manner when $H, |\MA| = O(1)$ (\cref{cor:linear-misspec-logloss}), deriving a computational lower bound in our setting fundamentally requires exploiting the long horizon.

The main new technical ingredient in our proof is the observation that decreasing the signal strength in the noisy parity distribution (i.e. sending the noise level $\eta$ towards $1/2$) correspondingly decreases the norm of the polynomial approximator. In particular, if $\eta \geq 1/2- 1/\poly(d)$, then we can set $B = O(1)$. Of course, when the noise is so close to uniform, an agnostic learner could achieve small Hellinger distance without learning the parity function, by simply outputting the uniform distribution. This is where we use the long horizon to ``boost'' the signal: instead of trying to learn the distribution of $(x, a)$ where $x \sim \Unif(\{-1,1\}^n)$ and $a = (-1)^\xi \prod_{i \in S} x_i$ for noise $\xi \sim \Ber(\eta)$ and parity set $S \subseteq [n]$, we try to learn the distribution of $(x,a_1,\dots,a_H)$ where $a_h = (-1)^{\xi_h} \prod_{i \in S} x_i$ for independent random variables $\xi_1,\dots,\xi_H \sim \Ber(\eta)$. For $H \sim 1/\eta^2$, the effective signal strength is constant, and since $\eta$ is small for each $h$, each conditional distribution $a_h\mid{}x$ admits a low-norm polynomial approximation.

\paragraph{Construction of hard instance} Formally, for parameters $n,H \in \NN$
, we let $M = M_{n,H}$ be the $H$-step autoregressive MDP with context space $\{-1,1\}^n$, initial context distribution $\Unif(\{-1,1\}^n)$, action space $\{-1,1\}$, and horizon $H$. We define a policy class $\Pi$ consisting of autoregressive linear models, where the features are monomials in the initial context $x$.

\begin{definition}[Policy class]\label{def:comp-lb-policy-class}
Let $n,t \in \NN$ and set $d := \sum_{i=0}^t \binom{n}{i}$. Identify $[d]$ with the collection of all subsets of $[n]$ of size at most $t$. We define a feature map $\phiver_{n,t}: \{-1,1\}^n \times \{-1,1\}^\st \to \RR^d$ by 
\[\phiver_{n,t}(x,a_{1:h})_T := a_h \prod_{i \in T} x_i\]
for each $T \subseteq [n]$ with $|T| \leq t$. Let $\Theta = \{\theta \in \RR^d:\norm{\theta}_1 \leq 1\}$. We then define $\Pi_{n,t,H} := \{\pi_\theta:\theta\in\Theta\}$ where $\pi_\theta = (\pi_{\theta,h})_{h=1}^H$ is defined by
\begin{equation}
  \pi_{\theta,h}(a_h\mid{}x,a_{1:h-1}) := \frac{\exp(\langle \theta,
    \phiver_{n,t}(x,a_{1:h})\rangle)}{\sum_{a'_h \in \MA}
    \exp(\langle \theta,\phiver_{n,t}(x,a_{1:h-1},a'_h)\rangle)}.
    \label{eq:comp-lb-density}
\end{equation}
\end{definition}

Next, we define the family of possible expert policies that our data may be generated by, which is parametrized by an unknown subset $S \subseteq [n]$.\footnote{Note that rather than defining the conditional distributions $\pistar_h(a_h\mid{}x,a_{1:h-1})$ for each $h$, we are directly defining the conditional distribution $a_{1:H}\mid{} x$ that would be generated autoregressively in $M$ under $\pistar$; however, this is equivalent.}

\begin{definition}[Noisy parity policies]
  \label{def:parity_policy}
Let $n,H \in \NN$. Let $\eta \in [0,1/2)$ and $S \subseteq [n]$. For $x \in \{-1,1\}^n$, we define $\pi^\st_{n,S,\eta,H}: \{-1,1\}^n \to \Delta(\{-1,1\}^H)$ so that $\pi^\st_{n,S,\eta,H}(\cdot\mid{}x)$ is the distribution of $\left((-1)^{\xi_h}\prod_{j \in S} x_j\right)_{h=1}^H$, where $\xi_1,\dots,\xi_H \sim \Ber(1/2-\eta)$ are independent.
\end{definition}

We also introduce the following notation for the trajectory distribution induced in $M$ by a noisy parity policy $\pistar_{n,S,\eta,H}$ (note that it corresponds to drawing $x \sim \Unif(\MX)$ and then $y \sim \pistar_{n,S,\eta,H}(\cdot\mid{}x)$).

\begin{definition}
Let $n,H \in \NN$ and $\eta \in [0,1/2)$. Let $S \subseteq [n]$. Then we define $\BP^\st_{n,S,H,\eta} := \BP^{\pi^\st_{n,S,\eta,H}}$. We also define $\BQ_{n,H} := \BP^\st_{n,S,H,0}$ (notice that the latter distribution does not depend on $S$).
\end{definition}

For example, $\BP^\st_{n,S,1,\eta}$ is the same as the noisy parity distribution $\BP^\st_{n,S,\eta}$ in \cref{def:parity}, and similarly $\BQ_{n,1} = \BQ_n$. \loose

With this notation, our goal is to show that an autoregressive learning algorithm with the guarantees specified in \cref{thm:comp-lb-app} enables learning the set $S$ from samples, and that this violates \cref{assumption:lpn}. To this end, there are three pieces to the proof. First, we show (\cref{subsec:generating-samples}) that given \emph{standard} LPN samples, i.e. samples from $\BP^\st_{n,S,\gamma}$ for some constant $\gamma \in (0,1/2)$ and unknown set $S \subseteq [n]$, one can efficiently generate samples from $\BP^\star_{n,S,H,\eta}$, so long as $\eta \ll 1/\sqrt{H}$. Second, we show (\cref{subsec:bounding-misspec}) that for any $S \subseteq [n]$, the joint distribution $\BP^\star_{n,S,H,\eta}$ has small misspecification (in Hellinger distance) with respect to $\Pi$. Third, we show
(\cref{subsec:policy-to-parity}) that learning a policy with small Hellinger distance to $\BP^\star_{n,S,H,\eta}$ enables recovering $S$.

\subsection{Step 1: Generating Samples}\label{subsec:generating-samples}

We start by showing that, given a sample from $\BP^\star_{n,S,1,1/4}$ for some unknown set $S \subseteq [n]$, we can efficiently generate a sample from a distribution close to $\BP^\star_{n,S,H,\eta}$, where $\eta = \gamma/(C\sqrt{H})$ for some constant $C$ and parameter $\gamma \in (0,1)$ that we will choose later. Essentially, the signal can be efficiently ``spread out'' across the $H$ steps (\cref{lemma:spread-signal}).

The following lemma is crucial to this reduction: it shows that given a noisy measurement of some bit $b \in \{-1,1\}$, and two distributions $p(-1)$ and $p(1)$, if $p(-1)$ and $p(1)$ have bounded density ratios, then it is possible to generate a sample from $p(b)$ (despite not observing $b$ directly).

\begin{lemma}\label{lemma:factor-bsc}
There is a polynomial-time algorithm $\Alg_{\ref{lemma:factor-bsc}}$ with the following property. Let $H \in \NN$, $\eta \in (0,1/2)$, and $p: \{-1,1\} \to \Delta([H])$. Suppose that \[\max\left\{\norm{\frac{p(1)}{p(-1)}}_\infty, \norm{\frac{p(-1)}{p(1)}}_\infty\right\} \leq \frac{1-\eta}{\eta}\]
where we define $0/0 = 1$. Then for any fixed $b \in \{-1,1\}$, for $\xi \sim \Ber(\eta)$, the output of $\Alg_{\ref{lemma:factor-bsc}}(p,(-1)^\xi b,\eta)$ has marginal distribution $p(b)$. %
\end{lemma}

\begin{proof}[\pfref{lemma:factor-bsc}]
Let $P \in \RR^{H \times 2}$ be the matrix with columns $p(-1), p(1) \in \RR^H$. Define the matrix
\[Q := P \begin{bmatrix} \frac{1-\eta}{1-2\eta} & -\frac{\eta}{1-2\eta} \\ -\frac{\eta}{1-2\eta} & \frac{1-\eta}{1-2\eta} \end{bmatrix}.\]
Observe that \[\sum_{i=1}^n Q_{i1} = \sum_{i=1}^H \frac{1-\eta}{1-2\eta} P_{i1} - \frac{\eta}{1-2\eta} P_{i2} = 1\] where the final equality uses that $\sum_{i=1}^H P_{i1} = \sum_{i=1}^H P_{i2} = 1$. Moreover, for each $i \in [H]$,
$Q_{i1} \geq 0$ by assumption that $\norm{p(1)/p(-1)}_\infty \leq (1-\eta)/\eta$. Thus, the first column of $Q$ represents a distribution $q(-1)$ over $[H]$. Similarly, the second column represents a distribution $q(1)$ over $[H]$. On input $(p,b')$, we define the algorithm $\Alg_{\ref{lemma:factor-bsc}}(p,b',\eta)$ to sample and output $x \sim q(b')$.

Now observe that the marginal distribution of $x =\Alg_{\ref{lemma:factor-bsc}}(p,(-1)^\xi b,\eta)$ when $b=-1$ and $\xi\sim\Ber(\eta)$ is 
\[Q \begin{bmatrix} 1-\eta \\ \eta \end{bmatrix} = P \begin{bmatrix} \frac{1-\eta}{1-2\eta} & -\frac{\eta}{1-2\eta} \\ -\frac{\eta}{1-2\eta} & \frac{1-\eta}{1-2\eta} \end{bmatrix} \begin{bmatrix} 1-\eta & \eta \\ \eta & 1-\eta \end{bmatrix} e_1 = p(-1).\]
Similarly, the marginal distribution of $x$ when $b=1$ is exactly $p(1)$.
\end{proof}

We now construct the desired reduction. Given a sample $(x,y)$ from $\BP^\star_{n,S,1,1/4}$, note that $y$ is a noisy measurement of $\prod_{i \in S} x_i$; we would like to produce $H$ independent samples from the distribution on $\{-1,1\}$ with bias $1/2 + \eta\prod_{i \in S} x_i$. If we could produce a sample $k$ from the binomial distribution $\Bin(H, 1/2 + \eta\prod_{i \in S} x_i)$, then we would be done since we could output $(x, a_{1:H})$ where $(a_1,\dots,a_H)$ is a uniformly random string in $\{-1,1\}^H$ subject to the constraint of containing $k$ ones. Unfortunately, the density ratio between $\Bin(H, 1/2 + \eta)$ and $\Bin(H, 1/2 - \eta)$ is not bounded unless $\eta = O(1/H)$, so we cannot directly apply \cref{lemma:factor-bsc}. Instead, we truncate the binomial distributions to the range $[H/2 - \tilde{O}(\sqrt{H}), H/2 + \tilde{O}(\sqrt{H})]$. The resulting distributions have bounded density ratios, and the truncation introduces negligible error, so long as $\eta = \tilde{O}(1/\sqrt{H})$.

\begin{lemma}\label{lemma:spread-signal}
There are universal constants $c_{\ref{lemma:spread-signal}}, C_{\ref{lemma:spread-signal}} > 0$ and a polynomial-time algorithm $\Alg_{\ref{lemma:spread-signal}}$ with the following property. Let $n,H \in \NN$ and $\gamma \in (0,1/2)$. For any $S \subseteq [n]$, for $(x,y) \sim \BP^\star_{n,S,1,1/4}$, the output of $\Alg_{\ref{lemma:spread-signal}}(x,y,\gamma,H)$ has marginal distribution $\mu$ satisfying \[\Dtv{\BP^\st_{n,S,H,\gamma/(C_{\ref{lemma:spread-signal}}\sqrt{H})}}{\mu} \leq 2\exp(-c_{\ref{lemma:spread-signal}}/\gamma^2).\]
\end{lemma}

\begin{proof}[\pfref{lemma:spread-signal}]
For notational convenience, let $f(n,k,\delta)$ denote the mass of the distribution $\Bin(n,\delta)$ at $k \in \{0,\dots,n\}$.  On input $(x,y,\gamma)$, the algorithm $\Alg_{\ref{lemma:spread-signal}}(x,y,\gamma,H)$ computes the function $p: \{-1,1\} \to \Delta(\{0,\dots,H\})$ where 
\begin{align}
p(-1)_k &\propto f\left(H, k, \frac{1}{2}-\frac{\gamma}{C_{\ref{lemma:spread-signal}}\sqrt{H}}\right) \mathbbm{1}\left[\frac{H}{2} - \frac{\sqrt{H}}{\gamma} \leq k \leq \frac{H}{2} + \frac{\sqrt{H}}{\gamma}\right], \\ 
p(1)_k &\propto f\left(H, k, \frac{1}{2}+\frac{\gamma}{C_{\ref{lemma:spread-signal}}\sqrt{H}}\right) \mathbbm{1}\left[\frac{H}{2} - \frac{\sqrt{H}}{\gamma} \leq k \leq \frac{H}{2} + \frac{\sqrt{H}}{\gamma}\right].
\end{align}
It then computes $K := \Alg_{\ref{lemma:factor-bsc}}(p, y, 1/4)$ (cf. \cref{lemma:factor-bsc}) and outputs $(x,y')$ where $y' \in \{-1,1\}^H$ is uniformly random subject to the constraint $|\{h \in [H]: y'_h = 1\}| = K$.

We now analyze the algorithm. Observe that 
\begin{align}
&\Prr_{X \sim \Bin(H, \frac{1}{2}-\frac{\gamma}{C_{\ref{lemma:spread-signal}}\sqrt{H}})}\left[\frac{H}{2} - \frac{\sqrt{H}}{\gamma} \leq X \leq \frac{H}{2} + \frac{\sqrt{H}}{\gamma}\right] \\
&= \Prr_{Y \sim \Bin(H, \frac{1}{2}+\frac{\gamma}{C_{\ref{lemma:spread-signal}}\sqrt{H}})}\left[\frac{H}{2} - \frac{\sqrt{H}}{\gamma} \leq Y \leq \frac{H}{2} + \frac{\sqrt{H}}{\gamma}\right]
\end{align}
because $H-X$ and $Y$ are identically distributed. Thus, for any $k \in \{0,\dots,n\}$, we have either $p(-1)_k = p(1)_k = 0$ or else $|H/2 - k| \leq \sqrt{H}/\gamma$ and hence
\[\frac{p(-1)_k}{p(1)_k} = \frac{f\left(H, k, \frac{1}{2}-\frac{\gamma}{C_{\ref{lemma:spread-signal}}\sqrt{H}}\right)}{f\left(H, k, \frac{1}{2}+\frac{\gamma}{C_{\ref{lemma:spread-signal}}\sqrt{H}}\right)} = \left(\frac{1-2\gamma/(C_{\ref{lemma:spread-signal}}\sqrt{H})}{1+2\gamma/(C_{\ref{lemma:spread-signal}}\sqrt{H})}\right)^{2k-H} \leq 2\]
so long as $C_{\ref{lemma:spread-signal}}>0$ is a sufficiently large constant. Thus, $\norm{p(-1)/p(1)}_\infty \leq 2$. Similarly, $\norm{p(1)/p(-1)}_\infty \leq 2$. 

Condition on $x$. We have $y = (-1)^\xi \prod_{i \in S} x_i$ where $\xi \sim \Ber(1/4)$. It follows from \cref{lemma:factor-bsc} and the preceding bounds that $K$ has distribution $p(\prod_{i \in S} x_i)$. If $K$ had distribution $\Bin(H, 1/2+\gamma\prod_{i \in S} x_i/(C_{\ref{lemma:spread-signal}}\sqrt{H}))$ (i.e., if we did not truncate), then $y'$ would have distribution exactly $\pi^\st_{n,S,H,\gamma/(C_{\ref{lemma:spread-signal}}\sqrt{H})}(\cdot\mid{}x)$ (cf. \cref{def:parity_policy}), and thus $(x,y')$ would have distribution exactly $\BP^\st_{n,S,H,\gamma/(C_{\ref{lemma:spread-signal}}\sqrt{H})}$. We bound the error induced by truncation as follows. By the data processing inequality, we have that
\begin{align}
&\Dtv{\Law(x,y')}{\BP^\st_{n,S,H,\gamma/(C_{\ref{lemma:spread-signal}}\sqrt{H})}} \\
&= \EE_{x\sim \Unif(\{-1,1\}^n)} \Dtv{\Law(y'\mid{}x)}{\pi^\st_{n,S,H,\gamma/(C_{\ref{lemma:spread-signal}}\sqrt{H})}(\cdot\mid{}x)} \\ 
&\leq \EE_{x\sim \Unif(\{-1,1\}^n)} \Dtv{p\left(\prod_{i \in S} x_i\right)}{\pi^\st_{n,S,H,\gamma/(C_{\ref{lemma:spread-signal}}\sqrt{H})}(\cdot\mid{}x)} \\ 
&\leq 1 - \Prr_{Z \sim \Bin(H, \frac{1}{2}-\frac{\gamma}{C_{\ref{lemma:spread-signal}}\sqrt{H}})}\left[\frac{H}{2} - \frac{\sqrt{H}}{\gamma} \leq Z \leq \frac{H}{2} + \frac{\sqrt{H}}{\gamma}\right] \\ 
&\leq 2\exp(-c_{\ref{lemma:spread-signal}}/\gamma^2)
\end{align}
where the final inequality is by a Chernoff bound, and holds so long as $c_{\ref{lemma:spread-signal}}>0$ is sufficiently small.
\end{proof}

\subsection{Step 2: Bounding the Misspecification}\label{subsec:bounding-misspec}

Next, we argue that $\BP^\star_{n,S,H,\eta}$ is close in Hellinger distance to some policy in the class $\Pi_{n,t,H}$ (\cref{def:comp-lb-policy-class}), so long as $\eta \leq n^{-\omega(\sqrt{n})}$ (i.e. the noise is sufficiently close to uniform) and $t = \omega(\sqrt{n})$ (i.e. the policy class is sufficiently rich). See \cref{lemma:bic-hell} for the formal statement. Except for the 
choice of $\eta$ to be small, the proof closely follows the analogous arguments in \cite{diakonikolas2022hardness}. %

\begin{lemma}\label{lemma:polynomial-interpolation}
Let $t,K \in \NN$. There is a degree-$t$ polynomial $f: \RR \to \RR$ such that:
\begin{itemize}
\item $f(k) = (-1)^{\frac{K-k}{2}}$ for all integers $k \in [-t,t]$ with $k \equiv K \pmod{2}$
\item $\sum_{i=0}^{t} |a_i| \leq O(t^3)$ where $a_0,\dots,a_{t+1}$ are the coefficients of $f$.
\end{itemize}
\end{lemma}

\begin{proof}[\pfref{lemma:polynomial-interpolation}]
For notational convenience, let $T$ denote the set of integers $k \in [-t,t]$ with $k \equiv K \pmod{2}$. Define $f:\RR \to \RR$ by
\[f(x) = \sum_{k \in T} (-1)^k \frac{\prod_{j \in T \setminus \{k\}} (x - j)}{\prod_{j \in T \setminus \{k\}} (k-j)}.\]
Since $|T| \leq t+1$, it is clear that $f$ has degree at most $t$. It is also clear that $f(k) = (-1)^{\frac{K-k}{2}}$ for all $k \in T$. 

Suppose $K$ is even, so without loss of generality $t$ is even. For any $k \in [T]$, the polynomial $\prod_{j \in T \setminus \{k\}} (x-j)$ has $\ell_1$ coefficient norm at most $\prod_{j \in T \setminus \{k\}} (1 + |j|) \leq ((t+1)!!)^2$. Moreover, $\prod_{j \in T \setminus \{k\}} |k-j| \geq \prod_{j \in T \setminus \{0\}} |j| = (t!!)^2$. It follows that the coefficient norm of $f$ is at most $|T| \cdot (t+1)^2 = (t+1)^3$. Now suppose $K$ is odd, so without loss of generality $t$ is odd. For any $k \in [T]$, the polynomial $\prod_{j \in T \setminus \{k\}} (x-j)$ has coefficient norm at most $\prod_{j \in T \setminus \{k\}} (1 + |j|) \leq ((t+1)!!)^2$. Moreover, $\prod_{j \in T \setminus \{k\}} |k-j| \geq (t-1)!!(t+1)!!$. Thus the coefficient norm of $f$ is at most $|T| \cdot (t+1) = (t+1)^2$.
\end{proof}

\begin{lemma}\label{lemma:bic-hell}
There is a universal constant $c_{\ref{lemma:bic-hell}}$ so that the following holds. Let $n,t,H \in \NN$ and $\eta>0$. Let $S \subseteq [n]$. If $\eta < c_{\ref{lemma:bic-hell}} t^{-4} n^{-t}$, then
\[\min_{\pi \in \Pi_{n,t,H}} \Dhels{\BP^\st_{n,S,H,\eta}}{\BP^\pi} \leq 4\exp(-t^2/(2n)).\]
\end{lemma}

\begin{proof}[\pfref{lemma:bic-hell}]
By \cref{lemma:polynomial-interpolation}, there is a polynomial $f: \RR \to \RR$ with $f(k) = (-1)^{\frac{|S|-k}{2}}$ for all integers $k \in [-t,t]$ with $k \equiv |S| \pmod{2}$, and with $\ell_1$ coefficient norm $O(t^3)$. Define $g: \RR^n \to \RR$ by $g(x) = f(\sum_{i \in S} x_i)$. For any $x \in \{-1,1\}^n$ with $|\sum_{i \in S} x_i| \leq t$, we have $\sum_{i \in S} x_i \equiv |S| \pmod{2}$, so 
\[g(x) = (-1)^{\sum_{i \in S} \frac{1-x_i}{2}} = (-1)^{\#\{i \in S: x_i = -1\}} = \prod_{i \in S} x_i.\]
Moreover, $g$ can be represented as a degree-$t$ polynomial in $x_1,\dots,x_n$, with $\ell_1$ coefficient norm at most $O(t^4 n^t)$. Let $\tilde g$ be the square-free reduction of $g$ on $\{-1,1\}^n$, and let $(c_T)_T$ be the coefficient vector of $\tilde g$. Define $\theta \in \RR^d$ by $\theta_T := \alpha c_T$ where $\alpha := \frac{1}{2}\log(\frac{1+2\eta}{1-2\eta})$. Observe that $\norm{\theta}_1 \leq \alpha \cdot O(t^4 K^t) \leq O(\eta t^4 n^t) \leq 1$ by assumption that $\eta \leq c_{\ref{lemma:bic-hell}} t^{-4} n^{-t}$, so long as $c_{\ref{lemma:bic-hell}} > 0$ is a sufficiently small universal constant. Thus, $\theta \in \Theta$, and so the policy $\pi_\theta$ defined in \cref{eq:comp-lb-density} lies in $\Pi_{n,t,H}$. Moreover, for any $h \in [H]$ and $(x,a_{1:h}) \in \{-1,1\}^n \times \{-1,1\}^h$, we have by \eqref{eq:comp-lb-density} that
\begin{align}
\pi_\theta(a_h\mid{}x,a_{1:h-1})
&= \frac{\exp\left(\sum_{|T| \leq t} \alpha c_T \phiver_{n,t}(x,a_{1:h})_T\right)}{\sum_{a'_h \in \{-1,1\}}\exp\left(\sum_{|T| \leq t} \alpha c_T \phiver_{n,t}(x,a_{1:h-1},a'_h)_T\right)} \\
&= \frac{\exp\left(\alpha a_h g(x)\right)}{\sum_{a'_h \in \{-1,1\}}\exp\left(\alpha a'_h g(x)\right)} \\ 
&= \frac{1}{1 + \exp(-2\alpha a_h g(x))}.
\end{align}
If $|\sum_{i \in S} x_i| \leq t$, then in particular we have
\begin{align}
\pi_\theta\left(\prod_{i \in S} x_i \,\middle|\,x,a_{1:h-1}\right)
= \frac{1}{1 + \exp(-2\alpha)}
= \frac{1}{2} + \eta
\end{align}
by choice of $\alpha$. Thus, $\BP^{\pi_\theta}(\cdot\mid{}x)$ and $\BP^\st_{n,S,H,\eta}(\cdot\mid{}x)$ are identical (as distributions over $\{-1,1\}^H$) for any $x \in \{-1,1\}^n$ such that $|\sum_{i \in S} x_i| \leq t$. We conclude that
\begin{align}
\Dhels{\BP^{\pi_\theta}}{\BP^\st_{n,S,H,\eta}}
&= \EE_{x \sim \Unif(\{-1,1\}^n)} [\Dhels{\BP^{\pi_\theta}(\cdot|x)}{\BP^\st_{n,S,H,\eta}(\cdot|x)} ]\\ 
&\leq 2 \Prr_{x \sim \Unif(\{-1,1\}^n)}\left[\left|\sum_{i \in S} x_i\right| > t\right] \\ 
&\leq 4\exp\left(-\frac{t^2}{2n}\right)
\end{align}
by Hoeffding's inequality. This completes the proof.
\end{proof}

\subsection{Step 3: From Policies to Parity Functions}\label{subsec:policy-to-parity}

Next, we show that for any policy $\pihat$ such that $\BP^{\pihat}$ is close in Hellinger distance to the true distribution over trajectories, if we can sample from the conditional distribution $\pihat(\cdot\mid{}x) \in \Delta(\{-1,1\}^H)$ for any given $x \in \{-1,1\}^n$, then we can predict the parity function $\prod_{i \in S} x_i$. The idea is to sample multiple times from $\pihat(\cdot\mid{}x)$ and take majority; each individual trajectory gives a fairly weak signal since $\eta \sqrt{H}$ is sub-constant, but after boosting, the predictor has low error.

\begin{lemma}\label{lemma:hellinger-to-secret}
There is an algorithm $\Alg_{\ref{lemma:hellinger-to-secret}}$ with the following property. Let $n, H \in \NN$ and $S \subseteq [n]$. Given access to a conditional sampling oracle $\wh\MO$ for a policy $\pihat: \{-1,1\}^n \to \{-1,1\}^H$ and inputs $x \in \{-1,1\}^n$ and $\delta,\gamma>0$, it holds that
\[\Prr\left[\Alg_{\ref{lemma:hellinger-to-secret}}^{\wh\MO}(X,\delta,\gamma) \neq \prod_{i \in S} X_i\right] \leq \delta + \frac{6C_{\ref{lemma:spread-signal}}\log(1/\delta)}{\gamma^2} \sqrt{\Dhels{\BP^{\pihat}}{\BP^{\pi^\st_{n,S,H,\gamma/(C_{\ref{lemma:spread-signal}}\sqrt{H})}}}}\]
where the probability is over $X \sim \Unif(\{-1,1\}^n)$. Moreover, the time complexity of $\Alg^{\wh{\cO}}_{\ref{lemma:hellinger-to-secret}}$ is $\poly(n,H,\gamma^{-1},\log(1/\delta))$.
\end{lemma}

\begin{proof}[\pfref{lemma:hellinger-to-secret}]
Fix a realization $X = x \in \{-1,1\}^n$. The algorithm $\Alg_{\ref{lemma:hellinger-to-secret}}^{\wh\MO}$ does the following on input $(x,\delta,\gamma)$. Set $N := 6C_{\ref{lemma:spread-signal}}^2\gamma^{-2}\log(1/\delta).$ The algorithm draws $N$ independent samples $y^{(1)},\dots,y^{(N)} \sim \pihat(\cdot\mid{}x)$, and outputs 
\[\Maj \{y^{(i)}_j: i \in [N], j \in [H]\}.\]
Let $\MO$ be the conditional sampling oracle for $\pi^\st_{n,S,H,\gamma/(C_{\ref{lemma:spread-signal}}\sqrt{H})}$. Then in the execution of $\Alg_{\ref{lemma:hellinger-to-secret}}^{\MO}(x,\delta,\gamma)$, we have that $(y^{(i)}_j: i \in [N], j \in [H])$ are $NH$ independent and identically distributed random variables with $\Pr[y^{(i)}_j \neq \prod_{i \in S} x_i] = 1/2 - \gamma/(C_{\ref{lemma:spread-signal}}\sqrt{H})$. It follows that
\begin{align}
\Prr\left[\Alg_{\ref{lemma:hellinger-to-secret}}^{\MO}(x,\delta,\gamma) \neq \prod_{i \in S} x_i\right]
&= \Prr\left[\sum_{i,j} \mathbbm{1}\left[y^{(i)}_j \neq \prod_{i \in S} x_i\right] \geq \frac{NH}{2}\right] \\ 
&\leq \exp\left(-\left(\frac{2\gamma}{C_{\ref{lemma:spread-signal}}\sqrt{H}}\right)^2 \cdot \frac{NH}{6}\right) \\ 
&\leq \delta.
\end{align}
Therefore by the data processing inequality,
\begin{align}
\Prr\left[\Alg_{\ref{lemma:hellinger-to-secret}}^{\wh\MO}(x,\delta,\gamma) \neq \prod_{i \in S} x_i\right]
&\leq \delta + \Prr\left[\Alg_{\ref{lemma:hellinger-to-secret}}^{\wh\MO}(x,\delta,\gamma) \neq \Alg_{\ref{lemma:hellinger-to-secret}}^{\MO}(x,\delta,\gamma)\right] \\ 
&\leq \delta + N \cdot \Dtv{\pihat(\cdot\mid{}x)}{\pi^\st_{n,S,H,\gamma/(C_{\ref{lemma:spread-signal}}\sqrt{H})}(\cdot\mid{}x)}.
\end{align}
Taking expectation over $X \sim \Unif(\{-1,1\}^n)$, we get that
\begin{align}
\Prr\left[\Alg_{\ref{lemma:hellinger-to-secret}}^{\wh\MO}(X,\delta,\gamma) \neq \prod_{i \in S} X_i\right]
&\leq \delta + N \cdot \Dtv{\BP^{\pihat}}{\BP^\st_{n,S,H,\gamma/(C_{\ref{lemma:spread-signal}}\sqrt{H})}} \\ 
&\leq \delta + N \cdot \sqrt{\Dhels{\BP^{\pihat}}{\BP^\st_{n,S,H,\gamma/(C_{\ref{lemma:spread-signal}}\sqrt{H})}}}
\end{align}
as claimed.
\end{proof}

\subsection{Putting Everything Together: Proof of Theorem \ref*{thm:comp-lb-main}}\label{subsec:comp-lb-proof}

We now restate and prove \cref{thm:comp-lb-app}, and hence \cref{thm:comp-lb-main}. The proof is a straightforward consequence of \cref{lemma:spread-signal,lemma:bic-hell,lemma:hellinger-to-secret} together with appropriate parameter choices.

\begin{proof}[\pfref{thm:comp-lb-main}]
Suppose that $\Alg$ is a learning algorithm that satisfies the guarantees specified in the theorem statement, with parameters $C,c>0$. We design a algorithm $\Adv$ (short for ``Adversary'') %
that, for any $n$, distinguishes between $\BP^\st_{n,S,1,1/4}$ and $\MD = \BQ_{n,1}$. In particular, $\Adv$ takes as input $1^n$ and a sampling oracle $\MO$ for some distribution $\MD$ over $\{-1,1\}^n \times \{-1,1\}$. $\Adv$ then has the following behavior.

Set $t := n^{4/(4+c)}$, $d := \sum_{i=0}^t \binom{n}{i}$, $H := n^{3t}$, $\epsilon := 1/H^2$, $\gamma := \sqrt{\frac{c_{\ref{lemma:spread-signal}}}{\log(2H)}}$, $\eta := \gamma/(C_{\ref{lemma:spread-signal}}\sqrt{H})$. Also set $T = (dH/\epsilon)^C$. First, $\Adv$ draws $T$ independent samples $(x\ind{i},y\ind{i})_{i=1}^T$ from $\MO$, and for each $i \in [T]$ computes $o\ind{i} \sim \Alg_{\ref{lemma:spread-signal}}(x\ind{i},y\ind{i},\gamma)$. Next, $\Adv$ simulates $\Alg$ with inputs $(o\ind{i})_{i=1}^T$, $\MA := \{-1,1\}$, and $\epsilon$. Recall that $\Alg$ requires access to two computational oracles, which $\Adv$ simulates efficiently as follows:
\begin{enumerate}
\item Feature oracle: when $\Alg$ queries $(x,a_{1:h}) \in \MX\times\MA^h$ for some $h \in [H]$, $\Adv$ returns $\phiver_{n,t}(x,a_{1:h})$.
\item Projection oracle: when $\Alg$ queries $\theta \in \RR^d$, $\Adv$ returns the projection of $\theta$ onto $\Theta$, which is the unit $\ell_1$ ball, using the method of \cite{duchi2008efficient}.
\end{enumerate}
The output of $\Alg$ is a collection of circuits $\MC_{\pihat_h}$ that sample from the conditional distributions of a policy $\pihat=(\pihat_h)_{h=1}^H$; chaining these together gives a conditional sampler $\wh\MO$ for the distribution of $a_{1:H}$ under $\pihat$ for any given $x$. $\Adv$ draws a fresh sample $(\xtest,\ytest) \sim \MO$ and computes $\wh y := \Alg_{\ref{lemma:hellinger-to-secret}}^{\wh\MO}(\xtest,1/100,\gamma)$ (cf. \cref{lemma:hellinger-to-secret}). Finally, $\Adv$ outputs $\mathbbm{1}[\ytest = \wh{y}]$.

\paragraph{Analysis} First suppose that $\MO$ is a sampling oracle for $\BP^\st_{n,S,1,1/4}$ for some unknown $S \subseteq [n]$. Let $\mu$ be the distribution of each $o\ind{i}$. By \cref{lemma:spread-signal} and choice of $\gamma$, we have 
\begin{equation} \Dhel{\BP^\st_{n,S,H,\eta}}{\mu} \leq \sqrt{2\Dtv{\BP^\st_{n,S,H,\eta}}{\mu}} \leq 2\exp(-c_{\ref{lemma:spread-signal}}/\gamma^2) \leq \frac{1}{H}.
\label{eq:sample-to-truth}
\end{equation}
Moreover, by \cref{lemma:bic-hell} with $K := n$ and the fact that %
$\eta < 1/\sqrt{H} = n^{-1.5t} \leq c_{\ref{lemma:bic-hell}}t^{-4}n^{-t}$ for sufficiently large $n$, we have 
\begin{equation}
\min_{\pi\in\Pi_{n,t,H}} \Dhels{\BP^\st_{n,S,H,\eta}}{\BP^\pi} \leq 4e^{-t^2/(2n)} \leq e^{-(3t\log n)^{1-\frac{c}{2}}} = e^{-(\log H)^{1-\frac{c}{2}}}
\label{eq:bic-applied}
\end{equation}
where the second inequality holds by choice of $t$, for sufficiently large $n$, and the final equality holds by choice of $H$. We now invoke the guarantee of $\Alg$. Since the parameter space $\Theta$ consists of vectors $\theta$ with $\norm{\theta}_1 \leq 1$, and the range of the feature map $\phiver_{n,t}$ is contained in $[-1,1]^d$, \cref{ass:linear-norm-bounds-main} is satisfied with parameters $B = \sqrt{d}$ and $\Bdot = 1$. Moreover, by construction the action space has size $2$. Thus, we get that the time complexity of $\Alg$ (modulo oracle calls) is $\poly(d,H,1/\epsilon)$, and with probability at least $9/10$ it holds that 
\begin{equation}\Dhels{\BP^{\pihat}}{\mu} \lesssim \epsilon + e^{(\log \max(d,H))^{1-c}} \cdot \min_{\pi\in\Pi_{n,t,H}} \Dhels{\BP^\pi}{\mu}.
\label{eq:cmis-bound}\end{equation}
Combining \cref{eq:sample-to-truth,eq:bic-applied,eq:cmis-bound}, we get that in an event $\ME$ that occurs with probability at least $9/10$,
\begin{align}
\Dhel{\BP^{\pihat}}{\BP^\st_{n,S,H,\eta}}
&\leq \Dhel{\BP^{\pihat}}{\mu} + \frac{1}{H} \\
&\leq \sqrt{\epsilon + e^{(\log \max(d,H))^{1-c}} \cdot \min_{\pi\in\Pi_{n,t,H}} \Dhels{\BP^\pi}{\mu}} + \frac{1}{H} \\ 
&\leq \sqrt{\epsilon + e^{(\log \max(d,H))^{1-c}} \cdot \left(\frac{1}{H}+e^{-(\log H)^{1-\frac{c}{2}}}\right)} + \frac{1}{H} \\ 
&\leq O(e^{-\frac{1}{3}(\log H)^{1-\frac{c}{2}}})
\end{align}
where the final inequality also uses the choice of $\epsilon := 1/H^2$ and the fact that $d \leq H$. In event $\ME$, we have by \cref{lemma:hellinger-to-secret}, choice of $\gamma$, and the above bound, that
\[\Prr\left[\wh y \neq \prod_{i \in S} \xtest_i\right] \leq \frac{1}{100} + O(\log^2 H) \cdot O(e^{-\frac{1}{3}(\log H)^{1-\frac{c}{2}}}) \leq \frac{1}{99}\]
so long as $n$ is sufficiently large. Hence, the total probability of the event $\{\wh y \neq \prod_{i \in S} \xtest_i\}$ is at most $\Pr[\overline \ME] + 1/99 \leq 1/8$. Now $\ytest = (-1)^{\xi} \prod_{i \in S} \xtest$ where $\xi \sim \Ber(1/4)$, so $\Pr[\wh y \neq \ytest] \leq 3/8$. Thus, the distinguisher outputs $1$ with probability at least $5/8$. %

On the other hand, suppose that $\MO$ is a sampling oracle for $\BQ_{n,H}$. Since $\ytest$ is uniformly random conditioned on $\xtest$, we have that $\ytest$ is independent of $\wh y$ and hence $\Adv$ outputs $1$ with probability exactly $1/2$.

Finally, note that the time complexity of $\Adv$ is $\poly(d,H,1/\epsilon) = n^{O(n^{4/(4+c)})} = 2^{n^{1-\Omega(1)}}$, since the time complexity is dominated by the simulation of $\Alg$, and both computational oracles can be implemented in time $\poly(d,H)$. This contradicts \cref{assumption:lpn}.
\end{proof}

\begin{remark}[On computational hardness of regret minimization]\label{remark:il-hardness}
By \cref{eq:tv_equivalence} and the quadratic equivalence between Hellinger distance and TV-distance, one can convert \cref{thm:comp-lb-app}
into a statement about the computational hardness of approximate regret minimization with unknown, worst-case bounded reward function, though this requires examining misspecification level in the construction. In fact, in the proof of \cref{thm:comp-lb-app} it is not necessary to go through Hellinger distance at all (except for the fact that the theorem concerns agnostic estimation in Hellinger): \cref{lemma:bic-hell} achieves the bound on squared Hellinger misspecification through bounding TV-misspecification, and similarly \cref{lemma:hellinger-to-secret} achieves the error bound in terms of Hellinger distance via an error bound in terms of TV-distance. Thus, a slightly more direct argument proves that there is no computationally efficient algorithm achieving  \[\Dtv{\BP^{\pihat}}{\BP^{\pistar}} \lesssim \epsilon + e^{(\log \max(d,H))^{1-c}} \cdot \min_{\pi \in \Pi} \Dtv{\BP^\pi}{\BP^{\pistar}}.\]
Via \cref{eq:tv_equivalence}, this precisely shows hardness of approximate regret minimization with worst-case bounded reward.
\end{remark}

\newpage
\section{Proof of Theorem \ref*{thm:chunk-kr-informal} (Computational-Statistical Tradeoff)}\label{app:comp-ub}

In this section we prove \cref{thm:chunk-kr-informal}, stated formally below as \cref{thm:chunk-kr-main}. We start by formally introducing the problem setting, expanding upon the discussion in \cref{sec:computational}.

\paragraph{Problem setting} Let $\MX$ be the context space and let $\MA = \{0,1\}$ be the action space. Let $d,H \in \NN$. Let $M$ be the $H$-step autoregressive MDP with context space $\MX$, action space $\MA$, and some initial context distribution $\MD \in \Delta(\MX)$. Let $\phi:\MX\times\MA^\st \to \RR^d$ be a $d$-dimensional feature mapping, and let $\Theta \subset \RR^d$ be a convex parameter set. Define the set of autoregressive linear policies as $\Pi := \{\pi_\theta:\theta\in\Theta\}$ where $\pi_\theta=  (\pi_{\theta,h})_{h=1}^H$ is as defined in \cref{eq:linear}. We assume that for any $(x,a_{1:h}) \in \MX\times\MA^h$ (where $0 \leq h \leq H$), we may query $\phi(x,a_{1:h})$ in time $\poly(d,H)$. Additionally, we assume that the features and parameter space satisfy the following assumption for a known parameter $L$.\footnote{We do not require a projection oracle for $\Theta$, since the algorithm will relax to a larger policy class depending only on $L$.}

\begin{assumption}[Norm bounds]\label{ass:linear-norm-bounds-ub}
  Let $L\geq 1$ be a parameter. It holds that $\norm{\phi(x,a_{1:h})}_2 \leq L$ for all $(x,a_{1:h}) \in \MX\times\MA^\st$ and $\norm{\theta}_2 \leq L$ for all $\theta\in\Theta$.
\end{assumption}

\begin{algorithm}[tp]
\caption{$\ChunkKR((x\ind{i},a\ind{i}_{1:H})_{i=1}^n, K,L, \epsilon)$}\label{alg:chunkkr}
\begin{algorithmic}[1]
\State \textbf{input:} Samples $(x\ind{i},a_{1:H}\ind{i})_{i=1}^n$, chunk size $K$, norm bound $L$, desired excess error $\epsilon$.
\State \multiline{Set $B \gets (L^2H^22^{3H+2}e^{2L^2+1}/\epsilon)^{C_{\ref{thm:kernel-apx}}L^2H}$, $\gamma \gets e^{-2L^2H-1}2^{-H}$, $\epapx \gets e^{-3L^2H-2}2^{-3H/2}\epsilon/H$, \\and $\epopt \gets \epsilon/H$.}
\For{$h \in \{K, 2K, 3K, \dots, H\}$} 
    \State Define $\pihat_{h+1-K:h}: \MX \times \MA^{h-K} \to \Delta(\MA^K)$ by
    \[\pihat_{h+1-K:h} \gets \KernRho((x\ind{i,h},a\ind{i}_{h+1-K:h})_{i=1}^n, B, \gamma, \epapx,\epopt),\]
    where $x\ind{i,h} := (x\ind{i}, a\ind{i}_{1:h-K}).$
    \Comment{See \cref{alg:kernrho} for the description of $\KernRho$.}
    \For{$j \in \{h+1-K,\dots,h\}$}
        \State Define $\pihat_j: \MX\times\MA^{j-1} \to \MA$ by
        \[\pihat_j(a_j\mid{}x,a_{1:j-1}) := \frac{\sum_{a'_{j+1:h} \in \MA^{h-j}} \pihat_{h+1-K:h}(a_{h+1-K:j},a'_{j+1:h})\mid{}x,a_{1:h-K})}{\sum_{a'_{j:h} \in \MA^{h+1-j}} \pihat_{h+1-K:h}(a_{h+1-K:j-1},a'_{j:h})\mid{}x,a_{1:h-K})}.\]
    \EndFor
\EndFor
\State \textbf{return:} policy $\pihat = (\pihat_j)_{j=1}^H$.
\end{algorithmic}
\end{algorithm}

We can now formally restate the desired result, which shows that the $\ChunkKR$ algorithm (\cref{alg:chunkkr}) achieves the approximation guarantee in \cref{thm:chunk-kr-informal}.

\begin{theorem}[Formal statement of \cref{thm:chunk-kr-informal}]\label{thm:chunk-kr-main}
There is a constant $C_{\ref{thm:chunk-kr-main}}>0$ such that the following guarantee for $\ChunkKR$ (\cref{alg:chunkkr}) holds. Let $\delta \in (0,1/2)$, $\epsilon \in (0,1)$, $K \in [H]$, $L \geq 1$, and $n \in \NN$. Suppose that \cref{ass:linear-norm-bounds-ub} holds with parameter $L$, and that 
\[n \geq (2^{L^2+K}H/\epsilon)^{C_{\ref{thm:kern-rho-main}}L^2K} \log(H/\delta).\] 
Let $(x\ind{i},a\ind{i})_{i=1}^n$ be i.i.d. samples from $\bbP^\st\ldef{}\bbP^{\pistar}\in \Delta(\MX\times\MA^H)$. Then with probability at least $1-\delta$, the output $\pihat\gets \ChunkKR((x\ind{i},a\ind{i})_{i=1}^n,K,L, \epsilon)$ satisfies
\[\Dhels{\BP^{\pihat}}{\BP^\star}
\lesssim \epsilon + \frac{H}{K} \min_{\pi\in\Pi} \Dhels{\BP^{\pi}}{\BP^\star}.\]
The time complexity of the algorithm is $\poly(n,H)$, and sampling from $\pihat$ can be done in time $\poly(n, H)$.
\end{theorem}

Henceforth we suppose that \cref{ass:linear-norm-bounds-ub} holds with parameter $L$; we omit restating this assumption in subsequent theorem and lemma statements.

\paragraph{Organization of this appendix} In \cref{subsec:comp-ub-overview}, we outline the proof of \cref{thm:chunk-kr-main}. In this proof, the main ingredient is \cref{thm:kern-rho-main}, an analysis of the subroutine $\KernRho$ (\cref{alg:kernrho}). In \cref{subsec:kernel-apx,subsec:kernrho-statistical,subsec:kernrho-computational} we assemble the key lemmas for the proof of \cref{thm:kern-rho-main}, and in \cref{subsec:kernrho-proof} we complete the proof of \cref{thm:kern-rho-main}. Finally, in \cref{subsec:chunkkr-proof} we complete the proof of \cref{thm:chunk-kr-main} and hence \cref{thm:chunk-kr-informal}.

\subsection{Algorithm and Proof Overview}\label{subsec:comp-ub-overview}

In this section we provide an overview of the algorithm $\ChunkKR$ (\cref{alg:chunkkr}) and outline the proof of \cref{thm:chunk-kr-main}. As discussed in \cref{sec:comp-lb}, the key subroutine of $\ChunkKR$ is an algorithm $\KernRho$ (\cref{alg:kernrho}), which learns misspecified autoregressive linear models with optimal approximation ratio, and with time complexity scaling polynomially in $d$ but exponentially in the horizon $H$.

As shown in \cref{alg:chunkkr}, the full algorithm $\ChunkKR$ divides the horizon into $H/K$ chunks of length $K$. For each chunk $\{h+1-K,\dots,h\}$, it applies $\KernRho$ to learn the distribution of $a_{h+1-K:h}$ under $\BP^\st$, conditioned on the initial context $x$ and the first $h-K$ actions $a_{1:h-K}$. In particular, the entire tuple $(x, a_{1:h-K})$ is interpreted as a ``context'' in a new autoregressive MDP with horizon $K$. We show that if $\BP^\st$ is close to some autoregressive linear model, then this new distribution is close to an autoregressive linear model in the new MDP; moreover, if we learn each chunk $\{h+1-K,\dots,h\}$ up to squared Hellinger distance $\epsilon_h$, then we learn the overall model up to squared Hellinger distance $\sum_{i=0}^{H/K-1} \epsilon_{(i+1)K}$. It follows that if $\KernRho$ has optimal approximation ratio, then $\ChunkKR$ has approximation ratio $O(H/K)$. We defer the formal analysis of $\ChunkKR$ to \cref{subsec:chunkkr-proof}; the interim is devoted to the analysis of $\KernRho$. We now outline the proof of the following guarantee.

\begin{algorithm}[tp]
\caption{$\KernRho((x\ind{i},a\ind{i}_{1:H})_{i=1}^n, B,\gamma,\epapx,\epopt)$}\label{alg:kernrho}
\begin{algorithmic}[1]
\State \textbf{input:} Samples $(x\ind{i},a_{1:H}\ind{i})_{i=1}^n$; RKHS norm bound $B\geq 1$; program tolerances $\gamma,\epapx>0$; optimization error tolerance $\epopt>0$.
\State For each $1 \leq i \leq n$ and $a \in \MA^H$, compute $\uop{x\ind{i}}{a}$ by
\[\uop{x\ind{i}}{a}_h \gets 
\begin{cases} 
\phi(x\ind{i},a_{1:h-1}, 1) - \phi(x\ind{i},a_{1:h-1}, 0) & \text{ if } a_h = 0 \\ 
\phi(x\ind{i},a_{1:h-1}, 0) - \phi(x\ind{i},a_{1:h-1}, 1) & \text{ if } a_h = 1
\end{cases}.
\]
\State Compute $\Sigma \in \RR^{n2^H \times n2^H}$ by
\begin{equation} 
\Sigma_{(i,a),(i',a')} \gets K(\uop{x\ind{i}}{a}, \uop{x\ind{i'}}{a'}) = \prod_{h=1}^H \frac{1}{1 - \frac{1}{2}\langle \uop{x\ind{i}}{a}_h, \uop{x\ind{i'}}{a'}_h\rangle}.
\label{eq:sigma-def}
\end{equation}
\State $T \gets 4C_{\ref{thm:sp-pgd}}^2 2^{2H+2} B/(\gamma^3 \epopt^2)$, $\eta \gets B\gamma^{3/2}2^{-H-1}\sqrt{2/T}$, $\epproj \gets 1/(16BT^4)$
\State Let $\Oproj$ implement $\epproj$-approximate projection onto the (convex) set
\begin{align} 
\MY &:= \Bigg\{y \in \RR^{n2^H}:  \left(\forall i\in[n],a\in\MA^H: e_{i,a} \Sigma^{1/2} y \geq \gamma\right) \\ 
&\qquad\land \left(\forall i \in [n]: 1-\epapx \leq \sum_{a\in\MA^H} e_{i,a} \Sigma^{1/2} y \leq 1+\epapx\right)\Bigg\}.
\label{eq:y-def}
\end{align}
\State Let $\Ovec$ implement evaluation queries to the vector field
\[g(\alphatil,\betatil) := \left(\grad_{\alphatil}\left[\frac{1}{n}\sum_{i=1}^n \tau\left(\frac{e_{i,a\ind{i}}\Sigma^{1/2}\alphatil}{e_{i,a\ind{i}}\Sigma^{1/2}\betatil}\right)\right],-\grad_{\betatil}\left[\frac{1}{n}\sum_{i=1}^n \tau\left(\frac{e_{i,a\ind{i}}\Sigma^{1/2}\alphatil}{e_{i,a\ind{i}}\Sigma^{1/2}\betatil}\right)\right]\right).\]
\State Compute $(\alphatil_t,\betatil_t)_{t=1}^T \gets \PGD(\Ovec,\Oproj\oplus\Oproj,T,\eta)$.\label{line:pgd-apply}
\Comment{See \cref{alg:pgd}.}
\State Compute 
\begin{equation} \wh\alpha \gets \Sigma^{-1/2} \Proj_{\vspan(\Sigma)} \left(\frac{1}{T}\sum_{t=1}^T \alphatil_t\right).\label{line:wh-alpha}\end{equation}
\State \textbf{return:} policy $\pihat:\MX \to \Delta(\MA^H)$ defined by
\[\pihat(\cdot|x) := \argmin_{\mu \in \Delta_\gamma(\MA^H)} \sum_{a \in \MA^H} \left| \mu(a) - \sum_{i \in [n]}\sum_{a' \in \MA^H} \wh\alpha_{i,a'} K(\uop{x\ind{i}}{a'}, \uop{x}{a})\right|.\]
\end{algorithmic}
\end{algorithm}

\begin{theorem}[Main guarantee for $\KernRho$]\label{thm:kern-rho-main}
There is a constant $C_{\ref{thm:kern-rho-main}}>0$ so that the following guarantee for $\KernRho$ (\cref{alg:kernrho}) holds. Let $\delta \in (0,1/2)$ and $\epsilon \in (0,1)$. Suppose that $n \geq (2^{L^2+H}/\epsilon)^{C_{\ref{thm:kern-rho-main}}L^2H} \log(1/\delta)$. Fix an arbitrary policy $\pistar$, and let $(x\ind{i},a\ind{i})_{i=1}^n$ be i.i.d. samples from $\BP^{\star}\ldef\BP^{\pistar}$. Then with probability at least $1-\delta$, the output $\pihat\gets \KernRho((x\ind{i},a\ind{i})_{i=1}^n, B,\gamma,\epapx,\epopt)$, with $B := (L^2H2^{3H+2}e^{2L^2 H+1}/\epsilon)^{C_{\ref{thm:kernel-apx}}L^2H}$, $\gamma := e^{-2L^2 H-1}2^{-H}$, $\epapx := e^{-3L^2 H-2}2^{-3H/2}\epsilon$, and $\epopt := \epsilon$, satisfies
\[\Dhels{\BP^{\pihat}}{\BP^\star}
\leq \frac{88}{3}\min_{\pi\in\Pi} \Dhels{\BP^{\pi}}{\BP^\star} + O\left(\epsilon\right).\]
The time complexity of the algorithm is $\poly(n)$, and sampling from $\pihat$ can be done in time $\poly(n)$.
\end{theorem}

See \cref{alg:kernrho} for pseudocode for $\KernRho$. The main idea is to implement an improper relaxation of the $\rho$-estimator $\rhobc$ from \cref{sec:optimal}. This is motivated by the fact that the min-max objective solved by $\rhobc$ is convex-concave in \emph{policy} space, yet even for autoregressive linear models, the objective is not convex-concave in \emph{parameter} space\footnote{Nor is the set of autoregressive linear models convex in policy space.}---at least, not with the natural parametrization $\theta \mapsto \pi_\theta$. However, any autoregressive linear model $\pi_\theta$ \emph{can} be approximated by a function in an infinite-dimensional reproducing kernel Hilbert space (RKHS) with efficiently computable kernel $K$. This motivates our basic approach: relax the program to the RKHS, and use the ``kernel trick'' to reduce back to a finite-dimensional program. This approach was pioneered by \cite{shalev2011learning} for agnostic learning of halfspaces; to compare, our relaxation requires additional care to ensure that the statistical properties of $\rhobc$ are preserved, and the resulting program is a convex-concave min-max program rather than a convex minimization program, since it is based on $\rhobc$ rather than Empirical Risk Minimization.

\paragraph{Kernel approximation (\cref{subsec:kernel-apx})}
We now describe the kernel function, the approximation result, and how it suggests a relaxation of $\rhobc$. We begin by defining a convenient reparameterization of the feature map $\phi$.
\begin{definition}[Joint feature map]\label{def:uvec}
For each $x \in \MX$ and $a_{1:H} \in \MA^H$, define $\uf(x,a_{1:H}) \in (\RR^d)^H$ by
\[\uf(x,a_{1:H})_h := \frac{1}{2L}\left(\phi(x,a_{1:h-1},1-a_h) - \phi(x,a_{1:h-1},a_h)\right)\]
for each $h \in [H]$. %
\end{definition}

Let $\ball$ denote the Euclidean unit ball in $\RR^d$. Note that $\uf(x,a_{1:H})_h \in \ball$ for each $h$, as a consequence of \cref{ass:linear-norm-bounds-ub}.

\begin{definition}[Kernel function] \label{def:kernel} Define $K: (\ball)^H \times (\ball)^H \to \RR$ by  
\[K(u_{1:H},u'_{1:H}) := \prod_{h=1}^H \frac{1}{1 - \frac{1}{2} \langle u_h, u'_h\rangle}.\]
\end{definition}

This kernel function coincides with that of \cite{shalev2011learning} when we set $H = 1$. In \cref{subsec:kernel-apx}, we describe an infinite-dimensional mapping $\psi: (\ball)^H \to \RR^\NN$ (\cref{def:psi}) with the following properties. First, $\psi$ induces the kernel $K$ (i.e., satisfies $\langle \psi(\cdot),\psi(\cdot)\rangle = K(\cdot,\cdot)$). Second, any autoregressive linear policy has sequence-level density $\pi_\theta(a_{1:H}\mid{} x)$ approximated by a bounded linear function of $\psi(\uf(x,a_{1:H}))$:

\begin{theorem}[Kernel approximation of autoregressive linear policies]\label{thm:kernel-apx}
There is a constant $C_{\ref{thm:kernel-apx}}>0$ so that the following holds. Let $\theta \in \Theta$ and $\epsilon>0$. There is some $v_\theta \in \RR^\NN$ such that $\norm{v_\theta}_2^2 \leq (L^2H2^H/\epsilon)^{C_{\ref{thm:kernel-apx}}L^2H}$ and, for all $x \in \MX$,
\[\sum_{a_{1:H} \in \MA^H} \left|\pi_\theta(a_{1:H}\mid{}x) - \langle v_\theta, \psi(\uf(x,a_{1:H}))\rangle\right| \leq \epsilon.\]
\end{theorem}

Notice that the norm bound in \cref{thm:kernel-apx} scales exponentially with the horizon $H$ and the parameter $L$ from \cref{ass:linear-norm-bounds-ub}, but not the dimension $d$ of the original features. This is crucial since the norm bound will be directly reflected in the sample complexity of $\KernRho$ (via Rademacher bounds for generalization), and hence in the time complexity.

\paragraph{Relaxing $\rhobc$ to the RKHS} Recall from \cref{sec:optimal} that $\rhobc$ is defined in terms of $\tau: (0,\infty) \to \RR$ defined as $\tau(z) = \frac{\sqrt{1/z} - 1}{\sqrt{1/z} + 1}.$
\cref{thm:kernel-apx} suggests relaxing $\rhobc$ to the following program:
\begin{equation} \wh v \gets \argmin_{v \in \RR^\NN:\norm{v}_2^2\leq B} \max_{w \in \RR^\NN:\norm{w}_2^2\leq B} \frac{1}{n} \sum_{i=1}^n \tau\left(\frac{\langle v,\psi(u_{1:H}\ind{i})\rangle}{\langle w, \psi(u_{1:H}\ind{i})\rangle}\right)\label{eq:kernel-rho-first-attempt}\end{equation}
where $B := (L^2H2^H/\epsilon)^{C_{\ref{thm:kernel-apx}}L^2H}$, and $u_{1:H}\ind{i} = \uf(x\ind{i},a_{1:H}\ind{i})$. Since $(x,y) \mapsto \tau(x/y)$ is convex-concave (\cref{lemma:tau-convex-concave}), this program is convex-concave, albeit infinite-dimensional. Unfortunately, analyzing the program as written leads to statistical issues: in the standard analysis of $\rhobc$, the key properties relating the population-level loss $\EE_{x\sim p^\st}[\tau(p(x)/q(x))]$ to the Hellinger distances $\Dhels{p}{p^\st}$ and $\Dhels{q}{p^\st}$ (\cref{lemma:rho-estimator-bounds}) crucially use that $p,q$ are distributions. In \cref{eq:kernel-rho-first-attempt}, not all $v,w \in \RR^\NN$ correspond to distributions, so it unclear whether the corresponding losses relate to any useful error metric. Even worse, $\tau$ is only well-defined on $(0,\infty)$, but the argument of $\tau$ in \cref{eq:kernel-rho-first-attempt} could be negative. Finally, even if the argument were always non-negative, $\tau$ is non-Lipschitz near $0$, which poses issues for generalization arguments based on Rademacher complexity. We fix all of these issues by adding additional constraints to ensure that $v$ and $w$ approximately correspond to conditional distributions with densities bounded above zero, at least when conditioning on the observed contexts:
\begin{equation} \wh v \gets \argmin_{v \in \wh V} \sup_{w \in \wh W} \frac{1}{n} \sum_{i=1}^n \tau\left(\frac{\langle v,\psi(u_{1:H}\ind{i})\rangle}{\langle w, \psi(u_{1:H}\ind{i})\rangle}\right)
\label{eq:kern-rho-second-attempt}
\end{equation}
where
\begin{align} 
\wh V := \wh W := \Big\{v \in \RR^\NN: \left(\norm{v}_2^2 \leq B\right) &\land \left(\forall i \in [n], a \in \MA^H: \langle v, \psi(\uf(x\ind{i},a_{1:H}))\rangle \geq \gamma\right) \\ 
&\land \left(\forall i \in [n]: 1-\epapx \leq \langle v, s\ind{i}\rangle \leq 1+\epapx\right)\Big\}.
\end{align}
where $s\ind{i} = \sum_{a \in \MA^H} \psi(\uf(x\ind{i},a_{1:H}))$, and $\gamma,\epapx>0$ are parameters defined in \cref{thm:kern-rho-main}. Proving \cref{thm:kern-rho-main} now requires (1) showing that the program \cref{eq:kern-rho-second-attempt} is statistically efficient, and (2) it can be reduced to a finite-dimensional program and efficiently solved.

\paragraph{Statistical analysis (\cref{subsec:kernrho-statistical})} The main results of \cref{subsec:kernrho-statistical} are (1) \cref{lemma:minmax-ub}, which shows that the min-max value of \cref{eq:kern-rho-second-attempt} can be bounded by the best-in-class Hellinger distance of $\pistar$ with respect to $\Pi$, and (2) \cref{lemma:hellinger-by-loss}, which shows that for any potential solution $v$ to the program \cref{eq:kern-rho-second-attempt}, if we convert it to a conditional distribution $\pibar^v$, the Hellinger distance from $\pibar^v$ to $\pistar$ can be bounded in terms of the best-in-class Hellinger distance and the min-max loss. Together, \cref{lemma:minmax-ub,lemma:hellinger-by-loss} imply that if we can approximately solve \cref{eq:kern-rho-second-attempt} (and compute the corresponding policy $\pibar^{\wh v}$), then we achieve the statistical guarantee required for \cref{thm:kern-rho-main}.

To prove these lemmas, we use the constraints on $\wh V = \wh W$ to show that with high probability any $v \in \wh V$ approximately corresponds to some real conditional distribution $\pibar^v$ for most contexts (\cref{lemma:tv-to-simplex}). We then 
 use standard Rademacher bounds, applied to an everywhere-Lipschitz mollification of the loss in \cref{eq:kern-rho-second-attempt}, to show that the empirical loss concentrates for all $v,w \in \wh V=\wh W$. Finally, we use \cref{thm:kernel-apx} together with the choice of $B$ to show that there is some $v \in \wh V$ for which $\pibar^v$ has near-optimal Hellinger distance to $\pistar$ (\cref{lemma:best-to-pibar}). With these tools, the desired statistical guarantees then follow from \cref{lemma:rho-estimator-bounds}.

\paragraph{Computational analysis (\cref{subsec:kernrho-computational})} While the program defined in \cref{eq:kern-rho-second-attempt} is a convex-concave min-max program with convex constraint sets, but it is infinite-dimensional. To reduce to finite dimensions, we essentially use a generalization of the Representer Theorem to min-max losses. In particular, it suffices to optimize over $v \in \wh V$ and $w \in \wh W$ that are linear combinations of the vectors $\{\psi(\uf(x\ind{i},a_{1:H})): i \in [n], a_{1:H}\in \MA^H\}$. For such vectors, the loss in \cref{eq:kern-rho-second-attempt} and the linear constraints can be written explicitly in terms of the kernel function and the coefficients of the linear combination: for any $v := \sum_{i,a} \alpha_{i,a_{1:H}} \psi(\uf(x\ind{i},a_{1:H}))$ and any $(x',a'_{1:H}) \in \MX\times\MA^H$, we can write
\[\langle v, \psi(\uf(x,a_{1:H}))\rangle = \sum_{i,a_{1:H}} \alpha_{i,a_{1:H}} K(\uf(x\ind{i},a_{1:H}),\uf(x',a'_{1:H})).\]
Additionally, since $K$ is positive-semidefinite, the Euclidean norm constraint translates to an ellipsoid constraint. Ultimately, we get the following program:
\begin{equation}
\wh\alpha \gets \argmin_{\alpha \in \wh J} \max_{\beta \in \wh K} \emploss(\alpha,\beta)
\label{eq:kern-rho-finite}
\end{equation}
where
\[\emploss(\alpha,\beta) := \frac{1}{n}\sum_{i=1}^n \tau \left( \frac{\sum_{j=1}^n \sum_{a \in \MA^H} \alpha_{j,a} K(\uop{x\ind{j}}{a}, \uop{x\ind{i}}{a\ind{i}})}{\sum_{j=1}^n \sum_{a \in \MA^H} \beta_{j,a} K(\uop{x\ind{j}}{a}, \uop{x\ind{i}}{a\ind{i}})}\right)\]
and
\begin{align}
\wh J := \wh K := \Big\{ \alpha \in \RR^{n2^H}: &\left(\sum_{j,j'=1}^n \sum_{a,a' \in \MA^H} \alpha_{j,a}\alpha_{j',a'}K(\uop{x\ind{j}}{a},\uop{x\ind{j'}}{a'}) \leq B \right) \\
&\land \left(\forall i \in [n],a\in\MA^H: \sum_{j=1}^n \sum_{a' \in \MA^H} \alpha_{j,a'} K(\uop{x\ind{j}}{a'},\uop{x\ind{i}}{a}) \geq \gamma\right) \\ 
&\land \left(\forall i \in [n]: 1-\epapx \leq \sum_{j=1}^n \sum_{a,a' \in \MA^H} \alpha_{j,a'} K(\uop{x\ind{j}}{a'},\uop{x\ind{i}}{a}) \leq 1+\epapx\right)\Big\}.
\end{align}
Now, \cref{eq:kern-rho-finite} is a convex-concave min-max program with convex constraints, and the constraint sets lie in $n2^H$ dimensions. We would like to solve it using projected gradient descent-ascent---see e.g. \cite[Theorem 5.1]{bubeck2015convex}. There is one remaining technical detail: there is no evident Euclidean norm bound on the constraint sets $\wh J, \wh K$, since the kernel matrix $\Sigma \in \RR^{n2^H \times n2^H}$ that is implicit in the ellipsoid constraint (see \cref{eq:sigma-def} for the explicit definition) could be arbitrarily ill-conditioned. To fix this, we apply a change-of-basis by $\Sigma^{-1/2}$ and observe that the loss function is still Lipschitz in the new basis. This results in the program solved by $\KernRho$ in \cref{alg:kernrho}.\footnote{We omit the norm bound in the definition of the constraint set (\cref{eq:y-def}), since projected gradient descent-ascent provides implicit regularization. However, adding in the norm bound would somewhat improve the rate, at the cost of a more complex projection oracle.}

\subsection{Kernel Approximation of Autoregressive Policies}\label{subsec:kernel-apx}

In this section we define the mapping $\psi: (\ball)^H \to \RR^\NN$ that induces the kernel function $K$ (\cref{def:kernel}), and prove \cref{thm:kernel-apx}. This material is a straightforward generalization of analogous results in \cite{shalev2011learning} to our autoregressive linear setting.

\begin{definition}\label{def:psi}
Identify $\NN$ with $I^H$ where $I := \emptyset\sqcup[d]\sqcup[d]^2 \sqcup\dots$. For any tuple $\bfi \in I$, write $|\bfi|$ to denote the length of $\bfi$ (for example, if $\bfi\in [d]^{j}$, then $\abs{\bfi}=j$). 

Define mapping $\psi: (\ball)^H \to \RR^\NN$ so that for any $u_{1:H} \in (\ball)^H$, the value of $\psi(u_{1:H})$ at index $\bfi_{1:H} \in I^H$ is 
\[\psi(u_{1:H})_{\bfi_{1:H}} := \prod_{h=1}^H 2^{-|\bfi_h|/2} \prod_{k=1}^{|\bfi_h|} u_{h,\bfi_{h,k}}.\]
\end{definition}

The following lemma shows that $\psi$ induces the kernel function $K$ (and as a byproduct, that $K$ is positive semi-definite).

\begin{lemma}\label{lemma:psi-k}
For any $u_{1:H},u'_{1:H} \in (\ball)^H$, we have
\[K(u_{1:H},u'_{1:H}) = \langle \psi(u_{1:H}), \psi(u'_{1:H})\rangle.\]
\end{lemma}

\begin{proof}[\pfref{lemma:psi-k}]
We have
\begin{align}
\langle \psi(u_{1:H}),\psi(u'_{1:H})\rangle 
&= \sum_{\bfi_{1:H} \in I^H} \prod_{h=1}^H 2^{-|\bfi_h|} \prod_{k=1}^{|\bfi_h|} u_{h,\bfi_{h,k}}u'_{h,\bfi_{h,k}} \\ 
&= \prod_{h=1}^H \sum_{\bfi_h \in I}
2^{-|\bfi_h|} \prod_{k=1}^{|\bfi_h|} u_{h,\bfi_{h,k}}u'_{h,\bfi_{h,k}} \\
&= \prod_{h=1}^H \sum_{j=0}^\infty 2^{-j} \sum_{\bfi_h \in [d]^j} \prod_{k=1}^j u_{h,\bfi_{h,k}}u'_{h,\bfi_{h,k}} \\ 
&= \prod_{h=1}^H \sum_{j=0}^\infty 2^{-j} \langle u_h,u'_h\rangle^j \\ 
&= K(u_{1:H},u'_{1:H})
\end{align}
as claimed.
\end{proof}

To prove \cref{thm:kernel-apx}, we first show that for any $\theta \in \Theta$, the density $\pi_\theta(a_{1:H}\mid{}x)$ can be approximated by a product of Taylor series in the variables $\langle \uf(x,a_{1:H})_h, \theta\rangle$ (\cref{lemma:policy-poly-apx}), where the coefficients of the Taylor series satisfy a certain decay condition dependent on the norm bound $L$ from \cref{ass:linear-norm-bounds-ub}. We then show that any such product of Taylor series is a bounded linear function of $\psi(\uop{x}{a_{1:H}})$ (\cref{lemma:poly-to-rkhs}).

\begin{lemma}\label{lemma:policy-poly-apx}
Let $\epsilon \in (0,1)$ and suppose $L \geq 2$. There is a Taylor series $p(t) = \sum_{j=0}^\infty \beta_j t^j$ with $\sum_{j=0}^\infty \beta_j^2 2^j \leq (L^2H/\epsilon)^{O(L^2)}$, such that for all $(x,a_{1:H}) \in \MX\times\MA^H$ and $\theta \in \Theta$, it holds that
\[\left|\pi_\theta(a_{1:H}\mid{}x) - \prod_{h=1}^H p(\langle \uop{x}{a_{1:H}}_h,\theta/L\rangle)\right| \leq \epsilon.\]
\end{lemma}

\begin{proof}[\pfref{lemma:policy-poly-apx}]
Recall that $\uop{x}{a_{1:H}}_h = \frac{1}{2L}\left(\phi(x,a_{1:h-1},1-a_h) - \phi(x,a_{1:h})\right)$ (\cref{def:uvec}), so that
\begin{align}
\pi_\theta(a_{1:H}\mid{}x)
&= \prod_{h=1}^H \pi_{\theta,h}(a_h\mid{}x,a_{1:h-1}) \\ 
&= \prod_{h=1}^H \frac{\exp(\langle \phi(x,a_{1:h}),\theta\rangle)}{\exp(\langle \phi(x,a_{1:h-1},1-a_h),\theta\rangle) + \exp(\langle \phi(x,a_{1:h}),\theta\rangle)} \\
&= \prod_{h=1}^H \sigma\left(\left\langle \uop{x}{a_{1:H}}_h, \theta/L\right\rangle\right),
\end{align}
where $\sigma: \RR \to (0,\infty)$ is defined by $\sigma(z) = 1/(1 + e^{2L^2 z})$. Notice that the argument of $\sigma$ above lies in $[-1,1]$, by \cref{ass:linear-norm-bounds-ub}.

By \cite[Lemma 2.5]{shalev2011learning}, there is a Taylor series $p$ satisfying the stated coefficient bound, with 
\[|\sigma(z) - p(z)| \leq \epsilon/(2H)\]
for all $z \in [-1,1]$. Since $\sigma(z) \in (0,1)$ for all $z \in \RR$, it follows that
\begin{align} 
&\left|\prod_{h=1}^H \sigma(\langle \uop{x}{a_{1:H}}_h, \theta'\rangle) - \prod_{h=1}^H p(\langle \uop{x}{a_{1:H}}_h, \theta'\rangle)\right| \\
&\leq 
\sum_{h=1}^H \left|\prod_{k=1}^{h-1} \sigma(\langle \uop{x}{a_{1:H}}_k, \theta'\rangle)\right| |\sigma(\langle \uop{x}{a_{1:H}}_h, \theta'\rangle) - p(\langle \uop{x}{a_{1:H}}_h, \theta'\rangle)| \left|\prod_{k=h+1}^H p(\langle \uop{x}{a_{1:H}}_k, \theta'\rangle)\right| \\ 
&\leq \sum_{h=1}^H \frac{\epsilon}{2H} (1 + \epsilon/(2H))^H \\ 
&\leq \epsilon,
\end{align}
where we have written $\theta' := \theta/L$.
\end{proof}

\begin{lemma}\label{lemma:poly-to-rkhs}
Let $B>0$ and let $p: \RR \to \RR$ be a Taylor series $p(t) = \sum_{j=0}^\infty \beta_j t^j$ with $\sum_{j=0}^\infty \beta_j^2 2^j \leq B$. For any $\theta \in \Theta$ there is some $v_\theta \in \RR^\NN$ such that for all $u_{1:H} \in (\ball)^H$,
\[\langle v_\theta,\psi(u_{1:H})\rangle = \prod_{h=1}^H p(\langle u_h, \theta/L\rangle).\]
Moreover, $\norm{v_\theta}_2^2 \leq B^H.$
\end{lemma}

\begin{proof}[\pfref{lemma:poly-to-rkhs}]Write $\theta' := \theta/L \in \ball$. Recall that we identified $\NN$ with $I^H$ where $I = \emptyset\sqcup[d]\sqcup[d]^2 \sqcup\dots$, and that for any $\bfi\in I$ we write $|\bfi|$ to denote the length of $\bfi$. Define $v_\theta$ at index $\bfi_{1:H} \in I^H$ to have value 
\[(v_\theta)_{\bfi_{1:H}} := \prod_{h=1}^H 2^{|\bfi_h|/2} \beta_{|\bfi_h|} \prod_{k=1}^{|\bfi_h|} \theta'_{ \bfi_{h,k}}.\]
Then for any $u_{1:H} \in (\ball)^H$,
\begin{align}
\langle v_\theta,\psi(u_{1:H})\rangle
&= \sum_{\bfi_{1:H}\in I^H} \prod_{h=1}^H \beta_{|\bfi_h|} \prod_{k=1}^{|\bfi_h|} \theta'_{\bfi_{h,k}} u_{h,\bfi_{h,k}} \\ 
&= \prod_{h=1}^H \sum_{j=0}^\infty \beta_j \sum_{\bfi_h \in [d]^j} \prod_{k=1}^j \theta'_{\bfi_{h,k}} u_{h,\bfi_{h,k}} \\ 
&= \prod_{h=1}^H \sum_{j=0}^\infty \beta_j \langle \theta',u_h\rangle^j \\ 
&= \prod_{h=1}^H p(\langle \theta',u_h\rangle).
\end{align}
Similarly,
\begin{align}
\norm{v_\theta}_2^2
&= \sum_{\bfi_{1:H}\in I^H} \prod_{h=1}^H 2^{|\bfi_h|} \beta_{|\bfi_h|}^2 \prod_{k=1}^{|\bfi_h|} (\theta'_{\bfi_{h,k}})^2 \\ 
&= \prod_{h=1}^H \sum_{j=0}^\infty 2^j \beta_j^2 \norm{\theta'}_2^{2j} \\ 
&\leq B^H
\end{align}
where the final inequality uses the fact that $\norm{\theta'}_2 \leq 1$.
\end{proof}

The proof of \cref{thm:kernel-apx} is now straightforward from the above lemmas.

\begin{proof}[Proof of \cref{thm:kernel-apx}]
Let $\epsilon > 0$. By \cref{lemma:policy-poly-apx}, there is a Taylor series $p(t) = \sum_{j=0}^\infty \beta_j t^j$ with $\sum_{j=0}^\infty \beta_j^{2j} 2^j \leq (L^2 H2^H/\epsilon)^{O(L^2)}$, such that for all $a_{1:H} \in \MA^H$ and $\theta \in \Theta$, it holds that
\[\left|\pi_\theta(a_{1:H}\mid{}x) - \prod_{h=1}^H p(\langle \uop{x}{a_{1:H}}_h,\theta/L\rangle)\right| \leq \epsilon/2^H.\]
By \cref{lemma:poly-to-rkhs} applied to $p$, for every $\theta\in\Theta$ there is some $v_\theta\in\RR^\NN$ such that $\norm{v_\theta}_2^2 \leq (L^2 H/\epsilon)^{O(L^2 H)}$ and, for all $u_{1:H} \in (\ball)^H$,
\[\langle v_\theta,\psi(u_{1:H})\rangle = \prod_{h=1}^H p(\langle u_h, \theta/L\rangle).\]
It follows that for any $\theta \in \Theta$ and $(x,a_{1:H}) \in \MX\times\MA^H$, since $\uop{x}{a_{1:H}} \in (\ball)^H$ (by \cref{ass:linear-norm-bounds-ub}),
\[\left|\pi_\theta(a_{1:H}\mid{}x) - \langle v_\theta, \psi(\uop{x}{a_{1:H}})\rangle)\right| \leq \epsilon/2^H.\]
The result now follows from summing over $a_{1:H} \in \MA^H$.
\end{proof}

\subsection{Statistical Analysis for $\KernRho$}\label{subsec:kernrho-statistical}

In this section we prove \cref{lemma:minmax-ub,lemma:hellinger-by-loss}, which together show that if we can approximately solve \cref{eq:kern-rho-second-attempt}, then we achieve the statistical guarantee required for \cref{thm:kern-rho-main}. For purposes of the analysis (particularly since $\KernRho$ effectively solves a relaxation of \cref{eq:kern-rho-second-attempt} and its solution may not lie in $\wh V$), it is convenient to define analogues of $\wh V$ (the infinite-dimensional constraint set) and $\wh J$ (the finite-dimensional constraint set before the change-of-basis) with different parameter choices:
\begin{align} 
\wh V(B',\gamma',\epapx') := \Big\{v \in \RR^\NN: \left(\norm{v}_2^2 \leq B'\right) &\land \left(\forall i \in [n], a \in \MA^H: \langle v, \psi(\uop{x\ind{i}}{a})\rangle \geq \gamma'\right) \\ 
&\land \left(\forall i \in [n]: 1-\epapx' \leq \langle v, s\ind{i}\rangle \leq 1+\epapx'\right)\Big\}
\end{align}
where as before, $s\ind{i} = \sum_{a\in\MA^H} \psi(\uop{x\ind{i}}{a})$, so that $\wh V(B,\gamma,\epapx) = \wh V$, and
\begin{align}
\wh J(B',\gamma',\epapx') := \Big\{ \alpha \in \RR^{n2^H}: &\left(\sum_{j,j'=1}^n \sum_{a,a' \in \MA^H} \alpha_{j,a}\alpha_{j',a'}K(\uop{x\ind{j}}{a},\uop{x\ind{j'}}{a'}) \leq B' \right) \\
&\land \left(\forall i \in [n],a\in\MA^H: \sum_{j=1}^n \sum_{a' \in \MA^H} \alpha_{j,a'} K(\uop{x\ind{j}}{a'},\uop{x\ind{i}}{a}) \geq \gamma'\right) \\ 
&\land \left(\forall i \in [n]: 1-\epapx' \leq \sum_{j=1}^n \sum_{a,a' \in \MA^H} \alpha_{j,a'} K(\uop{x\ind{j}}{a'},\uop{x\ind{i}}{a}) \leq 1+\epapx'\right)\Big\},
\end{align}
so that $\wh J(B,\gamma,\epapx) = \wh J$.

\paragraph{Rounding to a true conditional distribution} For any $v \in \wh V$, we may consider the function
\begin{equation} f^v(a_{1:H}\mid{}x) := \langle v, \psi(\uop{x}{a_{1:H}})\rangle.\label{eq:fv}\end{equation}
Under the constraints of $\wh V$, it holds for each observed context $x\ind{i}$ that $f^v(\cdot\mid{}x\ind{i})$ is $\epapx$-close to a valid distribution, and in particular to some distribution with densities lower bounded by $\gamma$. However, it may not be close for all $x\in\MX$. Moreover, $\KernRho$ ultimately needs to output (a sampler for) a valid conditional distribution (policy). Below, we define $\pibar^v$ as the closest conditional distribution to $f^v$ with all densities lower bounded by $\gamma$. 
It is most convenient to work with this object throughout the analysis. 

\begin{definition}\label{def:pibarv}
For each $v \in \RR^\NN$, define 
\[\pibar^v \in \argmin_{\pi: \MX \to \Delta_\gamma(\MA^H)} \EE_x\brk*{ \sum_{a \in \MA^H} \left| \langle v, \psi(u(x,a_{1:h-1})_{h=1}^H)\rangle - \pi(a_{1:H}\mid{}x)\right|}\]
where $\Delta_\gamma(\MA^H)$ denotes the set of distributions $p \in \Delta(\MA^H)$ such that $p(a) \geq \gamma$ for all $a \in \MA^H$.
\end{definition}

\begin{remark}
As we will see later, constructing $\pibar^v$ (for appropriately represented $v$) is not as computationally intractable as it looks; essentially, $\pibar^v$ can be computed on a context-by-context basis.
\end{remark}

\subsubsection{Feasibility}

The following lemma shows that for the optimal choice of $v \in \wh V$, the Hellinger distance of $\pibar^v$ from $\pistar$ is not much larger than the best-in-class Hellinger distance. The proof uses \cref{thm:kernel-apx}, the definition of the constraint set $\wh V$ from \cref{eq:kern-rho-second-attempt}, and the fact that any policy $\pi \in \Pi$ has conditional densities bounded away from $0$.\loose

\begin{lemma}\label{lemma:best-to-pibar}
Suppose that $B \geq (L^2H 2^{2H+2}/\epapx)^{C_{\ref{thm:kernel-apx}}L^2H}$ and $\gamma \leq (e^{-2L^2 H} - \epapx)2^{-H-1}$. Then
\[\min_{v \in \wh V} \Dhels{\BP^{\pistar}}{\BP^{\pibar^v}} \leq 2\min_{\theta \in \Theta} \Dhels{\BP^{\pistar}}{\BP^{\pi_\theta}} + 4\epapx.\]
Moreover, the set $\wh V(B/2, 2\gamma, \epapx/2)$ is non-empty.
\end{lemma}

\begin{proof}[\pfref{lemma:best-to-pibar}]
Pick $\theta^\star \in \argmin_{\theta \in \Theta} \Dhels{\BP^{\pistar}}{\BP^{\pi_\theta}}$. By \cref{thm:kernel-apx}, there is some $v^\star \in \RR^\NN$ such that $\norm{v^\star}_2^2 \leq (L^2H2^{2H+1}/\epapx)^{C_{\ref{thm:kernel-apx}}L^2H}$ and, for all $x \in \MX$ and $a \in \MA^H$,
\begin{equation} \left|\pi_{\theta^\star}(a_{1:H}\mid{}x) - \langle v^\star, \psi(\uop{x}{a_{1:H}})\rangle\right| \leq \epapx 2^{-H-1}.\label{eq:pithetastar-to-v}
\end{equation}
By the lemma assumption, $(L^2H2^{2H+1}/\epapx)^{C_{\ref{thm:kernel-apx}}L^2H} \leq B/2$. By \cref{lemma:softmax-density-lb}, for all $x \in \MX$ and $a \in \MA^H$, we have $\pi_{\theta^\star}(a_{1:H} \mid{} x) \geq e^{-2L^2 H} 2^{-H}$, so $\langle v^\star, \psi(\uop{x}{a_{1:H}})\rangle \geq (e^{-2L^2 H} - \epapx) 2^{-H} \geq 2\gamma$ by the lemma assumption. Moreover for each $x \in \MX$, since $\pi_{\theta^\star}(\cdot\mid{}x)$ is a distribution, \cref{eq:pithetastar-to-v} implies that
\[1-\frac{\epapx}{2} \leq \sum_{a_{1:H} \in \MA^H} \langle v^\star, \psi(\uop{x}{a_{1:H}})\rangle \leq 1 + \frac{\epapx}{2}.\]
Thus, $v^\star \in \wh V(B/2,2\gamma,\epapx/2) \subseteq \wh V$, and $\pi_{\theta^\star}(\cdot\mid{}x) \in \Delta_\gamma(\MA^H)$ (\cref{def:pibarv}) for all $x \in \MX$. The latter means that by definition of $\pibar^{v^\star}$,
\begin{align} 
&\EE_x \sum_{a_{1:H} \in \MA^H} \left| \langle v^\star, \psi(\uop{x}{a_{1:H}})\rangle - \pibar^{v^\star}(a_{1:H}\mid{}x)\right|\\ &\leq \EE_x \sum_{a_{1:H} \in \MA^H} \left| \langle v^\star, \psi(\uop{x}{a_{1:H}})\rangle - \pi_{\theta^\star}(a_{1:H}\mid{}x)\right| \\
&\leq \epapx.\label{eq:vstar-nearness}\end{align}
Finally, we compute that
\begin{align}
\Dhels{\BP^{\pistar}}{\BP^{\pibar^{v^\star}}}
&\leq 2\Dhels{\BP^{\pistar}}{\BP^{\pi_{\theta^\star}}} + 2\Dhels{\BP^{\pi_{\theta^\star}}}{\BP^{\pibar^{v^\star}}} \\
&= 2\Dhels{\BP^{\pistar}}{\BP^{\pi_{\theta^\star}}} + 2\EE_x \brk*{\sum_{a_{1:H} \in \MA^H} \left(\sqrt{\pi_{\theta^\star}(a_{1:H}\mid{}x)} - \sqrt{\pibar^{v^\star}(a_{1:H}\mid{}x)}\right)^2} \\ 
&\leq 2\Dhels{\BP^{\pistar}}{\BP^{\pi_{\theta^\star}}} + 2\EE_x\brk*{\sum_{a_{1:H} \in \MA^H} \left|\pi_{\theta^\star}(a_{1:H}\mid{}x) - \pibar^{v^\star}(a_{1:H}\mid{}x)\right|} \\
&\leq 2\Dhels{\BP^{\pistar}}{\BP^{\pi_{\theta^\star}}} + 4\epapx
\end{align}
where the final inequality is by \cref{eq:vstar-nearness} and the triangle inequality. 
\end{proof}

\begin{lemma}\label{lemma:softmax-density-lb}
For any $\theta \in \Theta$, for all $x \in \MX$ and $a_{1:H} \in \MA^H$, it holds that
\[\pi_\theta(a_{1:H}\mid{}x) \geq e^{-2L^2 H} 2^{-H}.\]
\end{lemma}

\begin{proof}[\pfref{lemma:softmax-density-lb}]
Note that $\pi_\theta(a_{1:H}\mid{}x) = \prod_{h=1}^H \pi_{\theta,h}(a_h\mid{}x,a_{1:h-1})$. The lemma is therefore a consequence of \cref{lemma:auto-linear-density-bound} with $\Bdot := L^2$ and $|\MA| = 2$.
\end{proof}

\subsubsection{Generalization}

Next, we prove that the empirical loss from \cref{eq:kern-rho-second-attempt} concentrates near the population loss of the \emph{rounded} policies $\pibar^v,\pibar^w$, uniformly over $v,w \in \wh V(\Blarge,\gamma,\epapx)$ (which is a quantitative relaxation of the constraint sets $\wh V = \wh W$). In this section, we assume $(x\ind{i},a_{1:H}\ind{i})_{i=1}^n$ are i.i.d. trajectories from $\BP^{\pistar}$, and we write $u\ind{i} := \uop{x\ind{i}}{a_{1:H}\ind{i}}$. We also fix a parameter $\Blarge \geq 1$. Note that $\wh V$ and the relaxation $\wh V(\Blarge,\gamma,\epapx)$ are random sets, since they are defined in terms of the data $(x\ind{i},a_{1:H}\ind{i})_{i=1}^n$.

\begin{lemma}\label{lemma:tau-conc}
There is a constant $C_{\ref{lemma:tau-conc}}>0$ so that the following holds. Let $\delta \in (0,1/2)$. If $n \geq C_{\ref{lemma:tau-conc}} 2^H \Blarge \log(1/\delta)/(\gamma^3 \epstat^2)$, then with probability at least $1-\delta$, it holds for any $v,w \in \wh V(\Blarge,\gamma,\epapx)$ that
\[\left|\frac{1}{n} \sum_{i=1}^n \tau\left(\frac{\langle v,\psi(u\ind{i})\rangle}{\langle w, \psi(u\ind{i})\rangle}\right) - \En^{\pistar}\left[ \tau\left(\frac{\pibar^{v}(a_{1:H}\mid{}x)}{\pibar^w(a_{1:H}\mid{}x)}\right)\right]\right| \lesssim \frac{\epapx+\epstat}{\gamma^{3/2}}.\]
\end{lemma}

To prove \cref{lemma:tau-conc}, we start by showing that with high probability, every policy $\pibar^v$ is close to the approximate policy $f^v$ from \cref{eq:fv}, on average over contexts $x$. Since the constraints enforce closeness on observed contexts, this follows from standard Rademacher bounds applied to the distance function ``context $x$ maps to the distance of $f^v(\cdot\mid{}x)$ from $\Delta_\gamma(\MA^H)$''.

\begin{lemma}\label{lemma:tv-to-simplex}
There is a constant $C_{\ref{lemma:tv-to-simplex}}>0$ so that the following holds. Let $\delta \in (0,1/2)$. If $n \geq C_{\ref{lemma:tv-to-simplex}} 2^{4H} \Blarge \log(1/\delta) / \epstat^2$, then with probability at least $1-\delta$, it holds for all $v \in \wh V(\Blarge,\gamma,\epapx)$ that
\[\EE_x\brk*{ \sum_{a_{1:H} \in \MA^H} \left| \langle v, \psi(\uop{x}{a_{1:H}})\rangle - \pibar^v(a_{1:H}\mid{}x)\right|} \leq \epapx + \epstat.\]
\end{lemma}

\begin{proof}[\pfref{lemma:tv-to-simplex}]
For each $a \in \MA^H$, define a function class $\MF_a := \{f_{v,a}: \MX \to \RR \mid v \in \RR^\NN, \norm{v}_2^2 \leq \Blarge\}$
where
\[f_{v,a}(x) := \langle v, \psi(\uop{x}{a})\rangle.\]
For any $u = \uop{x}{a}$, we have $\norm{\psi(u)}_2^2 = K(u,u) = \prod_{h=1}^H \frac{1}{1-\frac{1}{2}\norm{\uop{x}{a}_h}_2^2} \leq 2^H$ since $\norm{\uop{x}{a}_h}_2 \leq 1$ for each $h \in [H]$. Thus, the Gaussian complexity (cf. \cref{def:gaussian_complexity}) of $\MF_a$ is bounded as $\MG_n(\MF_a) \leq \sqrt{2^{H+1}\Blarge/n}$. Now define $\iota: \RR^{\MA^H} \to \RR$ by \[\iota(z) := \min_{\mu \in \Delta_\gamma(\MA^H)} \sum_{a \in \MA^H} |z - \mu(a)|.\] Let $\MF$ be the class of functions $\{f_v: \MX \to \RR: v \in \RR^\NN, \norm{v}_2^2 \leq \Blarge\}$ where
\[f_v(x) := \iota(f_{v,a}(x): a \in \MA^H) = \min_{\mu \in \Delta_\gamma(\MA^H)} \sum_{a \in \MA^H} \left| f_{v,a}(x) - \mu(a)\right|.\]
Since $\iota$ is $\sqrt{|\MA^H|} = 2^{H/2}$-Lipschitz with respect to the Euclidean norm, it follows from \cref{lem:rc-composition} that $\MG_n(\MF) \leq 2^{2H+1} \sqrt{\Blarge/n}$. Moreover, for each $v \in \RR^\NN$ with $\norm{v}_2^2 \leq \Blarge$ and each $x \in \MX$, we know that $|f_v(x)| \leq 1 + \sum_{a \in \MA^H} |f_{v,a}(x)| \leq 1+\sqrt{2^H \Blarge}$. By \cref{lem:unif-conv} and assumption on $n$, it holds with probability at least $1-\delta$ that for all $v \in \RR^\NN$ with $\norm{v}_2^2 \leq \Blarge$,
\begin{equation} \left|\frac{1}{n}\sum_{i=1}^n f_{v}(x\ind{i}) - \EE_x [f_{v}(x)]\right| \leq \epstat.
\label{eq:dist-to-simplex-conc}
\end{equation}
Condition on this event and fix $v \in \wh V(\Blarge,\gamma,\epapx)$. We know $\norm{v}_2^2 \leq \Blarge$ so the bound \eqref{eq:dist-to-simplex-conc} holds. Moreover, for each $i \in [n]$, we know from the definition of $\wh V(\Blarge,\gamma,\epapx)$ that $f_{v,a}(x\ind{i}) \geq \gamma$ for all $a \in \MA^H$, and similarly $\sum_{a \in \MA^H} f_{v,a}(x\ind{i}) \in [1-\epapx,1+\epapx]$. Thus, there is $\mu \in \Delta_\gamma(\MA^H)$ with $\sum_{a \in \MA^H} |f_{v,a}(x\ind{i}) - \mu(a)| \leq \epapx$. Hence, $f_v(x\ind{i}) \leq \epapx$. Since this holds for all $i \in [n]$, invoking \eqref{eq:dist-to-simplex-conc} gives that $\EE_x[f_{v}(x)] \leq \epapx + \epstat$. Since for each $x \in \MX$, $\pibar^v(\cdot\mid{}x)$ minimizes $\sum_{a \in \MA^H} |f_{v,a}(x) - \mu(a)|$ over all $\mu \in \Delta_\gamma(\MA^H)$, and this minimum value is exactly $f_v(x)$, it follows that
\[\EE_x\brk*{\sum_{a \in \MA^H} |f_{v,a}(x) - \pibar^v(a\mid{}x)|} \leq \epapx + \epstat\]
as claimed.
\end{proof}

Next, we would like to prove a statement of the form ``for all $v, w \in \wh V$ with $\norm{v}_2^2,\norm{w}_2^2 \leq \Blarge$, the empirical loss at $v,w$ concentrates near the population loss''. Unfortunately, since $\tau$ is only defined on $(0,\infty)$, the empirical loss and the naively-defined population loss are not well-defined on this entire parameter space, and moreover the subspace where they are well-defined is data-dependent. Instead, we mollify the loss so that it is well-defined and Lipschitz on the entire parameter space (and equals the original loss for all $v,w \in \wh V(\Blarge,\gamma,\epapx)$). We then invoke standard generalization bounds for Rademacher complexity to show that the mollified empirical loss concentrates (\cref{lemma:mol-tau-conc}).

\begin{definition}
Define $F: \RR^2 \to \RR$ by
\[F(x,y) := \tau\left(\frac{x}{y}\right)\mol{\gamma}{x}\mol{\gamma}{y}\]
where 
\[\mol{\gamma}{x} = \begin{cases} 1 & \text{ if } x \geq \gamma \\ 0 & \text{ if } x \leq \gamma/2 \\ \frac{2x}{\gamma}-1 & \text{ if } \gamma/2 < x < \gamma \end{cases}.\]
\end{definition}

\begin{lemma}\label{lemma:mol-tau-lipschitz}
The function $F$ is $6\gamma^{-3/2}$-Lipschitz with respect to the $\ell_1$ norm.
\end{lemma}

\begin{proof}[\pfref{lemma:mol-tau-lipschitz}]
Note that $\tau(z) \in [-1,1]$ and $|\tau'(z)| \leq -1/\sqrt{z}$ for all $z>0$. Thus, for any $x,y \in\bbR$ %
at which $F$ is differentiable, we have
\begin{align}
\left|\frac{\partial}{\partial x} F(x,y)\right|
&= \left|\frac{1}{y} \tau'\left(\frac{x}{y}\right) \mol{\gamma}{x}\mol{\gamma}{y} + \tau\left(\frac{x}{y}\right) \frac{2 \cdot \mathbbm{1}[\gamma/2 \leq x \leq \gamma]}{\gamma} \mol{\gamma}{y} \right | \\ 
&\leq \frac{1}{\gamma^{3/2}} + \frac{2}{\gamma}.
\end{align}
Since $F(x,y) = -F(y,x)$, the same bound holds on $\left|\frac{\partial}{\partial y} F(x,y)\right|$. The lemma follows.
\end{proof}

\begin{lemma}\label{lemma:mol-tau-conc}
There is a constant $C_{\ref{lemma:mol-tau-conc}}>0$ so that the following holds. Let $\delta \in (0,1/2)$. If $n \geq C_{\ref{lemma:mol-tau-conc}} 2^H \Blarge \log(1/\delta)/(\gamma^3 \epstat^2)$, then with probability at least $1-\delta$, it holds for all $v,w \in \RR^\NN$ with $\norm{v}_2^2 , \norm{w}_2^2 \leq \Blarge$ that
\begin{align}
\left|\frac{1}{n} \sum_{i=1}^n F(\langle v,\psi(u\ind{i})\rangle,\langle w, \psi(u\ind{i})\rangle - \En^{\pistar} \left[F(\langle v,\psi(\uop{x}{a_{1:H}})\rangle,\langle w, \psi(\uop{x}{a_{1:H}})\rangle) \right]\right| \leq \epstat
\label{eq:F-conc}
\end{align}
\end{lemma}

\begin{proof}[\pfref{lemma:mol-tau-conc}]
Define a function class $\MF := \{f_v: \MX \times \MA^H \to \RR \mid v \in \RR^\NN, \norm{v}_2^2 \leq \Blarge\}$
where
\[f_v(x,a) := \langle v, \psi(\uop{x}{a})\rangle.\]
For any $u = \uop{x}{a}$, we have $\norm{\psi(u)}_2^2 = K(u,u) = \prod_{h=1}^H \frac{1}{1-\frac{1}{2}\norm{\uop{x}{a}_h}_2^2} \leq 2^H$, since $\norm{\uop{x}{a}_h}_2 \leq 1$ for each $h \in [H]$. Thus, the Gaussian complexity of $\MF$ is bounded as $\MG_n(\MF) \leq \sqrt{2^{H+1} \Blarge/n}$. Now define the function class \[\mathcal{\widetilde F} := \{\widetilde f_{v,w}: \MX \times \MA^H \to \RR \mid v,w \in \RR^\NN, \norm{v}_2^2 \leq \Blarge, \norm{w}_2^2 \leq \Blarge\}\]
where
\[\widetilde f_{v,w}(x,a) := F(f_v(x,a),f_w(x,a)).\]
By \cref{lemma:mol-tau-lipschitz,lem:rc-composition} we have $\MG_n(\mathcal{\widetilde{F}}) \leq 12\sqrt{2^{H+2} \Blarge/(n\gamma^3)}$. Moreover, by definition of $F$, we know that all functions in $\mathcal{\widetilde F}$ have range in $[-1,1]$. Thus, by \cref{lem:unif-conv} and choice of $n$, the bound \eqref{eq:F-conc} holds for all $v,w \in \RR^\NN$ with $\norm{v}_2^2,\norm{w}_2^2 \leq \Blarge$ with probability at least $1-\delta$.
\end{proof}

We can now prove \cref{lemma:tau-conc} by combining \cref{lemma:tv-to-simplex,lemma:mol-tau-conc}.

\begin{proof}[Proof of \cref{lemma:tau-conc}]
By \cref{lemma:tv-to-simplex}, \cref{lemma:mol-tau-conc}, and the lemma assumption that \[n \geq C_{\ref{lemma:tau-conc}} 2^H\Blarge \log(1/\delta)/(\gamma^3\epstat^2),\] so long as $C_{\ref{lemma:tau-conc}}>0$ is a sufficiently large constant, we have with probability at least $1-\delta$ that the events of both \cref{lemma:tv-to-simplex} and \cref{lemma:mol-tau-conc} hold. Condition henceforth on the intersection of these events. For any $v,w \in \wh V(\Blarge,\gamma,\epapx)$, we have
\begin{align}
\frac{1}{n} \sum_{i=1}^n \tau\left(\frac{\langle v,\psi(u\ind{i})\rangle}{\langle w, \psi(u\ind{i})\rangle}\right)
&= \frac{1}{n}\sum_{i=1}^n F\left(\langle v,\psi(u\ind{i})\rangle,  \langle w, \psi(u\ind{i})\rangle\right) \\ 
&\leq \En^{\pi^\star}\left[ F\left(\left\langle v,\psi\left(\uop{x}{a_{1:H}}\right)\right\rangle,  \left\langle w, \psi\left(\uop{x}{a_{1:H}}\right)\right\rangle\right)\right] + \epstat \\ 
&\leq \En^{\pi^\star}\left[ F\left(\pibar^{v}(a_{1:H}\mid{}x), \pibar^w(a_{1:H}\mid{}x)\right)\right] + \epstat\\ 
&+ \frac{6}{\gamma^{3/2}}\Big(\En^{\pistar} \left|\left\langle v,\psi\left(\uop{x}{a_{1:H}}\right)\right\rangle - \pibar^{v}(a_{1:H}\mid{}x)\right| \\
&\qquad+ \left|\left\langle w,\psi\left(\uop{x}{a_{1:H}}\right)\right\rangle - \pibar^w(a_{1:H}\mid{}x)\right|\Big) \\ 
&\leq \En^{\pi^\star}\left[ F\left(\pibar^{v}(a_{1:H}\mid{}x), \pibar^w(a_{1:H}\mid{}x)\right)\right] + \epstat + \frac{12}{\gamma^{3/2}}(\epapx+\epstat)
\end{align}
where the first equality is because $v,w \in \wh V(\Blarge,\gamma,\epapx)$, so $\langle v, \psi(u\ind{i})\rangle \geq \gamma$ and $\langle w, \psi(u\ind{i})\rangle \geq \gamma$; the first inequality is by the event of \cref{lemma:mol-tau-conc}; the second inequality is by triangle inequality and \cref{lemma:mol-tau-lipschitz}; and the third inequality is by the event of \cref{lemma:tv-to-simplex}. Now
\begin{align}
\En^{\pistar}\left[ F\left(\pibar^{v}(a_{1:H}\mid{}x), \pibar^w(a_{1:H}\mid{}x)\right)\right]
&= \En^{\pistar} \left[\tau\left(\frac{\pibar^{v}(a_{1:H}\mid{}x)}{\pibar^w(a_{1:H}\mid{}x)}\right)\right]
\end{align}
since $\pibar^v(\cdot\mid{}x),\pibar^w(\cdot\mid{}x) \in \Delta_\gamma(\MA^H)$ for all $x \in \MX$. This proves one direction of the claimed inequality, and the other direction follows by a symmetric argument
\end{proof}

\subsubsection{Completing the Statistical Analysis}

We can now complete the statistical analysis of \cref{eq:kern-rho-second-attempt} using \cref{lemma:best-to-pibar}, \cref{lemma:tau-conc}, and the classical inequality for the $\rho$-estimator (\cref{lemma:rho-estimator-bounds}). The following lemmas together show that for any $v \in \wh V(\Blarge,\gamma,\epapx)$ with near-optimal min-max loss, the policy $\pibar^v$ achieves near-optimal error (as measured by trajectory-level Hellinger distance).

\begin{lemma}\label{lemma:minmax-ub}
Suppose that $B \geq (L^2H 2^{2H+2}/\epapx)^{C_{\ref{thm:kernel-apx}}L^2H}$ and $\gamma \leq (e^{-2L^2 H} - \epapx)2^{-H-1}$. In the event of \cref{lemma:tau-conc}, it holds that
\[\min_{v \in \wh V} \max_{w \in \wh W(\Blarge,\gamma,\epapx)} \frac{1}{n} \sum_{i=1}^n \tau\left(\frac{\langle v,\psi(u\ind{i})\rangle}{\langle w, \psi(u\ind{i})\rangle}\right) \leq 8\min_{\theta \in \Theta} \Dhels{\BP^{\pistar}}{\BP^{\pi_\theta}} + O\left(\frac{\epapx + \epstat}{\gamma^{3/2}}\right)\]
\end{lemma}

\begin{proof}[\pfref{lemma:minmax-ub}]
Condition on the event of \cref{lemma:tau-conc}, so that for any $v,w \in \wh V(\Blarge,\gamma,\epapx)$ we have
\begin{align}
\frac{1}{n} \sum_{i=1}^n \tau\left(\frac{\langle v,\psi(u\ind{i})\rangle}{\langle w, \psi(u\ind{i})\rangle}\right)
&\leq \En^{\pistar}\left[ \tau\left(\frac{\pibar^{v}(a_{1:H}\mid{}x)}{\pibar^w(a_{1:H}\mid{}x)}\right)\right]+ O\left(\frac{\epapx+\epstat}{\gamma^{3/2}}\right) \\ 
&\leq 4\Dhels{\BP^{\pistar}}{\BP^{\pibar^v}} - \frac{3}{8} \Dhels{\BP^{\pistar}}{\BP^{\pibar^w}} + O\left(\frac{\epapx+\epstat}{\gamma^{3/2}}\right)\\ 
&\leq 4\Dhels{\BP^{\pistar}}{\BP^{\pibar^v}}+ O\left(\frac{\epapx+\epstat}{\gamma^{3/2}}\right)
\end{align}
where the second inequality is by \cref{lemma:rho-estimator-bounds} and the basic equality $\frac{\BP^{\pibar^v}(x,a_{1:H})}{\BP^{\pibar^w}(x,a_{1:H})} = \frac{\pibar^v(a_{1:H}\mid{}x)}{\pibar^w(a_{1:H}\mid{}x)}$. Thus, for any $v \in \wh V$, we have
\begin{equation}
\max_{w \in \wh W(\Blarge,\gamma,\epapx)} \frac{1}{n} \sum_{i=1}^n \tau\left(\frac{\langle v,\psi(u\ind{i})\rangle}{\langle w, \psi(u\ind{i})\rangle}\right) \leq 4\Dhels{\BP^{\pistar}}{\BP^{\pibar^v}} + O\left(\frac{\epapx+\epstat}{\gamma^{3/2}}\right).\label{eq:sup-loss-bound}
\end{equation}

Minimizing over $v \in \wh V$ and applying \cref{lemma:best-to-pibar} (via the assumed bounds on $B$, $\gamma$, and $\epapx$), we get
\begin{equation}
\min_{v \in \wh V} \max_{w \in \wh W(\Blarge,\gamma,\epapx)} \frac{1}{n} \sum_{i=1}^n \tau\left(\frac{\langle v,\psi(u\ind{i})\rangle}{\langle w, \psi(u\ind{i})\rangle}\right) \leq 8\min_{\theta \in \Theta} \Dhels{\BP^{\pistar}}{\BP^{\pi_\theta}} + O\left(\frac{\epapx+\epstat}{\gamma^{3/2}}\right)
\end{equation}
which completes the proof.
\end{proof}

\begin{lemma}\label{lemma:hellinger-by-loss}
Suppose that $B \geq (L^2H 2^{2H+2}/\epapx)^{C_{\ref{thm:kernel-apx}}L^2H}$ and $\gamma \leq (e^{-2L^2H} - \epapx)2^{-H-1}$. In the event of \cref{lemma:tau-conc}, for all $v \in \wh V(\Blarge,\gamma,\epapx)$,
\[\Dhels{\BP^{\pistar}}{\BP^{\pibar^v}} \leq \frac{64}{3}\min_{\theta \in \Theta} \Dhels{\BP^{\pistar}}{\BP^{\pi_\theta}} + \frac{8}{3}\max_{w \in \wh W} \frac{1}{n} \sum_{i=1}^n \tau\left(\frac{\langle v,\psi(u\ind{i})\rangle}{\langle w, \psi(u\ind{i})\rangle}\right) + O\left(\frac{\epapx+\epstat}{\gamma^{3/2}}\right).\]
\end{lemma}

\begin{proof}[\pfref{lemma:hellinger-by-loss}]
Condition on the event of \cref{lemma:tau-conc}, so that for any $v,w \in \wh V(\Blarge,\gamma,\epapx)$ we have
\begin{align}
\frac{1}{n} \sum_{i=1}^n \tau\left(\frac{\langle v,\psi(u\ind{i})\rangle}{\langle w, \psi(u\ind{i})\rangle}\right)
&\geq \En^{\pistar}\left[ \tau\left(\frac{\pibar^{v}(a_{1:H}\mid{}x)}{\pibar^w(a_{1:H}\mid{}x)}\right)\right]- O\left(\frac{\epapx+\epstat}{\gamma^{3/2}}\right) \\ 
&\geq \frac{3}{8}\Dhels{\BP^{\pistar}}{\BP^{\pibar^v}} - 4 \Dhels{\BP^{\pistar}}{\BP^{\pibar^w}} - O\left(\frac{\epapx+\epstat}{\gamma^{3/2}}\right)
\end{align}
where the second inequality is by \cref{lemma:rho-estimator-bounds}. Thus, for any $v \in \wh V(\Blarge,\gamma,\epapx)$, we have
\begin{align}
\max_{w \in \wh W} \frac{1}{n} \sum_{i=1}^n \tau\left(\frac{\langle v,\psi(u\ind{i})\rangle}{\langle w, \psi(u\ind{i})\rangle}\right)
&\geq \frac{3}{8}\Dhels{\pistar}{\pibar^v} - 4\min_{w \in \wh W} \Dhels{\BP^{\pistar}}{\BP^{\pibar^w}} - O\left(\frac{\epapx+\epstat}{\gamma^{3/2}}\right) \\ 
&\geq \frac{3}{8}\Dhels{\pistar}{\pibar^v} - 8\min_{\theta \in \Theta} \Dhels{\pistar}{\pi_\theta} - O\left(\frac{\epapx+\epstat}{\gamma^{3/2}}\right)
\end{align}
where the second inequality is by \cref{lemma:best-to-pibar}. Rearranging completes the proof.
\end{proof}

\subsection{Computational Analysis for $\KernRho$}\label{subsec:kernrho-computational}

We now analyze $\KernRho$ (\cref{alg:kernrho}) itself. In particular, we show that it computes a succinct representation $\wh \alpha$ of an approximately optimal solution $\wh v$ to the infinite-dimensional program defined in \cref{eq:kern-rho-second-attempt}:

\begin{lemma}\label{lemma:kernrho-minmax}
Let $\epopt>0$ and suppose that $\Blarge \geq 2C_{\ref{thm:sp-pgd}}^2 2^{H+1}B/(\gamma^{3/2}\epopt)$. Define \[\wh v := \sum_{j=1}^n \sum_{a\in\MA^H} \wh \alpha_{j,a} \psi(\uop{x\ind{j}}{a})\] where $\wh\alpha$ is the parameter computed in \cref{line:wh-alpha} of \[\KernRho((x\ind{i},a\ind{i}_{1:H})_{i=1}^n, B,\gamma,\epapx,\epopt).\]
Then $\wh v \in \wh V(\Blarge,\gamma,\epapx)$ and
\[\max_{w \in \wh W} \frac{1}{n} \sum_{i=1}^n \tau\left(\frac{\langle \wh v, \psi(\uop{x\ind{i}}{a\ind{i}})\rangle}{\langle w, \psi(\uop{x\ind{i}}{a\ind{i}})\rangle}\right) \leq \min_{v\in\wh V} \max_{w \in \wh W(\Blarge,\gamma,\epapx)} \frac{1}{n} \sum_{i=1}^n \tau\left(\frac{\langle v, \psi(\uop{x\ind{i}}{a\ind{i}})\rangle}{\langle w, \psi(\uop{x\ind{i}}{a\ind{i}})\rangle}\right) + \epopt.\]
Moreover, the time complexity of $\KernRho$ with these parameters is $\poly(n,2^H,B,1/\gamma,1/\epopt)$, and for any $x\in\MX$, $\pihat(\cdot\mid{}x)$ can be explicitly computed in time $\poly(n,2^H)$.
\end{lemma}

To prove \cref{lemma:kernrho-minmax}, we combine two representational facts---\cref{lemma:representer,lemma:reverse-representer}, which together allow translating back and forth between the infinite-dimensional space and finite-dimensional space---with \cref{lemma:optimization-loss-guarantee}, which uses convexity-concavity of $\tau$ (\cref{lemma:tau-convex-concave}) and standard guarantees for projected gradient descent-ascent to show that $\wh \alpha$ is an approximately optimal solution to the finite-dimensional program defined in \cref{eq:kern-rho-finite}.

\subsubsection{Representational Results}

The following lemma shows that any vector $\alpha$ in the finite-dimensional constraint set $J(\Blarge,\gamma,\epapx)$ corresponds a vector $v$ in the infinite-dimensional constraint set $\wh V(\Blarge,\gamma,\epapx)$.

\begin{lemma}\label{lemma:reverse-representer}
For each $\alpha \in \wh J(\Blarge,\gamma,\epapx)$, the vector $v = \sum_{j=1}^n \sum_{a \in \MA^H} \alpha_{j,a} \psi(\uop{x\ind{j}}{a})$ satisfies 
\[\langle v, \psi(\uop{x\ind{i}}{a'})\rangle = \sum_{j=1}^n \sum_{a \in \MA^H} \alpha_{j,a} K(\uop{x\ind{j}}{a},\uop{x\ind{i}}{a'})\]
for all $i \in [n]$ and $a' \in \MA^H$, and moreover $v \in \wh V(\Blarge,\gamma,\epapx)$.
\end{lemma}

\begin{proof}[\pfref{lemma:reverse-representer}]
The display equation is immediate from the fact that $\langle \psi(\cdot),\psi(\cdot)\rangle = K(\cdot,\cdot)$ (\cref{lemma:psi-k}). The fact that $v \in \wh V(\Blarge,\gamma,\epapx)$ then follows from the display equation and the definitions of $\wh J(\Blarge,\gamma,\epapx)$ and $\wh V(\Blarge,\gamma,\epapx)$. %
\end{proof}

The converse of \cref{lemma:reverse-representer} is not true; not every vector in $\wh V(\cdot,\cdot,\cdot)$ can be expressed as a linear combination of $\{\psi(\uop{x\ind{j}}{a}): j\in[n], a \in \MA^H\}$. However, for every $v$ there does exist some $v'$ that (1) can be expressed as a linear combination, and (2) is equivalent to $v$ for all intents and purposes, i.e. $\norm{v'}_2 \leq \norm{v}_2$ and $\langle v,\psi(\uop{x\ind{i}}{a})\rangle = \langle v',\psi(\uop{x\ind{i}}{a})\rangle$ for all $i\in[n]$ and $a \in \MA^H$; note that the loss function only depends on such inner products. This fact is also the basis for the Representer Theorem for ERM in an RKHS. Formally, we need the following result.

\begin{lemma}\label{lemma:representer}
For any $v \in \wh V$, there is $\alpha \in \wh J$ with 
\[\langle v, \psi(\uop{x\ind{i}}{a'})\rangle = \sum_{j=1}^n \sum_{a \in \MA^H} \alpha_{j,a} K(\uop{x\ind{j}}{a},\uop{x\ind{i}}{a'})\]
for all $i \in [n]$ and $a' \in \MA^H$.
\end{lemma}

\begin{proof}[\pfref{lemma:representer}]
This is a consequence of standard facts about Hilbert spaces. Let $Y$ be the span of the vectors $\{\psi(\uop{x\ind{j}}{a}): j \in [n], a \in \MA^H\}$ in $\ell^2$. Since $\ell^2$ is a Hilbert space and $Y$ is a closed subspace of $\ell^2$, for any $v \in \wh V$ there are $y \in Y$ and $z \in Y^\perp$ such that $v=y+z$. By definition of $Y$, there is $\alpha \in \RR^{n \times 2^H}$ such that 
\[y = \sum_{j=1}^n \sum_{a \in \MA^H} \alpha_{j,a} \psi(\uop{x\ind{j}}{a}).\]
By definition of $z$, for each $j \in [n]$ and $a \in \MA^H$ we have $\langle z, \psi(\uop{x\ind{j}}{a})\rangle = 0$, and hence
\begin{align}
\langle v, \psi(\uop{x\ind{j}}{a})\rangle
&= \langle y, \psi(\uop{x\ind{j}}{a})\rangle \\ 
&= \sum_{j'=1}^n \sum_{a'\in\MA^H} \alpha_{j',a'} \langle \psi(\uop{x^{j'}}{a'}), \psi(\uop{x\ind{j}}{a})\rangle \\  
&= \sum_{j'=1}^n \sum_{a'\in\MA^H} \alpha_{j',a'} K(\uop{x^{j'}}{a'},\uop{x\ind{j}}{a}).
\end{align}
This proves the lemma's stated equality, and it remains to show $\alpha \in \wh J$. Using the above fact and the definition of $\wh V$, we get for any $i \in [n]$ and $a \in \MA^H$ that 
\[\sum_{j=1}^n \sum_{a' \in \MA^H} \alpha_{j,a'} K(\uop{x\ind{j}}{a'},\uop{x\ind{i}}{a}) = \langle v, \psi(\uop{x\ind{i}}{a})\rangle \geq \gamma\]
and similarly
\[\sum_{j=1}^n \sum_{a,a' \in \MA^H} \alpha_{j,a'} K(\uop{x\ind{j}}{a'},\uop{x\ind{i}}{a}) = \sum_{a \in \MA^H} \langle v, \psi(\uop{x\ind{i}}{a})\rangle \in [1-\epapx,1+\epapx].\]
Finally,
\[\sum_{j,j'=1}^n \sum_{a,a' \in \MA^H} \alpha_{j,a}\alpha_{j',a'}K(\uop{x\ind{j}}{a},\uop{x^{(j')}}{a'}) = \langle y,y\rangle \leq \langle v,v\rangle \leq B.\]
where the first inequality is since $\langle y,z\rangle = 0$ and hence $\langle v,v\rangle = \langle y,y\rangle + \langle z,z\rangle$. We conclude that $\alpha \in \wh J$.
\end{proof}

\subsubsection{Optimization Guarantee for Finite-Dimensional Program}

\begin{lemma}\label{lemma:optimization-loss-guarantee}
Let $\epopt>0$ and suppose that $\Blarge \geq 2C_{\ref{thm:sp-pgd}}^2 2^{H+1}B/(\gamma^{3/2}\epopt)$. Then the parameter $\wh\alpha$ computed in \cref{line:wh-alpha} of \[\KernRho((x\ind{i},a\ind{i}_{1:H})_{i=1}^n, B,\gamma,\epapx,\epopt)\] satisfies $\wh\alpha \in \wh J(\Blarge,\gamma,\epapx)$ and
\[\max_{\beta \in \wh K} \emploss(\wh \alpha, \beta) \leq \min_{\alpha \in \wh J} \max_{\beta \in \wh K(\Blarge,\gamma,\epapx)} \emploss( \alpha, \beta) + \epopt.\]
Moreover, the time complexity of $\KernRho$ with these parameters is $\poly(n,2^H,B,1/\gamma,1/\epopt)$, and for any $x\in\MX$, $\pihat(\cdot\mid{}x)$ can be explicitly computed in time $\poly(n,2^H)$.
\end{lemma}

\begin{proof}[\pfref{lemma:optimization-loss-guarantee}]
Define $f: \MY \times \MY \to \RR$ by 
\[f(\alphatil,\betatil) := \frac{1}{n}\sum_{i=1}^n \tau\left(\frac{e_{i,a\ind{i}}\Sigma^{1/2}\alphatil}{e_{i,a\ind{i}}\Sigma^{1/2}\betatil}\right),\]
where $\Sigma$ is defined in \cref{eq:sigma-def} and $\MY$ is defined in \cref{eq:y-def}. Notice that by definition of $\Sigma$, we have $\emploss(\alpha,\beta) = f(\Sigma^{1/2}\alpha,\Sigma^{1/2}\beta)$ for all $\alpha,\beta$, where $\emploss$ is the empirical loss function defined in \cref{eq:kern-rho-finite}.

We check the conditions of \cref{thm:sp-pgd}. Observe that $\MY$ is defined by intersection of linear constraints; hence, $\MY$ is convex. By \cref{lemma:tau-convex-concave} and the fact that $e_{i,a\ind{i}} \Sigma^{1/2} \alphatil$ is a linear function of $\alphatil$, we get that $f$ is convex in $\alphatil$, and similarly that $f$ is concave in $\betatil$. Next, since $|\tau'(z)| \leq 1/\sqrt{z}$ for all $z \in (0,\infty)$, we have for all $\alphatil,\alphatil',\betatil \in \MY$ that 
\begin{align}
|f(\alphatil,\betatil) - f(\alphatil',\betatil)|
&\leq \max_{i \in [n]} \left|\tau\left(\frac{e_{i,a\ind{i}}\Sigma^{1/2}\alphatil}{e_{i,a\ind{i}}\Sigma^{1/2}\betatil}\right) - \tau\left(\frac{e_{i,a\ind{i}}\Sigma^{1/2}\alphatil'}{e_{i,a\ind{i}}\Sigma^{1/2}\betatil}\right)\right| \\ 
&\leq \max_{i \in [n]} \frac{ \left|\frac{e_{i,a\ind{i}}\Sigma^{1/2}\alphatil}{e_{i,a\ind{i}}\Sigma^{1/2}\betatil} - \frac{e_{i,a\ind{i}}\Sigma^{1/2}\alphatil'}{e_{i,a\ind{i}}\Sigma^{1/2}\betatil}\right|}{\min\left(\sqrt{\frac{e_{i,a\ind{i}}\Sigma^{1/2}\alphatil}{e_{i,a\ind{i}}\Sigma^{1/2}\betatil}}, \sqrt{\frac{e_{i,a\ind{i}}\Sigma^{1/2}\alphatil'}{e_{i,a\ind{i}}\Sigma^{1/2}\betatil}}\right)} \\ 
&\leq \frac{1}{\gamma} \sqrt{\frac{1+\epapx}{\gamma}} \cdot \max_{i \in [n]} \left|e_{i,a\ind{i}}\Sigma^{1/2}\alphatil-e_{i,a\ind{i}}\Sigma^{1/2}\alphatil'\right| \\ 
&\leq \frac{2}{\gamma} \sqrt{\frac{4}{\gamma}} \norm{\alphatil - \alphatil'}_2 \cdot \max_{i \in [n]} \sqrt{\Sigma_{(i,a\ind{i}),(i,a\ind{i})}} \\ 
&\leq \frac{2^{H+2}}{\gamma^{3/2}} \norm{\alphatil-\alphatil'}_2
\end{align}
where the third inequality uses the fact that \[\gamma \leq e_{i,a}\Sigma^{1/2} y \leq 1+\epapx\] for all $y \in \MY$, the fourth inequality uses Cauchy-Schwarz, and the final inequality uses \cref{eq:sigma-def}. Hence, $f$ is $2^{H+2}\gamma^{-3/2}$-Lipschitz in $\alphatil$, with respect to the Euclidean norm. A symmetric argument, using the fact that $\tau(1/z)=-\tau(z)$, shows that $f$ is also $2^{H+2}\gamma^{-3/2}$-Lipschitz in $\betatil$. Finally, by definition $\Oproj$ is an $\epproj$-approximate projection oracle for $\MY$ with $\epproj = 1/(16BT^4)$, and $\Ovec$ implements queries to the vector field $g(\alphatil,\betatil) = (\grad_{\alphatil} f(\alphatil,\betatil), -\grad_{\betatil} f(\alphatil,\betatil))$. Thus, applying \cref{thm:sp-pgd} with $R := B$, we get
\[\max_{\betatil\in\MY: \norm{\betatil}_2 \leq \sqrt{B}} f\left(\frac{1}{T}\sum_{t=1}^T \alphatil_t, \betatil\right) - \min_{\alphatil\in\MY: \norm{\alphatil}_2 \leq \sqrt{B}} f\left(\alphatil,\frac{1}{T}\sum_{t=1}^T \betatil_t\right) \leq C_{\ref{thm:sp-pgd}} \left(\frac{2^{H+1}\sqrt{B}}{\gamma^{3/2}\sqrt{T}}\right) \leq \epopt\] 
by choice of $T := 4C_{\ref{thm:sp-pgd}}^2 2^{2H+2} B/(\gamma^3 \epopt^2)$. Moreover, 
$\frac{1}{T}\sum_{t=1}^T \alphatil_t, \frac{1}{T}\sum_{t=1}^T \betatil_t \in \MY$ with \[\norm{\frac{1}{T}\sum_{t=1}^T \alphatil_t}_2, \norm{\frac{1}{T}\sum_{t=1}^T \betatil_t}_2 \leq C_{\ref{thm:sp-pgd}} \sqrt{BT} \leq \Blarge\] by lemma assumption. By definition of $\wh\alpha$, observe that $\frac{1}{T}\sum_{t=1}^T \alphatil_t = \Sigma^{1/2} \wh\alpha + y$ for some $y \in \ker(\Sigma)$. The fact that $\frac{1}{T}\sum_{t=1}^T \alphatil_t \in \MY$ and $\norm{\frac{1}{T}\sum_{t=1}^T \alphatil_t}_2 \leq \Blarge$ implies that $\Sigma^{1/2} \wh\alpha \in \MY$ and $\norm{\Sigma^{1/2}\wh\alpha}_2 \leq \Blarge$, and so $\wh\alpha \in \wh J(\Blarge,\gamma,\epapx)$. Similarly, there is some $\wh\beta \in \wh K(\Blarge,\gamma,\epapx)$ and $y' \in \ker(\Sigma)$ such that $\frac{1}{T}\sum_{t=1}^T\betatil_t = \Sigma^{1/2} \wh \beta + y'$.

For any $\beta \in \wh K$, we have $\Sigma^{1/2} \beta \in \MY$ and $\norm{\Sigma^{1/2}\beta}_2 \leq \sqrt{B}$, so 
\[\max_{\beta \in \wh K} \emploss(\wh\alpha,\beta) = \max_{\beta \in \wh K} f(\Sigma^{1/2} \wh\alpha, \Sigma^{1/2}\beta) = \max_{\beta \in \wh K} f\left(\frac{1}{T}\sum_{t=1}^T \alphatil_t, \Sigma^{1/2}\beta\right) \leq \max_{\betatil \in \MY:\norm{\betatil}_2\leq \sqrt{B}}  f\left(\frac{1}{T}\sum_{t=1}^T \alphatil_t, \betatil\right).\]
Next, for any $\alpha \in \wh J$, we have $\Sigma^{1/2}\alpha \in \MY$ and $\norm{\Sigma^{1/2}\alpha}_2 \leq \sqrt{B}$, so
\begin{align}
\min_{\alpha \in \wh J} \max_{\beta \in \wh K(\Blarge,\gamma,\epapx)} \emploss(\alpha,\beta) 
&\geq \min_{\alpha \in \wh J} \emploss(\alpha,\wh\beta) \\
&\geq \min_{\alpha \in \wh J} f(\Sigma^{1/2} \alpha, \Sigma^{1/2} \wh\beta)\\
&\geq \min_{\alphatil \in \MY:\norm{\alphatil}_2 \leq \sqrt{B}} f(\alphatil, \Sigma^{1/2}\wh\beta)\\ 
&= \min_{\alphatil \in \MY:\norm{\alphatil}_2 \leq \sqrt{B}} f\left(\alphatil,\frac{1}{T}\sum_{t=1}^T \betatil_t\right).
\end{align}
We conclude that 
\[\max_{\beta \in \wh K} \emploss(\wh \alpha, \beta) \leq \min_{\alpha \in \wh J} \max_{\beta \in \wh K(\Blarge,\gamma,\epapx)} \emploss(\wh \alpha, \beta) + \epopt\]
as claimed. It remains to analyze the time complexity. Excluding the final step of the algorithm (computing $\pihat$), the claimed bound is immediate from the description of $\PGD$ (\cref{alg:pgd}) together with the choice of parameter $T$ and the fact that both oracles in $\KernRho$ can be implemented in polynomial time. In particular, \cref{lemma:projection-time} shows that $\Oproj$ can be implemented in polynomial time (since it is straightforward to check that the queries to the projection oracle will have polynomially-bounded norm), and it is evident from direct differentation that $\Ovec$ can be implemented in polynomial time. 

Now we argue that for any given $x\in\MX$, $\pihat(\cdot\mid{}x)$ can be explicitly computed in time $\poly(n,2^H)$. Indeed, this only requires $\poly(n,2^H)$ evaluations of the kernel function, followed by $\ell_1$ projection of a $n2^H$-dimensional vector onto $\Delta_\gamma(\MA^H)$. Evaluations of the kernel function are efficient by \cref{def:kernel}, and the projection step can be implemented efficiently by greedily increasing all coordinates which are less than $\gamma$, and then either greedily increasing or decreasing the largest coordinate(s) until the sum is exactly $1$.
\end{proof}

In the preceding proof, we used the following technical lemmas:

\begin{lemma}\label{lemma:tau-convex-concave}
The function $(x,y) \mapsto \tau(x/y)$ with domain $(0,\infty)^2$ is convex in $x$ and concave in $y$.
\end{lemma}

\begin{proof}[\pfref{lemma:tau-convex-concave}]
We can check that $\tau(x) = \frac{2}{1+\sqrt{x}} - 1$, so $\tau'(x) = -\frac{1}{\sqrt{x}(1+\sqrt{x})^2}$, which is non-decreasing in $x$. This establishes convexity of $(x,y) \mapsto \tau(x/y)$ in $x$. Similarly, $\tau(1/y) = 1 - \frac{2}{1+\sqrt{y}} = -\tau(y)$, which establishes concavity in $y$.
\end{proof}

\begin{lemma}\label{lemma:projection-time}
For any $\epproj > 0$ and query point $q$ with norm at most $N$, the $\epproj$-approximate projection oracle $\Oproj$ can be implemented in time $\poly(n,2^H, \log(N/(\epproj\epapx\gamma)))$. 
\end{lemma}

\begin{proof}[\pfref{lemma:projection-time}]
We apply the ellipsoid method with function $z \mapsto \norm{z - q}_2^2$ and constraint set $\MY \cap \{z: \norm{z}_2 \leq 2N\}$ (\cref{eq:y-def}), which admits an efficient separating hyperplane oracle. By definition, the set lies in $\RR^{n2^H}$ and is contained in a Euclidean ball of radius $2N$. Moreover, by \cref{lemma:best-to-pibar}, there is some $v^\star \in \wh V(B/2,2\gamma,\epapx/2)$. By \cref{lemma:representer}, there is $\alphatil \in \MY$ with $\langle v,\psi(\uop{x\ind{i}}{a})\rangle = e_{i,a}^\top \Sigma^{1/2} \alphatil$ for all $i,a$, and thus (by definition of $\wh V(B/2,2\gamma,\epapx/2)$), $e_{i,a}^\top \Sigma^{1/2} \alphatil \geq 2\gamma$ for all $i,a$ and $1-\epapx/2 \leq \sum_{a \in \MA^H} e_{i,a}^\top \Sigma^{1/2} \alphatil \leq 1+\epapx/2$ for all $i$. Moreover $\norm{\alphatil}_2^2 = \norm{v^\star}_2^2 \leq B/2$. Since $\norm{e_{i,a}}_\Sigma \leq 2^H$ for all $i,a$, it follows that for any $y \in \RR^{n2^H}$ with $\norm{y-\alphatil}_2 \leq \min(\epapx 2^{-3H/2-1}, \gamma 2^{-H/2-1})$, we have $\norm{y}_2 \leq N$ and $e_{i,a}^\top \Sigma^{1/2} y \geq \gamma$ and $1-\epapx \leq \sum_{a \in \MA^H} e_{i,a}^\top \Sigma^{1/2} y \leq 1+\epapx$ for all $i$, so that $y$ lies in the constraint set. Thus, the constraint set contains a Euclidean ball of radius $\min(\epapx 2^{-3H/2-1}, \gamma 2^{-H/2-1})$. Finally, note that the function $z \mapsto \norm{z-q}_2^2$ has range bounded in $[0,9N^2]$. Thus, we can conclude from \cite[Theorem 2.4]{bubeck2015convex} that the ellipsoid method finds, in time $\poly(n,2^H, \log(N/(\epproj\epapx\gamma)))$, a point $\hat z \in \MY$ satisfying 
\[\norm{\hat z - q}_2^2 \leq \min\{\norm{z-q}_2^2: z \in \MY, \norm{z}_2 \leq 2N\} + \epproj.\]
But we know that
\[\min\{\norm{z-q}_2^2: z \in \MY, \norm{z}_2 \leq 2N\} = \min\{\norm{z-q}_2^2: z \in \MY\}\]
since $\norm{q}_2 \leq N$ and $\norm{\alphatil}_2 \leq N$. Thus, the ellipsoid method implements an $\epproj$-approximate projection oracle.
\end{proof}

\subsubsection{Optimization Guarantee for Infinite-Dimensional Program}

We now prove \cref{lemma:kernrho-minmax} by appealing to \cref{lemma:optimization-loss-guarantee} as well as \cref{lemma:reverse-representer,lemma:representer}.\vspace{0.5em}

\begin{proof}[Proof of \cref{lemma:kernrho-minmax}]
The time complexity bound is immediate from \cref{lemma:optimization-loss-guarantee}; it remains to prove the inequality. By \cref{lemma:optimization-loss-guarantee} and assumption on $\Blarge$, we have $\wh\alpha \in \wh J(\Blarge,\gamma,\epapx)$. By \cref{lemma:reverse-representer}, we have \[\wh v := \sum_{j=1}^n \sum_{a \in \MA^H} \wh\alpha_{j,a} \psi(\uop{x\ind{j}}{a}) \in  \wh V(\Blarge,\gamma,\epapx).\] 
By \cref{lemma:representer}, for each $w \in \wh W$ there is some $\beta^w \in \wh K$ such that 
\[\langle w, \psi(\uop{x\ind{i}}{a\ind{i}})\rangle = \sum_{j=1}^n \sum_{a \in \MA^H} \beta^w_{j,a} K(\uop{x\ind{j}}{a},\uop{x\ind{i}}{a\ind{i}})\]
for all $i \in [n]$. The analogous relation also holds for $\wh v$ and $\wh \alpha$, by \cref{lemma:reverse-representer}. Hence, for each $w \in \wh W$,
\[\frac{1}{n} \sum_{i=1}^n \tau\left(\frac{\langle \wh v, \psi(\uop{x\ind{i}}{a\ind{i}})\rangle}{\langle w, \psi(\uop{x\ind{i}}{a\ind{i}})\rangle}\right) = \emploss(\wh \alpha, \beta^w),\]
so supremizing over $w \in \wh W$ gives
\begin{align} 
\max_{w \in \wh W} \frac{1}{n} \sum_{i=1}^n \tau\left(\frac{\langle \wh v, \psi(\uop{x\ind{i}}{a\ind{i}})\rangle}{\langle w, \psi(\uop{x\ind{i}}{a\ind{i}})\rangle}\right) 
&\leq \max_{\beta \in \wh K} \emploss(\wh \alpha, \beta) \\ 
&\leq \min_{\alpha \in \wh J} \max_{\beta \in \wh K(\Blarge,\gamma,\epapx)} \emploss(\alpha,\beta) + \epopt\label{eq:loss-bound-minmax}
\end{align}
where the second inequality is by \cref{lemma:optimization-loss-guarantee}. Now fix any $v^\star \in \wh V$. By \cref{lemma:representer}, there is some $\alpha^\star \in \wh J$ such that
\[\langle v^\star, \psi(\uop{x\ind{i}}{a\ind{i}})\rangle = \sum_{j=1}^n \sum_{a \in \MA^H} \alpha^\star_{j,a} K(\uop{x\ind{j}}{a},\uop{x\ind{i}}{a\ind{i}})\]
for all $i \in [n]$. For each $\beta \in \wh K(\Blarge,\gamma,\epapx)$, defining $w^\beta = \sum_{j=1}^n \sum_{a \in \MA^H} \beta_{j,a} \psi(\uop{x\ind{j}}{a})$ we have $w^\beta \in \wh W(\Blarge,\gamma,\epapx)$ by \cref{lemma:reverse-representer}, and the analogous relation to the above holds for $w^\beta$ and $\beta$. Thus,
\[\frac{1}{n} \sum_{i=1}^n \tau\left(\frac{\langle v^\star, \psi(\uop{x\ind{i}}{a\ind{i}})\rangle}{\langle w^\beta, \psi(\uop{x\ind{i}}{a\ind{i}})\rangle}\right) = \emploss(\alpha^\star, \beta).\]
We conclude that
\begin{align}
\min_{\alpha \in \wh J} \max_{\beta \in \wh K(\Blarge,\gamma,\epapx)} \emploss(\alpha,\beta) 
&\leq \max_{\beta \in \wh K(\Blarge,\gamma,\epapx)} \emploss(\alpha^\star, \beta) \\ 
&\leq \max_{w \in \wh W(\Blarge,\gamma,\epapx)} \frac{1}{n} \sum_{i=1}^n \tau\left(\frac{\langle v^\star, \psi(\uop{x\ind{i}}{a\ind{i}})\rangle}{\langle w, \psi(\uop{x\ind{i}}{a\ind{i}})\rangle}\right).
\end{align}
Since $v^\st\in\wh V$ was arbitrary, it follows that
\[\min_{\alpha \in \wh J} \max_{\beta \in \wh K(\Blarge,\gamma,\epapx)} \emploss(\alpha,\beta)  \leq \min_{v\in\wh V}\max_{w \in \wh W(\Blarge,\gamma,\epapx)} \frac{1}{n} \sum_{i=1}^n \tau\left(\frac{\langle v, \psi(\uop{x\ind{i}}{a\ind{i}})\rangle}{\langle w, \psi(\uop{x\ind{i}}{a\ind{i}})\rangle}\right).\]
Substituting into \cref{eq:loss-bound-minmax} completes the proof.
\end{proof}

\subsection{Proof of Theorem \ref*{thm:kern-rho-main}}\label{subsec:kernrho-proof}

\begin{proof}[Proof of \cref{thm:kern-rho-main}]
For purposes of the analysis, set $\Blarge := 2C_{\ref{thm:sp-pgd}}^2 2^{H+1}B/(\gamma^{3/2}\epopt)$ and $\epstat := \gamma^{3/2}\epsilon$. Condition on the event that the bound from \cref{lemma:tau-conc} holds, which occurs with probability at least $1-\delta$ over the data $(x\ind{i},a_{1:H}\ind{i})_{i=1}^n$ from $\BP^{\pistar}$, since $n \geq C_{\ref{lemma:tau-conc}} 2^H \Blarge \log(1/\delta)/(\gamma^3 \epstat^2)$ by theorem assumption, so long as $C_{\ref{thm:kern-rho-main}}$ is a sufficiently large constant. 

Recall the definition of $\wh\alpha$ from \cref{line:wh-alpha}. By \cref{lemma:kernrho-minmax} and choice of $\Blarge$, we have \[\wh v := \sum_{j=1}^n \sum_{a \in \MA^H} \wh\alpha_{j,a} \psi(\uop{x\ind{j}}{a}) \in  \wh V(\Blarge,\gamma,\epapx).\] Thus, applying \cref{lemma:hellinger-by-loss} to $\wh v$ (note that $B \geq (L^2H 2^{2H+2}/\epapx)^{C_{\ref{thm:kernel-apx}}L^2H}$ and $\gamma \leq (e^{-2L^2H} - \epapx)2^{-H-1}$, and we have conditioned on the event of \cref{lemma:tau-conc}, so the conditions of the lemma are satisfied),
\begin{align}
\Dhels{\BP^{\pistar}}{\BP^{\pibar^{\wh v}}}
&\leq \frac{64}{3} \min_{\theta \in \Theta} \Dhels{\BP^{\pistar}}{\BP^{\pi_\theta}} + \frac{8}{3} \max_{w \in \wh W} \frac{1}{n} \sum_{i=1}^n \tau\left(\frac{\langle \wh v, \psi(\uop{x\ind{i}}{a\ind{i}})\rangle}{\langle w, \psi(\uop{x\ind{i}}{a\ind{i}})\rangle}\right) + O\left( \frac{\epapx + \epstat}{\gamma^{3/2}}\right).
\end{align}
By \cref{lemma:kernrho-minmax} and choice of $\Blarge$, we have
\begin{align} 
\max_{w \in \wh W} \frac{1}{n} \sum_{i=1}^n \tau\left(\frac{\langle \wh v, \psi(\uop{x\ind{i}}{a\ind{i}})\rangle}{\langle w, \psi(\uop{x\ind{i}}{a\ind{i}})\rangle}\right)
\leq \min_{v\in\wh V}\max_{w \in \wh W(\Blarge,\gamma,\epapx)} \frac{1}{n} \sum_{i=1}^n \tau\left(\frac{\langle  v, \psi(\uop{x\ind{i}}{a\ind{i}})\rangle}{\langle w, \psi(\uop{x\ind{i}}{a\ind{i}})\rangle}\right) + \epopt.
\end{align}
By \cref{lemma:minmax-ub}, we have
\begin{align}
\min_{v \in \wh V} \max_{w \in \wh W(\Blarge,\gamma,\epapx)} \frac{1}{n} \sum_{i=1}^n \tau\left(\frac{\langle v,\psi(\uop{x\ind{i}}{a\ind{i}})\rangle}{\langle w, \psi(\uop{x\ind{i}}{a\ind{i}})\rangle}\right) \leq 8\min_{\theta \in \Theta} \Dhels{\BP^{\pi^\star}}{\BP^{\pi_\theta}} + O\left(\frac{\epapx + \epstat}{\gamma^{3/2}}\right)
\end{align}
Putting everything together, we get
\begin{align}
\Dhels{\BP^{\pistar}}{\BP^{\pibar^{\wh v}}}
&\leq \frac{88}{3}\min_{\theta \in \Theta} \Dhels{\BP^{\pistar}}{\BP^{\pi_\theta}} + O\left(\epopt + \frac{\epapx+\epstat}{\gamma^{3/2}}\right).
\end{align}
Substituting in the chosen values of $\epopt,\epapx,\epstat,\gamma$, and observing that $\pibar^{\wh v}$ is exactly the policy $\pihat$ produced by $\KernRho$, gives the claimed result. The time complexity bound is immediate from \cref{lemma:minmax-ub}.
\end{proof}

\subsection{Proof of Theorem \ref*{thm:chunk-kr-main}/Theorem \ref*{thm:chunk-kr-informal}}\label{subsec:chunkkr-proof}

We now complete the proof of \cref{thm:chunk-kr-main}, which formally proves \cref{thm:chunk-kr-informal}.

\begin{proof}[Proof of \cref{thm:chunk-kr-main}]
First, we remark that the distribution $\BP^{\pihat}$ where $\pihat = (\pihat_j)_{j=1}^H$ is the output of $\ChunkKR$ is identical to the distribution autoregressively induced by the ``chunked'' policies $\pihat_{h+1-K:h}$ for $h \in \{K,2K,\dots,H\}$; indeed, $\pihat_j(a_j\mid{} x,a_{1:j-1})$ is defined precisely to be the conditional distribution of $a_j$ given $a_{h+1-K:j-1}$ under $a_{h+1-K:h} \sim \pihat_{h+1-K:h}(\cdot\mid{}x,a_{1:h-K})$. Thus, sampling from $\BP^{\pihat}$ is equivalent to sampling $x$, then successively sampling $a_{1:K} \sim \pihat_{1:K}(\cdot\mid{}x)$, followed by $a_{K+1:2K} \sim \pihat_{K+1:2K}(\cdot\mid{}x,a_{1:K})$ and so forth.

For each $h \in \{K,2K,\dots,H\}$ let $\BP^\star_{h+1-K:h}(\cdot\mid{}x,a_{1:h-K})$ denote the marginal distribution of $a_{h+1-K:h}$ under $(x,a_{1:H}) \sim \BP^\star$ conditioned on $(x,a_{1:h-K})$. Observe that $\pihat_{h+1-K:h}(\cdot\mid{}x,a_{1:h-K})$ is precisely the analogous conditional distribution under $(x,a_{1:H}) \sim \BP^{\pihat}$. Also let $\BP^\star_{:h}$ denote the marginal distribution of $(x, a_{1:h})$ under $(x,a_{1:H}) \sim \BP^\star$, and let $\BP^\star_{:h-K} \circ \pihat_{h+1-K:h}$ denote the distribution of $(x,a_{1:h})$ obtained by sampling $(x,a_{1:h-K}) \sim \BP^\star_{:h-K}$ and then $a_{h+1-K:h} \sim \pihat_{h+1-K:h}(\cdot\mid{}x,a_{1:h-K})$. By \cref{lemma:hell-chain-bound}, we have
\begin{align} 
\Dhels{\BP^\star}{\BP^{\pihat}} 
&\leq 7 \cdot \EE_{(x,a_{1:H}) \sim \BP^\star}\left[\sum_{i=0}^{H/K-1} \Dhels{\BP^\star_{iK+1:(i+1)K}(\cdot\mid{}x,a_{1:iK})}{\pihat_{iK+1:(i+1)K}(\cdot\mid{}x,a_{1:iK})}\right] \\ 
&= 7 \cdot \sum_{i=0}^{H/K-1} \Dhels{\BP^\star_{:(i+1)K}}{\BP^\star_{:iK} \circ \pihat_{iK+1:(i+1)K}}\label{eq:chunk-1}
\end{align}
where the equality is by \cref{lemma:hell-cond}. Now consider the execution of $\ChunkKR$ and fix some particular $h \in \{K,2K,\dots,H\}$. Observe that each $(x\ind{i,h}, a\ind{i}_{h+1-K:h})$ has joint distribution $\BP^\star_{:h}$. We now apply \cref{thm:kern-rho-main} to this data, taking the parameter $H$ in \cref{thm:kern-rho-main} to be $K$. By the theorem assumption on $n$ and the parameter choices in $\ChunkKR$, we get that with probability at least $1-\delta/H$, 
\begin{equation}
\Dhels{\BP^\star_{:h}}{\BP^\star_{:h-K} \circ \pihat_{h+1-K:h}} \leq \frac{88}{3} \min_{\theta \in \Theta} \Dhels{\BP^\star_{:h}}{\BP^\star_{:h-K} \circ \pi_{\theta,h+1-K:h}} + \frac{\epsilon}{H}\label{eq:chunk-2}
\end{equation}
where here $\BP^\star_{:h-K} \circ \pi_{\theta,h+1-K:h}$ is the distribution over $(\MX \times \MA^{h-K}) \times \MA^K$ induced by sampling $(x,a_{1:h-K}) \sim \BP^\star_{:h-K}$ and then autoregressively sampling $a_{h+1-K:h} \sim \pi_\theta(\cdot\mid{}x,a_{1:h-K})$. Condition henceforth on the event that \cref{eq:chunk-2} holds for all $h \in \{K,2K,\dots,H\}$, which occurs with probability at least $1-\delta$. Combining \cref{eq:chunk-1,eq:chunk-2}, we get
\begin{align}
\Dhels{\BP^\st}{\BP^{\pihat}}
&\lesssim \epsilon + \sum_{i=0}^{H/K-1} \min_{\theta\in\Theta} \Dhels{\BP^\st_{:(i+1)K}}{\BP^\st_{:iK} \circ \pi_{\theta,iK+1:(i+1)K}} \\ 
&\leq \epsilon + \min_{\theta\in\Theta} \sum_{i=0}^{H/K-1} \Dhels{\BP^\st_{:(i+1)K}}{\BP^\st_{:iK} \circ \pi_{\theta,iK+1:(i+1)K}} \\ 
&= \epsilon + \min_{\theta\in\Theta} \EE_{(x,a_{1:H})\sim\BP^\st}\left[\sum_{i=0}^{H/K-1} \Dhels{\BP^\st_{iK+1:(i+1)K}(\cdot\mid{}x,a_{1:iK})}{\pi_{\theta,iK+1:(i+1)K}(\cdot\mid{}x,a_{1:iK})}\right] \\ 
&\lesssim \epsilon + \frac{H}{K}\min_{\theta\in\Theta} \Dhels{\BP^\st}{\BP^{\pi_\theta}}
\end{align}
where the equality is by \cref{lemma:hell-cond} and the final inequality is by \cref{lemma:hell-reverse-chain-bound}. Finally, the time complexity bound is immediate from \cref{thm:kern-rho-main}.
\end{proof}

\subsection{Technical Lemmas}

\subsubsection{Information Theory}

\begin{lemma}[{\citet[Lemma D.2]{foster2024online}}]\label{lemma:hell-chain-bound}
Let $n \in \NN$ and let $\MX$ be a set. Let $\BP, \BQ \in \Delta(\MX^n)$. Then 
\[\Dhels{\BP}{\BQ} \leq 7 \cdot \EE_{x \sim \BP}\left[ \sum_{i=1}^n \Dhels{\BP_i(\cdot\mid{}x_{1:i-1})}{\BQ_i(\cdot\mid{}x_{1:i-1})},\right]\]
where $\BP_i(\cdot\mid{}x_{1:i-1})$ is the marginal of $x_i$ under $x \sim \BP$ conditioned on $x_{1:i-1}$, and $\BQ_i(\cdot\mid{}x_{1:i-1})$ is the marginal of $x_i$ under $x \sim \BQ$ conditioned on $x_{1:i-1}$.
\end{lemma}

\begin{lemma}[e.g. {\citet[Proposition 7.5(4)]{polyanskiy2024information}}]\label{lemma:hell-cond}
For any two joint distributions $\BP,\BQ$ over random variables $(X,Y)$,
\[\Dhels{\BP_{X,Y}}{\BP_X\BQ_{Y\mid{}X}} = \EE_{x\sim \BP_X}[\Dhels{\BP_{Y\mid{}X=x}}{\BQ_{Y\mid{}X=x}}].\]
\end{lemma}

\begin{lemma}[{\citet[Lemma A.9]{foster2021statistical}}]\label{lemma:hell-reverse-chain-bound}
For any two joint distributions $\BP,\BQ$ over random variables $(X,Y)$,\loose
\[\Dhels{\BP_{X,Y}}{\BP_X\BQ_{Y\mid{}X}} \leq 4\Dhels{\BP_{X,Y}}{\BQ_{X,Y}}.\]
\end{lemma}

The following bound provides a converse to \cref{lemma:hell-chain-bound}, though it loses a factor of $n$; it follows from applying \cref{lemma:hell-reverse-chain-bound} (in conjunction with \cref{lemma:hell-cond}) and the data processing inequality to individually upper bound each term of the summation by $\Dhels{\BP}{\BQ}$.

\begin{corollary}\label{cor:hell-reverse-chain-bound}
Let $n \in \NN$ and let $\MX$ be a set. Let $\BP, \BQ \in \Delta(\MX^n)$. Then 
\[\EE_{x \sim \BP}\left[ \sum_{i=1}^n \Dhels{\BP_i(\cdot\mid{}x_{1:i-1})}{\BQ_i(\cdot\mid{}x_{1:i-1})}\right] \leq 4n \cdot \Dhels{\BP}{\BQ}\]
where $\BP_i(\cdot\mid{}x_{1:i-1})$ is the marginal of $x_i$ under $x \sim \BP$ conditioned on $x_{1:i-1}$, and $\BQ_i(\cdot\mid{}x_{1:i-1})$ is the marginal of $x_i$ under $x \sim \BQ$ conditioned on $x_{1:i-1}$.
\end{corollary}

\subsubsection{Generalization Theory}

\begin{definition}
\label{def:gaussian_complexity}
For a set $\MX$ and a class $\MF$ of functions $f : \MX \to \RR$, and $n \in \NN$, the \emph{Gaussian complexity} of $\MF$ with respect to samples $x_1, \ldots, x_n \in \MX$ is
\[\MG_n(\MF; x_{1:n}) := \frac{1}{n}\EE_{\xi_{1:n} \sim N(0,1)}\left[\sup_{f\in\MF} \sum_{i=1}^n \xi_i f(x_i) \right].\] 
We write $\MG_n(\MF) = \sup_{x_{1:n}} \MG_n(\MF;x_{1:n})$.
\end{definition}

\begin{lemma}[Composition of Gaussian complexities e.g. {\cite[Lemma B.6]{golowich2024exploring}}]
  \label{lem:rc-composition}
  Let $\MX$ be a set. Fix $A,L \in \NN$ and let $\MF_1, \ldots, \MF_A$ be classes of functions mapping $\MX$ to $\RR$. Let $\phi : \RR^A \to \RR$ be $L$-Lipschitz with respect to the Euclidean distance on $\RR^A$. Let $\MF$ be the class of real-valued functions on $\MX$ defined as follows:\loose
  \begin{align}
\MF := \left\{ x \mapsto \phi(f_1(x), \ldots, f_A(x)) \ : \ f_1 \in \MF_1, \ldots, f_A \in \MF_A \right\}\nonumber.
  \end{align}
  Then for all $n \in \NN$, 
  \begin{align}
\MG_n(\MF) \leq L \sum_{a=1}^A \MG_n(\MF_a)\nonumber.
  \end{align}
\end{lemma}

\begin{lemma}[{\citet[Theorem 26.5]{shalev2014understanding} + \citet[Exercise 5.5]{wainwright2019high}}] \label{lem:unif-conv}
  Suppose $\MX$ is a set and $\MF$ is a class of functions $f : \MX \to [-B, B]$ for some $B > 0$. Suppose $P$ is a distribution on $\MX$. Then for any $n \in \NN$ and $\delta \in (0,1)$, with probability at least $1-\delta$ over an i.i.d.~sample $X_1, \ldots, X_n \sim P$, it holds that\loose
  \begin{align}
\sup_{f \in \MF} \left| \EE_{X \sim P}[f(X)] - \frac 1n \sum_{i=1}^n f(X_i) \right| \leq \sqrt{2\pi}\MG_n(\MF) + 4B\sqrt{\frac{2 \log(4/\delta)}{n}}.\nonumber
  \end{align}
\end{lemma}

\subsubsection{Optimization}

\begin{algorithm}
\caption{$\PGD(\Ovec,\Oproj,T,\eta)$: approximate projected gradient descent}\label{alg:pgd}
\begin{algorithmic}
\Require Vector field oracle $\Ovec: \RR^n \to \RR^n$; projection oracle $\Oproj: \RR^n \to \RR^n$; iteration complexity $T \in \NN$; step size $\eta>0$
\State $x_1 \gets \Oproj(0)$
\For{$1 \leq t \leq T-1$}
    \State $y_{t+1} \gets x_t - \eta \Ovec(x_t)$
    \State $x_{t+1} \gets \Oproj(y_{t+1})$
\EndFor
\State \textbf{return} $(x_t)_{t=1}^T$
\end{algorithmic}
\end{algorithm}

\begin{definition}
Let $\MX \subset \RR^n$ be a compact set. An $\epproj$-approximate projection oracle $\Oproj$ for $\MX$ takes input $y \in \RR^n$ and returns $x \in \MX$ such that
\[\norm{y-x}_2^2 \leq \min_{x' \in \MX} \norm{y-x'}_2^2 + \epproj.\]
\end{definition}

\begin{lemma}\label{lemma:apx-proj-pythag}
Let $\MX \subset \RR^n$ be a convex set, and let $\Oproj$ be an $\epproj$-approximate projection oracle for $\MX$. For any $x \in \MX$ and $y \in \RR^n$, it holds that 
\[\norm{y - x}_2^2 \geq (1-\sqrt{\epproj})\norm{\Oproj(y)-x}_2^2 - \sqrt{\epproj}\]
\end{lemma}

\begin{proof}[\pfref{lemma:apx-proj-pythag}]
For any $\alpha \in (0,1)$, we have $\alpha x + (1-\alpha)\Oproj(y) \in \MX$, so
\[\norm{\Oproj(y) - y}_2^2 \leq \norm{\Oproj(y) - y + \alpha ( x - \Oproj(y))}_2^2 + \epproj.\]
Therefore
\[2\alpha \langle y - \Oproj(y), \Oproj(y) - x\rangle \geq -\alpha^2 \norm{x - \Oproj(y)}_2^2 - \epproj.\]
Setting $\alpha = \sqrt{\epproj}$, we use the above bound to get
\begin{align}
\norm{y-x}_2^2
&= \norm{y-\Oproj(y)}_2^2 + \norm{\Oproj(y)-x}_2^2 + 2\langle y-\Oproj(y), \Oproj(y)-x\rangle \\ 
&\geq (1-\sqrt{\epproj})\norm{\Oproj(y)-x}_2^2 - \sqrt{\epproj}
\end{align}
as claimed.
\end{proof}

\begin{lemma}[Modification of {\cite[Theorem 4.2]{bubeck2015convex}}]\label{lemma:pgd}
Let $L,R \geq 1$ and $T \in \NN$. Let $\MX \subset \RR^n$ be a convex set and let $g: \RR^n \to \RR^n$ be a vector field. Suppose that $\norm{g(x)}_2 \leq L$ for all $x \in \MX$. Let $\Oproj$ be an $\epproj$-approximate projection oracle for $\MX$. If $\epproj \leq 1/(R^2 T^4)$, the iterates $(x_t)_{t=1}^T \gets \PGD(g,\Oproj,T, \frac{R}{L}\sqrt{\frac{2}{T}})$ satisfies, for any $x \in \MX$ with $\norm{x}_2 \leq R$,
\[\frac{1}{T}\sum_{t=1}^T \langle g(x_t), x_t - x\rangle \leq O\left(\frac{RL}{\sqrt{T}}\right).\]
Moreover, $x_t \in \MX$ and $\norm{x_t}_2 \leq O(R\sqrt{T})$ for all $t \in [T]$.
\end{lemma}

\begin{proof}[\pfref{lemma:pgd}]
For notational convenience, write $g_t := g(x_t)$ for each $t \in [T]$. For any $t \in [T]$, we have
\begin{align}
\langle g_t, x_t - x\rangle
&= \frac{1}{\eta} \langle x_t - y_{t+1}, x_t - x \rangle \\ 
&= \frac{1}{2\eta} \left(\norm{x_t-y_{t+1}}_2^2 + \norm{x_t - x}_2^2 - \norm{y_{t+1}-x}_2^2 \right) \\ 
&\leq \frac{1}{2\eta} \left(\norm{x_t-y_{t+1}}_2^2 + \norm{x_t - x}_2^2 - \norm{x_{t+1}-x}_2^2 \right) + \frac{\sqrt{\epproj}(1+\norm{x_{t+1}-x}_2^2)}{2\eta} \\ 
&\leq \frac{1}{2\eta} \left(\norm{x_t - x}_2^2 - \norm{x_{t+1}-x}_2^2 \right) + \frac{\sqrt{\epproj}(1+\norm{x_{t+1}-x}_2^2)}{2\eta} + \frac{\eta L^2}{2}
\end{align}
where the first inequality uses \cref{lemma:apx-proj-pythag} and the second inequality uses that $\norm{x_t - y_{t+1}}_2 = \eta\norm{g_t}_2 \leq \eta L$. Averaging the above bound and telescoping,
\[\frac{1}{T}\sum_{t=1}^T \langle g_t,x_t - x\rangle \leq \frac{\norm{x_1 - x}_2^2}{2\eta T} + \frac{\eta L^2 }{2} + \frac{\sqrt{\epproj}}{2\eta} \max_{t \in [T]} (1 + \norm{x_{t+1} - x}_2^2).\]
For each $t \in [T]$, we have
\begin{align}
\norm{x_{t+1} - x}_2
&\leq (1+2\sqrt{\epproj})\norm{y_{t+1}-x}_2 + \epproj^{1/4} \\ 
&\leq (1+2\sqrt{\epproj})\left(\norm{x_t - x}_2 + \eta L + \epproj^{1/4}\right) \\ 
&\leq e^{2t\sqrt{\epproj}} \norm{x_1-x}_2 + \sum_{s=1}^t e^{2s\sqrt{\epproj}} (\eta L + \epproj^{1/4}) \\ 
&\leq e^2 \norm{x_1 - x}_2 + e^2 T (\eta L + \epproj^{1/4})
\end{align}
where the first inequality is by \cref{lemma:apx-proj-pythag}, and the last inequality is by assumption that $\epproj \leq 1/T^2$. Moreover, again by \cref{lemma:apx-proj-pythag},
\[\norm{x_1 - x}_2 \leq e^{2\sqrt{\epproj}} \norm{x}_2 + \epproj^{1/4}.\]
Since $\norm{x}_2 \leq R$, we conclude that
\[\frac{1}{T}\sum_{t=1}^T \langle g_t,x_t - x\rangle \leq O\left(\frac{R^2+\sqrt{\epproj}}{2\eta T} + \frac{\eta L^2 }{2} + \frac{\sqrt{\epproj}}{2\eta}(1 + R^2 + T^2\eta^2 L^2 + T^2\sqrt{\epproj})\right).\]
Substituting in $\eta = (R/L)\sqrt{2/T}$ and using the assumption that $\epproj \leq 1/(R^2 T^4)$ gives
\[\frac{1}{T}\sum_{t=1}^T \langle g_t,x_t - x\rangle \leq O\left(\frac{LR}{\sqrt{T}}\right).\]
Moreover, $\norm{x_t - x}_2 \leq O(R \sqrt{T})$ as claimed. The fact that $x_t \in \MX$ is by definition of $\Oproj$.
\end{proof}

\begin{theorem}[Modification of {\cite[Theorem 5.1]{bubeck2015convex}}]\label{thm:sp-pgd}
There is a universal constant $C_{\ref{thm:sp-pgd}}>0$ so that the following holds. Let $L,R\geq 1$ and $T \in \NN$. Let $\MX \subset \RR^n$ be a convex set and let $f: \RR^n \times \RR^n \to \RR$ be a function. Suppose that for each $y \in \MX$, $f(\cdot,y)$ is convex and $L$-Lipschitz w.r.t. $\norm{\cdot}_2$ on $\MX$, and that for each $x \in \MX$, $f(x,\cdot)$ is concave and $L$-Lipschitz w.r.t. $\norm{\cdot}_2$ on $\MX$. Define $g(x,y) = (\grad_x f(x,y), -\grad_y f(x,y))$. Let $\Oproj$ be an $(\epproj/2)$-approximate projection oracle for $\MX$ with $\epproj \leq 1/(8R^2 T^4)$. Then $(x_t,y_t)_{t=1}^T \gets \PGD(g, \Oproj \oplus \Oproj, T, \frac{R}{L}\sqrt{\frac{2}{T}})$ satisfies
\begin{equation} \max_{y \in \MX: \norm{y}_2 \leq R} f\left(\frac{1}{T}\sum_{t=1}^T x_t, y\right) - \min_{x \in \MX: \norm{x}_2 \leq R} f\left(x, \frac{1}{T}\sum_{t=1}^T y_t\right) \leq C_{\ref{thm:sp-pgd}} \cdot \left(\frac{LR}{\sqrt{T}}\right)
\label{eq:sp-guarantee}
\end{equation}
and $\frac{1}{T}\sum_{t=1}^T x_t \in \MX$, $\frac{1}{T}\sum_{t=1}^T y_t \in \MX$ with $\norm{\frac{1}{T}\sum_{t=1}^T x_t}_2,\norm{\frac{1}{T}\sum_{t=1}^T y_t}_2 \leq C_{\ref{thm:sp-pgd}} \cdot R\sqrt{T}$.
\end{theorem}

\begin{proof}[\pfref{thm:sp-pgd}]
We apply \cref{lemma:pgd} with set $\MX \times \MX \subset \RR^{2n}$ and vector field $g$. For any $(x,y) \in \MX \times \MX$, we have
\[\norm{g(x,y)}_2^2 = \norm{\grad_x f(x,y)}_2^2 + \norm{\grad_y f(x,y)}_2^2 \leq 2L^2\]
by $L$-Lipschitzness of $f(\cdot,y)$ and $f(x,\cdot)$. Next, observe that the projection oracle $\Oproj \oplus \Oproj$ defined by $(\Oproj \oplus \Oproj)(x,y) := (\Oproj(x),\Oproj(y))$ is a $2\epproj$-approximate projection oracle for $\MX\times\MX$. Thus, \cref{lemma:pgd} gives for any $(x,y) \in \MX\times\MX$ with $\norm{x}_2,\norm{y}_2 \leq R$ that
\[\frac{1}{T}\sum_{t=1}^T \langle \grad_x f(x_t,y_t), x_t - x\rangle - \langle \grad_y f(x_t,y_t), y_t - y\rangle \leq O\left(\frac{RL}{\sqrt{T}}\right).\]
Now for each $t \in [T]$, by convexity of $f(\cdot,y_t)$, we have
\[f(x,y_t) - f(x_t,y_t) \geq \langle \grad_x f(x_t,y_t), x - x_t\rangle.\]
Similarly, by concavity of $f(x_t,\cdot)$,
\[f(x_t, y) - f(x_t,y_t) \leq \langle \grad_y f(x_t,y_t), y - y_t\rangle.\]
Summing, we get
\[f(x_t,y) - f(x,y_t) \leq \langle \grad_x f(x_t,y_t), x_t - x\rangle - \langle \grad_y f(x_t,y_t), y_t - y\rangle.\]
Finally, convexity of $f(\cdot,y)$ and concavity of $f(x,\cdot)$ gives
\begin{align}
f\left(\frac{1}{T}\sum_{t=1}^T x_t,y\right) - f\left(x,\frac{1}{T}\sum_{t=1}^T y_t\right) 
&\leq \frac{1}{T}\sum_{t=1}^T f(x_t,y) - f(x,y_t) \\ 
&\leq \frac{1}{T}\sum_{t=1}^T \langle \grad_x f(x_t,y_t), x_t - x\rangle - \langle \grad_y f(x_t,y_t), y_t - y\rangle \\
&\leq O\left(\frac{RL}{\sqrt{T}}\right).
\end{align}
Since this bound holds for all $x,y \in \MX$ with $\norm{x}_2,\norm{y}_2 \leq R$, we have proven \cref{eq:sp-guarantee}. The containments $\frac{1}{T}\sum_{t=1}^T x_t \in \MX$, $\frac{1}{T}\sum_{t=1}^T y_t \in \MX$ and norm bounds $\norm{\frac{1}{T}\sum_{t=1}^T x_t}_2,\norm{\frac{1}{T}\sum_{t=1}^T y_t}_2 \leq C_{\ref{thm:sp-pgd}} \cdot R\sqrt{T}$ are immediate from convexity of $\MX$ and the guarantee of \cref{lemma:pgd} that $(x_t,y_t) \in \MX\times\MX$ with $\norm{(x_t,y_t)}_2^2 \leq O(R\sqrt{T})$ for all $t \in [T]$.\loose
\end{proof}

\end{document}